\documentclass[11pt,a4paper]{article}
\usepackage{times,latexsym}
\usepackage{url}
\usepackage[T1]{fontenc}

\usepackage{amsfonts}
\usepackage{amssymb}
\usepackage{amsmath}
\usepackage{natbib}
\usepackage{bm}

\usepackage{enumitem}

\usepackage[dvipsnames]{xcolor}

\usepackage{hyperref}
\usepackage{fancyvrb}

\newcommand{\spell}[1]{\overline{#1}}
\newcommand{\yield}{\mathcal{Y}}
\newcommand{\func}{\phi}
\newcommand{\prodrule}{\psi}

\newcommand{\Ff}{\mathcal{F}}
\newcommand{\trees}{\mathcal{T}}
\newcommand{\DL}{\operatorname{D}}

\newcommand{\START}{\mathsf{START}}

\newcommand{\T}{\mathsf{T}}
\newcommand{\NT}{\mathsf{NT}}

\newcommand{\cvgarg}[0]{attribute}
\newcommand{\Cvgarg}[0]{Attribute}

\usepackage{natbib}
\usepackage{fullpage}
\usepackage{mathpazo}

\usepackage{dsfont}
\usepackage{graphicx}
\usepackage{xcolor}
\usepackage{amsmath}
\usepackage{amsfonts}

\usepackage{amsmath}
\usepackage{amssymb}
\usepackage{mathtools}
\usepackage{amsthm}

\usepackage[capitalize,noabbrev]{cleveref}

\theoremstyle{plain}
\newtheorem{theorem}{Theorem} 

\newtheorem{lemma}[theorem]{Lemma}

\theoremstyle{definition}
\theoremstyle{example}
\newtheorem{example}[theorem]{Example}
\newtheorem{definition}[theorem]{Definition}

\theoremstyle{remark}

\newcommand\pfun{\mathrel{\ooalign{\hfil$\mapstochar$\hfil\cr$\to$\cr}}}

\usepackage[utf8]{inputenc} 
\usepackage[T1]{fontenc}    
\usepackage{hyperref}       
\usepackage{url}            
\usepackage{booktabs}       
\usepackage{amsfonts}       
\usepackage{nicefrac}       
\usepackage{microtype}      
\usepackage{xcolor}         
\usepackage{subcaption}

\usepackage{tikz}

\usepackage{xspace,mfirstuc,tabulary}

\title{A Theory of Emergent In-Context Learning as Implicit Structure Induction }

\author{
  Michael Hahn\\
  Saarland University\\
  \texttt{mhahn@lst.uni-saarland.de} \\
  \and
  Navin Goyal\\
  Microsoft Research India\\
  \texttt{navingo@microsoft.com}\\ 
}

\date{}

\begin{document}
\maketitle
\begin{abstract}
Scaling large language models (LLMs) leads to an emergent capacity to learn in-context from example demonstrations.
Despite progress, theoretical understanding of this phenomenon remains limited.
We argue that in-context learning relies on recombination of compositional operations found in natural language data.
We derive an information-theoretic bound showing how in-context learning abilities arise from generic next-token prediction when the pretraining distribution has sufficient amounts of compositional structure, under linguistically motivated assumptions.
A second bound provides a theoretical justification for the empirical success of prompting LMs to output intermediate steps towards an answer.
To validate theoretical predictions, we introduce a controlled setup for inducing in-context learning; unlike previous approaches, it accounts for the compositional nature of language. 
Trained transformer LMs can perform in-context learning for a range of tasks, in a manner consistent with the theoretical results.
Mirroring real-world LMs in a miniature setup, in-context learning emerges when scaling parameters and data, and LMs perform  better when prompted to output intermediate steps. 
Probing shows that in-context learning is supported by a representation of the input's compositional structure.
Taken together, these results provide a step towards theoretical understanding of emergent behavior in large language models.
\end{abstract}

Large language models (LLMs), trained only on next-word prediction, can perform novel tasks by completing a prompt consisting of example demonstrations, without any parameter updating \citep{BrownMRSKDNSSAA20}.
This ability, termed \textit{in-context learning} (ICL), is emergent in the sense that it arises without specialized training data or objectives, simply by scaling models and computation \citep{DBLP:journals/corr/abs-2206-07682}.
This phenomenon has recently been the focus of much research, but theoretical understanding is limited.
Aiming to build theoretical understanding, recent work has studied in-context learning in miniaturized controlled settings, investigating how transformers could learn to solve simple classification or regression tasks in context 
\citep{DBLP:conf/iclr/XieRL022,DBLP:journals/corr/abs-2211-15661,DBLP:journals/corr/abs-2208-01066,DBLP:journals/corr/abs-2205-05055,Oswald2022TransformersLI,DBLP:journals/corr/abs-2212-10559}.
However, existing studies do not take into account the highly compositional nature of language data, modeling the pretraining data either in terms of an unstructured set of HMMs \citep{DBLP:conf/iclr/XieRL022}, or as consisting of prompts formatted analogously to the test tasks.
Such setups make it hard to account for a lot of the remarkable flexibility that real-world LLMs show:
They can perform broad ranges of in-context tasks with varying prompt formats, and they can be prompted to provide intermediate steps leading to an answer, often dramatically improving performance \cite[e.g.][]{Wei2022Chain,DBLP:journals/corr/abs-2112-00114,wang-etal-2022-iteratively,Suzgun2022ChallengingBT}.
The emergence of such behavior remains largely mysterious.

We argue that these abilities can arise through recombination of compositional structure found in linguistic data, which we formalize in terms of grammar formalisms long studied in the linguistic literature.
We first investigate when an idealized predictor performing next-token prediction can perform in-context learning from demonstrations.
Theorem~\ref{theorem:theorem1} describes how broad ICL skills arise when the pretraining distribution contains a sufficient amount of compositional structure. 
Based on this result, we introduce a novel controlled scenario in which in-context learning from demonstrations emerges from next-token prediction. 
We define a suite of few-shot test tasks, defined in first-order logic relative to a logical world model, and evaluate language models (LMs) trained on training datasets with varying amounts of diverse compositional structure.
Unlike training datasets closely mirroring constrained scenarios proposed in previous work, text generated by compositional processes leads to broad ICL capabilities.
While pretraining cross-entropy decreases continuously, a wide variety of tasks emerge suddenly after varying amounts of pretraining.
Our theory also explains why prompting LLMs to provide intermediate steps makes ICL more effective (Theorem~\ref{theorem:cot}).
We probe the LM's inner workings and argue that representation learning supports the ICL ability.

Taken together, our key contributions are
\begin{enumerate}
	\item a theoretical analysis of the conditions under which generic next-token prediction leads to in-context learning from demonstrations in an idealized predictive model,  
	\item a controlled setup for studying in-context learning, in which in-context learning skills emerge for a broad set of tasks, including prompting LMs for providing intermediate steps.
\end{enumerate}

\begin{table}[h]
    \centering
    \footnotesize
    \begin{tabular}{lccccccccc}
    \hline
         & Theory & \textsc{Comp.} & LLM  \\ \hline
improves with prompt length         & \checkmark & \checkmark & \checkmark  \\
gets harder with $\DL[\tau_\func]$ & \checkmark & \checkmark & \checkmark\footnotemark  \\
\textsc{ChainOfThought} > Raw & \checkmark & \checkmark & \checkmark \footnotemark \\
\textsc{ChainOfThought} > \textsc{Explanation} & \checkmark & \checkmark & \checkmark \textsuperscript{\ref{footnote-label}} \\
\hline
gets harder with $|\Ff|$ & \checkmark & \checkmark & ? \\
does not get harder with $|\Omega|$ & \checkmark & \checkmark & ? \\
recombining skills never seen together & n.a. & \checkmark\footnotemark & ? \\ \hline
works for natural \& unnatural prompts & \checkmark & n.a. & \checkmark \\
\hline
    \end{tabular}

    \caption{Schematic properties of ICL as predicted by our theory, as exhibited by real transformers trained on the \textsc{Compositional} data, and observed for real-world LLMs.
Evidence for the first group of properties comes from all settings.
%Whereas improvement with prompt length is arguably explained by all theoretical approaches to ICL, the others are not predicted by existing approaches.
	A signature prediction of our theory is that ICL success depends on the complexity $\DL[\tau_\func]$ of a task's compositional description.
	We further derive a benefit of prompting LMs for intermediate steps before the answer (chain-of-thought prompting).
	Properties in the second group are hard to establish for real-world LLMs, but can be cleanly studied in our controlled setup.
The theory predicts scaling with the number of functions $|\Ff|$ and objects $|\Omega|$; we further experimentally observe recombination of skills never seen together in finite training data.
Robustness to unnatural prompt formats (third  group) is shared by our theory and real-world LLMs; it does not apply to our miniaturized training data which has no notion of ``naturalness''.
    }
    \label{tab:icl-properties}
\end{table}

\footnotetext[1]{Figure~\ref{fig:gpt3} provides evidence from InstructGPT.}
\footnotetext[2]{\label{footnote-label} \cite[e.g.][]{Wei2022Chain,DBLP:conf/emnlp/LampinenDCMTCMW22}; Figure~\ref{fig:gpt3}}
\footnotetext[3]{Figure~\ref{fig:recombination}}

\section{A Formal Learnability Bound for Learning from Demonstrations}\label{sec:theory}

In order to understand when broad ICL capabilities emerge from generic next-token prediction, we take the perspective of idealized learners with access to infinite training data and infinite capacity to fit the data distribution. 
We show that very general linguistically-motivated assumptions about the generative process underlying the pretraining data are sufficient to guarantee ICL capabilities for an idealized language model performing ordinary next-token prediction.

\subsection{Setup}
\paragraph{World Model.}
Both the pretraining data and the few-shot tasks are generated on the basis of some  finite universe $\Omega$ of objects.
The pretraining corpus consists of a collection of finite strings $d\in\Sigma^*$, referred to as \emph{documents}; $\Sigma$ is a finite set serving as the alphabet.
A ``spellout'' map $\omega \mapsto \spell{\omega}$ maps objects $\omega\in\Omega$ to their names $\spell{\omega} \in \Sigma$.

\begin{figure*}
    \centering

        \begin{tikzpicture}

		\node (label) at (-1.2,1)[draw=none, align=left, anchor=center]{\includegraphics[width=0.55\textwidth]{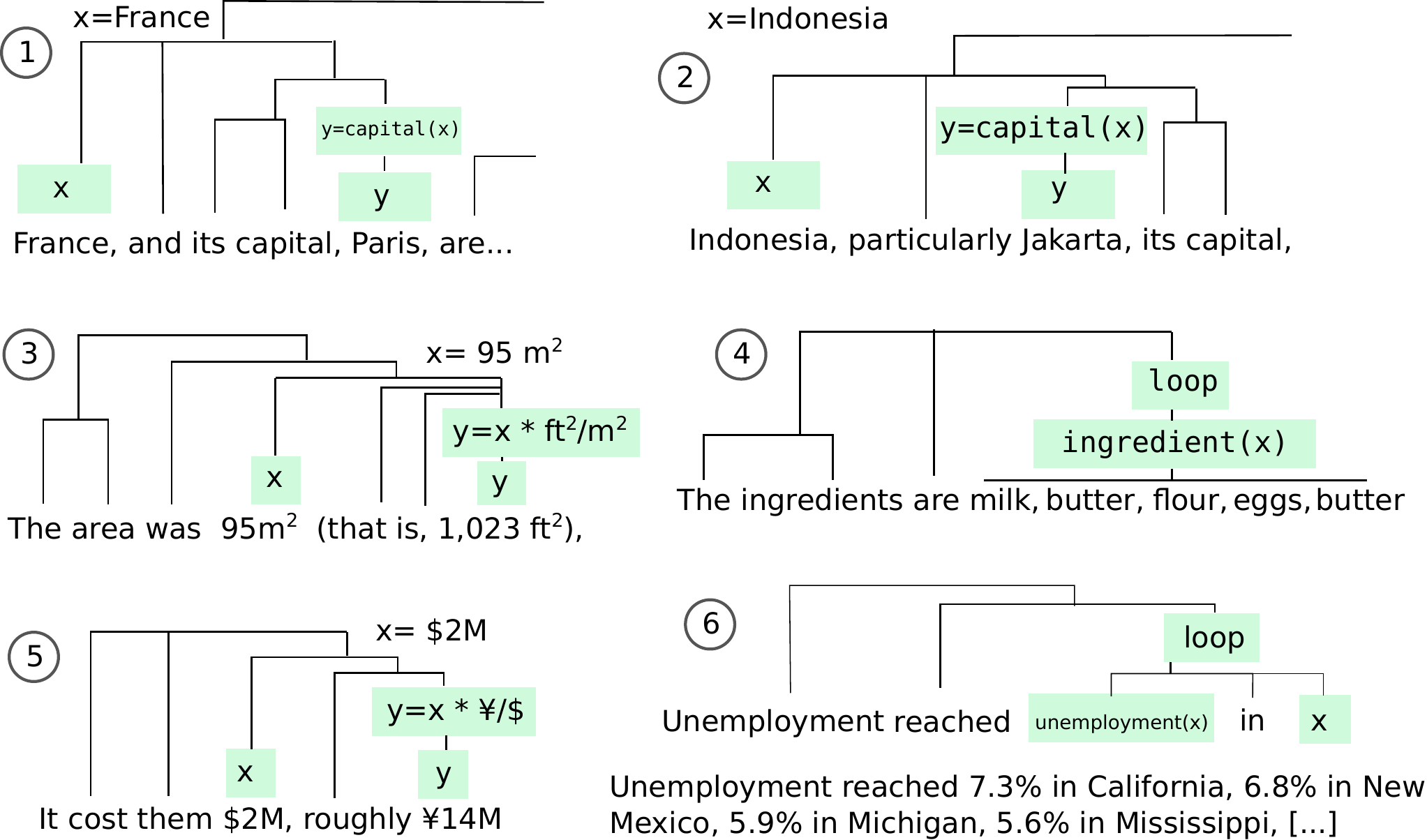}};
  		\node (label) at (-6,3)[draw=none, align=left, anchor=center]{\LARGE{A}};

		\node (label) at (6.5,1)[draw=none, align=left, anchor=center]{\includegraphics[width=0.35\textwidth]{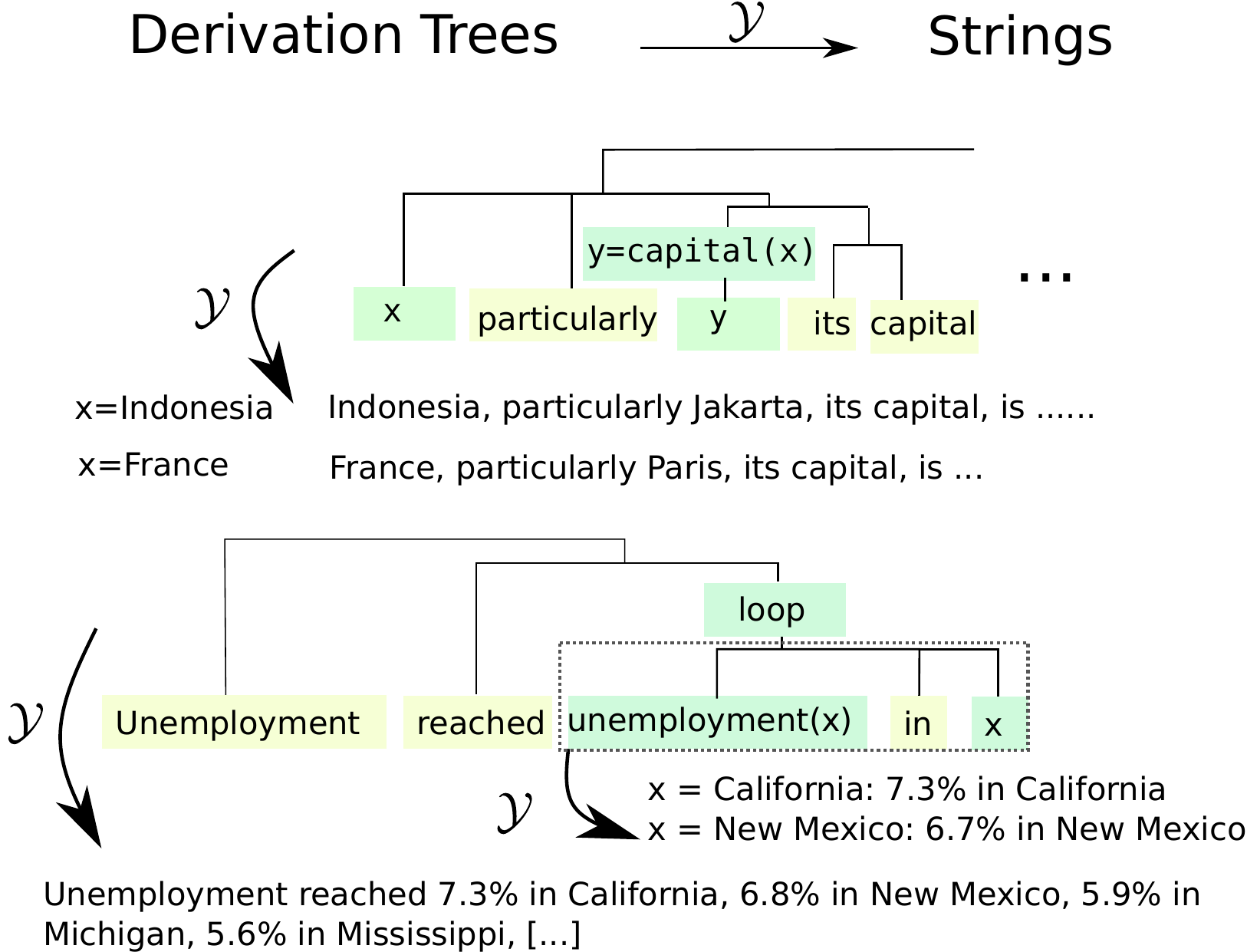}};
\node (label) at (3.5,3)[draw=none, align=left, anchor=center]{\LARGE{B}};

		\node (label) at (-2.5,-4.5)[draw=none, align=left, anchor=center]{\includegraphics[width=0.4\textwidth]{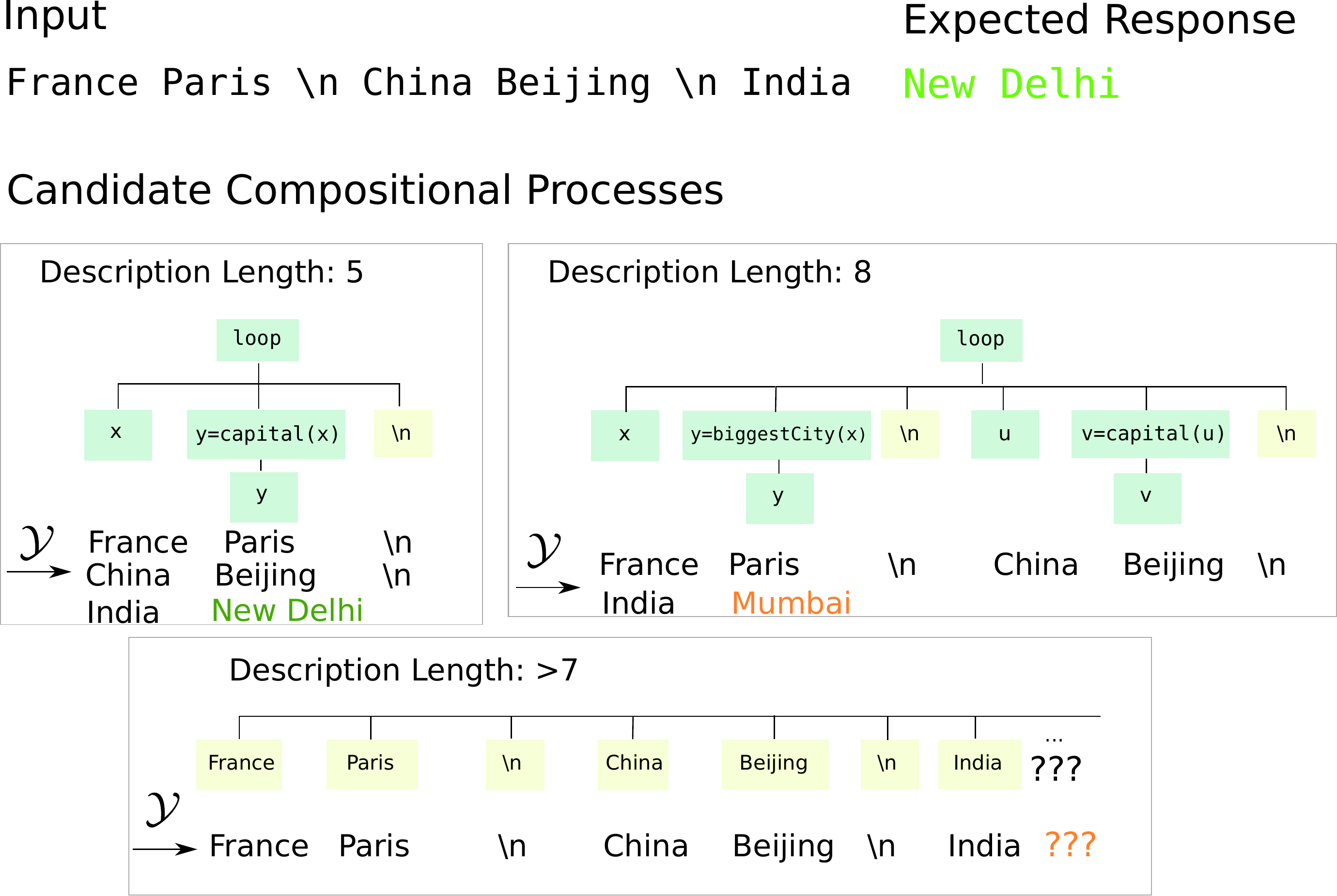}};
  \node (label) at (-6,-2.5)[draw=none, align=left, anchor=center]{\LARGE{C}};

		\node (label) at (6,-4.5)[draw=none, align=left, anchor=center]{\includegraphics[width=0.5\textwidth]{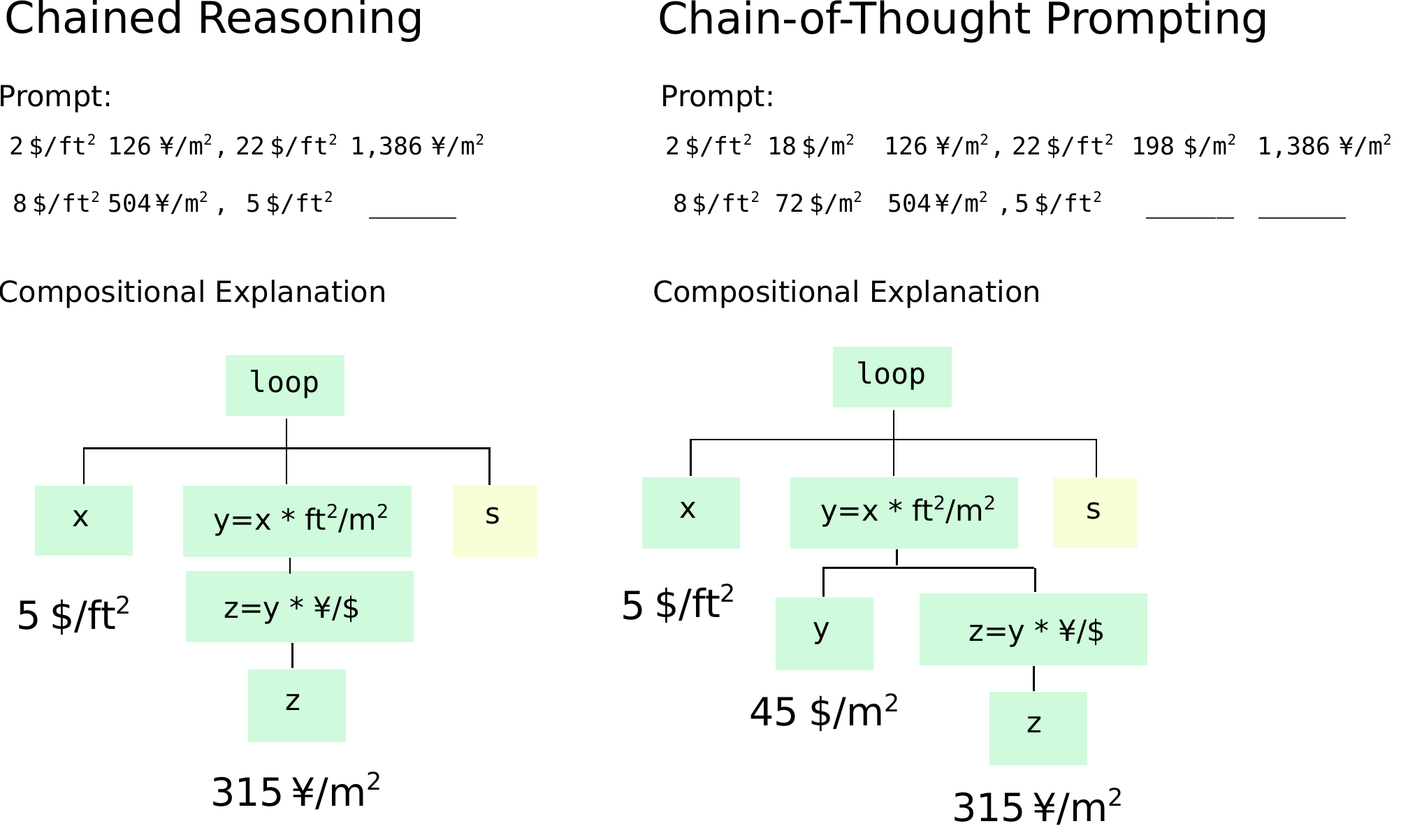}};
\node (label) at (1.5,-2.5)[draw=none, align=left, anchor=center]{\LARGE{D}};

\node (label) at (-2.1,-7.7)[draw=none, align=left, anchor=center]{\LARGE{E}};
		\node (label) at (2,-8.5)[draw=none, align=left, anchor=center]{\includegraphics[width=0.5\textwidth]{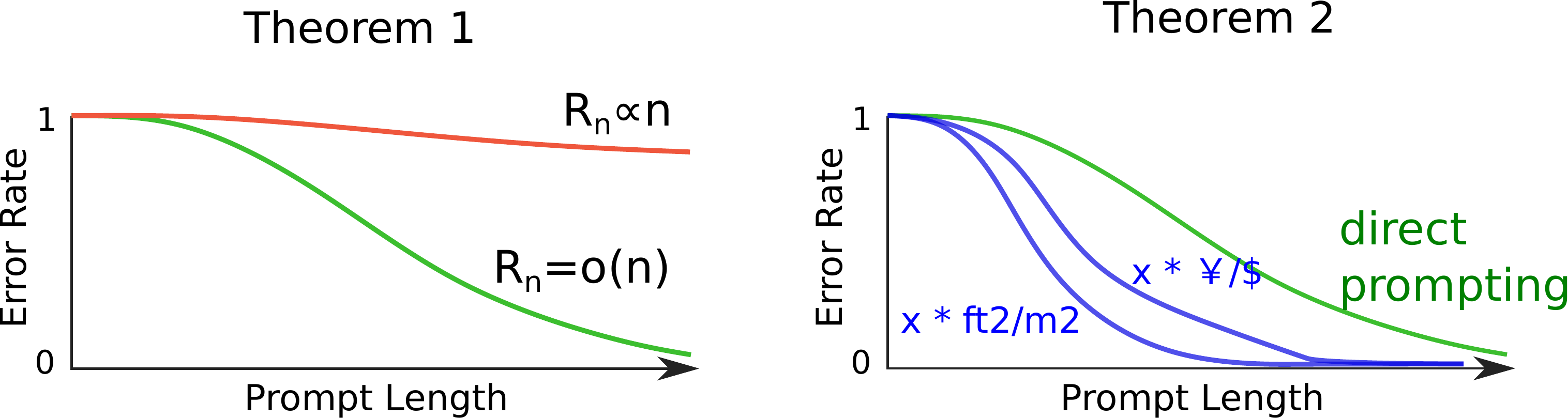}};

	\end{tikzpicture}

 \caption{ (A) Natural language text is generated by a compositional process. Following linguistic research, we assume that each document can be described in terms of a compositional description generated from a formal grammar. We highlight some examples of operations that are re-used and re-composed across documents and applied to different objects.
	(B)
	We formalize this generative processes in terms of a probabilistic grammar combining operations into derivation trees paired with a \emph{yield} operation expressing these into strings.
 Going beyond PCFGs, variables can be shared across subtrees (top), and subtrees can be iterated (bottom).
(C) When faced with a prompt, recombining compositional operations found in the training corpus provides a parsimonious explanation of the input (here, naming countries with their capitals), allowing an optimal predictor to correctly infer the response.
Noncompositional explanations (bottom) are not parsimonious.
Other compositional explanations (such as alternatingly naming largest cities and capitals, top right) can also explain the input, but are disfavored because they are less parsimonious.
	(D) This perspective extends to chained reasoning: the most parsimonious explanation of the prompt is in terms of a generative process combining reasoning steps.
	We illustrate this at the example of numerical reasoning (unit conversion), which can be solved by recombining two functions observed in the training data.
	In the chain-of-thought version, the intermediate step is made explicit.
 (E) Theorem~\ref{theorem:theorem1} guarantees convergence of errors to zero, as prompt length increases, when the compositional process can express $n$-fold repetition (as in (B) bottom).
 For a composite task as in (D), Theorem~\ref{theorem:cot} provides individual error bounds for the two steps in chain-of-thought prompting (D right), providing faster convergence than for direct prompting (D left).
	}
    \label{fig:compositional-setup}
\end{figure*}

\paragraph{Formalizing compositional document generation.}

We start by theoretically analyzing when ICL is possible for a predictor reflecting a linguistically-plausible generative process.
Over the past decades, the linguistics literature has proposed a substantial number of grammar formalisms intended to describe the compositional structure and distribution of sentences and text \cite[e.g.][]{Pollard1984GeneralizedPS,Joshi1985NaturalLP,Abney1996StochasticAG,Stabler1996DerivationalM,Steedman2004TheSP,Kallmeyer2010ParsingBC}. 
Rather than committing to any individual one of them, we eclectically condense key aspects into a simple formalism and analyze that one.
Our results then transfer to other formalisms in the literature; see Appendix~\ref{sec:design-choices}.
The intuitive scope of our formalism is described in Figure~\ref{fig:compositional-setup}A--B:
Text is produced from an inventory of building blocks through composition operations \citep[e.g.][]{chomsky1957syntactic,Goldberg2006ConstructionsAW}, and properties and operations can be recombined and applied to different objects \citep[e.g.][]{montague1973proper,Marcus1998RethinkingEC} (instantiated by attributes named $x,y$ in Figure~\ref{fig:compositional-setup}A).

Our formalization of linguistically-plausible compositional generative processes,
\textbf{Compositional Attribute Grammar} (CAG), consists of two components: (1) a probabilistic context-free grammar (PCFG) probabilistically generatomg derivation trees over finite sets of terminals and nonterminals  by applying a finite set of production rules, (2) a \emph{yield} operation $\mathcal{Y}$ recursively mapping trees to strings.
As outlined below, to sample a random string from the CAG, we first generate a derivation tree from the PCFG and then apply the yield operation on it.
We discuss the relation between this definition and the landscape of grammar formalisms, and how our theoretical results transfer to those, in Appendix~\ref{sec:design-choices}.

CAGs go beyond PCFGs in (i) conditioning  string generation on {\cvgarg}s passed across subtrees (such as the entities ``France'' or ``95 $m^2$'' in Figure~\ref{fig:compositional-setup}A), and (ii) allowing operations other than simple concatenation of strings derived by subtrees (Figure~\ref{fig:compositional-setup}A.4, 6). 
In a derivation tree, each node is associated with a list of {\cvgarg}s taking values in $\Omega$; its length given by the \emph{arity} $a_n \in \mathbb{N}$ of the node's nonterminal $n$.
The \emph{yield} operation $\mathcal{Y}$ recursively maps trees to strings. 
It takes into account a tree, {\cvgarg}s from $\Omega$, and a source $r$ of randomness (Figure~\ref{fig:compositional-setup}B):
\begin{equation}\label{eq:yield-general}
    \mathcal{Y}(\tau, \langle x_1,\dots,x_{a_n}\rangle, r) \in \Sigma^*, 
\end{equation}
where $\tau \in \mathcal{T}$ (the set of all derivation trees), $x_i \in \Omega$. 
For a tree consisting of only a terminal, (\ref{eq:yield-general}) is arbitrarily defined. 
For a tree with children, the yield is defined recursively: 
The yield of $\psi[\tau_1, ..., \tau_\ell]$---i.e., the tree rooted by a production rule $\psi$ with children trees $\tau_1, ..., \tau_\ell \in \mathcal{T}$---is some arbitrary concatenation of yields of children trees, $\yield(\tau_{i},\eta_j,r_j)$, where each $\eta_j$ is a tuple of {\cvgarg}s and $r_j$ are independent. Children may appear multiple times with different {\cvgarg}s; their ordering, multiplicity, and  {\cvgarg}s may depend on $\psi$, $r$, and the {\cvgarg}s $x_1, \dots, x_{a_n}$ of the parent (formal definition in Appendix~\ref{sec:yield}). 
Besides grammatical knowledge that is typically the focus in linguistic work on grammar formalisms, $\yield$ must also incorporate world knowledge that shapes the sentence distribution: In Figure~\ref{fig:compositional-setup}B (top), $\yield$ as applied to the node labeled \texttt{y=capital(x)} is responsible for passing the {\cvgarg} $y$ satisfying $y=capital(x)$ to a subtree; in Figure~\ref{fig:compositional-setup}B (bottom), it is---when applied to the \texttt{loop} node---responsible for repeating a subtree applied to different US states.

Each document $d$ in the corpus is generated by sampling a tree $\tau\in\mathcal{T}$ whose root nonterminal is a designated start symbol $\START$, with arity 0, and a random $r$ and setting $d := \mathcal{Y}(\tau, \langle\rangle, r) \in \Sigma^*$. 
We write $p(d)$ for the resulting distribution on $\Sigma^*$.
We refer to the number of nodes in a derivation tree $\tau \in \mathcal{T}$ as its \textbf{description length} $\DL[\tau]$.
For some constant $\rho>0$, $P(\tau \,|\,\START) \geq \exp(-\rho\cdot \DL[\tau])$ (Lemma~\ref{prop:prob-dl-bound}).

\paragraph{Regularity Assumptions.}
We make general regularity assumptions about the CAG, conceptually similar to those made in the HMM model of \citet{DBLP:conf/iclr/XieRL022}:
The set of derivation trees is closed under projection onto variable values, concatenation of yields (as in a standard CFG), and marginalizing of variables.
Documents have finite expected length, and all nonterminals can be used in generating some documents at probability bounded away from zero.
See Appendix~\ref{sec:assumptions} for formal definition.
These assumptions ensure that all strings in $\Sigma^*$ can be constructed at some (albeit small) nonzero probability, so that an idealized predictor makes well-defined next-token predictions on any input.

\paragraph{Iteration Complexity.}
Key to our learning bound will be the ability of CAGs to generate repetition of the same operation applied to different objects.
Natural language has various such operations (Figure~\ref{fig:compositional-setup}A), including lists (Figure~\ref{fig:compositional-setup}A.4) or the \textit{gapping} construction (Figure~\ref{fig:compositional-setup}A.6\footnote{A linguistically faithful analysis of gapping is slightly more complex than in Figure~\ref{fig:compositional-setup}A.6, see more in Appendix~\ref{sec:minimalist-iteration}.}); \citet{Ross1970GAPPINGAT}), the latter highly prominent in linguistic research and thought to elude context-free syntax \citep{Steedman1990GappingAC}.
Not all CAGs will have loop-like operations as in Figure~\ref{fig:compositional-setup}A, but they may have other compositional means of generating structures repeating an operation on different {\cvgarg}s.
We formalize this by associating to each CAG its \textit{Iteration Complexity} $R_n$:
For each $n \leq |\Omega|$, let $R_n$ be the smallest number such that 
the following holds for all $\theta \in \trees$, and all pairwise distinct $x_1, ..., x_n \in \Omega$.
We consider all trees $\tau\in\mathcal{T}$ ($a_{\tau} = a_{\theta}+1$) such that for all $\xi\in\Omega^{a_\tau}$, $\mathcal{Y}(\tau, \xi, r)$ with probability at least $p_\tau > 0$ has an infix whose distribution matches
\begin{equation}\label{eq:infix-repetition}
\mathcal{Y}(\theta,\langle x_1, \xi_{1\dots a_\tau}\rangle,r_1)...\mathcal{Y}(\theta,\langle x_N, \xi_{1\dots a_\tau}\rangle,r_n).
\end{equation}
There is always at least one such tree (Lemma~\ref{eq:always-some-repetition}).
We define $R_n$ by the requirement that, for at least one of these $\tau$,
\begin{equation}
\DL[\tau]  \leq  R_n + \DL[\theta] + \frac{1}{\rho} \log \left[p_\tau \cdot {{|\Omega|\choose n}}\right].
\end{equation}
Intuitively, $R_n$ indicates how much more complex repetition is compared to a single occurrence; the third term accounts for the number of different choices of $x_1,\dots,x_n$; it disappears in the simple case where the yield of $\tau$ contains (\ref{eq:infix-repetition}) for each sequence $x_1,\dots,x_n$ at equal probabilities $p_\tau = {{|\Omega|\choose n}}^{-1}$. 
A simple way of achieving $R_n = 1$ uses a production rule $\psi$ mapping a nonterminal to a single nonterminal, and a corresponding yield $\yield(\psi[\theta], \langle\rangle, r)$ of the form $\yield(\theta, \langle x_1\rangle, r_1) \dots \yield(\theta, \langle x_{n}\rangle, r_{n})$, with the permutation determined by $r$, as in Figure \ref{fig:compositional-setup}B bottom.\footnote{More generally, if the number $n$ of iterations produced by this nonterminal depends on $r$, with some probability distribution $p(n)$, then $R_n \leq 1-\frac{1}{\rho} \log \sum_{k=n}^\infty p(k)$ (Appendix, Example~\ref{ex:iteration-rn}).}

\subsection{Learnability Bound}
We now provide in-context learning guarantees for an idealized predictor reflecting the distribution of documents sampled from a CAG.
This autoregressive predictive distribution,
over strings over the alphabet $\Sigma \cup \{\$\}$, for each $n=1,2,\dots$, is given as:\footnote{This is defined for all $x_{1\dots n}$ except when $x_{1\dots n-1} \in \Sigma^+\$\Sigma^*$, which can never be followed by any symbol inside a document.}
\begin{equation}\label{eq:predictive}
    M(x_n|x_{1\dots n-1}) = \frac{\sum_{d\in\Sigma^*} p(d) \cdot \#_d(x_{1\dots n})}{\sum_{d\in\Sigma^*} p(d) \cdot \#_d(x_{1\dots n-1})} 
\end{equation}
where $\#_d(x_{1\dots n})$ is the number of times $x_{1\dots n}$ appears in $\$d\$$ (with \$ serving as beginning and end of sequence token). 
For longer strings, $M(x_{n\dots n+\Delta}|x_{1\dots n-1}) := \prod_{i=n}^{n+\Delta} M(x_i|x_{1\dots i-1})$.

Our learning bound is in terms of the description length of defining a function within the CAG. 
Formally, we say that a function $\func : \Omega \rightarrow \Omega^*$  
is \emph{expressed} by a derivation tree $\tau_\func$ with description length $\DL[\tau_\func]$ if: 
\begin{equation}\label{eq:f-expressible}
\yield(\tau_\func, \langle x\rangle, r) \equiv \spell{\func(x)}, \forall  r, \forall x \in \Omega
\end{equation}
For instance, ``capital'' or unit conversion are expressed by subtrees of description length $2$ in the examples in Figure~\ref{fig:compositional-setup}A--B. 
With these notions in place, we state our first theorem:

\begin{theorem}[Single-Step Prompting]\label{theorem:theorem1}
Let any CAG be given, satisfying the regularity assumptions, including the associated trees $\trees$, yield map $\yield$, and predictive distribution $M$, with the associated quantities $R_n$.
	Let $\func : \Omega \rightarrow \Omega^d$ be a function expressed by a derivation tree $\tau_\func\in\trees$.
 Let $\xi := x_1, x_2, ..., x_n \in \Omega$ ($n \leq |\Omega|$) be a sequence with the $x_i$ pairwise distinct, and let $s \in \Sigma$.
For $m=1,\dots,n$, consider the prompt $P_m$ given by
	\begin{equation}
		 \overline{x_1 \func(x_1)} s \spell{x_2 \func(x_2)} s \dots s \spell{x_{m-1} \func(x_{m-1})} s \spell{x_{m}},
	\end{equation}
 with expected completion $\spell{\func(x_m)}$.
Assume that predictions are made as
\begin{equation}
     \operatorname{arg} \operatorname{max}\limits_{\omega \in \Sigma^d} M(\omega s|P_m),
\end{equation}
with ties broken arbitrarily. 
On average across the choice of the sequence $x_1, x_2, ..., x_n$ (picked uniformly at random from length-$n$ sequences with pairwise-distinct entries), the summed zero-one loss on completing $P_1, ..., P_n$,
is bounded by
 \begin{equation}\label{eq:learning-bound}
  \mathcal{O}\left(R_n + \DL[\tau_\func]\right)
 \end{equation}
where $\mathcal{O}(\cdot)$ absorbs constants depending on the PCFG, $s$, and the average document length $\mathbb{E}[|d|]$, but not otherwise on $|\Omega|$, $\func$, or $n$.
\end{theorem}
We provide the proof in Appendix~\ref{proof:theorem1}.

\paragraph{Remarks.}
Dependence of (\ref{eq:learning-bound}) on $R_n$ cannot in general be avoided (Appendix~\ref{sec:optimality-bound}).
The bound~(\ref{eq:learning-bound}) is information-theoretic in nature, considering the idealized predictor (\ref{eq:predictive}).
We take up the empirical behavior of real transformers pretrained on finite data in Section~\ref{sec:experiments}.

Equation~\ref{eq:learning-bound} absorbs constants that depend on the PCFG backbone, but not on $|\Omega|$. Thus, while the bound might end up vacuous when $|\Omega|$ (and thus the maximum prompt length) is small, it will always become nonvacuous when taking $|\Omega|$ to infinity while fixing the PCFG.

Various variants and extensions can be proven with the same approach.
We assumed that the response has a fixed, known length $d$, because we did not make any assumptions about the separator.
An analogous theorem holds if the length of the response is unknown a priori, but the separator $s$ does not occur in any $\spell{\func(x)}$. Even the length of $\spell{x}$ may be taken as flexible if their set is prefix-free.
The bound is robust to changes in the prompt format, such as adding symbols between $x$ and $\func(x)$.
The function $\func$ can also be taken as stochastic; in this case, a regret bound comparing to an oracle predictor that knows the task from the start holds, with an analogous proof (Appendix~\ref{sec:stochastic-functions}).
An analogous statement further holds for functions $\func$ with multiple arguments, though an adapted definition of $R_n$ is then needed; we include such functions in our experiments (Section~\ref{sec:test-tasks}).

\paragraph{Proof Intuition.}
The intuition of the proof is that an optimal predictive model $M$ implicitly identifies the generative process $\tau\in\mathcal{T}$ underlying the prompt, in order to predict the next token (Figure~\ref{fig:compositional-setup}C).
One possibility is that the prompt was generated in some unstructured manner as a random concatenation of symbols (Figure~\ref{fig:compositional-setup}C bottom); another possibility is that the recurrence of pairs $(\spell{x},\spell{\func(x)})$ throughout the string is no coincidence, and that a generative process generating a prompt-like structure underlies it (Figure~\ref{fig:compositional-setup}C center).
If $M$ was trained on an unstructured corpus, there is no reason to prefer the second hypothesis:
appearance of structure is likely to be a coincidence under the corpus-generating process, and there is no reason to extrapolate it to future tokens.
On the other hand, when the pretraining data was generated by a compositional process (as in Figure~\ref{fig:compositional-setup}A), the most parsimonious explanation of the very peculiar format of the prompt $P_n$ is as structured repetition of a single operation, leading $M$ to predict $\spell{\func(x_{m+1})}$.
The key quantities modulating the preference for the second explanation are $\DL[\tau_\func]$ and $R_n$:
the smaller these are, the greater the advantage in parsimony of a structured explanation, and the more strongly $M$ will predict the pattern to continue, i.e., predict $\spell{\func(x_{m+1})}$ as the next token.
A more complex $\func$, as in  Figure~\ref{fig:compositional-setup}D (left) may take more examples, but will nonetheless ultimately be learned: every prediction error provides some information about the function $\func$; the number of errors is thus bounded by the complexity of the structure underlying the prompt.

\paragraph{Role of Iteration Complexity.}
The key to in-context learning is the parameter $R_n$, which measures how complex repetition of an operation is:
the slower the growth of $R_n$ with $n$, the better the error bound on ICL.
If $R_n = o(n)$, the error on $P_n$ must converge to zero as $n$ increases (Figure~\ref{fig:compositional-setup}E).
This capacity is unavailable in pure PCFGs, for which $R_n \equiv +\infty$ for $n>1$ (Appendix~\ref{sec:pcfg-no-icl})\footnote{Intuitively, this follows from the fact that the copy language $\{ww : w \in \Omega^*\}$ is not context-free.}, but it arises in mildly context-sensitive languages thought to be appropriate to natural language syntax; an example is the \textit{gapping} construction in Figure~\ref{fig:compositional-setup}A.6 \citep{Kallmeyer2010OnMC}.
See Appendix~\ref{sec:minimalist-iteration} for more on the linguistic background.

\paragraph{Description Length.}
In computing description lengths in Figure~\ref{fig:compositional-setup}, we assumed that concepts such as ``capital'', ``biggest city'', or ``USD to RMB'' were given by atomic nonterminals in the generative process; however, some of these operations might themselves be best thought of as composed (e.g., ``USD to RMB'' from ``multiplication'' and ``exchange of unit symbols'', or ``biggest city'' from ``city'' and ``populations''), accordingly impacting description length.
Our learning bound is stated in terms of the description length within the formal system, remaining agnostic about the description lengths of any of these specific real-world concepts.

\paragraph{(Un)Natural Prompts and Semantic Priors}
Theorem~\ref{theorem:theorem1} is stated for prompts that simply concatenate inputs $x_i$ and outputs $\func(x_i)$, but the statement holds equivalently for other regularly structured prompts.
Real-world LLMs are often prompted with more naturalistic prompts (``\texttt{volleyball is a sport  {\textbackslash}n onions are}...''), but can also deal with other prompt formats (``\texttt{volleyball: sport  {\textbackslash}n  onions: food}...'') \citep{rong21extrapolating,DBLP:journals/corr/abs-2202-12837} and can, at least when they are sufficiently large, even learn unnatural or permuted input-output mappings (``\texttt{volleyball: animal {\textbackslash}n onions: sport {\textbackslash}n broccoli: sport}...'') \citep{rong21extrapolating,Wei2023LargerLM}.
Indeed, the proof of Theorem \ref{theorem:theorem1} provides more or less favorable bounds for more or less natural prompts:
the bound in Equation~\ref{eq:learning-bound} is derived by bounding the cross-entropy that $M$ incurs on predicting all tokens in the prompt $P_{n+1}$.
Higher probability of natural examples  compared to less natural or even unnaturally permuted ones propagates to increased probability assigned by $M$ to the entire prompt, and thereby faster convergence of ICL.
The theory thus predicts correctly that naturalistic prompts are more successful than unnatural ones, but simultaneously that sufficiently strong predictive models can ultimately override semantic priors favoring natural completions \citep{Wei2023LargerLM}.
The relation of the error bound~(\ref{eq:learning-bound}) to the cross-entropy on the prompt also explains the empirical observation that prompts assigned higher LLM likelihood tend to lead to better ICL results \citep{Gonen2022DemystifyingPI}.

\subsection{Chain-of-Thought Prompting}\label{sec:cot-theory}

Empirical research has observed that ICL for complex tasks benefits when models are prompted to provide intermediate steps before the answer \citep[e.g.][\textit{chain of thought prompting}]{DBLP:journals/corr/abs-2112-00114, Wei2022Chain,Suzgun2022ChallengingBT}. 
We formally study this in the simple context of computing composed functions $\func_1 \circ \func_2$.
Here, chain-of-thought prompting conceptually corresponds to prompting the model to output both an intermediate step $\func_1(x_{n})$ and the result $\func_2(\func_1(x_{n}))$ (Figure~\ref{fig:compositional-setup}D).
Applying Theorem 1
to either direct prompting or a version with the intermediate step results in a bound depending on $\DL[\tau_{\func_1 \circ \func_2}]$. 
We now show a better bound for the chain-of-thought version, where the intermediate step is provided before the answer:
the error in each of the two steps can be bounded individually by the description of only one function. 
While one cannot, without further assumptions, expect a bound that holds pointwise for each pair $\func_1, \func_2$, we prove a bound that holds \emph{pointwise} on the component of interest and \emph{on-average} on the other component:
\begin{theorem}[Chain-of-Thought Prompting]\label{theorem:cot}
Let any CAG be given, satisfying the regularity assumptions, including the associated trees $\trees$, yield map $\yield$, and predictive distribution $M$, with the associated quantities $R_n$.
    Let $\func_1 : \Omega \rightarrow \Omega$ be a function expressed by a derivation tree $\tau_{\func_1}\in\trees$.
    Let $\func_2 : \Omega \rightarrow \Omega$.
    Let $s \in \Sigma$.
    Consider the prompt
    \begin{equation}\label{eq:first-prompt}
    P^{(1)}_n = \overline{x_1 \func_1(x_1) \func_2(\func_1(x_1))} s \dots s \overline{x_m \func_1(x_m) \func_2(\func_1(x_m))} s \overline{x_{m+1}}
    \end{equation}
    with expected completion $\overline{\func_1(x_{m+1})}$,
    or
   \begin{equation}\label{eq:second-prompt}
    P^{(2)}_m = \spell{x_1 \func_2(x_1) \func_1(\func_2(x_1))} s \dots \spell{x_m \func_2(x_m) \func_1(\func_2(x_m))} s \spell{x_{m+1} \func_2(x_{m+1})}
    \end{equation}
    with expected completion $\overline{\func_1(\func_2(x_{m+1}))}$.
    On average across arbitrary functions $\func_2$ and across pairwise distinct sequences $x_1, \dots, x_n \in \Omega$, and summed over $m=1, \dots, n$, the zero-one-error on each of the two prompts is bounded by
    \begin{equation}
        \mathcal{O}(R_n + \DL[\tau_{\func_1}])
    \end{equation}
    with constants depending on the PCFG, $s$, and $\mathbb{E}[|d|]$, but not $\func_1$, $|\Omega|$, or $n$.
\end{theorem}

We prove this in Appendix~\ref{sec:proof-cot}.
The proof idea is that in each step, the other function can be effectively ignored in inferring the compositional process.
While we focus on the composition of two functions, an analogous statement and proof hold for longer composition chains.

In the worst case, the two components could make errors on disjoint sets of inputs, so that the errors would add up, giving the same asymptotics as without chain-of-thought prompting.
But in the more realistic situation where both errors are high for short prompts and then go to zero once the tasks are identified, the overall error will just be the larger one of the two component tasks' errors (Figure~\ref{fig:compositional-setup}E right).
This indeed is close to what happens in our experiments (Figure~\ref{fig:cot-by-prompt}).

This theoretical benefit is not available when providing the intermediate step \emph{after} the solution, as an \emph{explanation} rather than a chain-of-thought.
In this version, the first step amounts to solving the composed task in one go, leading to an error bound only in terms of $\DL[\tau_{\func_1\circ\func_2}]$.
Indeed, in real-world LLMs, providing intermediate steps before the answer seems to be much more effective than providing it after the answer \citep{Wei2022Chain,DBLP:conf/emnlp/LampinenDCMTCMW22}.

\subsection{Comparison to \citet{DBLP:conf/iclr/XieRL022}}\label{sec:comparison-xie}
We discuss how this theoretical analysis of ICL relates to and differs from the analysis proposed by \citet{DBLP:conf/iclr/XieRL022}.
They model the pretraining data as a mixture of HMMs, and cast ICL as Bayesian identification of one of these mixture components (an analogous idea is sketched by \citet{yang203large}). 
Each HMM mixture component is thought to describe some type of text, e.g., Wikipedia biographies, newspaper articles, or tweets.
A prompt (e.g., \texttt{Albert Einstein was German {\textbackslash}n Mahatma Gandhi was Indian {\textbackslash}n Marie Curie was}) is then identified as a concatenation of text samples resembling some of these components (e.g., biographies).
Our analysis likewise can be understood in terms of Bayesian inference, in that $M$ implicitly identifies a generative process $\tau$ underlying the prompt.
The most important difference between our analysis and that of \citet{DBLP:conf/iclr/XieRL022} is that we aim to account for the flexible and open-ended nature of prompting capabilities in LLMs by leveraging the compositional nature of natural-language data: 
Whereas \citet{DBLP:conf/iclr/XieRL022} analyzed the task of recovering one of a \emph{fixed space} of HMMs which made up the training corpus, we explain ICL as identifying a task from an \emph{open-ended} hypothesis space of tasks \emph{compositionally recombining} operations found in the training corpus.
Whereas \citet{DBLP:conf/iclr/XieRL022} focused their discussion of ICL on entity-property associations (e.g., nationalities), our approach makes it possible to study ICL on tasks of varying complexity and structure within a single framework (Figure~\ref{fig:compositional-setup}D), including variants such as chain-of-thought prompting (Section~\ref{sec:cot-theory}).

The theorems in \citet{DBLP:conf/iclr/XieRL022} study general recoverability of HMM mixture components from a single sample in the presence of repeated low-probability transitions (caused by the prompt structure), with few assumptions about the HMM. 
However, the application of these theorems to explaining ICL of classification tasks (e.g. mapping famous people to their nationalities) relies on an important additional assumption, which closely relates to Iteration Complexity: 
HMM transitions happen within each mixture component (e.g., Wikipedia articles about people), but not between different mixture components (e.g., articles about people; articles about cities; tweets).
This encodes an implicit modeling assumption that similar text about different entities (e.g., biographies of different people) tends to appear contiguously in the pretraining corpus.
This assumption is a specific form of the more general idea that the generative process underlying natural language can produce repetition of the same operation applied to different entities, formalized by Iteration Complexity $R_n$.
Similar assumptions implicitly underlie other work aiming to empirically induce ICL in controlled setups  by training on prompt-like inputs \citep{DBLP:journals/corr/abs-2208-01066,DBLP:journals/corr/abs-2205-05055,DBLP:journals/corr/abs-2211-15661}.
We will experimentally compare such training data with CAG-based pretraining data.

\section{Experiments}\label{sec:experiments}

\subsection{Training Datasets}

\begin{figure*}
\centering

\includegraphics[width=0.98\textwidth]{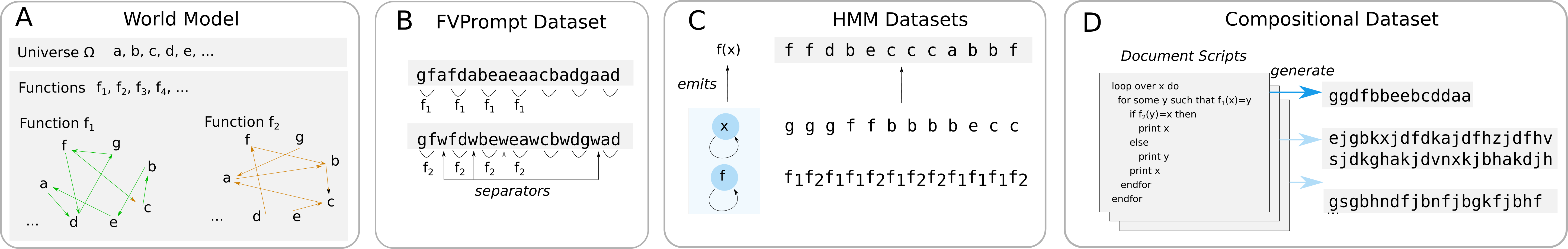}

	\caption{Setup and datasets.
	(A)
We assume a logical model $\mathcal{M}$ consisting of a universe $\Omega$ and a set of functions, visualized here as directed graphs.
	(B--D) We consider three types of generative processes for the training dataset used in pretraining.
	(B) The \textsc{FVPrompt} dataset directly encodes the functions in a prompt-based format.
	(C) The \textsc{HMM} datasets, adapting GINC \citep{DBLP:conf/iclr/XieRL022}, are generated by HMMs whose state space consists of object-function pairs.
	(D) In the \textsc{Compositional} dataset, documents are generated by scripts that have access to the world model and contain arbitrary but compositionally structured instructions; they correspond to a very simple CAG.
	}\label{eq:setup}
\end{figure*}

We have described an information-theoretic analysis characterizing when broad ICL capabilities become possible for an idealized predictive model reflecting a linguistically motivated string distribution.
We now empirically verify whether the predicted behavior can emerge in transformer models pretrained on finite data sampled from a CAG.
We define a suite of in-context learning tasks, and benchmark transformers pretrained on several types of controlled miniature datasets: some modeled closely on prior work, and one representing a minimal CAG.

\paragraph{World Models.}
All training datasets are based on a world model $\mathcal{M}$ consisting of a finite universe $\Omega$ and a set $\Ff$ of functions $f : \Omega \rightarrow \Omega$. 
We focus on functions of arity~1 for simplicity, and will logically define more complex functions in terms of these. 
In creating documents, we make the simple assumption that $\Omega=\Sigma$ and $\omega=\spell{\omega}$.

\paragraph{Function Value Prompts.} 
Our first dataset (\textsc{FVPrompt}, Figure~\ref{eq:setup}B) has a prompt-based format: pairs of $x$ and $f(x)$, for functions $f \in \Ff$ (fixed within a document), with an intervening separator randomly chosen (fixed within a document) from $\Omega$. 
Each prompt includes all $x\in\Omega$ in random order.
Analogous training datasets have been  used in prior work experimentally inducing ICL in constrained setups \citep{DBLP:journals/corr/abs-2205-05055,DBLP:journals/corr/abs-2208-01066,DBLP:journals/corr/abs-2211-15661,DBLP:journals/corr/abs-2301-07067}.

\paragraph{HMM5 and \textsc{HMMPerDoc}.}
Our next dataset (\textsc{HMM5}, Figure~\ref{eq:setup}C) closely follows the GINC dataset proposed by \citet{DBLP:conf/iclr/XieRL022}:
the dataset is generated by a mixture of five HMMs; each HMM state has the form $\langle x,f\rangle \in \Omega\times\Ff$ and emits $f(x)$; the two components evolve independently according to separate transition matrices.
The transition matrices are defined as in \citet{DBLP:conf/iclr/XieRL022}; we provide these in Appendix~\ref{sec:HMM} for reference.
This dataset has more diversity than \textsc{FVPrompt}, but still less than the unbounded state space arising from a PCFG-based compositional system. 
We thus considered a variant, \textsc{HMMPerDoc}, where separate permutation matrices were sampled for each document, and no mixing or averaging was applied (Appendix~\ref{sec:HMM}).

By assuming a fixed finite state space, the generative processes underlying \textsc{HMM5} and \textsc{HMMPerDoc} cannot express unbounded composition.
However, due to the specific factorized design, they incorporate a bias towards repeating the same operation (in this case, sequences of evaluations $f_i(x)$) on different objects; it can be viewed as a simple instantiation of loop operations.
Indeed, this modeling choice is key to \citet{DBLP:conf/iclr/XieRL022}'s argument for this as a model of ICL, whereby ICL arises because similar text about different entities (e.g., Wikipedia articles about different people) tends to appear contiguously in the pretraining corpus.

\paragraph{Compositional Document Scripts.}
Finally, we define a minimal CAG (\textsc{Compositional}, Figure~\ref{eq:setup}D), including minimal language features needed for our theoretical analysis:
a loop construct, a construct introducing new attributes standing in functional relationships to existing attributes---so that functions $f \in \Ff$ can be defined, a terminal that outputs the value of an attribute, and a conditional construct (``if-then-else'').
We ablate each component below.

This can be intuitively described as a minimalistic programming language; we'll refer to the derivation trees as \emph{document scripts} and write them as programs (Figure~\ref{fig:backus-naur}).
Attributes correspond to variables in a script.
Loops are executed for 10 random objects $\omega\in\Omega$.
The ``for some'' statement selects a random satisfying object if more than one exists, and is skipped if none exist.
The syntax tree of a script corresponds to the derivation tree $\tau$, and the stochastic map from scripts to output strings corresponds to the yield function $\yield$ in a CAG (see Appendix~\ref{sec:scripts-are-cvcg}).
Due to the ``for all'' construct, $R_n=1$ for $n \leq 10$.
In order to sample scripts for document generation, we defined a PCFG-like distribution over syntactically valid scripts (corresponding to the PCFG backbone of CAGs), favoring scripts generating documents within our LM's context length (64).
See Appendix~\ref{sec:acc:generating-scripts} for details,

We provide examples in Appendix~\ref{sec:sample-docs}.
Intuitively, this generative process represents agents that have access to the world model and produce text according to arbitrary but compositionally structured instructions.
Unlike natural language, the documents do not share systematic generalizations such as vocabulary or grammar, nor do they have function words indicating structural relations between words.
They make up for this lack of structure by containing increased amounts of repetition within a document. 
We discuss the relation to and differences from natural language in Section~\ref{sec:discussion}.

Our research questions are:
\begin{enumerate}
\item Does ICL appear when training real-world transformers on finite data generated from a minimal CAG? How do models trained on CAG data compare to models trained on the other datasets?
\item Are the predictions of Theorems 1--2 borne out, i.e., effect of description length, dependence on $|\Ff|$, independence from $|\Omega|$, advantage of chain-of-thought prompting?
\item In ICL, can transformers recombine operations never seen together during pretraining?
\item What are the dynamics of emergence?
\end{enumerate}

\subsection{Training Setup}

\paragraph{Models.}
We train GPT2-like \citep{Radford2019LanguageMA} models for next-token prediction using the HuggingFace transformers library \citep{Wolf2019Huggingface}.
Varying the numbers of dimensions, layers, and heads (Table~\ref{tab:model-sizes}), we considered models with 14M (small, 2 layers), 21M (medium, 3 layers), 42M (large, 6 layers), and 85M (XL, 12 layers) parameters. 
See Appendix~\ref{sec:training-details} for training details.

\begin{table}
	\centering
	\begin{tabular}{l|lllll}
   	       & $d$ & Heads & Layers & Parameters \\ \hline
		Small  & 64 & 2 & 2 & 14M \\ 
		  Medium & 128 & 2 & 3 & 21M \\ 
		Large  & 256 & 8 & 6 & 42M \\ 
		XL & 768 & 12 & 12 & 85M \\ 
\end{tabular}

	\caption{Model Sizes. The XL model has the same architecture as GPT2-Small, but a smaller vocabulary.}\label{tab:model-sizes}
\end{table}

\paragraph{Worlds.}
We focus on $|\Omega|=30$ and $|\Ff|=10$, and additionally varied $|\Ff|=5,10,20,30$ and $|\Omega|=30, 100,300$.
Functions were created randomly; $f_1$ was the identity function.
Furthermore, 
we designated functions $f_2$, $f_3$ such that no script included both $f_2$ and $f_3$: we did this in order to test whether models could generalize to contexts simultaneously requiring knowledge of both functions, which cannot be generated by scripts in the training distribution.

For each world $\mathcal{M} = \langle\Omega,\Ff\rangle$, we generated 500M tokens of training data, for each of the four data-generating procedures.
In the main setup, we additionally created five more worlds  and trained medium-size models on compositional data to verify robustness to sampling of the world (Appendix~\ref{sec:stability-world}).
Documents were concatenated, separated by a \textsc{StartOfSequence} symbol.
Data was fed to the model in portions of 64 tokens, across document boundaries.
Training was performed for up to 20 epochs, or until held-out cross-entropy stopped improving, which happened earlier for noncompositional datasets.

\begin{figure*}
\centering

\includegraphics[width=0.98\textwidth]{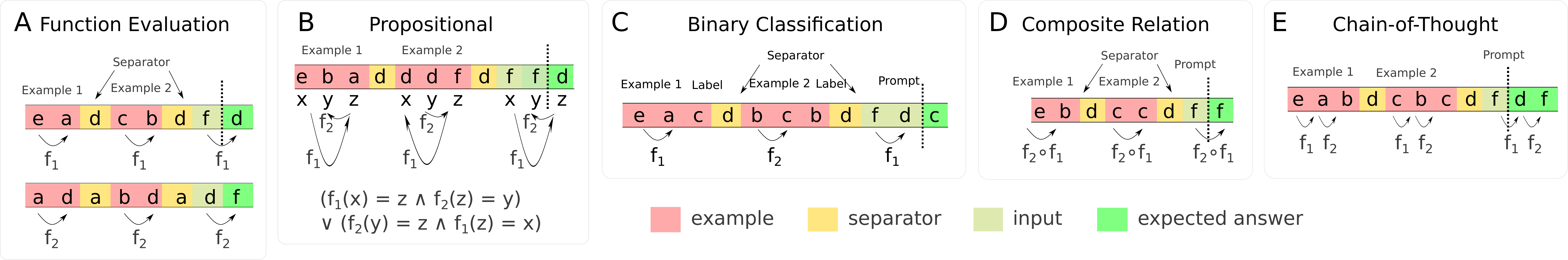}

	\caption{Test tasks.	(A) The \textsc{FunctionEvaluation} task prompts the model to apply a function to a given element.
	(B) \textsc{Propositional} tasks generalize this by asking the model to find a (unique) $z$ satisfying a relation $\varphi(x,y,z)$, where $x,y$ are the inputs. 
	(C) \textsc{Binary} tasks prompt the model to decide which of two funcional relationships holds between a pair of objects. 
	(D) The \textsc{Composition} Task prompts the model to compute the composition of two functions.
	(E) In the \textsc{Chain-of-Thought} version, the model is prompted to also produce an intermediate step. 
	Across (A--E), the models need to identify the task from the prompt without any further training or instruction.
	Note that there are no separate types reserved for labels or separators, and the models have no access to markup indicating the components of the prompt. Similar to large-scale LMs, the models need to figure out the structures of different prompts on-the-fly.
	}\label{fig:tasks}
\end{figure*}

\subsection{Test Tasks}\label{sec:test-tasks}

In order to evaluate ICL capabilities, we collected a suite of test tasks that are definable relative to the world $\mathcal{M} = (\Omega, \Ff)$ (Figure~\ref{fig:tasks}).
Tasks are defined using first-order logic formulas $\varphi$ with literals of the form $f(X)=Y$ ($f \in \Ff$), with
free variables matching the inputs $x \in \Omega$ (or $x,y \in \Omega$) and the expected output $z$. 
Given inputs $x \in \Omega$ (or $x,y$), the expected output $\func(x)$ (or $\func(x,y)$) is the $z$ satisfying the formula.
All prompts are chosen so that $z$ is uniquely determined for all included $(x,y)$, including the examples.
The simplest task is the \textsc{FunctionEvaluation} task (Figure~\ref{fig:tasks}A):
given $x \in \Omega$, compute $\func(x):=f_i(x)$---i.e., the $z$ such that $\varphi(x,y,z) \equiv (f_i(x)=z)$ holds.

\paragraph{Prompt Format.} 
We encoded all test tasks into a prompt-based format for evaluating ICL.
Across tasks, examples are separated by a random (but fixed within a prompt) separator $s \in \Omega$. 
When there are two input variables, 
each prompt is of the form 
\begin{equation}
    x_1 y_1 z_1 s x_2 y_2 z_2 s \ldots x_k y_k z_k s x_{k+1} y_{k+1}
\end{equation}
where each tuple $(x_i,y_i,z_i)$ satisfies the formula $\varphi(x_i,y_i,z_i)$ defining the task.
When there is one input variable, each $y_i$ is omitted.

\paragraph{Propositional.}
Our next set of tasks is defined by first-order formulas without quantifiers (see Appendix~\ref{sec:task-formulas} for list). Besides the function evaluation task, its \textsc{Inverse} (given $x \in \Omega$, output some $z$ such that $f(z)=x$) is also definable using one literal.
We next constructed more complex formulas with two input variables $x,y$.
Each formula was in DNF, such that each term had 2 literals.
We choose 2 literals because this allows encoding the functional relationships between the three variables $x,y,z$.
The number of literals in $\varphi$ provides a proxy for the description length $\DL[\tau_\func]$ in the minimal CAG (Appendix~\ref{sec:test-representability}).
For instance, one task (``missing link'', Figure~\ref{fig:tasks}B, \#2 in Appendix~\ref{sec:task-formulas}) assumes that either $x=f_i(f_j(y))$ or $y=f_i(f_j(x))$; $z$ then is the intermediate element (either $f_j(y)$ or $f_j(x)$); defined by
$\varphi(x,y,z) \equiv (f_j(x) = z \wedge f_i(z) = y) \vee (f_j(y) = z \wedge f_i(z) = x)$.

\paragraph{Composed.}
We furthermore considered a set of tasks that require reasoning about an unobserved variable, or, equivalently, require evaluating composed functions in one go (see Appendix~\ref{sec:task-formulas} for list).
One example (\textsc{Composition}, Figure~\ref{fig:tasks}C) is defined as $\varphi(x,z) \equiv \exists a : a=f_i(x) \wedge z=f_j(a)$; here, $z = f_j(f_i(x))$.

\paragraph{Binary Classification.}

Beyond multi-class classification, LLMs can solve novel classification and reasoning tasks.
To test for the emergence of such abilities, we created tasks that require discriminating between two types of examples and output a binary label $\ell_1, \ell_2 \in \Omega$ (fixed within a prompt).
For instance, the \textsc{RelationClassification} Task (Figure~\ref{fig:tasks}D) asks the model to decide 
 whether $y=f_i(x)$ ($z=\ell_1$) or $y=f_j(x)$ ($z=\ell_2$), assuming exactly one of these is true.

\paragraph{Experimental Details.}
All functions $f_i, f_j, ...$  are varied across prompts, but fixed  within a prompt.
We first exclude the designated functions $f_2,f_3$, and later evaluate performance on them separately.
We also excluded the identity function $f_1$.
Each input $x$ or $(x,y)$ appears at most once within a prompt.
Inputs where the answer is either ambiguous or undefined (e.g., in the \textsc{Inverse} task, if $z$ is the image of zero or two elements under $f_i$) were excluded.

We evaluate on prompts with an even number of examples, ranging from $2$ to $14$; this exhausts the LM's context size for some tasks.
In the binary classification tasks, the number of examples was balanced between the classes.
In the \textsc{Propositional} or \textsc{Composed} tasks involving disjunction, all disjuncts were represented equally, up to a difference of at most one to make up for non-divisibility.

Importantly, there are no separate types reserved for labels or separators, and the model has no access to markup indicating the components of the prompt: like real-world LLMs, models are asked to figure out the structures of different prompts on-the-fly.

\subsection{Results}

\paragraph{Compositional training dataset enables ICL on composed tasks.}
Figure~\ref{fig:comparison-baselines} shows results across tasks for the four training datasets, as a function of the training steps (\# of processed tokens), for the models with 85M parameters.
As expected, models trained on \textsc{HMM5} cannot solve the ICL tasks, and models trained on \textsc{FVPrompt} can only solve the function evaluation task.
Models trained on \textsc{HMMPerDoc} achieve above-chance performance on the \textsc{Propositional} tasks, though not on the \textsc{Binary} or \textsc{Composed} tasks.
For models trained on \textsc{Compositional} 
near-perfect accuracy is achieved on \textsc{FunctionEvaluation} and other tasks with few literals, but not on \textsc{Binary} tasks for which above-chance accuracy is achieved. 
\textsc{Binary} tasks and tasks with many literals are the most difficult.

Focusing on models trained on \textsc{Compositional}, we next investigated how accuracy scales with various parameters, focusing on the diverse \textsc{Propositional} tasks.

\paragraph{Scaling with prompt length.}
Theorem~\ref{theorem:theorem1} provides a bound on the summed errors across prompt lengths, guaranteeing faster convergence for functions with small description length.
Figure~\ref{fig:scaling}A shows that accuracy increases with prompt length.
In agreement with the theorem, longer prompts tend to be necessary for tasks with more literals (i.e., higher description length).

\paragraph{Increasing $|\Ff|$ makes ICL harder, increasing $|\Omega|$ does not.}
Another prediction of Theorem~\ref{theorem:theorem1} concerns the size of the world model.
Recall that Equation~\ref{eq:learning-bound} depends on the PCFG, $R_n$, $s$, and $\DL[\tau_\func]$.
As each $f \in \Ff$ has its own production (Figure~\ref{fig:backus-naur}), dependence on the PCFG is less favorable when $|\Ff|$ increases even if $\DL[\tau_\func]$ stays the same.
On the other hand, in our experimental setup, none of these parameters change with $|\Omega|$. 
We thus expect that increasing $|\Ff|$  should make tasks harder, but increasing $\Omega$ should not.\footnote{Equation~\ref{eq:learning-bound} depends on $s$. As our experiments randomize $s$ over $\Omega$, the PCFG is decoupled from $s$,  removing this dependence. See Remark 2 in Appendix~\ref{proof:theorem1}.}
This was borne out when re-fitting at $|\Ff|= 20, 30$  (Figure~\ref{fig:scaling}B), and at $|\Omega|=100, 300$ (Figure~\ref{fig:scaling}C).
Indeed, ICL accuracy \emph{improved} on \textsc{Propositional} (no improvement on other tasks, see Figure~\ref{fig:additional}) when increasing $|\Omega|$; most remarkably, tasks with eight literals are now learnt at almost perfect accuracies at the same prompt length (see Appendix~\ref{sec:effect-of-omega} for explanation).

\paragraph{Emergence with data and model size.}
We next evaluated the role of model size (Figure~\ref{fig:scaling}E).
Two observations are salient.
First, increasing model size leads to smaller gains on tasks with few literals (accuracy was already high for these tasks), and large gains on tasks with many literals.
Second, Figure~\ref{fig:scaling} show that accuracy follows a pattern of sudden emergence over the course of pretraining---in particular, in those cases where very high accuracy is ultimately reached (as when $|\Omega|=300$, Figure~\ref{fig:scaling}B): a period of flat at-chance performance precedes a sudden increase in accuracy, followed by a mostly flat phase.
This stands in contrast with the evolution of pretraining cross-entropy, which decreases continuously (Figure~\ref{fig:scaling}D).

\begin{figure*}
\centering
\includegraphics[width=0.65\textwidth]{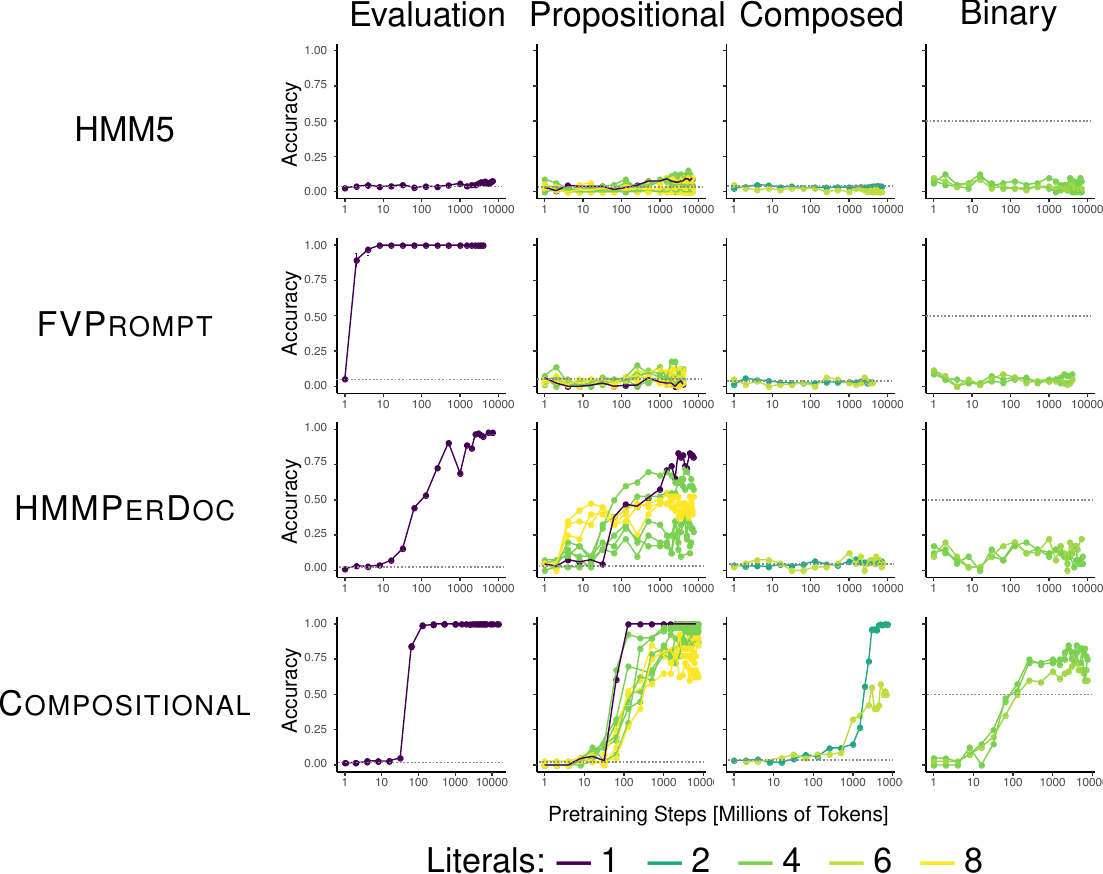}

	\caption{Compositional training data enables emergent in-context learning.
 For each pretraining dataset (rows) and each group of test tasks (columns), we show accuracy (y-axis) as a function of the number of tokens processed in pretraining (x-axis, in millions); one epoch corresponds to 500M tokens.
	Models were trained until held-out loss stopped improving but at most for 20 epochs ($10^{10}$ tokens).
	All results at 85M parameters and the longest prompt length (14 examples).
	Dotted lines indicate chance accuracy.
	The \textsc{HMM5} dataset does not lead to ICL.
	The \textsc{FVPrompt} dataset does not lead to generalization beyond the simple function evaluation task.
	With \textsc{HMMPerDoc} and \textsc{Compositional}, above-chance performance on composed tasks becomes possible.
	\textsc{Compositional} achieves above-chance accuracy on all task groups.
	}\label{fig:comparison-baselines}
\end{figure*}

\begin{figure*}
	\centering
\includegraphics[width=0.8\textwidth]{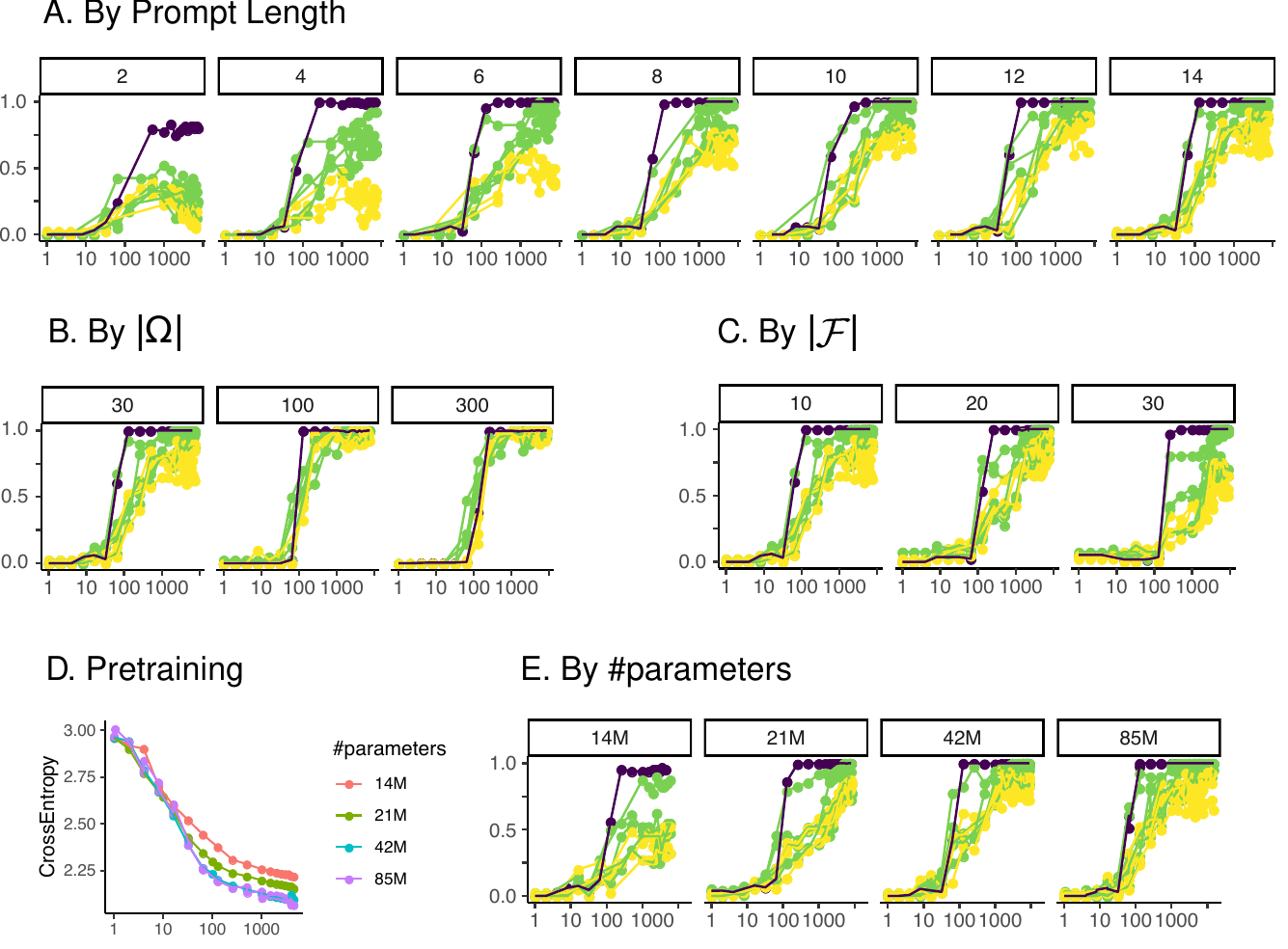}

\definecolor{col0}{HTML}{440154}
\definecolor{col1}{HTML}{46286D}
\definecolor{col2}{HTML}{414487}
\definecolor{col3}{HTML}{3B5E8B}
\definecolor{col4}{HTML}{2A788E}
\definecolor{col5}{HTML}{2B9089}
\definecolor{col6}{HTML}{22A884}
\definecolor{col7}{HTML}{5ABC6D}
\definecolor{col8}{HTML}{7AD151}
\definecolor{col9}{HTML}{C0DD40}
\definecolor{col10}{HTML}{FDE725}

	\textsf{Literals:	\textbf{\textcolor{col0}{---}} 1\ \ \ 
	\textbf{\textcolor{col6}{---}} 2\ \ \ 
	\textbf{\textcolor{col8}{---}} 4\ \ \ 
	\textbf{\textcolor{col9}{---}} 6\ \ \ 
	\textbf{\textcolor{col10}{---}} 8}

	\caption{Scaling of accuracy on the \textsc{Propositional} tasks. 
 As in Figure~\ref{fig:comparison-baselines}, the $x$-axis denotes the number of tokens processed in pretraining. otherwise stated, prompt length is 14, $|\Omega|=30$, $|\Ff|=10$, model size 85M parameters.}\label{fig:scaling}
\end{figure*}

\begin{figure*}
	\centering
       \begin{tikzpicture}
		\node (label) at (0,2)[draw=none, align=left, anchor=center]{\includegraphics[width=0.45\textwidth]{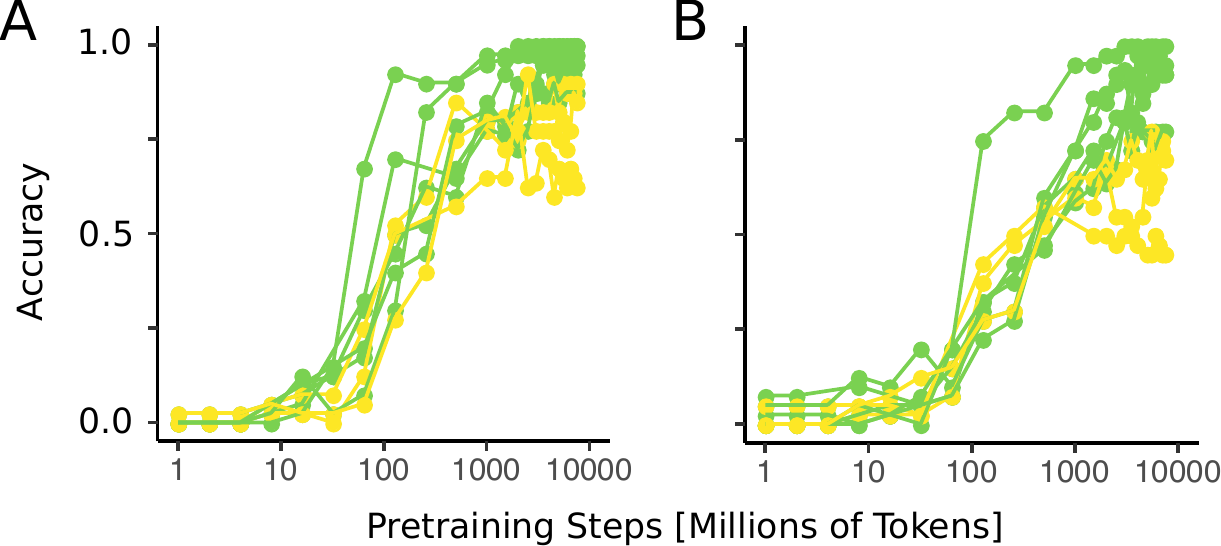} };
\end{tikzpicture}

\definecolor{col0}{HTML}{440154}
\definecolor{col1}{HTML}{46286D}
\definecolor{col2}{HTML}{414487}
\definecolor{col3}{HTML}{3B5E8B}
\definecolor{col4}{HTML}{2A788E}
\definecolor{col5}{HTML}{2B9089}
\definecolor{col6}{HTML}{22A884}
\definecolor{col7}{HTML}{5ABC6D}
\definecolor{col8}{HTML}{7AD151}
\definecolor{col9}{HTML}{C0DD40}
\definecolor{col10}{HTML}{FDE725}

	\textsf{Literals:	\textbf{\textcolor{col0}{---}} 1\ \ \ 
	\textbf{\textcolor{col6}{---}} 2\ \ \ 
	\textbf{\textcolor{col8}{---}} 4\ \ \ 
	\textbf{\textcolor{col9}{---}} 6\ \ \ 
	\textbf{\textcolor{col10}{---}} 8}

	\caption{Recombination: For \textsc{Propositional} tasks with at least two different functions, we show results when all function pairs appeared in overlapping sets of document scripts om pretraining (left), and when two functions appeared in disjoint sets of document scripts (right).}\label{fig:recombination}
\end{figure*}

\definecolor{col1}{HTML}{F8766D}
\definecolor{col2}{HTML}{7CAE00}
\definecolor{col3}{HTML}{00BFC4}
\definecolor{col4}{HTML}{C77CFF}

\begin{figure*}
\centering

        \begin{tikzpicture}

		\node (label) at (2.5,1.8)[draw=none, align=left, anchor=center]{	\footnotesize{\textsf{\textsc{Raw} }}};
		\node (label) at (1.4*4-0.1,1.8)[draw=none, align=left, anchor=center]{	\footnotesize{\textsf{\textsc{Explanation} }}};
		\node (label) at (1.4*6+0,1.8)[draw=none, align=left, anchor=center]{	\footnotesize{\textsf{\textsc{ChainOfThought} }}};

		\node (label) at (5,-1.5)[draw=none, align=left, anchor=center]{	\includegraphics[width=0.6\textwidth]{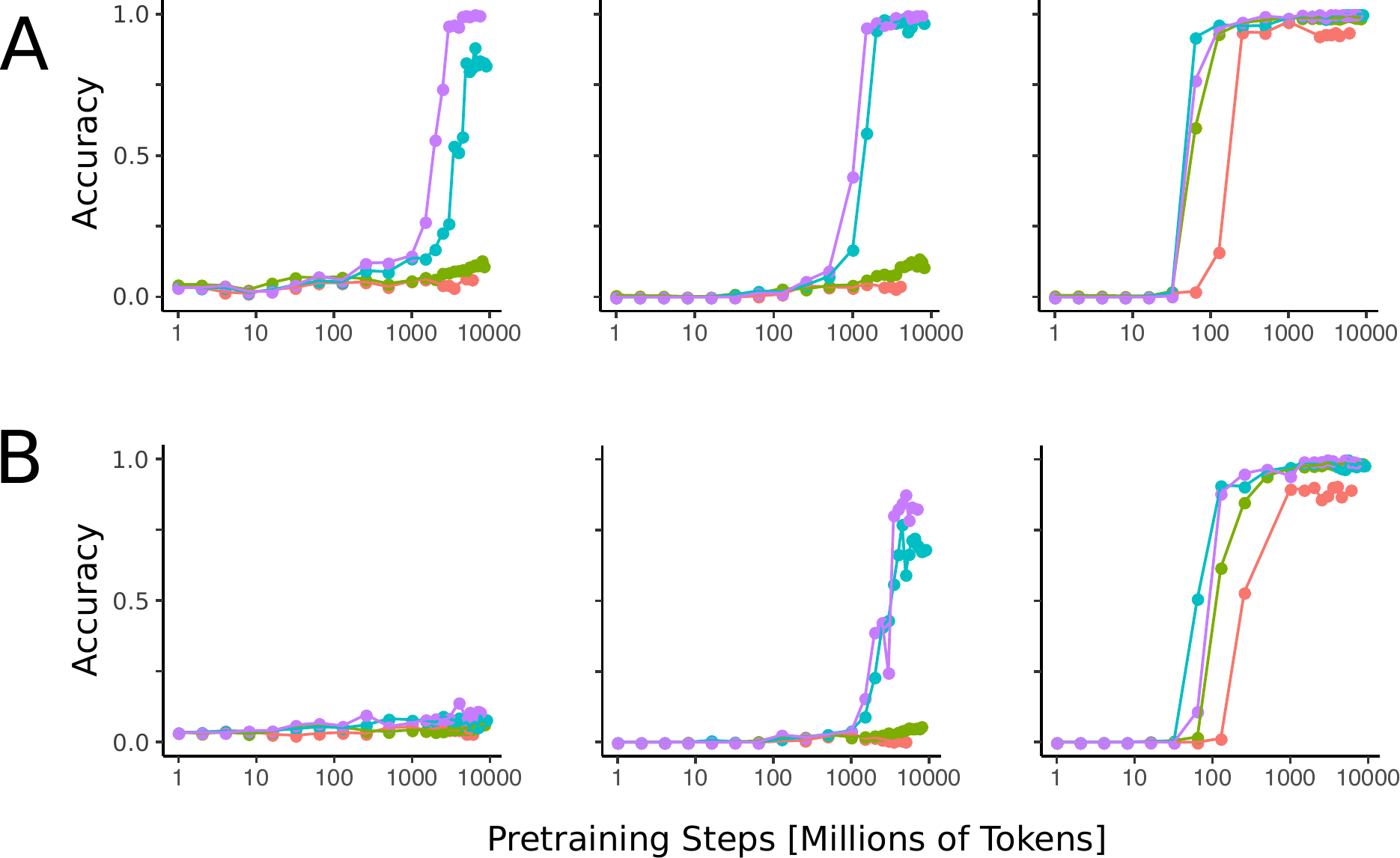} };

	\end{tikzpicture}

	\textsf{	Parameters:	
	\textbf{\textcolor{col1}{---}} 14M
	\textbf{	\textcolor{col2}{---}} 21M
	\textbf{	\textcolor{col3}{---}} 42M
	\textbf{	\textcolor{col4}{---}} 84M}

	\caption{Computing the composition of functions is hard, as this involves an unobserved variable (A), in particular, when it requires recombining functions $f_i,f_j$ that never appeared together in training (B).
Prompting the model to provide the intermediate step after the answer (\textsc{Explanation}) or before it (\textsc{ChainOfThought}) facilitates this.
	\textsc{ChainOfThought} transforms the task into a sequence of two simple tasks, and enables even the smallest model to solve the task.
	All results at the longest prompt length (14 examples).
	}\label{fig:composition}
\end{figure*}
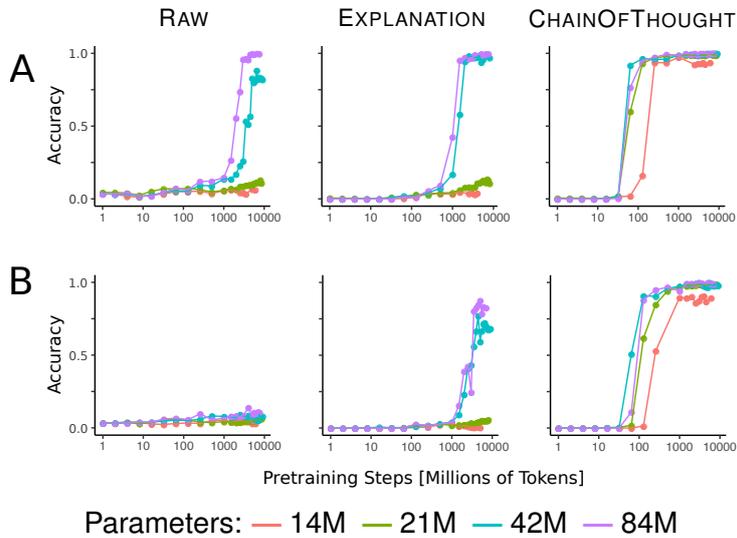

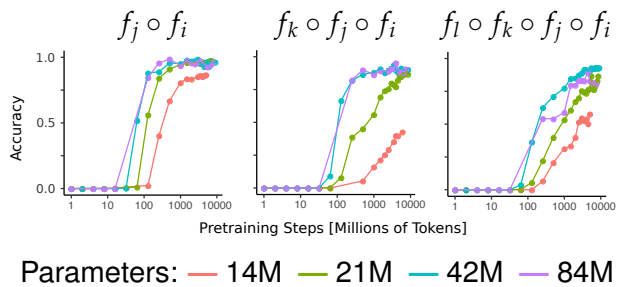
\begin{figure*}
	\centering
        \begin{tikzpicture}
	 \node (label) at (-1.7,4.5)[draw=none, align=left, anchor=center]{$f_j\circ f_i$};
	 \node (label) at (0.7,4.5)[draw=none, align=left, anchor=center]{$f_k\circ f_j \circ f_i$};
	 \node (label) at (3.2,4.5)[draw=none, align=left, anchor=center]{$f_l\circ f_k \circ f_j \circ f_i$};
	 \node (label) at (0.4,3)[draw=none, align=left, anchor=center]{\includegraphics[width=0.5\textwidth]{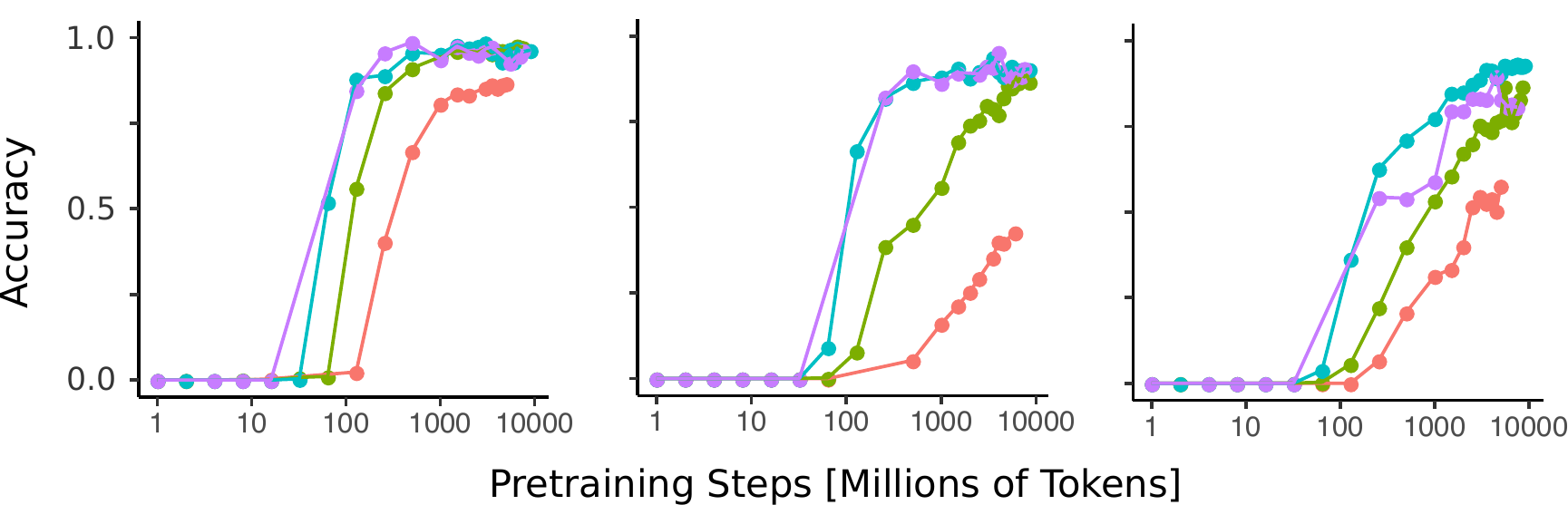}}; 
  \end{tikzpicture}

	\textsf{	Parameters:	
	\textbf{\textcolor{col1}{---}} 14M
	\textbf{	\textcolor{col2}{---}} 21M
	\textbf{	\textcolor{col3}{---}} 42M
	\textbf{	\textcolor{col4}{---}} 84M}

	\caption{Chain-of-thought reasoning for compositions of several functions; in the setting where $f_i, f_j$ never appeared in the same training document script (B in Figure~\ref{fig:composition}).
	For fair comparison, we show results at the largest length (8 examples) that fits into the transformer's context length for all tasks. 
}\label{fig:cot-by-length}
\end{figure*}

\begin{figure*}
\centering
\begin{tabular}{cccccccc}
	\textsf{A} & \textsf{B} \\
\includegraphics[width=0.4\textwidth]{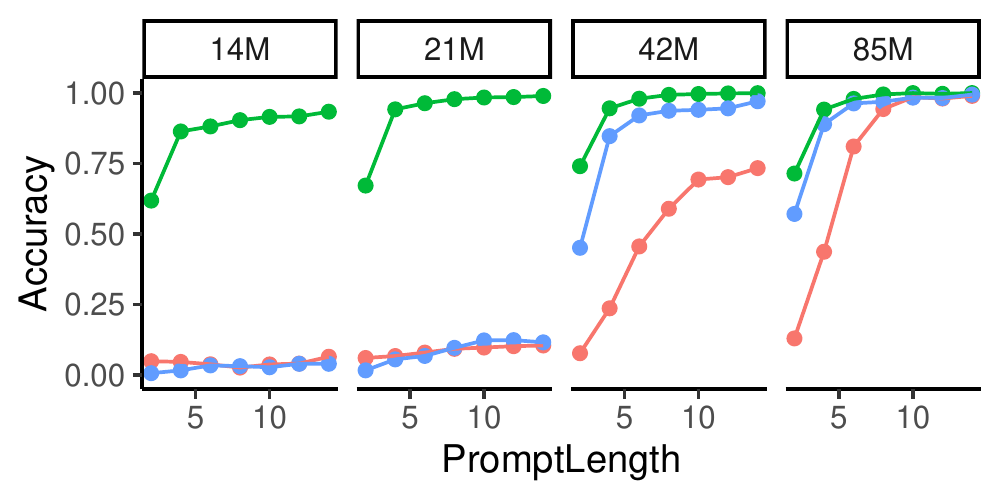} &
\includegraphics[width=0.4\textwidth]{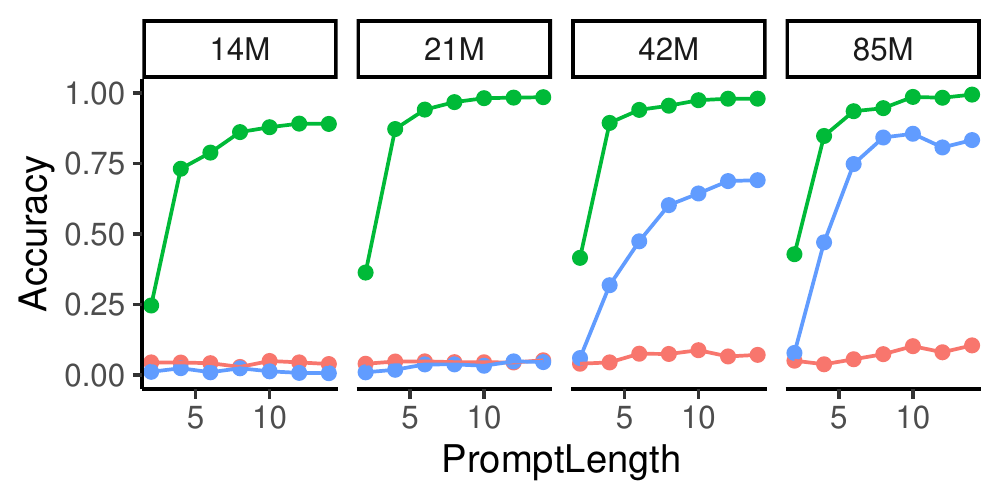}
\end{tabular}

	\textsf{{\LARGE\textcolor{red}{---}} Raw
\ \ \ \ \ 
{\LARGE\textcolor{blue}{---}} \textsc{Explanation}
\ \ \ \ \ 
	{\LARGE\textcolor{green}{---}} \textsc{ChainOfThought}}

\caption{ Accuracy for function composition task by prompt length and model size, at the end of pretraining.
A and B are as in Figure~\ref{fig:composition}.
As predicted by our theory, \textsc{ChainOfThought} achieves the best results.
\textsc{Explanation} only helps for large models (as found for real-world LLMs by \citet{DBLP:conf/emnlp/LampinenDCMTCMW22}), and does not work as well as \textsc{ChainOfThought} even for those.}
\label{fig:cot-by-prompt}
\end{figure*}

\paragraph{LMs can recombine functions.}

Next, for each of the \textsc{Propositional} tasks with at least two different functions, we created a variant where two of the functions involved were replaced by the two functions that never appeared in a single document script ($f_2$,$f_3$).
Solving these tasks thus requires recombining knowledge acquired from different portions of the training data.
We compare accuracy on these versions with the previous results in Figure~\ref{fig:recombination}.
The effect on the accuracies is small, with a significant drop affecting only some tasks with 8 literals.

\paragraph{Explanations help recombine abilities, but Chain-of-thought prompting helps more.}

To test our predictions from Section~\ref{sec:cot-theory}, we
compared two variants of the \textsc{Composition} task where the model was prompted to generate not only the answer, but also the unobserved variable.
In the \textsc{ChainOfThought} (CoT) version, the model was prompted to output the unobserved variable \emph{before} the answer (Figure~\ref{fig:tasks}E). In the \textsc{Explanation} version, the model was prompted to output the unobserved variable \emph{after} the answer.
In these versions, the models are prompted to produce two tokens.
We generated using greedy decoding, and counted responses as correct when the correct sequence was produced.
Theorem~\ref{theorem:cot} provides an improved bound in \textsc{ChainOfThought}, but not in \textsc{Explanation}.

Results, by data and model size, are shown in Figure~\ref{fig:composition}A.
Providing an explanation leads to somewhat earlier emergence of the task, but it only benefits large models (in line with empirical findings for real-world LLMs by \citet{DBLP:conf/emnlp/LampinenDCMTCMW22}).
In line with the theoretical prediction, gains from \textsc{ChainOfThought} are much stronger: it leads to early emergence even in the small model, not much later than the emergence of the simple function evaluation task (Figure~\ref{fig:comparison-baselines}).
This difference between producing intermediate reasoning steps before or after the answer mirrors the behavior of real-world LLMs  \citep{Wei2022Chain,DBLP:conf/emnlp/LampinenDCMTCMW22}.

We next considered composing $f_2, f_3$, which never appeared in the same script used for the pretraining corpus (Figure~\ref{fig:composition}B).
The raw task is now inaccessible even to the largest model.
Both ways of including the intermediate step make the task accessible to the large models; CoT succeeds even on the smallest model and with short prompts  (Figure~\ref{fig:cot-by-prompt}B).
This suggests that prompting the model to include the intermediate step helps it compositionally recombine abilities that were never used together in pretraining.

\textsc{ChainOfThought} continues to succeed for the composition of three and four functions, where single-step prompting is too difficult for all models (Figure~\ref{fig:cot-by-length}).

\paragraph{Ablating CAG Features.}
We created variants of the training data with (i) loops, (ii) variable introduction via ``for some'', (iii) conditions (if-then-else) 
ablated.
See Appendix~\ref{sec:ablations} for details.
We trained medium-size LMs (21M parameters) on each variant (Figure~\ref{fig:ablation}).
Ablating loops or variable introduction made ICL impossible; indeed, these are necessary to provide learning guarantees via Theorem 1 for any nontrivial task.
On the other hand, ablating conditions had no discernible negative impact even on \textsc{Propositional} or \textsc{Binary} tasks whose definitions involve disjunction.

\paragraph{Heldout analysis.}
We additionally trained the 21M parameter LM on a pretraining set where any document containing a substring that would be a valid prompt (with at least 4 examples) for any test task (including chain-of-thought) was removed, affecting $< 0.02\%$  of documents.\footnote{Of a sample of 1K excluded documents, 27\% \textsc{FunctionEvaluation} or \textsc{Inverse}, 71\% \textsc{Composition}, 3\% Task 12, 1.5\% \textsc{Composition} with CoT or explanation; 0\% other \textsc{Propositional} or \textsc{Binary} tasks. We believe that most matches to the composed functions are chance matches, i.e. to functions that are simple and happen to match on some inputs. Upon closer inspection, 66\% of the \textsc{Composition} examples were cases where $z=x$, which match \textsc{Composition} in the (a priori rare) case where $f_j(f_i(x))=x$, so that the function collapses to the much simpler identity function.}
We considered not just the sample prompts used for evaluation, but any sequence matching the specification of any of the tasks.
We considered substrings appearing in any place in a document, not just substrings following the StartOfSequence symbol.
Accuracies were highly correlated between the two versions (Figure~\ref{fig:heldout}).

\paragraph{Real-World LLMs.}
We investigated whether key tenets of our theory apply to a real-world LLM in the InstructGPT family (\texttt{text-davinci-003}), in a synthetic setup of strings over \texttt{a},\dots,\texttt{z} of length 10, with length-preserving operations from \citet{DBLP:journals/jair/HupkesDMB20} as basic functions (reverse; shift; swap first and last letter).
Almost none of the $>10^{14}$ strings are likely to have appeared in the training data.
See Appendix \ref{sec:gpt3=appendix} for details.
Results are shown in Figure~\ref{fig:gpt3}.
For \textsc{Propositional} tasks, accuracy increases with prompt length, with earlier increases for tasks defined with fewer literals.
\textsc{Binary} tasks are solved successfully, better than in our other experiments.
Finally, the \textsc{FunctionComposition} task is difficult; in line with Theorem 2 and our other experiments, \textsc{ChainOfThought} is more helpful in facilitating it than \textsc{Explanation}.

\begin{figure}
    \centering
    \begin{tikzpicture}

	    \node (label) at (0.5,2.7)[draw=none, align=center, anchor=center]{ \textsf{ A.  Propositional}};
	    \node (label) at (4.5,2.7)[draw=none, align=center, anchor=center]{ \textsf{ B.  Binary}};
	    \node (label) at (8.5,2.7)[draw=none, align=center, anchor=center]{  \textsf{ C. Compositiion}};

\node (label) at (0,0)[draw=none, align=center, anchor=center]{    \includegraphics[width=0.25\textwidth]{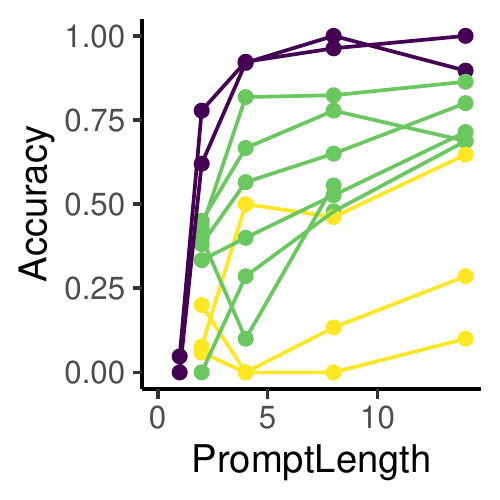}};
    \node (label) at (4,0)[draw=none, align=center, anchor=center]{\includegraphics[width=0.25\textwidth]{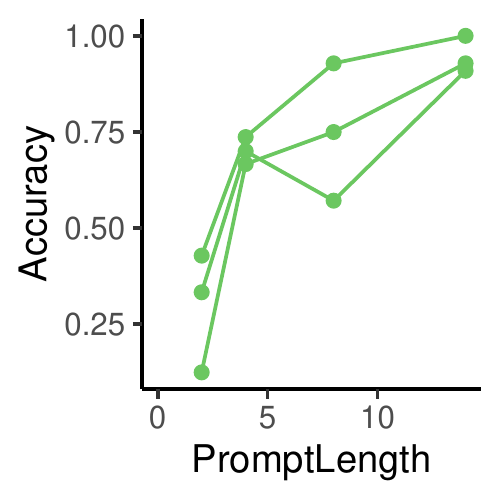}};
      \node (label) at (8,0)[draw=none, align=center, anchor=center]{  \includegraphics[width=0.25\textwidth]{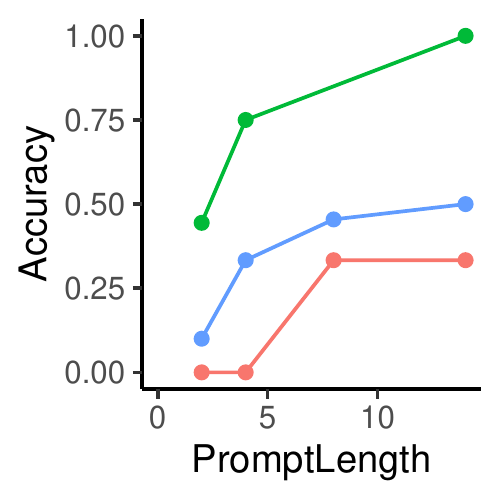}};
        \end{tikzpicture}

\definecolor{col0}{HTML}{440154}
\definecolor{col1}{HTML}{46286D}
\definecolor{col2}{HTML}{414487}
\definecolor{col3}{HTML}{3B5E8B}
\definecolor{col4}{HTML}{2A788E}
\definecolor{col5}{HTML}{2B9089}
\definecolor{col6}{HTML}{22A884}
\definecolor{col7}{HTML}{5ABC6D}
\definecolor{col8}{HTML}{7AD151}
\definecolor{col9}{HTML}{C0DD40}
\definecolor{col10}{HTML}{FDE725}

	\textsf{(A--B)\ \ \ 
Literals:	
	{\LARGE\textbf{\textcolor{col0}{---}}} 1\ \ \ 
	{\LARGE\textbf{\textcolor{col6}{---}}} 2\ \ \ \ \ \ \ 
	{\LARGE\textbf{\textcolor{col8}{---}}} 4\ \ \ \ \ \ \ 
	{\LARGE\textbf{\textcolor{col9}{---}}} 6\ \ \ \ \ \ \ 
	{\LARGE\textbf{\textcolor{col10}{---}}} 8}

	\textsf{(C) \ \ \ 
	{\LARGE\textcolor{red}{---}} Raw
\ \ \ \ \ 
{\LARGE\textcolor{blue}{---}} \textsc{Explanation}
\ \ \ \ \ 
	{\LARGE\textcolor{green}{---}} \textsc{ChainOfThought}}

    \caption{Results from InstructGPT (\texttt{text-davinci-003}) on a family of synthetic string manipulation tasks: \textsc{FunctionEvaluation} and \textsc{Propositional} tasks (A), \textsc{Binary} tasks (B), \textsc{FunctionComposition} (C) with direct prompting, \textsc{Explanation} and \textsc{ChainOfThought}.}
    \label{fig:gpt3}
\end{figure}
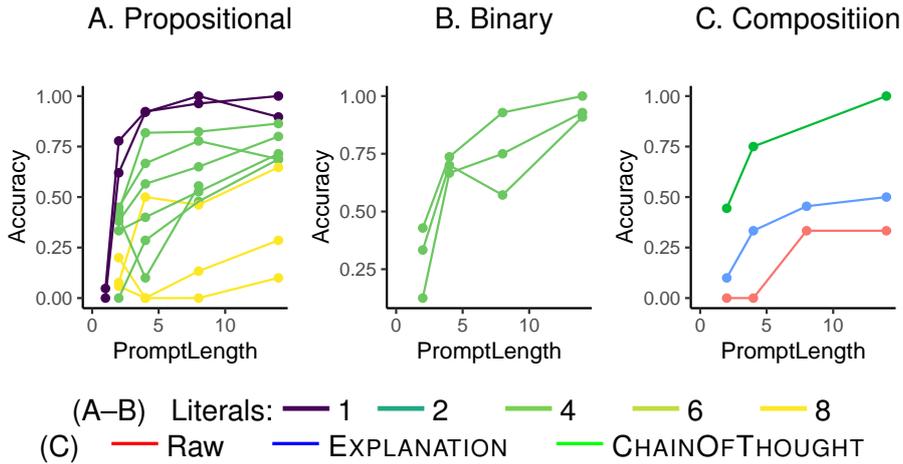

\subsection{Representation learning supports ICL}

Our theoretical analysis in Section~\ref{sec:theory} argues that ICL relies on identifying the compositional generative process underlying a prompt.
Here, we provide evidence from the LM's activations and attention patterns that they indeed induce the compositional structure underlying documents and prompts.
We target the 21M parameters model as it has a small number of heads and layers and yet is successful on almost all tasks.
We first visualize attention patterns in a chain-of-thought example, focusing on the final tokens for visibility, in Figure~\ref{eq:attention-correlation}A (see Appendix, Figure~\ref{fig:att-by-tasks} for other tasks).
The two heads in the lowest layer attend to the three preceding token to decreasing degree (1.1), or recent tokens with the same type (1.2), respectively.
The two heads in the second layer attend to diffuse average of recent tokens (2.1) or sharply the immediately preceding tokens (2.2).
In the third layer, one head (3.1) attends to the preceding tokens that are in structurally corresponding positions to the upcoming token---e.g. the head attends to previous delimiters at the end of each prompt example. Head (3.2) is similar if a bit more diffuse.
The pattern was more diffuse but analogous when removing the separator, i.e., the model does not rely on a recurring symbol to induce the structure (Appendix, Figures~\ref{fig:att-map-cot-no-sep} and \ref{fig:att-by-tasks}).

We next analyzed attention patterns across a sample of 300 random documents in the training corpus.
The attention head patterns generalized: heads in the lower layers attended in the same way as described (Figure~\ref{eq:attention-correlation}B).
Head (3.1) was best explained as attending to tokens that were \emph{structurally corresponding}---those tokens produced by executing the same line in the document script as the token now being predicted.
Head (3.2) was explained best---depending on random seed---either by the same or by the same shifted by one.

In addition to this correlational study, we next performed an interventional one: we intervened on the top-level attention heads in the trained model by masking out attention logits to non-structurally-corresponding positions.
If the function of the top-layer attention heads is indeed to attend to structurally corresponding positions, this intervention should improve performance, by effectively providing oracle access to the document's structure.
Indeed, the intervention consistently improved cross-entropy on the pretraining task when applied to the top layer heads, and hurt when applied to other layers (Appendix, Figure~\ref{fig:attention-intervention-ce}).
Similarly, it improved accuracy on the chain-of-thought task, in particular for short prompts (Figure~\ref{eq:attention-correlation}D).

\paragraph{Towards extracting the learned algorithm.}
This attention pattern suggests a general algorithmic approach:
first, the lower layers identify the structure of the document.
The third layer then collects information from all structurally matching positions, and predicts the structurally required token based on that information.
In this algorithm, there are two key tasks:
first, identify the structure underlying the input in the lower layers; second, perform analogical reasoning to predict the next token in the top layer.
Importantly, this algorithm equally works on simple \textsc{FunctionEvaluation} and chain-of-thought-prompting: once the structure has been induced, both tasks amount to predicting $f_i(x)$ for the last token $x$.
This helps explain why chain-of-thought prompting emerges so quickly and works even in the smallest model.

The remaining question is how the lower layers identify and represent document structure.
We hypothesize that the model learns to encode the logical relations holding among the tokens close by, using information gained from the three heads attending to recent tokens.
We can formalize the logical relations as a set of graphs with the immediately preceding tokens as vertices ($\{t-K+1, ... t\}$), and edges describing which functional relations hold.
Encoding this in a set of adjacency matrices, we obtain a tensor $M_{i,j,k}$ ($1\geq i,j \geq K$, $1\geq k \geq |\Ff|$) where $M_{i,j,k}$ holds iff $f(w_{t-i}) = w_{t-j}$.
We set the context size to $K=3$, and fitted a set of log-linear probes to decode $M$ from the activations in each layer (Figure~\ref{eq:attention-correlation}).
$M$ can be decoded with highest accuracy from the second layer.
Probe complexity as assessed using a prequential code  \citep{DBLP:conf/emnlp/VoitaT20} was $\approx 50\%$ of that for a control task \citep{Hewitt2019DesigningAI} where functions were randomized (but consistently within a probe), showing that the layer encodes nontrivial information.
Interestingly, probe accuracy undergoes a rapid increase at about $\approx$ 70M tokens of training, at a similar time as the emergence of many \textsc{Propositional} tasks.

\begin{figure*}
\centering
\includegraphics[width=0.98\textwidth]{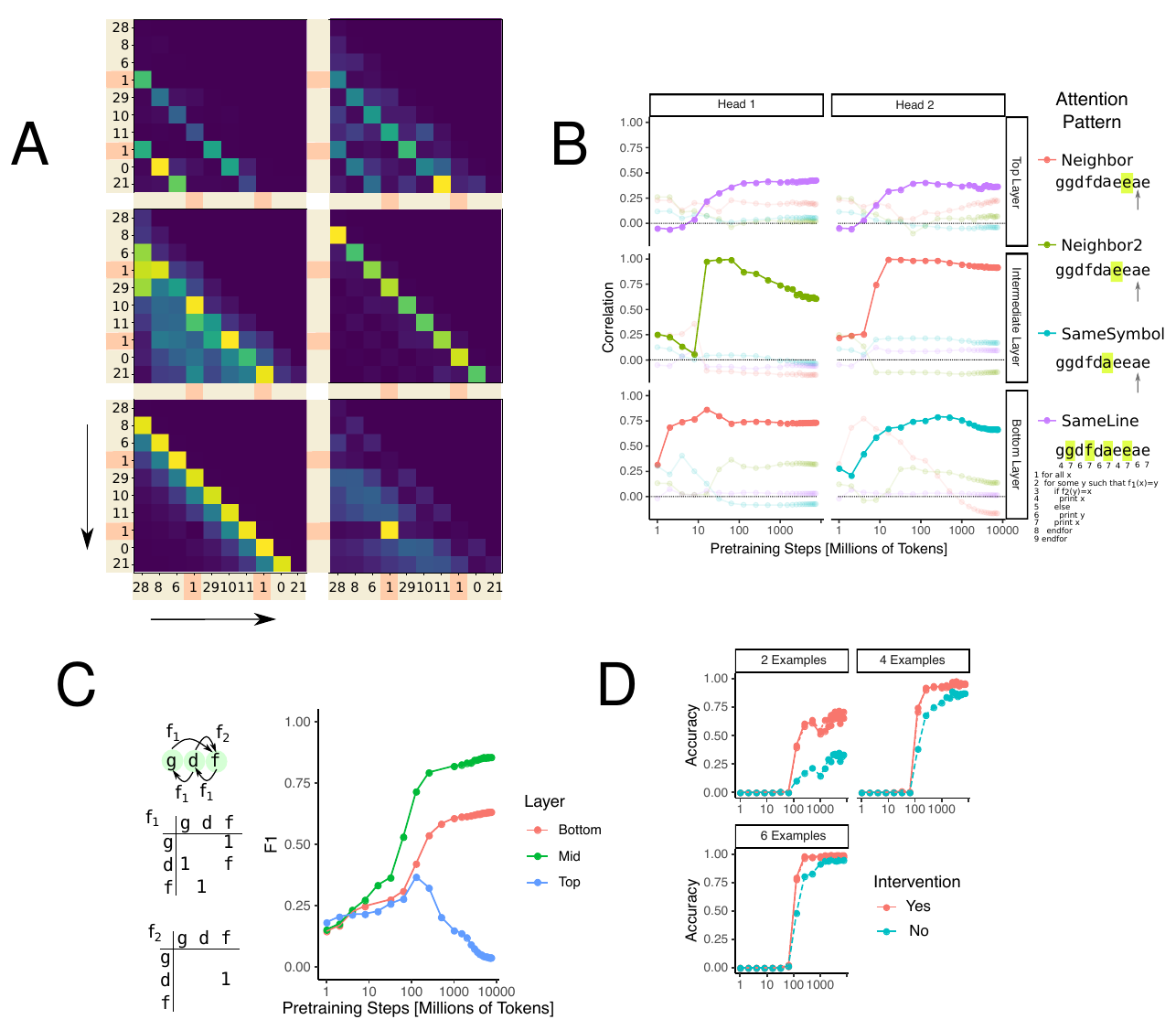}

	\caption{(A) Attention Maps: chain-of-thought prompting, 21M parameter model. Columns: heads; rows: layers, bottom layer is closer to the input. The prompt included 15 examples, we only show the last 10 tokens for clarity. Within each attention map, the $y$-axis indexes queries and the $x$-axis indexes keys. Orange background indicates the separator. (B) Correlation between attention scores and four attention pattern predictors across 300 documents. (C) Local relations among immediately preceding tokens can be summarized into a set of adjacency matrices with binary entries, one for each function $f_i$. A log-linear probe recovers these matrices at high $F_1$ from the intermediate layer. (D) Intervening on attention heads in the top layer, so that they can only attend to structurally-matching positions (without any other change to the trained model) improves model performance for short prompts (here, the same chain-of-thought task as in A).}\label{eq:attention-correlation}
\end{figure*}

\section{Discussion and Related Work}\label{sec:discussion}

\paragraph{Qualitative Properties of ICL.}
We summarize schematic properties of ICL in Table~\ref{tab:icl-properties}.
Properties in the first group are supported by our theory, real transformers pretrained on a minimal CAG, and real-world LLMs like GPT-3.
Improvement of ICL with longer prompts is arguably predicted by all theoretical or experimental approaches to understanding ICL.
On the other hand, the advantage for prompting LMs to provide intermediate steps, and effects of task complexity formalized by description length, are new predictions of our theory.
The first one is by now very well established empirically \citep[e.g.][]{Wei2022Chain, Suzgun2022ChallengingBT}; we found evidence for the second one by prompting InstructGPT on our test tasks (Figure~\ref{fig:gpt3}).
The second group in Table~\ref{tab:icl-properties} includes properties of ICL which are easy to manipulate in controlled setups.
The fact that ICL gets harder with increased $|\Ff|$ is intuitive. 
Invariance to, or even improvement with, increased $|\Omega|$ is consistent with empirical observations about benefits of increased vocabulary size \citep{DBLP:conf/iclr/XieRL022} and increased numbers of objects \citep{DBLP:journals/corr/abs-2205-05055}.
We empirically also found that transformers could recombine their knowledge of functions $f_i, f_j$ never seen together during training.
While it is likely that emergent behavior in LLMs involves a substantial amount of recombination of abilities, this has been hard to directly verify in real-world LLMs.

\paragraph{Theoretical Guarantees for In-Context Learning.}
We have proven information-theoretic bounds for in-context learning in an optimal predictor, under linguistically motivated assumptions on the pretraining dataset, stated in terms of a task's description length in the compositional structure underlying the pretraining data (Theorem~\ref{theorem:theorem1}).
As discussed in Section~\ref{sec:comparison-xie} this differs from bounds by \citet{DBLP:conf/iclr/XieRL022} in guaranteeing learning in an open-ended hypothesis space recombining compositional operations found in the pretraining data; we used this to prove a benefit for prompting models to provide intermediate steps (Theorem~\ref{theorem:cot}).
These results highlight the usefulness of considering linguistically motivated compositional generative processes in analyzing the theoretical foundations of LLMs.
Our bounds apply to an optimal predictor; extending them to finite-capacity models pretrained on finite data (as in our experiments) is an interesting problem for future research.
Our findings on representation learning suggests ways in which the transformer's architecture might be helpful.

\paragraph{ICL in Controlled Setups.}
Our work joins a line of recent work \citep{DBLP:journals/corr/abs-2211-15661,DBLP:journals/corr/abs-2208-01066,DBLP:journals/corr/abs-2205-05055,DBLP:conf/iclr/XieRL022,DBLP:journals/corr/abs-2301-07067} studying ICL in controlled miniature setups.
\citet{DBLP:journals/corr/abs-2208-01066,DBLP:journals/corr/abs-2211-15661} study the in-context learning abilities of transformers by directly meta-training on prompts from a candidate function class, finding that they can be meta-trained to learn, among others, linear functions and decision trees, from example prompts.
\citet{DBLP:journals/corr/abs-2301-07067} provide generalization bounds for such a meta-training setup.
\citet{DBLP:journals/corr/abs-2205-05055} empirically study a simple image-based multiclass classification dataset, where the training data also had a prompt-based format but the task was fixed across pretraining prompts and test trials targeted heldout classes.
Conceptually closest to our work, \citet{DBLP:conf/iclr/XieRL022} cast ICL as recovery of an HMM mixture component (see detailed discussion in Section~\ref{sec:comparison-xie}).
The key innovation in our study is that we account for compositional recombination of skills into composed tasks, attributing its emergence to the compositional structure of linguistic pretraining data.
This allows us to induce ICL on tasks of varying complexity, and account for the benefit of providing intermediate steps.

Another line of work has studied the ability of real-world LLMs to in-context-learn synthetic or even unnatural tasks resembling natural tasks but with exchanged labels
\citep{rong21extrapolating,DBLP:journals/corr/abs-2202-12837,Wei2023LargerLM}.
The ability of models both to infer the prompt structure and unnatural input-label mappings, at least when they are sufficiently large \citep{Wei2023LargerLM}, is conceptually well compatible with the idea that in-context learning relies on identifying the compositional structure underlying a prompt.

\paragraph{Emergence and Grokking.}
In many cases, we observed sudden emergence of abilities when a certain threshold of training steps had been crossed, while the pretraining loss decreased continuously.
This has commonalities with the \textit{grokking} phenomenon observed in supervised learning \citep{Power2022GrokkingGB}.
Grokking is thought to relate to the build-up of generalizable representations \citep{Liu2022TowardsUG,DBLP:journals/corr/abs-2302-03025}; in a similar vein, we found that emergence of many tasks coincided with improvement in structural representations.

\paragraph{Mechanisms of ICL.}
A recent line of work \citep{DBLP:journals/corr/abs-2211-15661,Oswald2022TransformersLI,DBLP:journals/corr/abs-2212-10559}
 provides in-principle proofs that transformers can implement optimization algorithms in context, hypothesizing that this underlies real-world ICL, but leaving open how generic natural-language text data would give rise to such abilities.
In our controlled setup, ICL relied in part on attention heads attending to structurally related positions earlier in the sequence, allowing the model to ``copy'' behavior. This may be related to a hypothesized role of \textit{induction heads} \citep{Olsson2022IncontextLA}, heads that copy symbols from past contexts matching a certain pattern-like description.

\paragraph{Role of Pretraining Data in ICL.}
\citet{DBLP:conf/naacl/ShinLAKKKCLPHS22} empirically investigate the role of pretraining corpora for ICL in real-world LLMs. Potentially relevant to our {\cvgarg}s on the role of recombination of compositional behavior, they found that ICL can result by combining corpora which individually do not give rise to it.
\citet{DBLP:conf/emnlp/RazeghiL0022} studied numerical reasoning capabilities in GPT-based LLMs, finding that term frequencies in the pretraining dataset had a strong impact on ICL performance.
\citet{DBLP:journals/corr/abs-2205-05055} empirically investigated the role of the training dataset in a simple image classification task where the pretraining data consisted of prompt-like sequences, finding that data-distributional properties such as a Zipfian distribution over classes were beneficial to ICL of held-out classes. 
Our setup differs in that we target an open-ended space of compositionally created test tasks and the benefit of providing intermediate steps.
However, relevant to our findings, they found that varying the assignment of classes to labels in pretraining improved ICL; this may encourage a model towards flexible recombination of knowledge.

\paragraph{Algorithmic Information Theory and Minimum Description Length.}
Theorems 1 and 2 suggest that ICL can work because prompts are compressible into compositional generation processes.
Our theoretical analysis has links to Algorithmic Information Theory \citep{LiVitanyi2008introduction} and the Minimum Description Length (MDL) principle \citep{Rissanen2010MinimumDL}, and to statistical estimators based on MDL \citep{Barron1991MinimumCD}.
Informally, Theorem~\ref{theorem:theorem1} is based on the idea that a predictor trained on compositionally generated data will show an implicit MDL-like bias.
In the case where documents are generated by a Turing-complete generative process, $\DL[\cdot]$ corresponds to Kolmogorov Complexity and $M$ corresponds to the \textit{universal prior} used in Solomonoff induction (\citep{solomonoff1964formal}; Definition 4.5.6 in \citet{LiVitanyi2008introduction}). 
In that setting, our theorem is similar to the completeness theorem for Solomonoff induction (Theorem 5.2.1 in \citet{LiVitanyi2008introduction}).
Turing-complete generative processes might be linguistically unrealistic as a model of text data; hence, we here explicitly derive ICL guarantees using a restricted and linguistically motivated class of generative processes. However,  there is an intriguing link to proposals conceptualizing intelligence as universal prediction akin to Solomonoff induction \citep{Hutter2004UniversalAI}.

\paragraph{Pretraining Data and Natural Language.}
While our theoretical results are grounded in a long tradition of research on language, the \textsc{Compositional} pretraining dataset only aims to provide a minimal CAG, and has key differences from real language.
Most prominently, document scripts repeat a lot of structure \emph{within} a document, whereas natural language repeats a lot of structure \emph{across} documents, which our scripts cannot model (recall that there are no grammatical or lexical conventions across documents).
This limitation is shared with the existing work aiming to induce ICL in controlled setups \citep{DBLP:journals/corr/abs-2208-01066,DBLP:journals/corr/abs-2205-05055,DBLP:conf/iclr/XieRL022,DBLP:journals/corr/abs-2211-15661}, though our experiments advance by incorporating compositional recursive structure and multi-step tasks.
Importantly, our theoretical analysis in Section~\ref{sec:theory} is valid independently of such restrictions, because it holds for a very broad class of linguistically realistic generative processes, and is (up to constants) robust to extending the CAG, even extensions adversarial to ICL (Appendix~\ref{sec:optimality-bound}).
Designing CAGs incorporating more realistic features of natural language is an interesting future research direction, and should enable more detailed miniaturized models of the emergence of ICL.
Linguistically richer CAGs could also pave the way to accounting for the role of instructions, which boost prompt effectiveness when prepended to demonstrations \citep{BrownMRSKDNSSAA20}; arguably, they provide additional cues regarding the generative process.

\paragraph{Neural Networks and Compositionality.}
The ability of neural networks to generalize compositionally has recently been the focus of much research \cite[e.g.][]{Kim2020COGSAC}.
Our research pursues a somewhat different direction, by studying how compositional structure in the pretraining data leads to ICL capabilities in the model.
In this sense, there is a relation to work showing that pretraining on structured data can imbue neural networks with increased abilities to model language or generalize compositionally \citep[e.g.][]{DBLP:conf/emnlp/PapadimitriouJ20,DBLP:conf/acl/Mueller0LWS22}.

\paragraph{Limitations and Future Work.}
Our theory and experiments focus on ICL for functions applying to objects from a fixed finite universe.
Real-world ICL is also successful on more complex and open-ended input, such as mapping a sentence to its sentiment, or mapping a text and a question to an answer. Formalizing this within our framework and deriving appropriate ICL bounds is an interesting problem for further research.
Relatedly, the analysis of chain-of-thought prompting in Theorem \ref{theorem:cot} focuses on the composition of a fixed sequence of functions. Real LLMs may also be able to find an appropriate sequence of reasoning steps that may itself be dependent on the input, and there will typically be multiple acceptable chains of reasoning leading to a correct answer. Generalizing Theorem 2 in this respect is an intriguing problem for future research.

In our theoretical guarantees, the measure of task difficulty is the description length in the compositional generative process underlying the pretraining data.
Quantifying it for naturalistic real-world tasks could provide a useful measure of ICL task difficulty for real-world LLMs.

\section{Conclusion}

We have provided an information-theoretic analysis of how generic next-token prediction can enable a model to learn from demonstrations in context.
We show a general error bound for the in-context learning of functions, under linguistically motivated assumptions about the generative process underlying the pretraining corpus.
Using this framework, we further prove a benefit for prompting models to provide intermediate steps towards an answer.
To validate the theoretical analysis, we introduce a new way of inducing in-context learning in a controlled miniature setup.
In this setup, we found sudden emergence of in-context learning when scaling parameters and data, recombination of abilities found in different parts of the pretraining data, 
and the theoretically predicted benefit of prompting for intermediate steps.
Taken together, these results provide a step towards theoretical understanding of emergent properties in large language models.

\section*{Acknowledgments}
M.H. gratefully acknowledges Saarland University's Department of Language Science and Technology and Saarland Informatics Campus for compute resources.

\bibliography{literature}

\begin{thebibliography}{83}
\providecommand{\natexlab}[1]{#1}
\providecommand{\url}[1]{\texttt{#1}}
\expandafter\ifx\csname urlstyle\endcsname\relax
  \providecommand{\doi}[1]{doi: #1}\else
  \providecommand{\doi}{doi: \begingroup \urlstyle{rm}\Url}\fi

\bibitem[Abney(1996)]{Abney1996StochasticAG}
Steven~P. Abney.
\newblock Stochastic attribute-value grammars.
\newblock \emph{ArXiv}, cmp-lg/9610003, 1996.

\bibitem[Aky{\"{u}}rek et~al.(2022)Aky{\"{u}}rek, Schuurmans, Andreas, Ma, and
  Zhou]{DBLP:journals/corr/abs-2211-15661}
Ekin Aky{\"{u}}rek, Dale Schuurmans, Jacob Andreas, Tengyu Ma, and Denny Zhou.
\newblock What learning algorithm is in-context learning? investigations with
  linear models.
\newblock \emph{CoRR}, abs/2211.15661, 2022.
\newblock \doi{10.48550/arXiv.2211.15661}.
\newblock URL \url{https://doi.org/10.48550/arXiv.2211.15661}.

\bibitem[Barron and Cover(1991)]{Barron1991MinimumCD}
Andrew~R. Barron and Thomas~M. Cover.
\newblock Minimum complexity density estimation.
\newblock \emph{IEEE Trans. Inf. Theory}, 37:\penalty0 1034--1054, 1991.

\bibitem[Blasi et~al.(2019)Blasi, Cotterell, Wolf{-}Sonkin, Stoll, Bickel, and
  Baroni]{DBLP:conf/acl/BlasiCWSBB19}
Dami{\'{a}}n~E. Blasi, Ryan Cotterell, Lawrence Wolf{-}Sonkin, Sabine Stoll,
  Balthasar Bickel, and Marco Baroni.
\newblock On the distribution of deep clausal embeddings: {A} large
  cross-linguistic study.
\newblock In Anna Korhonen, David~R. Traum, and Llu{\'{\i}}s M{\`{a}}rquez,
  editors, \emph{Proceedings of the 57th Conference of the Association for
  Computational Linguistics, {ACL} 2019, Florence, Italy, July 28- August 2,
  2019, Volume 1: Long Papers}, pages 3938--3943. Association for Computational
  Linguistics, 2019.
\newblock \doi{10.18653/v1/p19-1384}.
\newblock URL \url{https://doi.org/10.18653/v1/p19-1384}.

\bibitem[Boas and Sag(2012)]{Boas2012SignBasedCG}
Hans~Christian Boas and Ivan~A. Sag.
\newblock \emph{Sign-Based Construction Grammar}.
\newblock 2012.

\bibitem[Boullier(1999)]{Boullier1999ChineseNM}
Pierre Boullier.
\newblock Chinese numbers, mix, scrambling, and range concatenation grammars.
\newblock In \emph{Conference of the European Chapter of the Association for
  Computational Linguistics}, 1999.

\bibitem[Boullier(2000)]{Boullier2000RangeCG}
Pierre Boullier.
\newblock Range concatenation grammars.
\newblock In \emph{International Workshop/Conference on Parsing Technologies},
  2000.

\bibitem[Bresnan(2000)]{Bresnan1987LexicalfunctionalG}
Joan Bresnan.
\newblock \emph{Lexical Functional Syntax}.
\newblock 2000.

\bibitem[Brown et~al.(2020)Brown, Mann, Ryder, Subbiah, Kaplan, Dhariwal,
  Neelakantan, Shyam, Sastry, Askell, Agarwal, Herbert{-}Voss, Krueger,
  Henighan, Child, Ramesh, Ziegler, Wu, Winter, Hesse, Chen, Sigler, Litwin,
  Gray, Chess, Clark, Berner, McCandlish, Radford, Sutskever, and
  Amodei]{BrownMRSKDNSSAA20}
Tom~B. Brown, Benjamin Mann, Nick Ryder, Melanie Subbiah, Jared Kaplan,
  Prafulla Dhariwal, Arvind Neelakantan, Pranav Shyam, Girish Sastry, Amanda
  Askell, Sandhini Agarwal, Ariel Herbert{-}Voss, Gretchen Krueger, Tom
  Henighan, Rewon Child, Aditya Ramesh, Daniel~M. Ziegler, Jeffrey Wu, Clemens
  Winter, Christopher Hesse, Mark Chen, Eric Sigler, Mateusz Litwin, Scott
  Gray, Benjamin Chess, Jack Clark, Christopher Berner, Sam McCandlish, Alec
  Radford, Ilya Sutskever, and Dario Amodei.
\newblock Language models are few-shot learners.
\newblock In Hugo Larochelle, Marc'Aurelio Ranzato, Raia Hadsell,
  Maria{-}Florina Balcan, and Hsuan{-}Tien Lin, editors, \emph{Advances in
  Neural Information Processing Systems 33: Annual Conference on Neural
  Information Processing Systems 2020, NeurIPS 2020, December 6-12, 2020,
  virtual}, 2020.
\newblock URL
  \url{https://proceedings.neurips.cc/paper/2020/hash/1457c0d6bfcb4967418bfb8ac142f64a-Abstract.html}.

\bibitem[Chan et~al.(2022)Chan, Santoro, Lampinen, Wang, Singh, Richemond,
  McClelland, and Hill]{DBLP:journals/corr/abs-2205-05055}
Stephanie C.~Y. Chan, Adam Santoro, Andrew~K. Lampinen, Jane~X. Wang, Aaditya
  Singh, Pierre~H. Richemond, Jay McClelland, and Felix Hill.
\newblock Data distributional properties drive emergent in-context learning in
  transformers.
\newblock \emph{CoRR}, abs/2205.05055, 2022.
\newblock \doi{10.48550/arXiv.2205.05055}.
\newblock URL \url{https://doi.org/10.48550/arXiv.2205.05055}.

\bibitem[Chi(1999)]{Chi1999StatisticalPO}
Zhiyi Chi.
\newblock Statistical properties of probabilistic context-free grammars.
\newblock \emph{Comput. Linguistics}, 25:\penalty0 131--160, 1999.

\bibitem[Chomsky(1957)]{chomsky1957syntactic}
Noam Chomsky.
\newblock \emph{Syntactic structures.}
\newblock Mouton, The Hague, 1957.

\bibitem[Chomsky(1992)]{Chomsky1992TheMP}
Noam Chomsky.
\newblock The minimalist program.
\newblock 1992.

\bibitem[Chughtai et~al.(2023)Chughtai, Chan, and
  Nanda]{DBLP:journals/corr/abs-2302-03025}
Bilal Chughtai, Lawrence Chan, and Neel Nanda.
\newblock A toy model of universality: Reverse engineering how networks learn
  group operations.
\newblock \emph{CoRR}, abs/2302.03025, 2023.
\newblock \doi{10.48550/arXiv.2302.03025}.
\newblock URL \url{https://doi.org/10.48550/arXiv.2302.03025}.

\bibitem[Clark(2021)]{Clark2021StrongLO}
Alexander Clark.
\newblock Strong learning of some probabilistic multiple context-free grammars.
\newblock In \emph{Mathematics of Language}, 2021.

\bibitem[Dai et~al.(2022)Dai, Sun, Dong, Hao, Sui, and
  Wei]{DBLP:journals/corr/abs-2212-10559}
Damai Dai, Yutao Sun, Li~Dong, Yaru Hao, Zhifang Sui, and Furu Wei.
\newblock Why can {GPT} learn in-context? language models secretly perform
  gradient descent as meta-optimizers.
\newblock \emph{CoRR}, abs/2212.10559, 2022.
\newblock \doi{10.48550/arXiv.2212.10559}.
\newblock URL \url{https://doi.org/10.48550/arXiv.2212.10559}.

\bibitem[Garg et~al.(2022)Garg, Tsipras, Liang, and
  Valiant]{DBLP:journals/corr/abs-2208-01066}
Shivam Garg, Dimitris Tsipras, Percy Liang, and Gregory Valiant.
\newblock What can transformers learn in-context? {A} case study of simple
  function classes.
\newblock \emph{CoRR}, abs/2208.01066, 2022.
\newblock \doi{10.48550/arXiv.2208.01066}.
\newblock URL \url{https://doi.org/10.48550/arXiv.2208.01066}.

\bibitem[Ginzburg(2012)]{Ginzburg2012TheIS}
Jonathan Ginzburg.
\newblock \emph{The interactive stance : meaning for conversation}.
\newblock 2012.

\bibitem[Ginzburg and Sag(2001)]{Ginzburg2001InterrogativeIT}
Jonathan Ginzburg and Ivan~A. Sag.
\newblock \emph{Interrogative Investigations: The Form, Meaning, and Use of
  English Interrogatives}.
\newblock 2001.

\bibitem[Goldberg(2006)]{Goldberg2006ConstructionsAW}
Adele~E. Goldberg.
\newblock Constructions at work: The nature of generalization in language.
\newblock 2006.

\bibitem[Gonen et~al.(2022)Gonen, Iyer, Blevins, Smith, and
  Zettlemoyer]{Gonen2022DemystifyingPI}
Hila Gonen, Srini Iyer, Terra Blevins, Noah~A. Smith, and Luke Zettlemoyer.
\newblock Demystifying prompts in language models via perplexity estimation.
\newblock \emph{ArXiv}, abs/2212.04037, 2022.

\bibitem[Hale(2006)]{Hale2006UncertaintyAT}
John Hale.
\newblock Uncertainty about the rest of the sentence.
\newblock \emph{Cognitive science}, 30 4:\penalty0 643--72, 2006.

\bibitem[Hewitt and Liang(2019)]{Hewitt2019DesigningAI}
John Hewitt and Percy Liang.
\newblock Designing and interpreting probes with control tasks.
\newblock \emph{ArXiv}, abs/1909.03368, 2019.

\bibitem[Hunter and Dyer(2013)]{Hunter2013DistributionsOM}
Tim Hunter and Chris Dyer.
\newblock Distributions on minimalist grammar derivations.
\newblock In \emph{Mathematics of Language}, 2013.

\bibitem[Hupkes et~al.(2020)Hupkes, Dankers, Mul, and
  Bruni]{DBLP:journals/jair/HupkesDMB20}
Dieuwke Hupkes, Verna Dankers, Mathijs Mul, and Elia Bruni.
\newblock Compositionality decomposed: How do neural networks generalise?
\newblock \emph{J. Artif. Intell. Res.}, 67:\penalty0 757--795, 2020.
\newblock \doi{10.1613/jair.1.11674}.
\newblock URL \url{https://doi.org/10.1613/jair.1.11674}.

\bibitem[Hutter(2004)]{Hutter2004UniversalAI}
Marcus Hutter.
\newblock Universal artificial intelligence.
\newblock In \emph{Texts in Theoretical Computer Science. An EATCS Series},
  2004.

\bibitem[J{\"a}ger and Rogers(2012)]{Jger2012FormalLT}
Gerhard J{\"a}ger and James Rogers.
\newblock Formal language theory: refining the chomsky hierarchy.
\newblock \emph{Philosophical Transactions of the Royal Society B: Biological
  Sciences}, 367:\penalty0 1956 -- 1970, 2012.

\bibitem[Joshi(1985)]{Joshi1985NaturalLP}
Aravind~K. Joshi.
\newblock Natural language parsing: Tree adjoining grammars: How much
  context-sensitivity is required to provide reasonable structural
  descriptions?
\newblock 1985.

\bibitem[Kallmeyer(2010{\natexlab{a}})]{Kallmeyer2010OnMC}
Laura Kallmeyer.
\newblock On mildly context-sensitive non-linear rewriting.
\newblock \emph{Research on Language and Computation}, 8:\penalty0 341--363,
  2010{\natexlab{a}}.

\bibitem[Kallmeyer(2010{\natexlab{b}})]{Kallmeyer2010ParsingBC}
Laura Kallmeyer.
\newblock Parsing beyond context-free grammars.
\newblock In \emph{Cognitive Technologies}, 2010{\natexlab{b}}.

\bibitem[Kallmeyer and Romero(2004)]{Kallmeyer2004LTAGSW}
Laura Kallmeyer and Maribel Romero.
\newblock Ltag semantics with semantic unification.
\newblock In \emph{Tag}, 2004.

\bibitem[Kamp and Reyle(1993)]{Kamp1993FromDT}
Hans Kamp and Uwe Reyle.
\newblock From discourse to logic - introduction to modeltheoretic semantics of
  natural language, formal logic and discourse representation theory.
\newblock In \emph{Studies in Linguistics and Philosophy}, 1993.

\bibitem[Karlsson(2007)]{karlsson:2007-constraints}
Fred Karlsson.
\newblock Constraints on multiple center-embedding of clauses.
\newblock \emph{Journal of Linguistics}, pages 365--392, 2007.

\bibitem[Kim and Sells(2008)]{Kim2008EnglishSA}
Jong-Bok Kim and Peter Sells.
\newblock \emph{English Syntax: An Introduction}.
\newblock 2008.

\bibitem[Kim and Linzen(2020)]{Kim2020COGSAC}
Najoung Kim and Tal Linzen.
\newblock Cogs: A compositional generalization challenge based on semantic
  interpretation.
\newblock \emph{ArXiv}, abs/2010.05465, 2020.

\bibitem[Kim et~al.(2019)Kim, Dyer, and Rush]{DBLP:conf/acl/KimDR19}
Yoon Kim, Chris Dyer, and Alexander~M. Rush.
\newblock Compound probabilistic context-free grammars for grammar induction.
\newblock In Anna Korhonen, David~R. Traum, and Llu{\'{\i}}s M{\`{a}}rquez,
  editors, \emph{Proceedings of the 57th Conference of the Association for
  Computational Linguistics, {ACL} 2019, Florence, Italy, July 28- August 2,
  2019, Volume 1: Long Papers}, pages 2369--2385. Association for Computational
  Linguistics, 2019.
\newblock \doi{10.18653/v1/p19-1228}.
\newblock URL \url{https://doi.org/10.18653/v1/p19-1228}.

\bibitem[Kobele(2005)]{kobele2005features}
Gregory~M Kobele.
\newblock Features moving madly: A formal perspective on feature percolation in
  the minimalist program.
\newblock \emph{Research on Language and Computation}, 3\penalty0
  (2-3):\penalty0 391--410, 2005.

\bibitem[Lampinen et~al.(2022)Lampinen, Dasgupta, Chan, Mathewson, Tessler,
  Creswell, McClelland, Wang, and Hill]{DBLP:conf/emnlp/LampinenDCMTCMW22}
Andrew~K. Lampinen, Ishita Dasgupta, Stephanie C.~Y. Chan, Kory~W. Mathewson,
  Mh~Tessler, Antonia Creswell, James~L. McClelland, Jane Wang, and Felix Hill.
\newblock Can language models learn from explanations in context?
\newblock In Yoav Goldberg, Zornitsa Kozareva, and Yue Zhang, editors,
  \emph{Findings of the Association for Computational Linguistics: {EMNLP}
  2022, Abu Dhabi, United Arab Emirates, December 7-11, 2022}, pages 537--563.
  Association for Computational Linguistics, 2022.
\newblock URL \url{https://aclanthology.org/2022.findings-emnlp.38}.

\bibitem[Li and Vit{\'a}nyi(2008)]{LiVitanyi2008introduction}
Ming Li and Paul Vit{\'a}nyi.
\newblock \emph{An Introduction to Kolmogorov Complexity and its Applications}.
\newblock 3 edition, 2008.

\bibitem[Li et~al.(2023)Li, Ildiz, Papailiopoulos, and
  Oymak]{DBLP:journals/corr/abs-2301-07067}
Yingcong Li, M.~Emrullah Ildiz, Dimitris~S. Papailiopoulos, and Samet Oymak.
\newblock Transformers as algorithms: Generalization and implicit model
  selection in in-context learning.
\newblock \emph{CoRR}, abs/2301.07067, 2023.
\newblock \doi{10.48550/arXiv.2301.07067}.
\newblock URL \url{https://doi.org/10.48550/arXiv.2301.07067}.

\bibitem[Liu et~al.(2022)Liu, Kitouni, Nolte, Michaud, Tegmark, and
  Williams]{Liu2022TowardsUG}
Ziming Liu, Ouail Kitouni, Niklas~Stefan Nolte, Eric~J. Michaud, Max Tegmark,
  and Mike Williams.
\newblock Towards understanding grokking: An effective theory of representation
  learning.
\newblock \emph{ArXiv}, abs/2205.10343, 2022.

\bibitem[Manning and Sch{\"u}tze(1999)]{Manning1999FoundationsOS}
Christopher~D. Manning and Hinrich Sch{\"u}tze.
\newblock Foundations of statistical natural language processing.
\newblock 1999.

\bibitem[Marcus(1998)]{Marcus1998RethinkingEC}
Gary~F. Marcus.
\newblock Rethinking eliminative connectionism.
\newblock \emph{Cognitive Psychology}, 37:\penalty0 243--282, 1998.

\bibitem[Michaelis(1998)]{DBLP:conf/lacl/Michaelis98}
Jens Michaelis.
\newblock Derivational minimalism is mildly context-sensitive.
\newblock In Michael Moortgat, editor, \emph{Logical Aspects of Computational
  Linguistics, Third International Conference, LACL'98, Grenoble, France,
  December 14-16, 1998, Selected Papers}, volume 2014 of \emph{Lecture Notes in
  Computer Science}, pages 179--198. Springer, 1998.
\newblock \doi{10.1007/3-540-45738-0\_11}.
\newblock URL \url{https://doi.org/10.1007/3-540-45738-0\_11}.

\bibitem[Michaelis(2001)]{DBLP:conf/lacl/Michaelis01}
Jens Michaelis.
\newblock Transforming linear context-free rewriting systems into minimalist
  grammars.
\newblock In Philippe de~Groote, Glyn Morrill, and Christian Retor{\'{e}},
  editors, \emph{Logical Aspects of Computational Linguistics, 4th
  International Conference, {LACL} 2001, Le Croisic, France, June 27-29, 2001,
  Proceedings}, volume 2099 of \emph{Lecture Notes in Computer Science}, pages
  228--244. Springer, 2001.
\newblock \doi{10.1007/3-540-48199-0\_14}.
\newblock URL \url{https://doi.org/10.1007/3-540-48199-0\_14}.

\bibitem[Min et~al.(2022)Min, Lyu, Holtzman, Artetxe, Lewis, Hajishirzi, and
  Zettlemoyer]{DBLP:journals/corr/abs-2202-12837}
Sewon Min, Xinxi Lyu, Ari Holtzman, Mikel Artetxe, Mike Lewis, Hannaneh
  Hajishirzi, and Luke Zettlemoyer.
\newblock Rethinking the role of demonstrations: What makes in-context learning
  work?
\newblock \emph{CoRR}, abs/2202.12837, 2022.
\newblock URL \url{https://arxiv.org/abs/2202.12837}.

\bibitem[Montague(1973)]{montague1973proper}
Richard Montague.
\newblock The proper treatment of quantification in ordinary {E}nglish.
\newblock In \emph{Approaches to natural language}, pages 221--242. Springer,
  1973.

\bibitem[Mueller et~al.(2022)Mueller, Frank, Linzen, Wang, and
  Schuster]{DBLP:conf/acl/Mueller0LWS22}
Aaron Mueller, Robert Frank, Tal Linzen, Luheng Wang, and Sebastian Schuster.
\newblock Coloring the blank slate: Pre-training imparts a hierarchical
  inductive bias to sequence-to-sequence models.
\newblock In Smaranda Muresan, Preslav Nakov, and Aline Villavicencio, editors,
  \emph{Findings of the Association for Computational Linguistics: {ACL} 2022,
  Dublin, Ireland, May 22-27, 2022}, pages 1352--1368. Association for
  Computational Linguistics, 2022.
\newblock \doi{10.18653/v1/2022.findings-acl.106}.
\newblock URL \url{https://doi.org/10.18653/v1/2022.findings-acl.106}.

\bibitem[M{\"u}ller(2020)]{Mller2020GrammaticalT}
Stefan M{\"u}ller.
\newblock Grammatical theory: From transformational grammar to constraint-based
  approaches.
\newblock Language Science Press, 2020.
\newblock URL \url{https://library.oapen.org/handle/20.500.12657/46939}.

\bibitem[Nye et~al.(2021)Nye, Andreassen, Gur{-}Ari, Michalewski, Austin,
  Bieber, Dohan, Lewkowycz, Bosma, Luan, Sutton, and
  Odena]{DBLP:journals/corr/abs-2112-00114}
Maxwell~I. Nye, Anders~Johan Andreassen, Guy Gur{-}Ari, Henryk Michalewski,
  Jacob Austin, David Bieber, David Dohan, Aitor Lewkowycz, Maarten Bosma,
  David Luan, Charles Sutton, and Augustus Odena.
\newblock Show your work: Scratchpads for intermediate computation with
  language models.
\newblock \emph{CoRR}, abs/2112.00114, 2021.
\newblock URL \url{https://arxiv.org/abs/2112.00114}.

\bibitem[Olsson et~al.(2022)Olsson, Elhage, Nanda, Joseph, DasSarma, Henighan,
  Mann, Askell, Bai, Chen, Conerly, Drain, Ganguli, Hatfield-Dodds, Hernandez,
  Johnston, Jones, Kernion, Lovitt, Ndousse, Amodei, Brown, Clark, Kaplan,
  McCandlish, and Olah]{Olsson2022IncontextLA}
Catherine Olsson, Nelson Elhage, Neel Nanda, Nicholas Joseph, Nova DasSarma,
  T.~J. Henighan, Benjamin Mann, Amanda Askell, Yushi Bai, Anna Chen, Tom
  Conerly, Dawn Drain, Deep Ganguli, Zac Hatfield-Dodds, Danny Hernandez, Scott
  Johnston, Andy Jones, John Kernion, Liane Lovitt, Kamal Ndousse, Dario
  Amodei, Tom~B. Brown, Jack Clark, Jared Kaplan, Sam McCandlish, and
  Christopher Olah.
\newblock In-context learning and induction heads.
\newblock \emph{ArXiv}, abs/2209.11895, 2022.

\bibitem[Papadimitriou and Jurafsky(2020)]{DBLP:conf/emnlp/PapadimitriouJ20}
Isabel Papadimitriou and Dan Jurafsky.
\newblock Learning music helps you read: Using transfer to study linguistic
  structure in language models.
\newblock In Bonnie Webber, Trevor Cohn, Yulan He, and Yang Liu, editors,
  \emph{Proceedings of the 2020 Conference on Empirical Methods in Natural
  Language Processing, {EMNLP} 2020, Online, November 16-20, 2020}, pages
  6829--6839. Association for Computational Linguistics, 2020.
\newblock \doi{10.18653/v1/2020.emnlp-main.554}.
\newblock URL \url{https://doi.org/10.18653/v1/2020.emnlp-main.554}.

\bibitem[Pollard(1984)]{Pollard1984GeneralizedPS}
Carl Pollard.
\newblock Generalized phrase structure grammars, head grammars, and natural
  language, 1984.

\bibitem[Pollard and Sag(1994)]{pollard1994head}
Carl Pollard and Ivan~A Sag.
\newblock \emph{Head-driven phrase structure grammar}.
\newblock University of Chicago Press, 1994.

\bibitem[Portelance et~al.(2017)Portelance, Bergen, Bruno, and
  O'Donnell]{DBLP:journals/corr/abs-1710-11350}
Eva Portelance, Leon Bergen, Chris Bruno, and Timothy~J. O'Donnell.
\newblock Mildly context sensitive grammar induction and variational bayesian
  inference.
\newblock \emph{CoRR}, abs/1710.11350, 2017.
\newblock URL \url{http://arxiv.org/abs/1710.11350}.

\bibitem[Power et~al.(2022)Power, Burda, Edwards, Babuschkin, and
  Misra]{Power2022GrokkingGB}
Alethea Power, Yuri Burda, Harrison Edwards, Igor Babuschkin, and Vedant Misra.
\newblock Grokking: Generalization beyond overfitting on small algorithmic
  datasets.
\newblock \emph{ArXiv}, abs/2201.02177, 2022.

\bibitem[Radford et~al.(2019)Radford, Wu, Child, Luan, Amodei, and
  Sutskever]{Radford2019LanguageMA}
A.~Radford, Jeffrey Wu, R.~Child, David Luan, Dario Amodei, and Ilya Sutskever.
\newblock \emph{Language Models are Unsupervised Multitask Learners}.
\newblock OpenAI, 2019.

\bibitem[Razeghi et~al.(2022)Razeghi, IV, Gardner, and
  Singh]{DBLP:conf/emnlp/RazeghiL0022}
Yasaman Razeghi, Robert L.~Logan IV, Matt Gardner, and Sameer Singh.
\newblock Impact of pretraining term frequencies on few-shot numerical
  reasoning.
\newblock In Yoav Goldberg, Zornitsa Kozareva, and Yue Zhang, editors,
  \emph{Findings of the Association for Computational Linguistics: {EMNLP}
  2022, Abu Dhabi, United Arab Emirates, December 7-11, 2022}, pages 840--854.
  Association for Computational Linguistics, 2022.
\newblock URL \url{https://aclanthology.org/2022.findings-emnlp.59}.

\bibitem[Rissanen(2010)]{Rissanen2010MinimumDL}
Jorma Rissanen.
\newblock Minimum description length principle.
\newblock In \emph{Encyclopedia of Machine Learning}, 2010.

\bibitem[Rong(2021)]{rong21extrapolating}
Frieda Rong.
\newblock Extrapolating to unnatural language processing with gpt-3's
  in-context learning: The good, the bad, and the mysterious.
\newblock
  \url{https://ai.stanford.edu/blog/in-context-learning/#:~:text=Extrapolating%20to%20Unnatural%20Language%20Processing%20with%20GPT-3%27s%20In-context,GPT-3%27s%20ability%20to%20extrapolate%20to%20less%20natural%20inputs},
  2021.

\bibitem[Ross(1970)]{Ross1970GAPPINGAT}
John~Robert Ross.
\newblock Gapping and the order of constituents.
\newblock 1970.

\bibitem[Seki et~al.(1991)Seki, Matsumura, Fujii, and Kasami]{Seki1991OnMC}
Hiroyuki Seki, Takashi Matsumura, Mamoru Fujii, and Tadao Kasami.
\newblock On multiple context-free grammars.
\newblock \emph{Theor. Comput. Sci.}, 88:\penalty0 191--229, 1991.

\bibitem[Shieber(1985)]{Shieber1985EvidenceAT}
Stuart~M. Shieber.
\newblock Evidence against the context-freeness of natural language.
\newblock \emph{Linguistics and Philosophy}, 8:\penalty0 333--343, 1985.

\bibitem[Shin et~al.(2022)Shin, Lee, Ahn, Kim, Kim, Kim, Cho, Lee, Park, Ha,
  and Sung]{DBLP:conf/naacl/ShinLAKKKCLPHS22}
Seongjin Shin, Sang{-}Woo Lee, Hwijeen Ahn, Sungdong Kim, HyoungSeok Kim,
  Boseop Kim, Kyunghyun Cho, Gichang Lee, Woo{-}Myoung Park, Jung{-}Woo Ha, and
  Nako Sung.
\newblock On the effect of pretraining corpora on in-context learning by a
  large-scale language model.
\newblock In Marine Carpuat, Marie{-}Catherine de~Marneffe, and Iv{\'{a}}n
  Vladimir~Meza Ru{\'{\i}}z, editors, \emph{Proceedings of the 2022 Conference
  of the North American Chapter of the Association for Computational
  Linguistics: Human Language Technologies, {NAACL} 2022, Seattle, WA, United
  States, July 10-15, 2022}, pages 5168--5186. Association for Computational
  Linguistics, 2022.
\newblock \doi{10.18653/v1/2022.naacl-main.380}.
\newblock URL \url{https://doi.org/10.18653/v1/2022.naacl-main.380}.

\bibitem[Solomonoff(1964)]{solomonoff1964formal}
Ray~J. Solomonoff.
\newblock A formal theory of inductive inference, part ii.
\newblock \emph{Information and Control}, 7\penalty0 (2):\penalty0 224--254,
  1964.

\bibitem[Stabler(1996)]{Stabler1996DerivationalM}
E.~Stabler.
\newblock Derivational minimalism.
\newblock In \emph{Logical Aspects of Computational Linguistics}, 1996.

\bibitem[Steedman(1990)]{Steedman1990GappingAC}
Mark Steedman.
\newblock Gapping as constituent coordination.
\newblock \emph{Linguistics and Philosophy}, 13:\penalty0 207--263, 1990.

\bibitem[Steedman(2001)]{Steedman2004TheSP}
Mark Steedman.
\newblock \emph{The syntactic process}.
\newblock 2001.

\bibitem[Suzgun et~al.(2022)Suzgun, Scales, Scharli, Gehrmann, Tay, Chung,
  Chowdhery, Le, hsin Chi, Zhou, and Wei]{Suzgun2022ChallengingBT}
Mirac Suzgun, Nathan Scales, Nathanael Scharli, Sebastian Gehrmann, Yi~Tay,
  Hyung~Won Chung, Aakanksha Chowdhery, Quoc~V. Le, Ed~Huai hsin Chi, Denny
  Zhou, and Jason Wei.
\newblock Challenging big-bench tasks and whether chain-of-thought can solve
  them.
\newblock \emph{ArXiv}, abs/2210.09261, 2022.

\bibitem[Torr et~al.(2019)Torr, Stanojevi{\'c}, Steedman, and
  Cohen]{Torr2019WideCoverageNA}
John Torr, Milos Stanojevi{\'c}, Mark Steedman, and Shay~B. Cohen.
\newblock Wide-coverage neural a* parsing for minimalist grammars.
\newblock In \emph{Annual Meeting of the Association for Computational
  Linguistics}, 2019.

\bibitem[Vijay-Shanker and Weir(1994)]{VijayShanker1994TheEO}
K.~Vijay-Shanker and David~J. Weir.
\newblock The equivalence of four extensions of context-free grammars.
\newblock \emph{Mathematical systems theory}, 27:\penalty0 511--546, 1994.

\bibitem[Vijay-Shanker et~al.(1987)Vijay-Shanker, Weir, and
  Joshi]{VijayShanker1987CharacterizingSD}
K.~Vijay-Shanker, David~J. Weir, and Aravind~K. Joshi.
\newblock Characterizing structural descriptions produced by various
  grammatical formalisms.
\newblock In \emph{Annual Meeting of the Association for Computational
  Linguistics}, 1987.

\bibitem[Voita and Titov(2020)]{DBLP:conf/emnlp/VoitaT20}
Elena Voita and Ivan Titov.
\newblock Information-theoretic probing with minimum description length.
\newblock In Bonnie Webber, Trevor Cohn, Yulan He, and Yang Liu, editors,
  \emph{Proceedings of the 2020 Conference on Empirical Methods in Natural
  Language Processing, {EMNLP} 2020, Online, November 16-20, 2020}, pages
  183--196. Association for Computational Linguistics, 2020.
\newblock \doi{10.18653/v1/2020.emnlp-main.14}.
\newblock URL \url{https://doi.org/10.18653/v1/2020.emnlp-main.14}.

\bibitem[von Oswald et~al.(2022)von Oswald, Niklasson, Randazzo, Sacramento,
  Mordvintsev, Zhmoginov, and Vladymyrov]{Oswald2022TransformersLI}
Johannes von Oswald, Eyvind Niklasson, Ettore Randazzo, Joao Sacramento,
  Alexander Mordvintsev, Andrey Zhmoginov, and Max Vladymyrov.
\newblock Transformers learn in-context by gradient descent.
\newblock 2022.

\bibitem[Wang et~al.(2022)Wang, Deng, and Sun]{wang-etal-2022-iteratively}
Boshi Wang, Xiang Deng, and Huan Sun.
\newblock Iteratively prompt pre-trained language models for chain of thought.
\newblock In \emph{Proceedings of the 2022 Conference on Empirical Methods in
  Natural Language Processing}, pages 2714--2730, Abu Dhabi, United Arab
  Emirates, December 2022. Association for Computational Linguistics.
\newblock URL \url{https://aclanthology.org/2022.emnlp-main.174}.

\bibitem[Wang et~al.(2023)Wang, Zhu, and Wang]{yang203large}
Xinyi Wang, Wanrong Zhu, and William~Yang Wang.
\newblock Large language models are implicitly topic models: Explaining and
  finding good demonstrations for in-context learning.
\newblock \emph{ArXiv}, arXiv:2301.11916, 2023.

\bibitem[Wei et~al.(2022{\natexlab{a}})Wei, Tay, Bommasani, Raffel, Zoph,
  Borgeaud, Yogatama, Bosma, Zhou, Metzler, Chi, Hashimoto, Vinyals, Liang,
  Dean, and Fedus]{DBLP:journals/corr/abs-2206-07682}
Jason Wei, Yi~Tay, Rishi Bommasani, Colin Raffel, Barret Zoph, Sebastian
  Borgeaud, Dani Yogatama, Maarten Bosma, Denny Zhou, Donald Metzler, Ed~H.
  Chi, Tatsunori Hashimoto, Oriol Vinyals, Percy Liang, Jeff Dean, and William
  Fedus.
\newblock Emergent abilities of large language models.
\newblock \emph{CoRR}, abs/2206.07682, 2022{\natexlab{a}}.
\newblock \doi{10.48550/arXiv.2206.07682}.
\newblock URL \url{https://doi.org/10.48550/arXiv.2206.07682}.

\bibitem[Wei et~al.(2022{\natexlab{b}})Wei, Wang, Schuurmans, Bosma, Chi, Le,
  and Zhou]{Wei2022Chain}
Jason Wei, Xuezhi Wang, Dale Schuurmans, Maarten Bosma, Ed~H. Chi, Quoc Le, and
  Denny Zhou.
\newblock Chain of thought prompting elicits reasoning in large language
  models.
\newblock \emph{CoRR}, abs/2201.11903, 2022{\natexlab{b}}.
\newblock URL \url{https://arxiv.org/abs/2201.11903}.

\bibitem[Wei et~al.(2023)Wei, Wei, Tay, Tran, Webson, Lu, Chen, Liu, Huang,
  Zhou, and Ma]{Wei2023LargerLM}
Jerry~W. Wei, Jason Wei, Yi~Tay, Dustin Tran, Albert Webson, Yifeng Lu, Xinyun
  Chen, Hanxiao Liu, Da~Huang, Denny Zhou, and Tengyu Ma.
\newblock Larger language models do in-context learning differently.
\newblock 2023.

\bibitem[Wolf et~al.(2019)Wolf, Debut, Sanh, Chaumond, Delangue, Moi, Cistac,
  Rault, Louf, Funtowicz, and Brew]{Wolf2019Huggingface}
Thomas Wolf, Lysandre Debut, Victor Sanh, Julien Chaumond, Clement Delangue,
  Anthony Moi, Pierric Cistac, Tim Rault, R{\'{e}}mi Louf, Morgan Funtowicz,
  and Jamie Brew.
\newblock Huggingface's transformers: State-of-the-art natural language
  processing.
\newblock \emph{CoRR}, abs/1910.03771, 2019.
\newblock URL \url{http://arxiv.org/abs/1910.03771}.

\bibitem[Xie et~al.(2022)Xie, Raghunathan, Liang, and
  Ma]{DBLP:conf/iclr/XieRL022}
Sang~Michael Xie, Aditi Raghunathan, Percy Liang, and Tengyu Ma.
\newblock An explanation of in-context learning as implicit bayesian inference.
\newblock In \emph{The Tenth International Conference on Learning
  Representations, {ICLR} 2022, Virtual Event, April 25-29, 2022}.
  OpenReview.net, 2022.
\newblock URL \url{https://openreview.net/forum?id=RdJVFCHjUMI}.

\bibitem[Yang et~al.(2022)Yang, Levy, and Kim]{Yang2022UnsupervisedDC}
Songlin Yang, R.~Levy, and Yoon Kim.
\newblock Unsupervised discontinuous constituency parsing with mildly
  context-sensitive grammars.
\newblock \emph{ArXiv}, abs/2212.09140, 2022.

\bibitem[Zhao and Titov(2020)]{DBLP:conf/emnlp/ZhaoT20}
Yanpeng Zhao and Ivan Titov.
\newblock Visually grounded compound pcfgs.
\newblock In Bonnie Webber, Trevor Cohn, Yulan He, and Yang Liu, editors,
  \emph{Proceedings of the 2020 Conference on Empirical Methods in Natural
  Language Processing, {EMNLP} 2020, Online, November 16-20, 2020}, pages
  4369--4379. Association for Computational Linguistics, 2020.
\newblock \doi{10.18653/v1/2020.emnlp-main.354}.
\newblock URL \url{https://doi.org/10.18653/v1/2020.emnlp-main.354}.

\end{thebibliography}

\appendix

\begin{figure*}
\centering

	\includegraphics[width=0.3\textwidth]{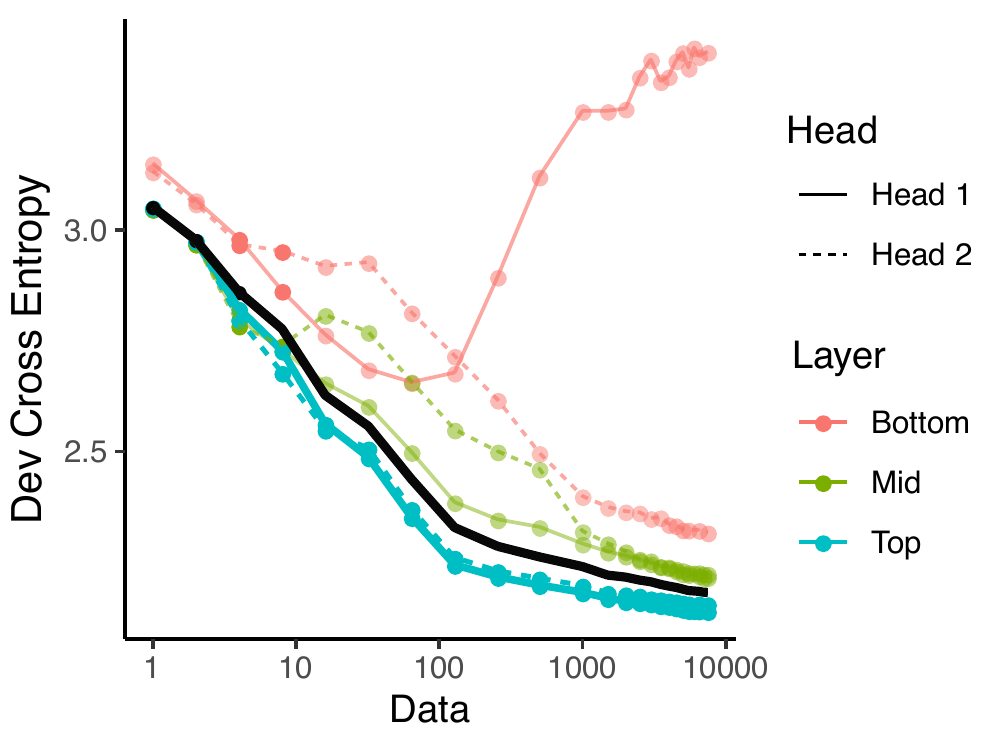}

		\caption{Intervening on attention heads by masking out non-structurally-corresponding positions reduces cross-entropy compared to trained model (black) when applied to top layer (blue), and increases it when applied to lower layers (red, green). Here, we simply intervene on the trained model, without any additional training. A benefit begins early in training, and persists.}\label{fig:attention-intervention-ce}
\end{figure*}

\begin{figure*}
\centering
        \begin{tikzpicture}

		\node (label) at (-0.75*2,-1.5)[draw=none, align=left, anchor=center]{ Accuracy};
		\node (label) at (0.75*0, -1.5)[draw=none, align=left, anchor=center]{ 0.0};
		\node (label) at (0.75*2, -1.5)[draw=none, align=left, anchor=center]{ 0.0};
		\node (label) at (0.75*4, -1.5)[draw=none, align=left, anchor=center]{ 0.0};
		\node (label) at (0.75*6, -1.5)[draw=none, align=left, anchor=center]{ 0.0};
		\node (label) at (0.75*8, -1.5)[draw=none, align=left, anchor=center]{ 0.0};
		\node (label) at (0.75*10,-1.5)[draw=none, align=left, anchor=center]{ 0.0};
		\node (label) at (0.75*12,-1.5)[draw=none, align=left, anchor=center]{ 0.86};
		\node (label) at (0.75*14,-1.5)[draw=none, align=left, anchor=center]{ 0.95};
		\node (label) at (0.75*16,-1.5)[draw=none, align=left, anchor=center]{ 0.98};
		\node (label) at (0.75*18,-1.5)[draw=none, align=left, anchor=center]{ 0.98};

		\node (label) at (-0.75*2,3.3)[draw=none, align=left, anchor=center]{ Data};
		\node (label) at (0.75*0,3.3)[draw=none, align=left, anchor=center]{ 4M};
		\node (label) at (0.75*2,3.3)[draw=none, align=left, anchor=center]{ 8M};
		\node (label) at (0.75*4,3.3)[draw=none, align=left, anchor=center]{ 16M};
		\node (label) at (0.75*6,3.3)[draw=none, align=left, anchor=center]{ 32M};
		\node (label) at (0.75*8,3.3)[draw=none, align=left, anchor=center]{ 64M};
		\node (label) at (0.75*10,3.3)[draw=none, align=left, anchor=center]{ 128M};
		\node (label) at (0.75*12,3.3)[draw=none, align=left, anchor=center]{ 256M};
		\node (label) at (0.75*14,3.3)[draw=none, align=left, anchor=center]{ 500M};
		\node (label) at (0.75*16,3.3)[draw=none, align=left, anchor=center]{ 1000M};
		\node (label) at (0.75*18,3.3)[draw=none, align=left, anchor=center]{ 7500M};
		\node (label) at (0.75*0,2)[draw=none, align=left, anchor=center]{ \includegraphics[width=0.15\textwidth]{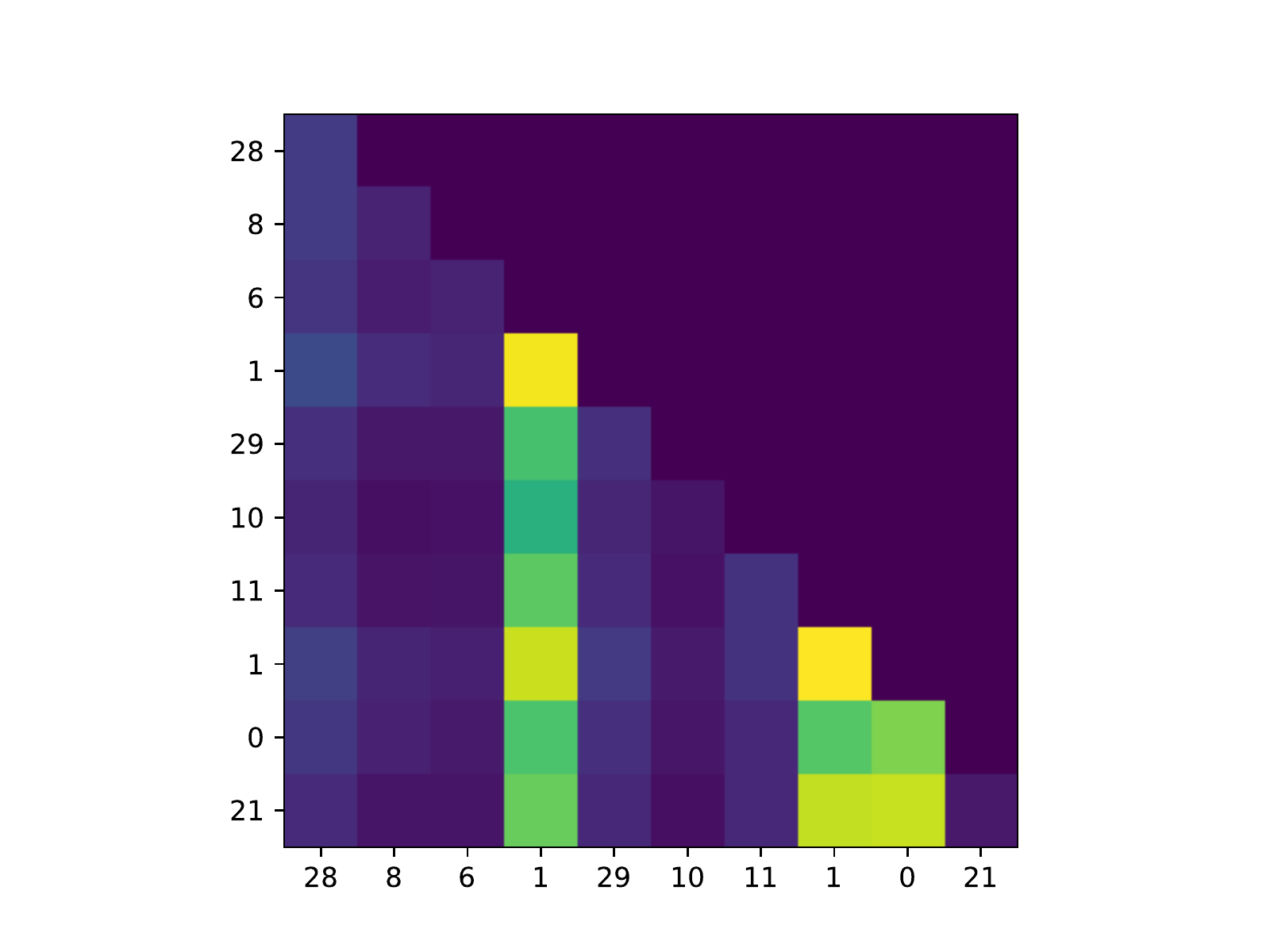}};
		\node (label) at (0.75*0,0)[draw=none, align=left, anchor=center]{ \includegraphics[width=0.15\textwidth]{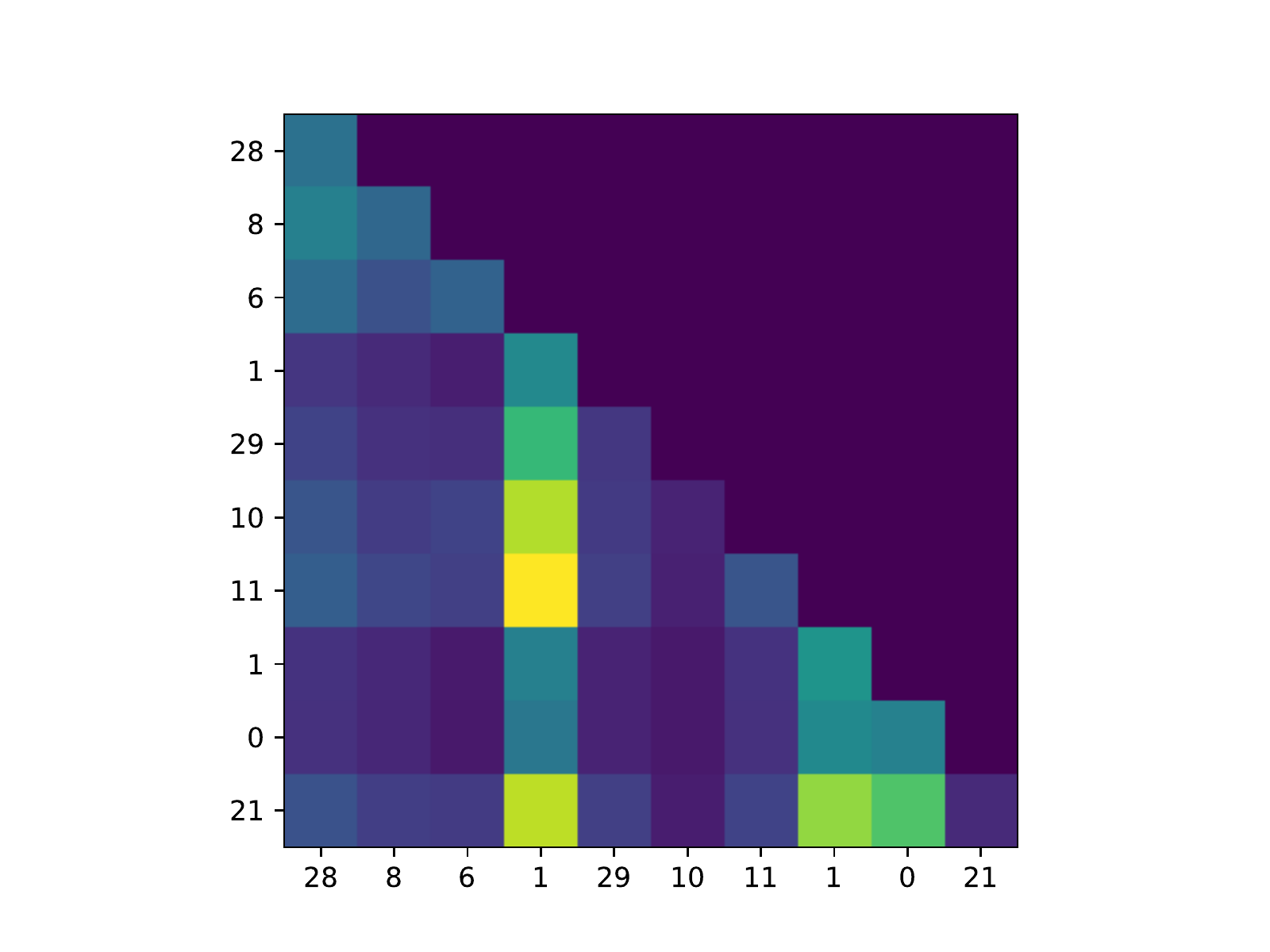}};
		\node (label) at (0.75*2,2)[draw=none, align=left, anchor=center]{ \includegraphics[width=0.15\textwidth]{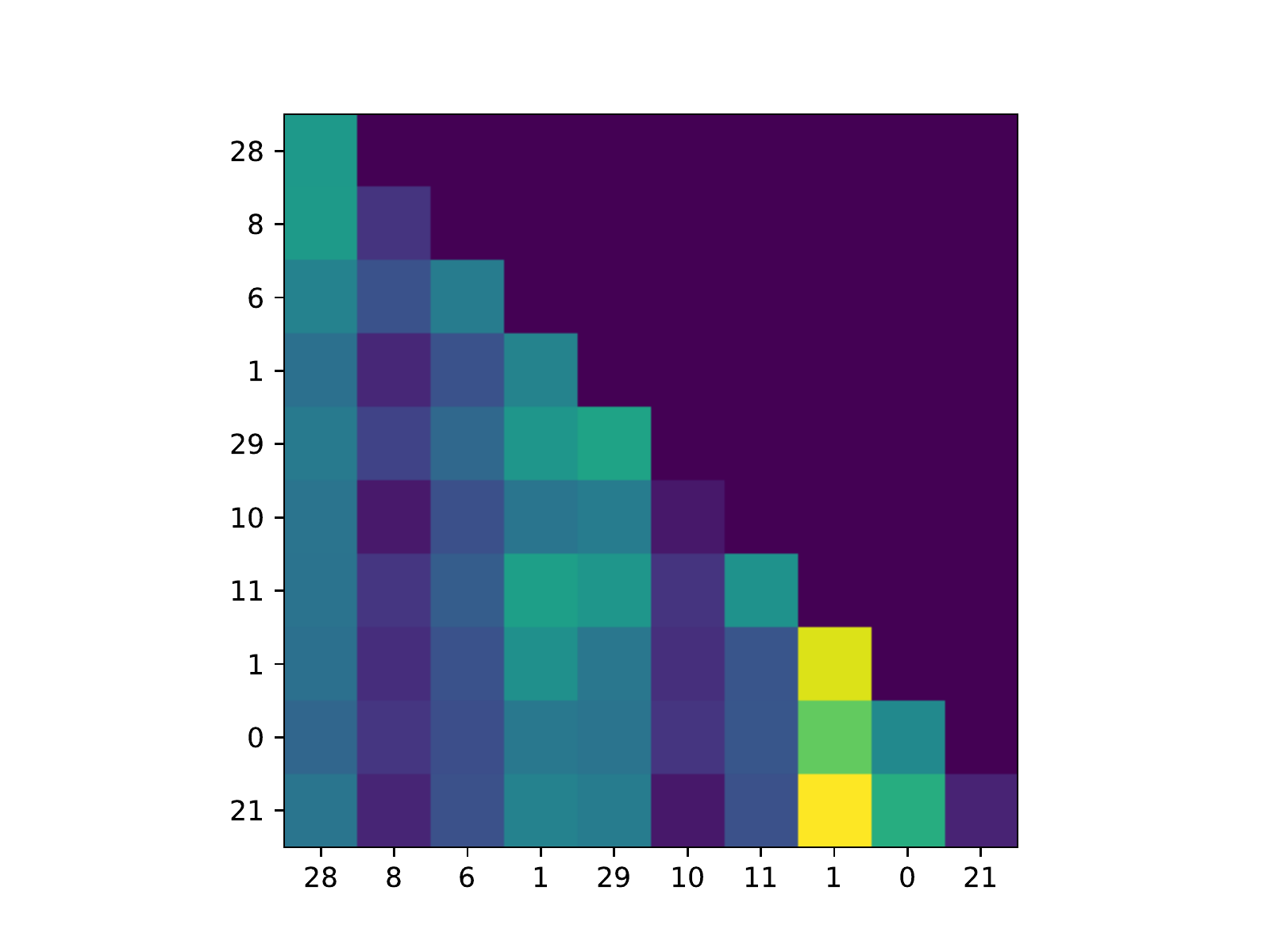}};
		\node (label) at (0.75*2,0)[draw=none, align=left, anchor=center]{ \includegraphics[width=0.15\textwidth]{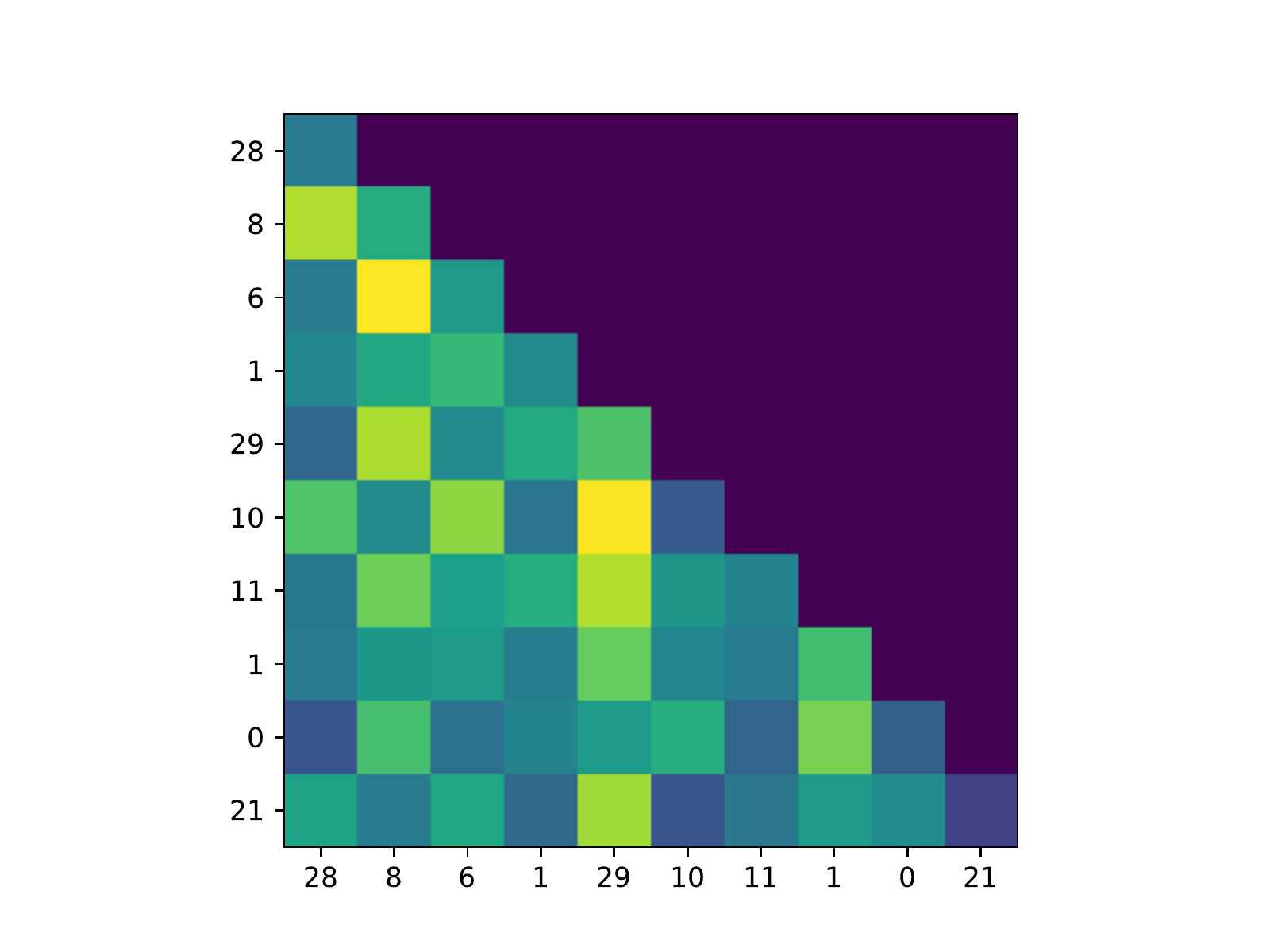}};
		\node (label) at (0.75*4,2)[draw=none, align=left, anchor=center]{ \includegraphics[width=0.15\textwidth]{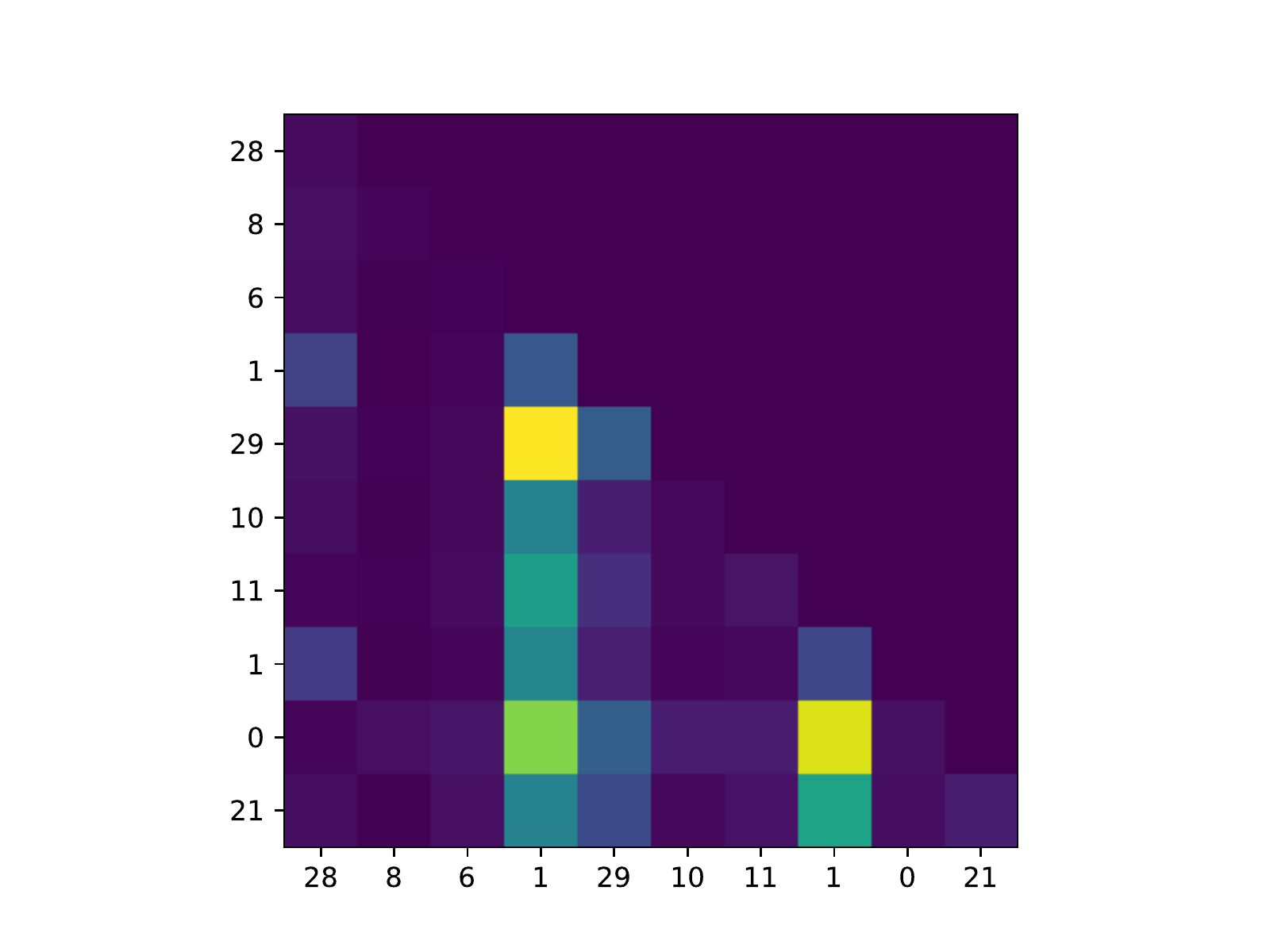}};
		\node (label) at (0.75*4,0)[draw=none, align=left, anchor=center]{ \includegraphics[width=0.15\textwidth]{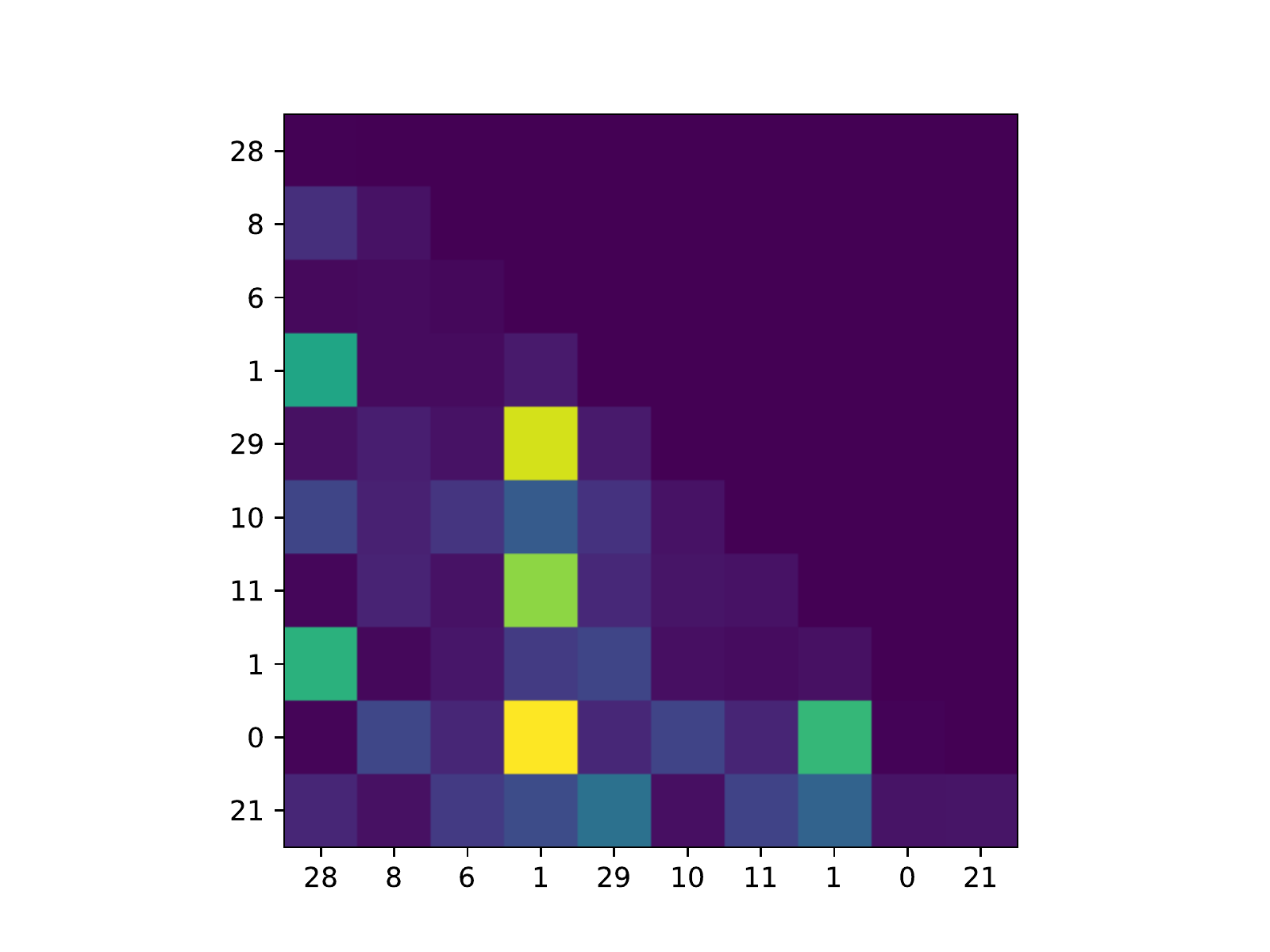}};
		\node (label) at (0.75*6,2)[draw=none, align=left, anchor=center]{ \includegraphics[width=0.15\textwidth]{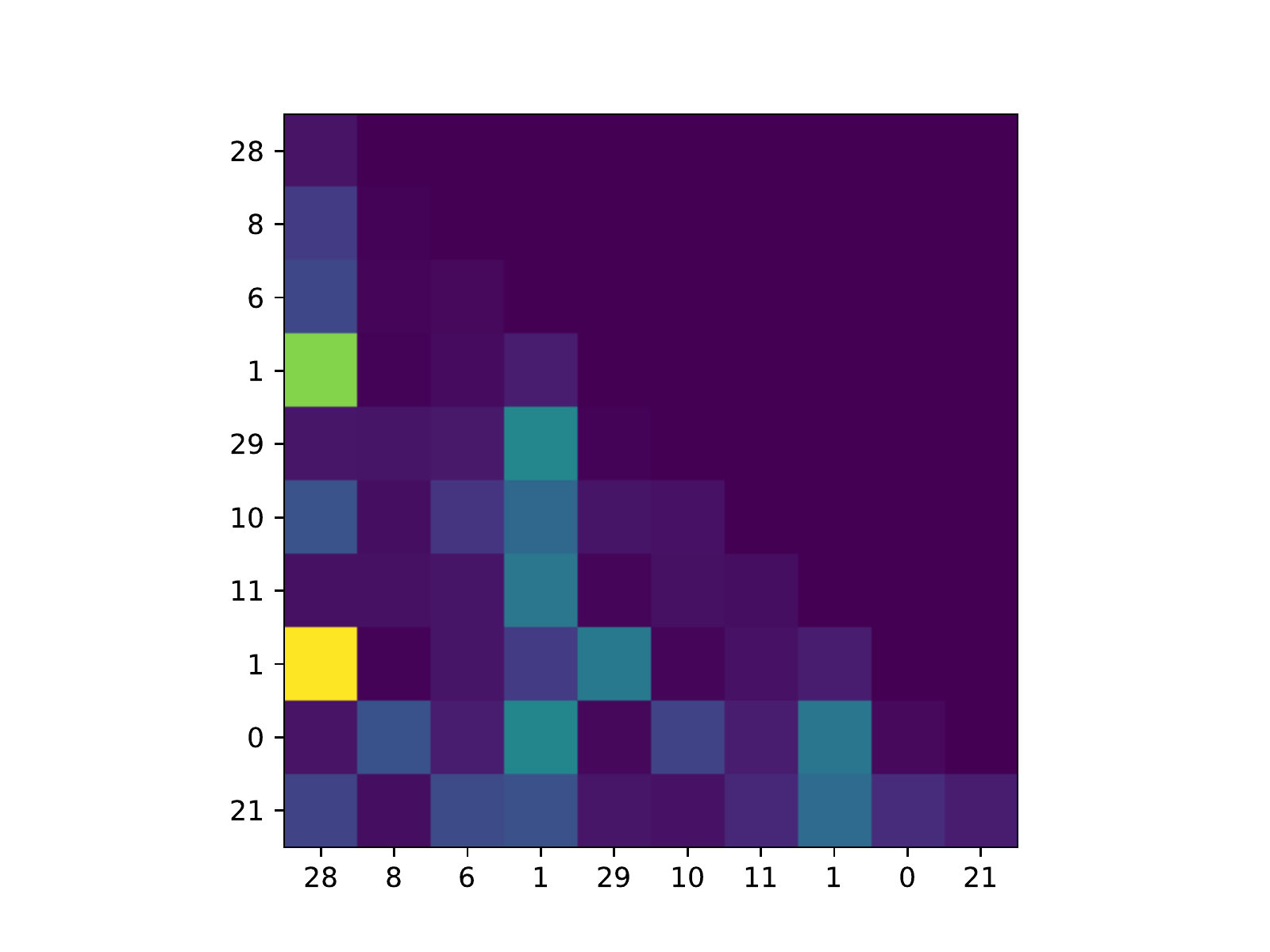}};
		\node (label) at (0.75*6,0)[draw=none, align=left, anchor=center]{ \includegraphics[width=0.15\textwidth]{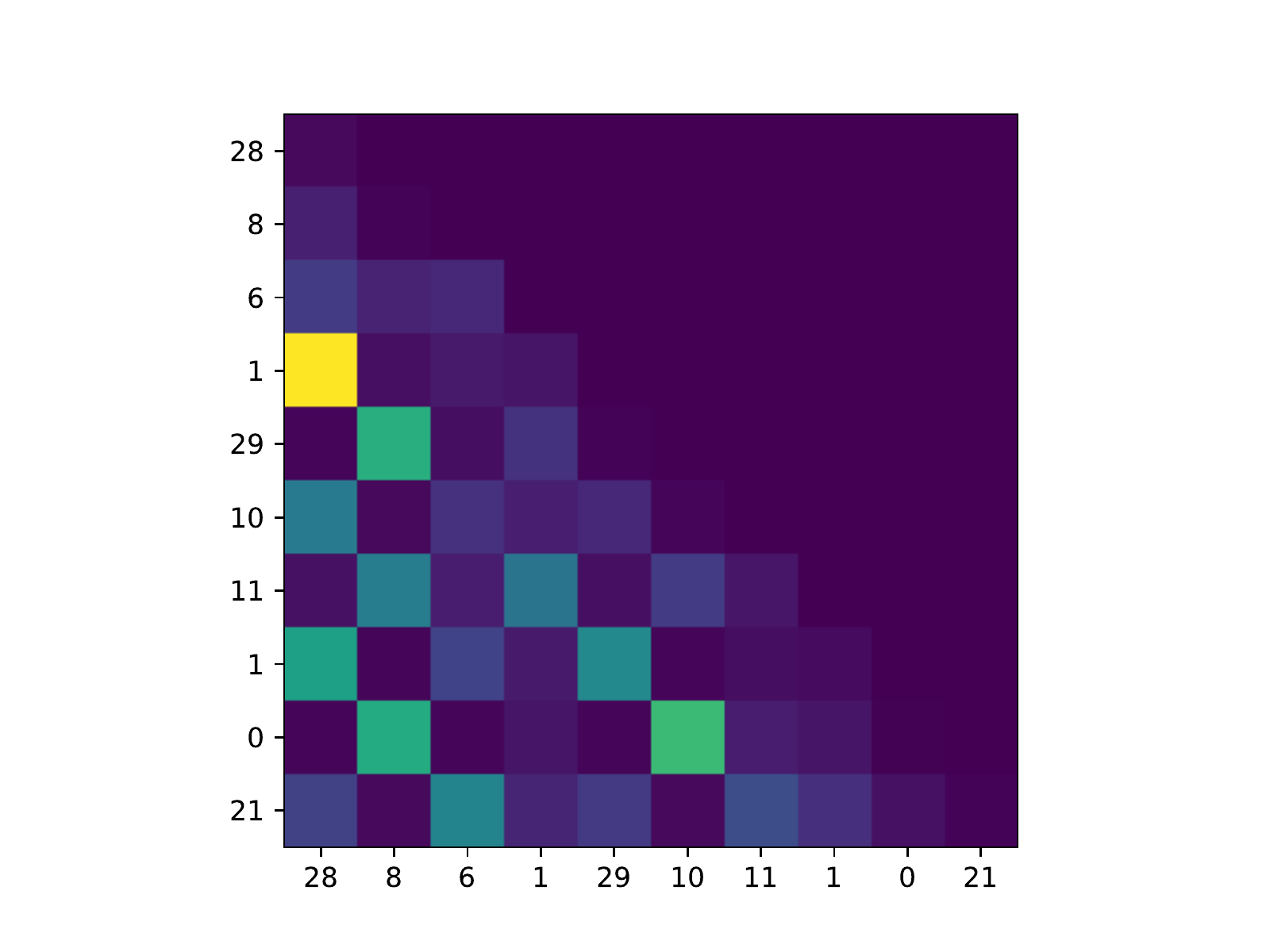}};
		\node (label) at (0.75*8,2)[draw=none, align=left, anchor=center]{ \includegraphics[width=0.15\textwidth]{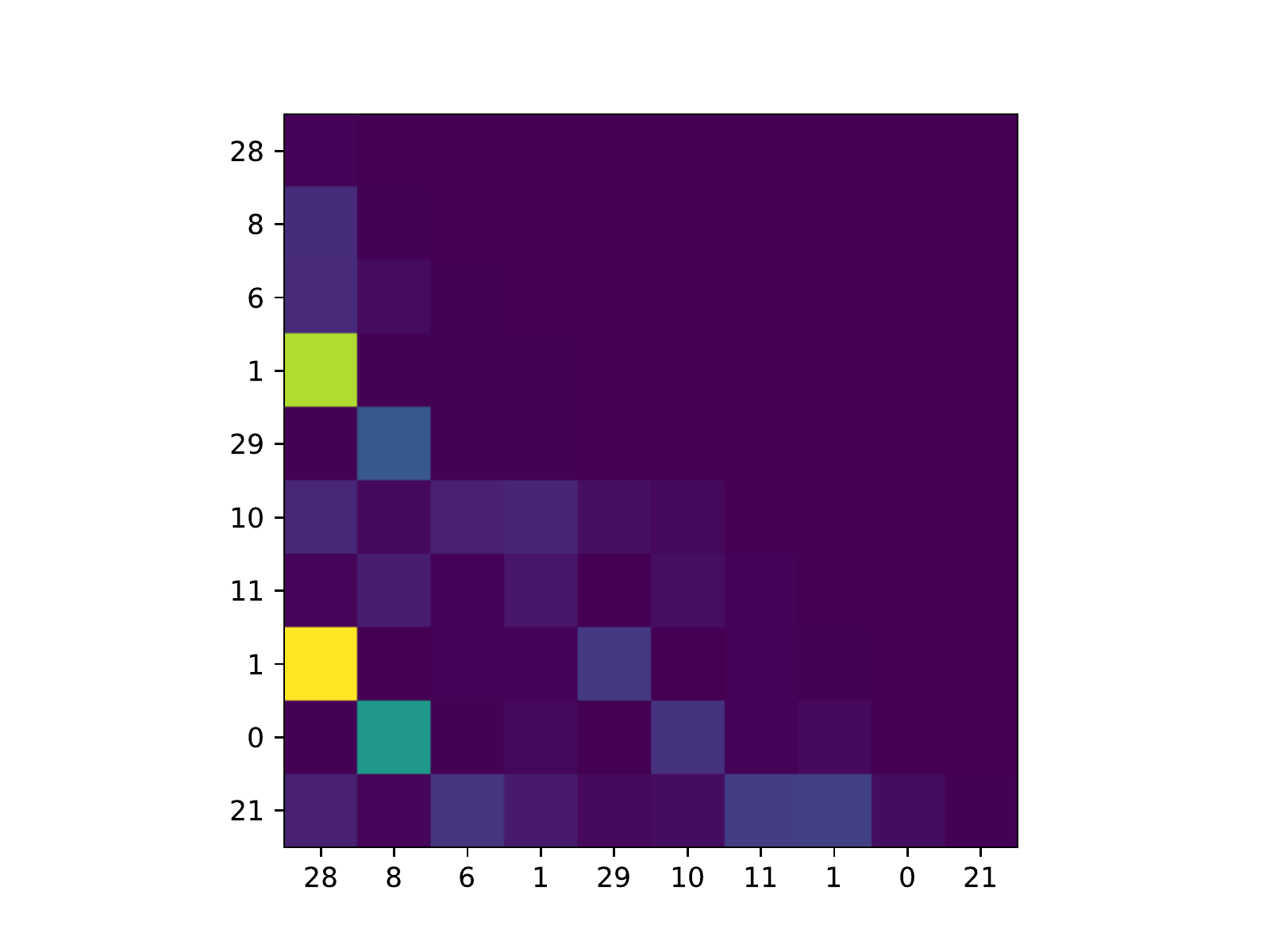}};
		\node (label) at (0.75*8,0)[draw=none, align=left, anchor=center]{ \includegraphics[width=0.15\textwidth]{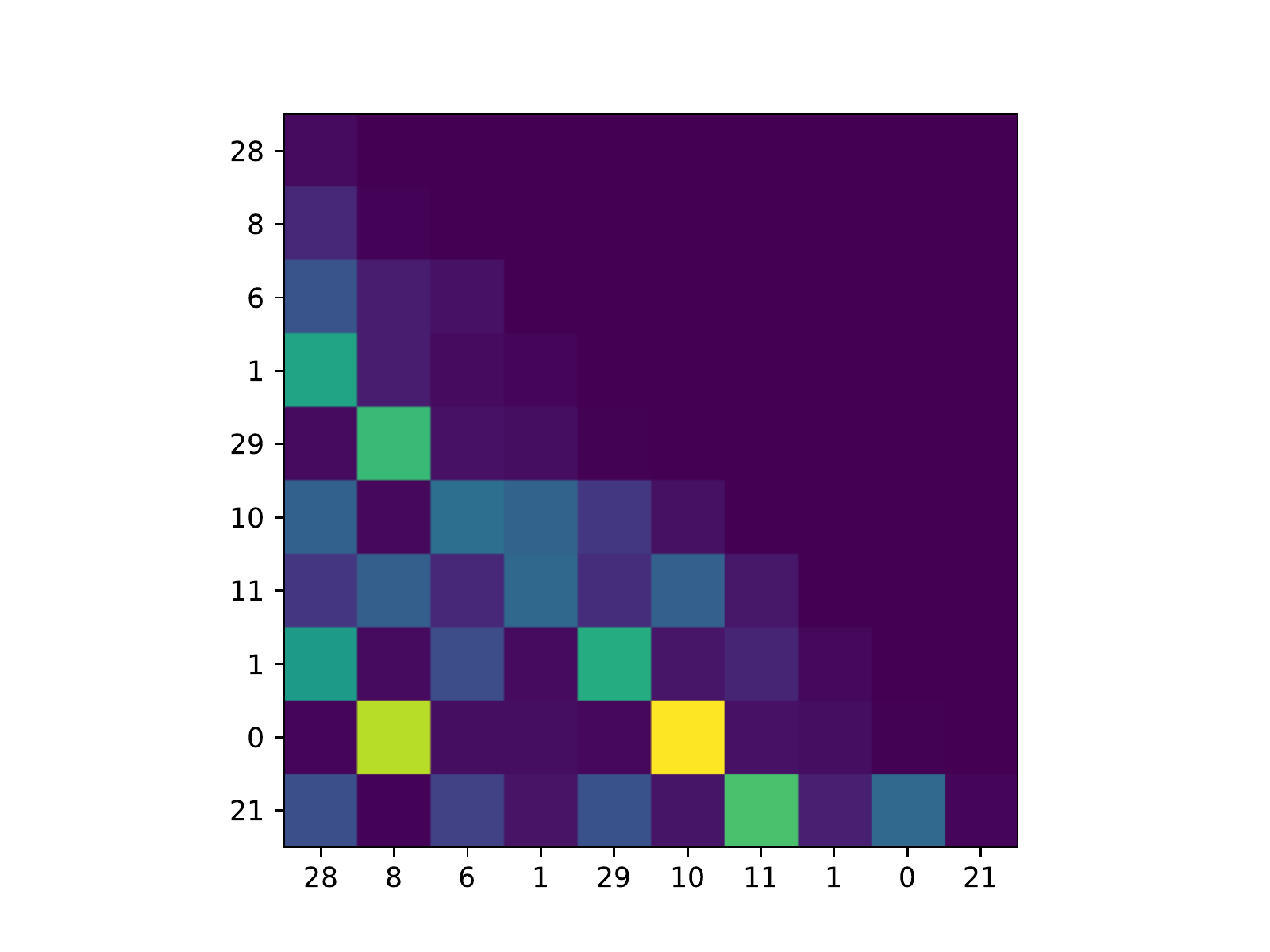}};
		\node (label) at (0.75*10,2)[draw=none, align=left, anchor=center]{ \includegraphics[width=0.15\textwidth]{probing_Sep_ChainOfThought_DiffRel_Long_End.py_medium.py_500_corpus.py_91226973.txt_293031_EPOCH_0.128.txt_2_0.pdf}};
		\node (label) at (0.75*10,0)[draw=none, align=left, anchor=center]{ \includegraphics[width=0.15\textwidth]{probing_Sep_ChainOfThought_DiffRel_Long_End.py_medium.py_500_corpus.py_91226973.txt_293031_EPOCH_0.128.txt_2_1.pdf}};
		\node (label) at (0.75*12,2)[draw=none, align=left, anchor=center]{ \includegraphics[width=0.15\textwidth]{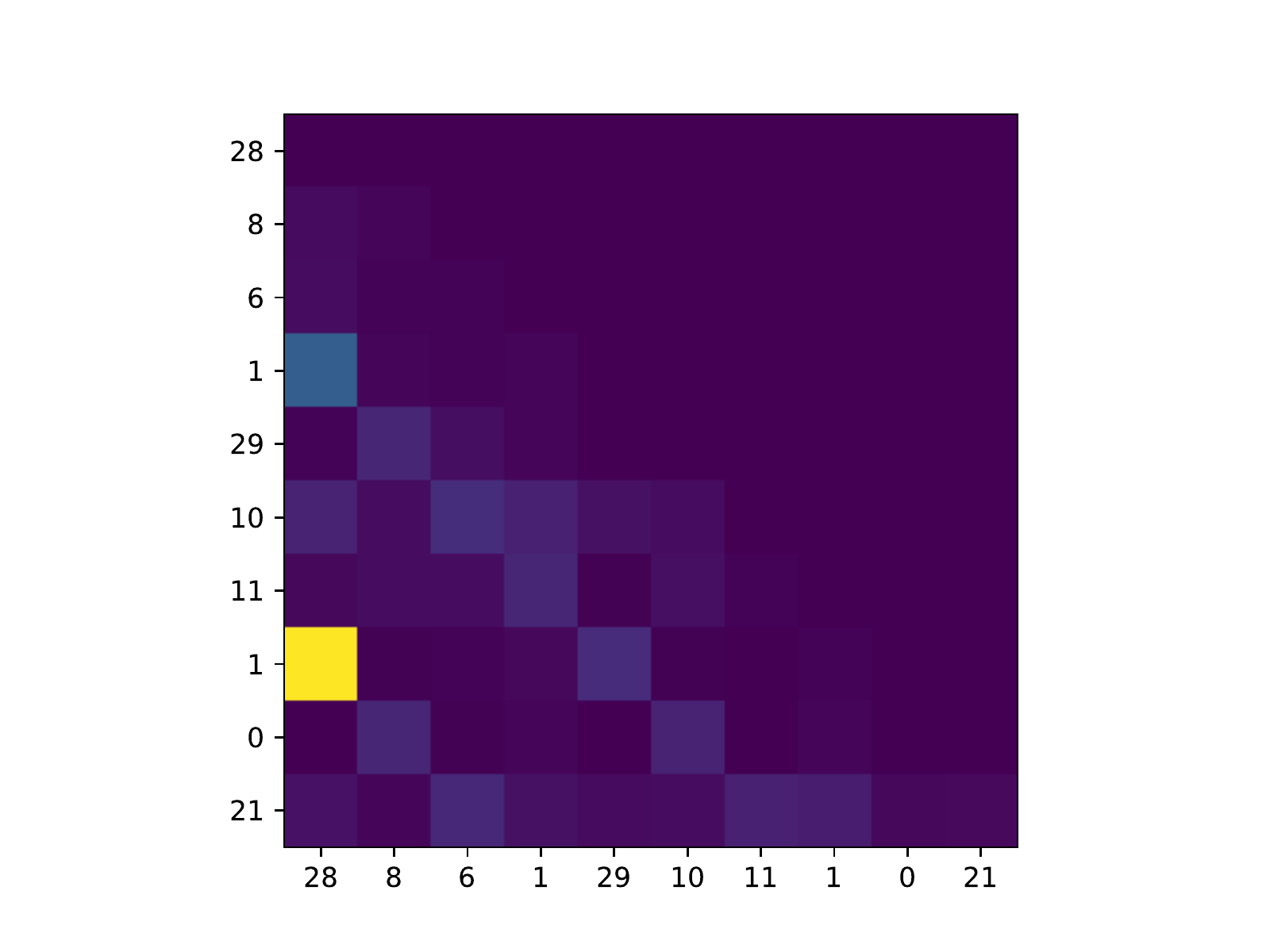}};
		\node (label) at (0.75*12,0)[draw=none, align=left, anchor=center]{ \includegraphics[width=0.15\textwidth]{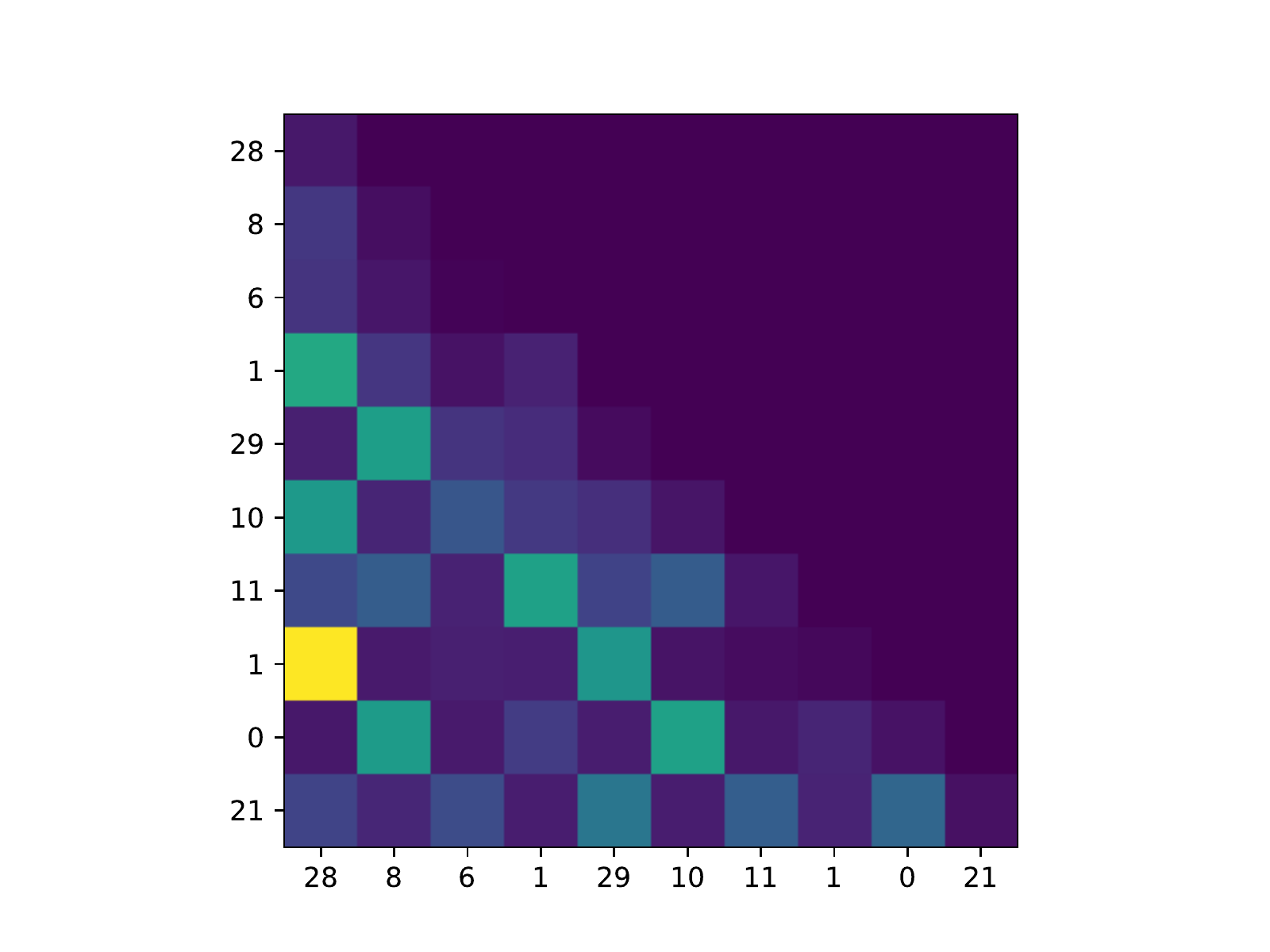}};
		\node (label) at (0.75*14,2)[draw=none, align=left, anchor=center]{ \includegraphics[width=0.15\textwidth]{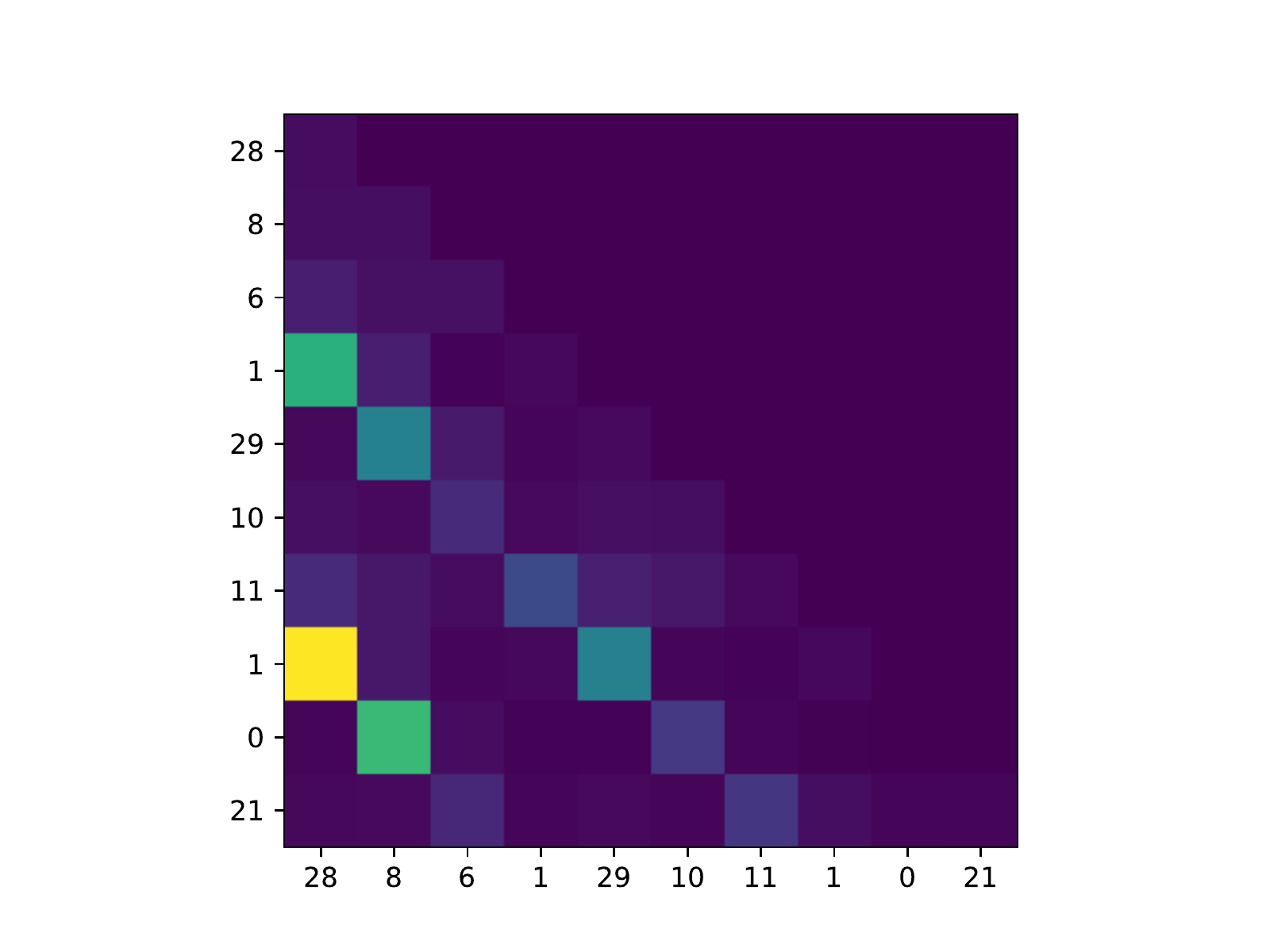}};
		\node (label) at (0.75*14,0)[draw=none, align=left, anchor=center]{ \includegraphics[width=0.15\textwidth]{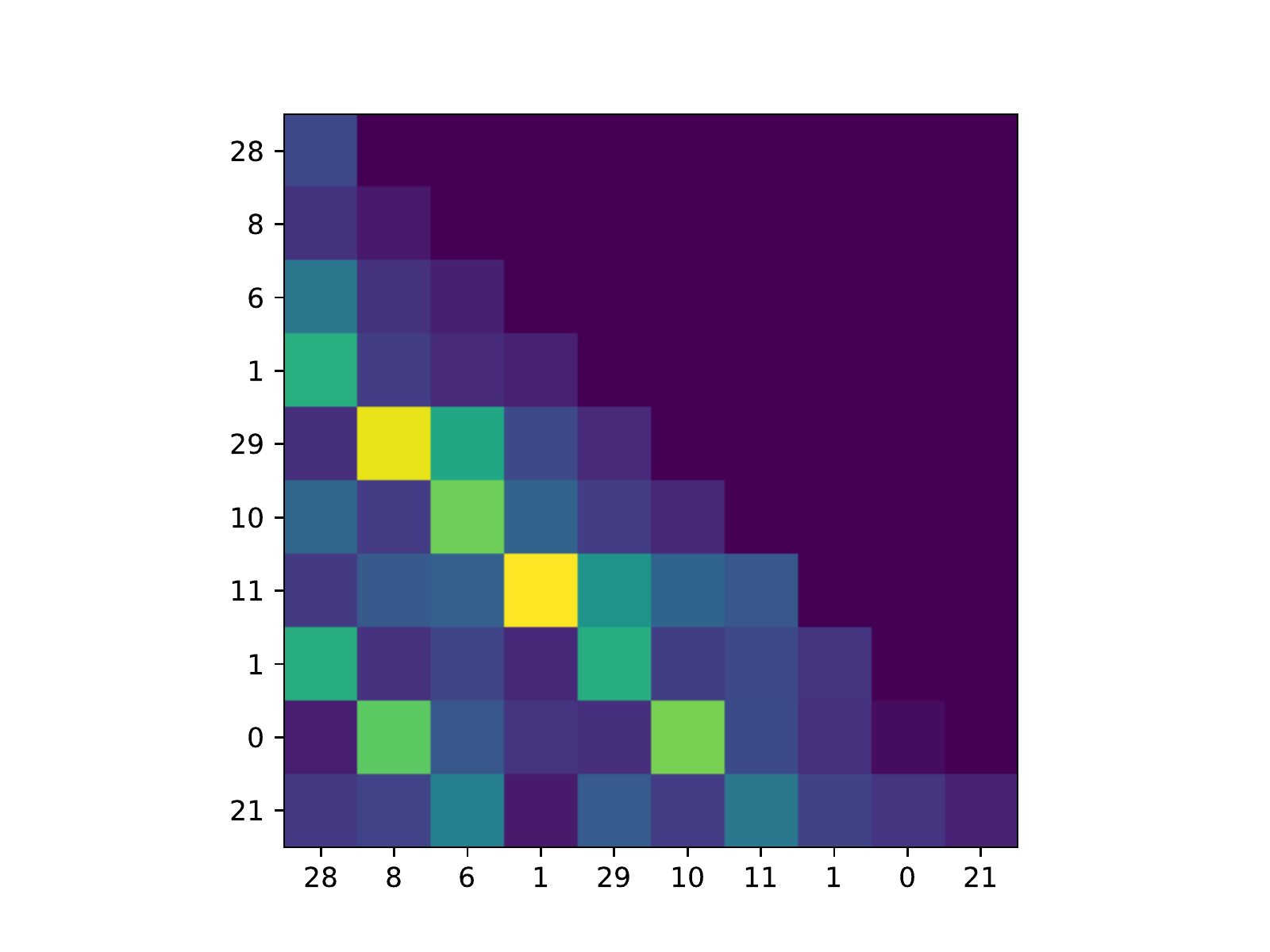}};
		\node (label) at (0.75*16,2)[draw=none, align=left, anchor=center]{ \includegraphics[width=0.15\textwidth]{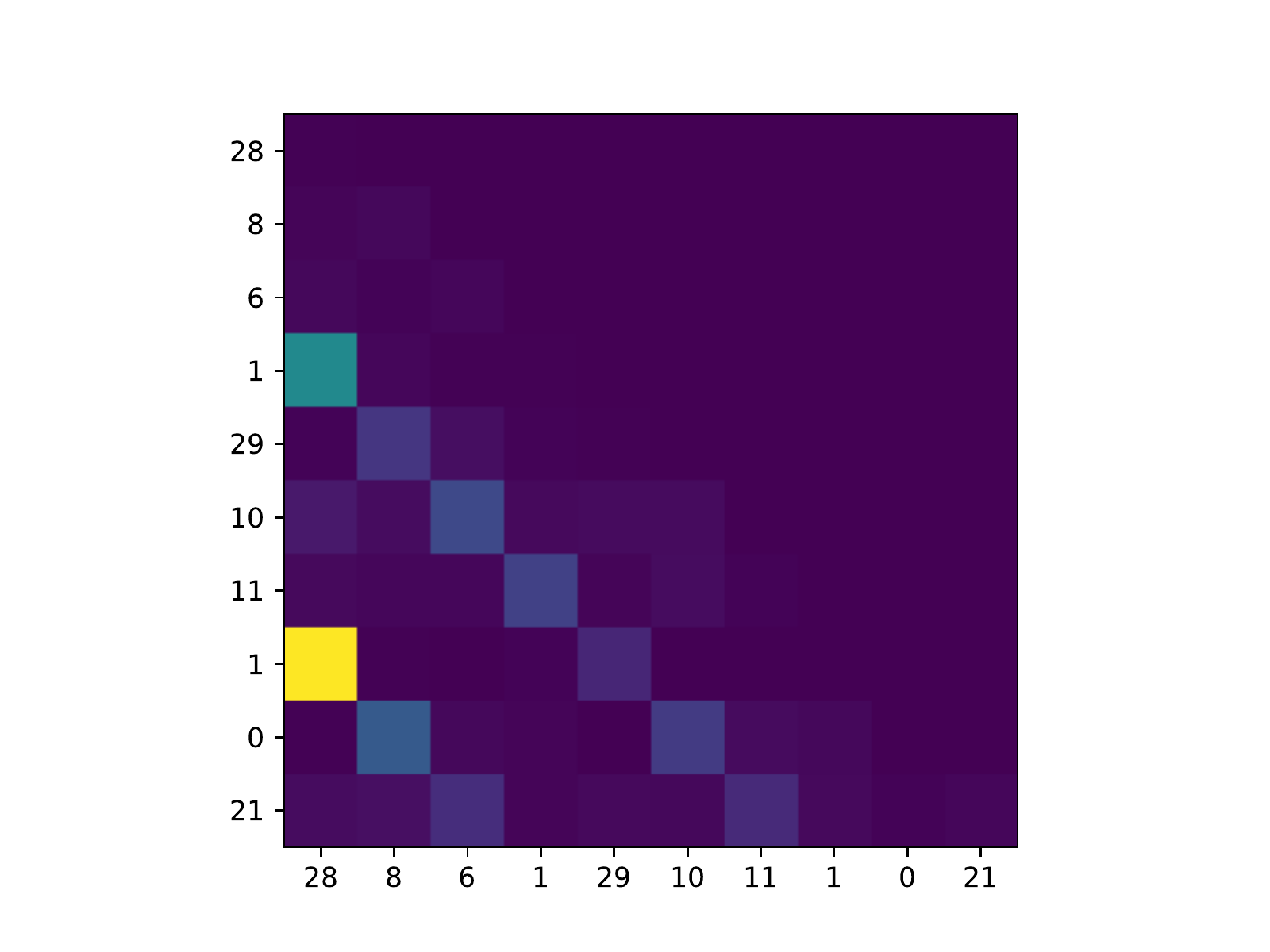}};
		\node (label) at (0.75*16,0)[draw=none, align=left, anchor=center]{ \includegraphics[width=0.15\textwidth]{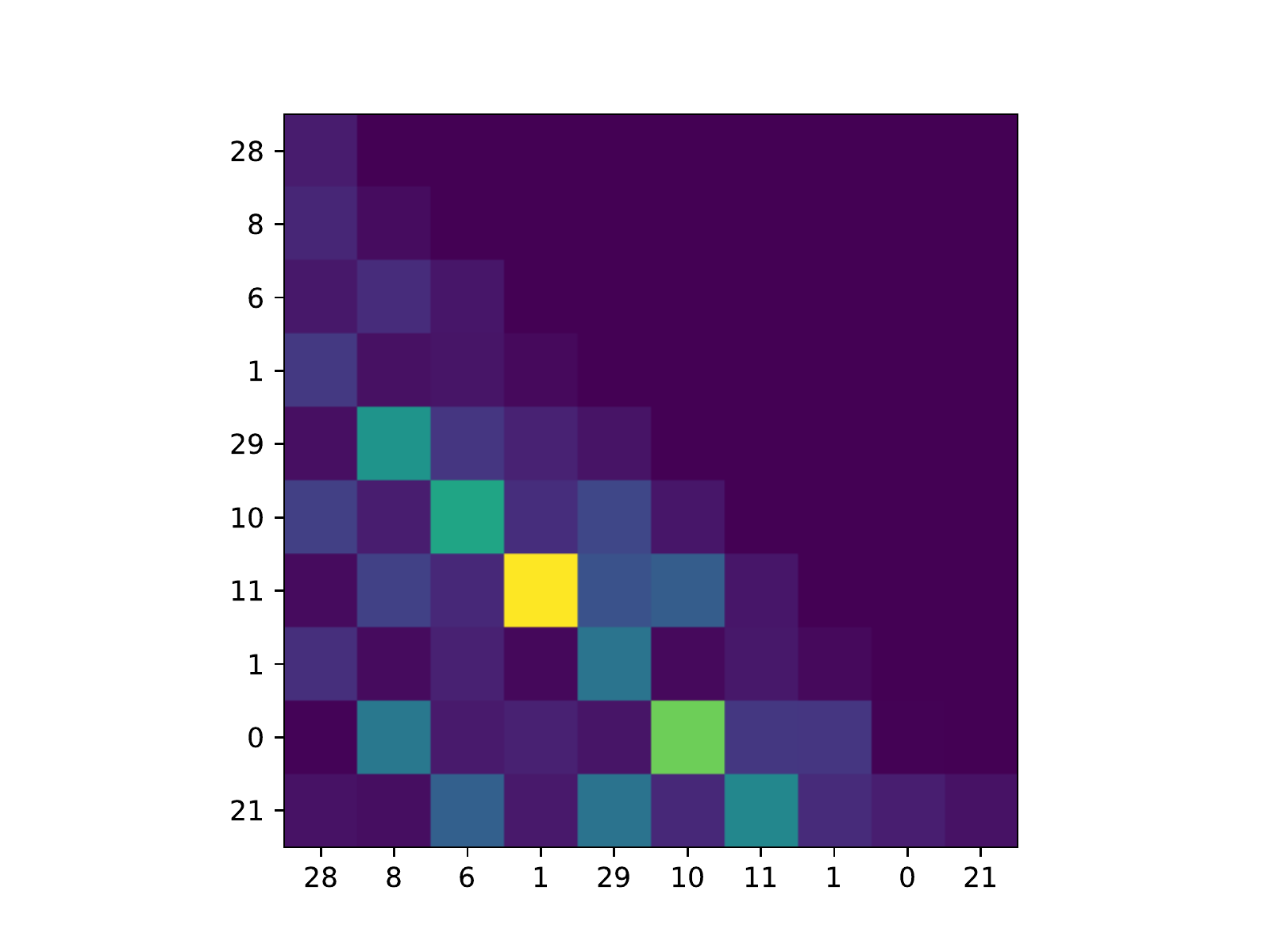}};
		\node (label) at (0.75*18,2)[draw=none, align=left, anchor=center]{ \includegraphics[width=0.15\textwidth]{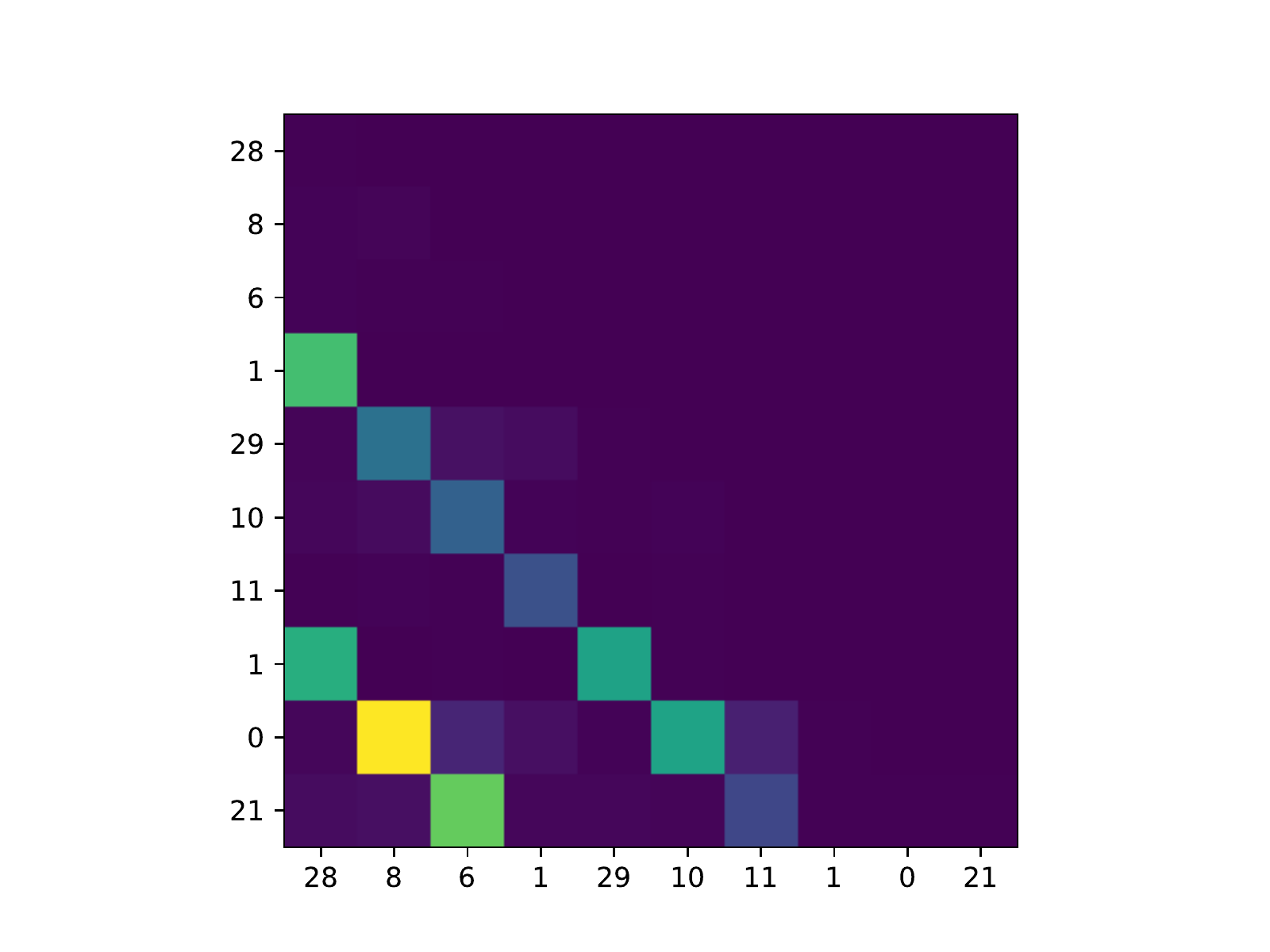}};
		\node (label) at (0.75*18,0)[draw=none, align=left, anchor=center]{ \includegraphics[width=0.15\textwidth]{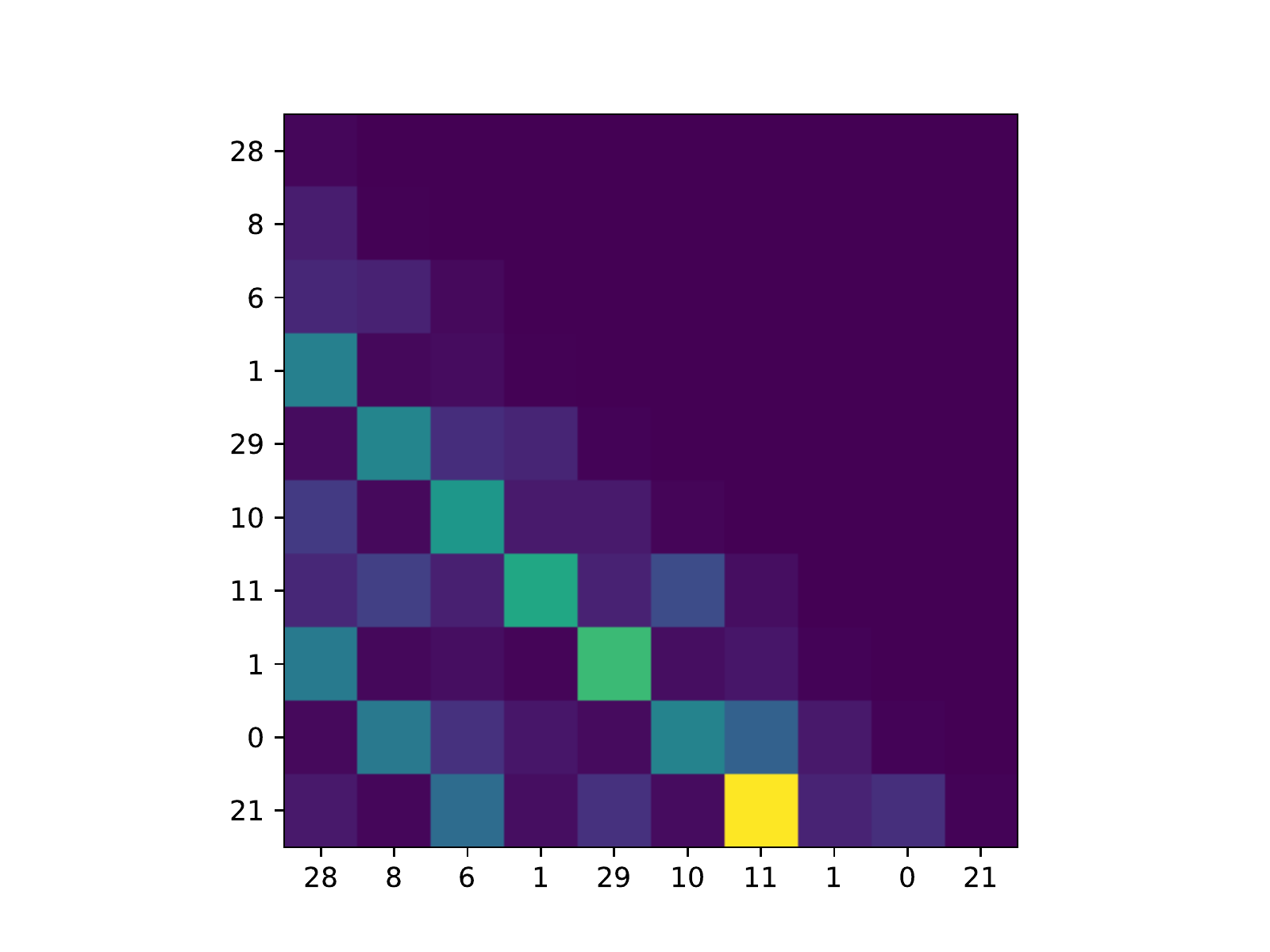}};
	\end{tikzpicture}

	\caption{Attention Maps: chain-of-thought prompting, 21M parameter model, for the two attention heads in the top layer, by the amount of pretraining measured in tokens.
	The rightmost facet corresponds to Figure~\ref{eq:attention-correlation}A.
	A periodic attention pattern becomes visible at around 64M training tokens, preceding the rapid emergence of high accuracy between 128M and 256M tokens.}
\end{figure*}

\begin{figure*}
\centering
        \begin{tikzpicture}

		\node (label) at (0,0)[draw=none, align=left, anchor=west]{ \includegraphics[width=0.6\textwidth]{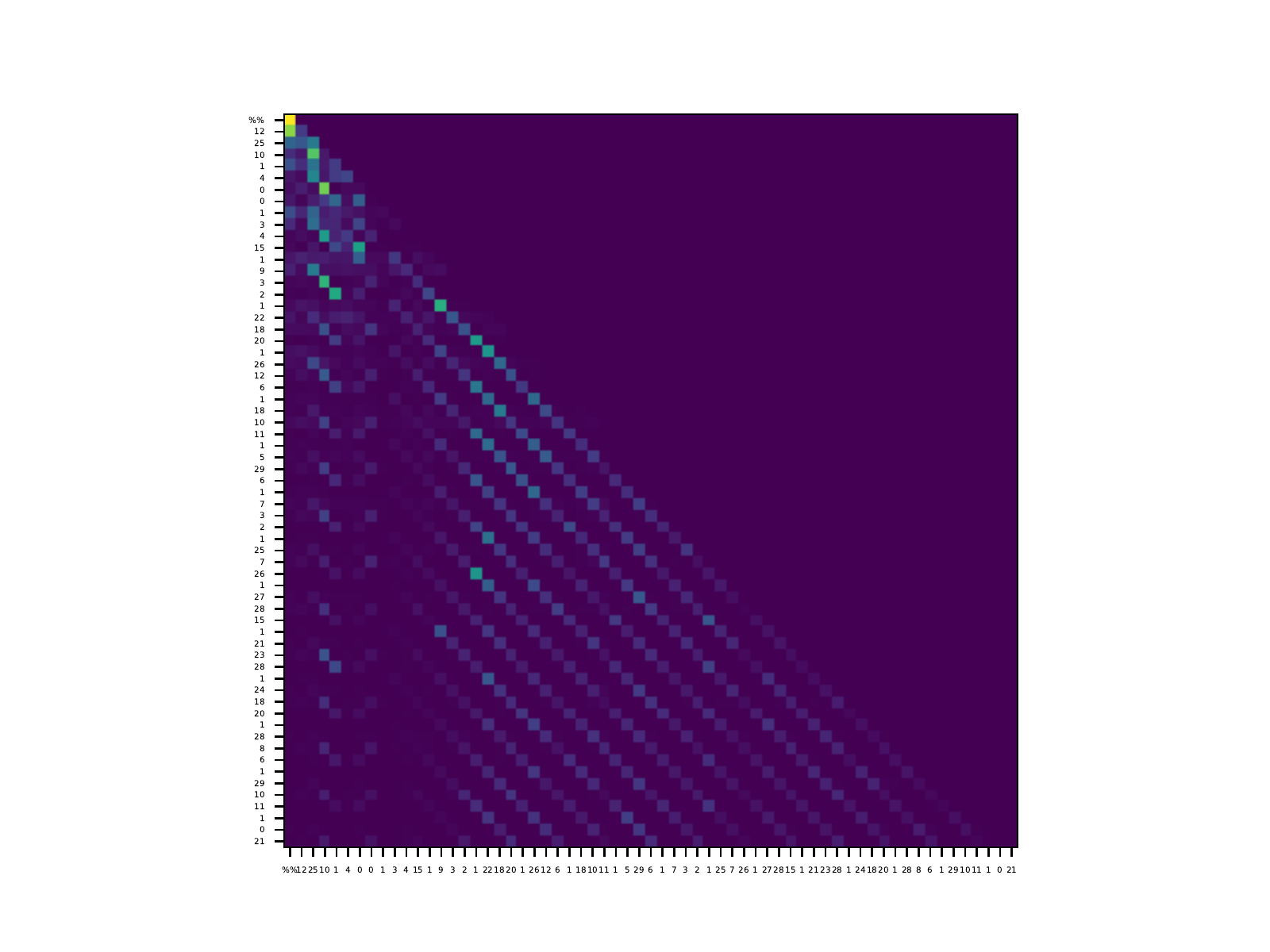}};
		\node (label) at (8,0)[draw=none, align=left, anchor=west]{ \includegraphics[width=0.6\textwidth]{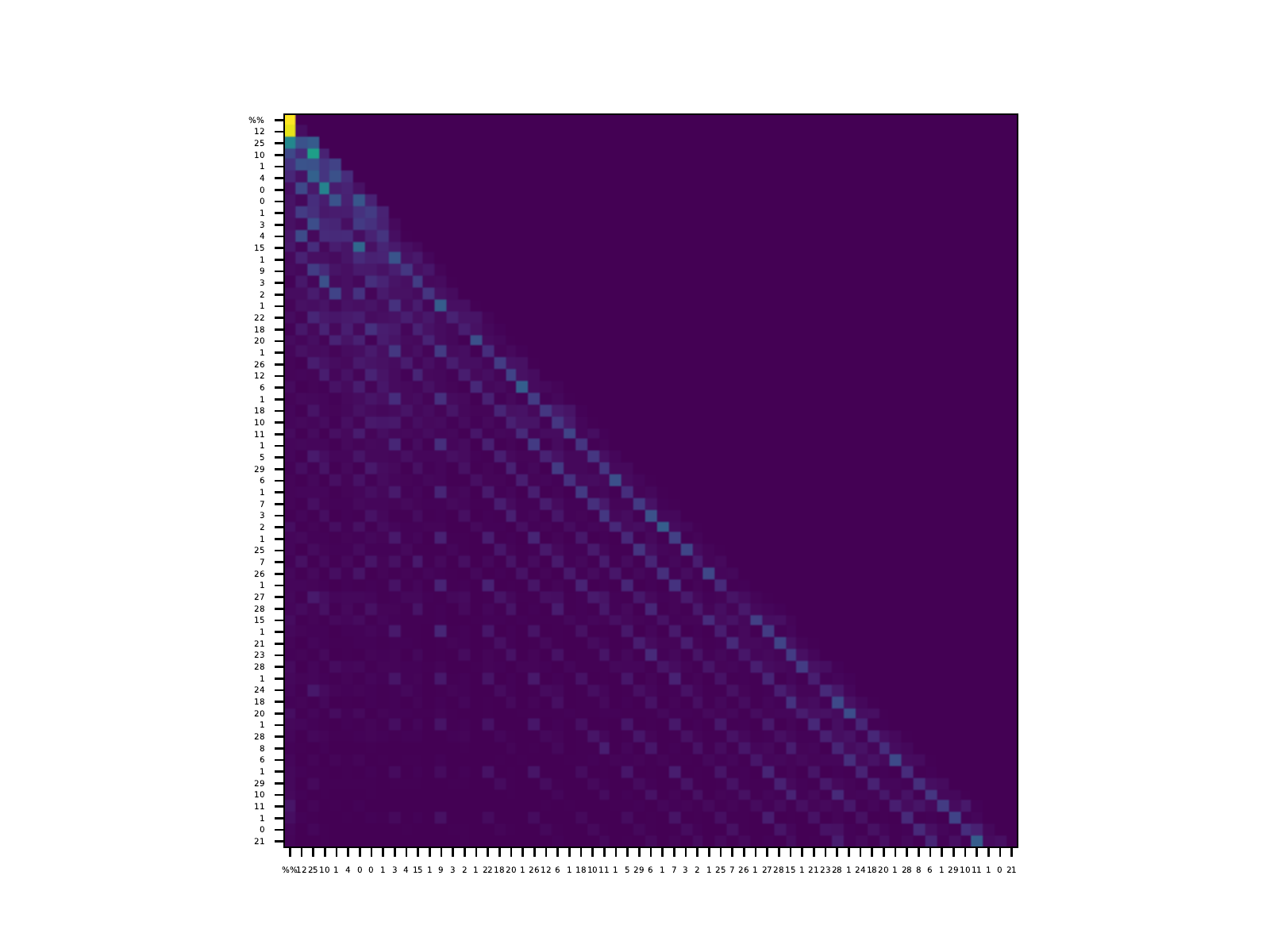}};
		\node (label) at (0,-7)[draw=none, align=left, anchor=west]{ \includegraphics[width=0.6\textwidth]{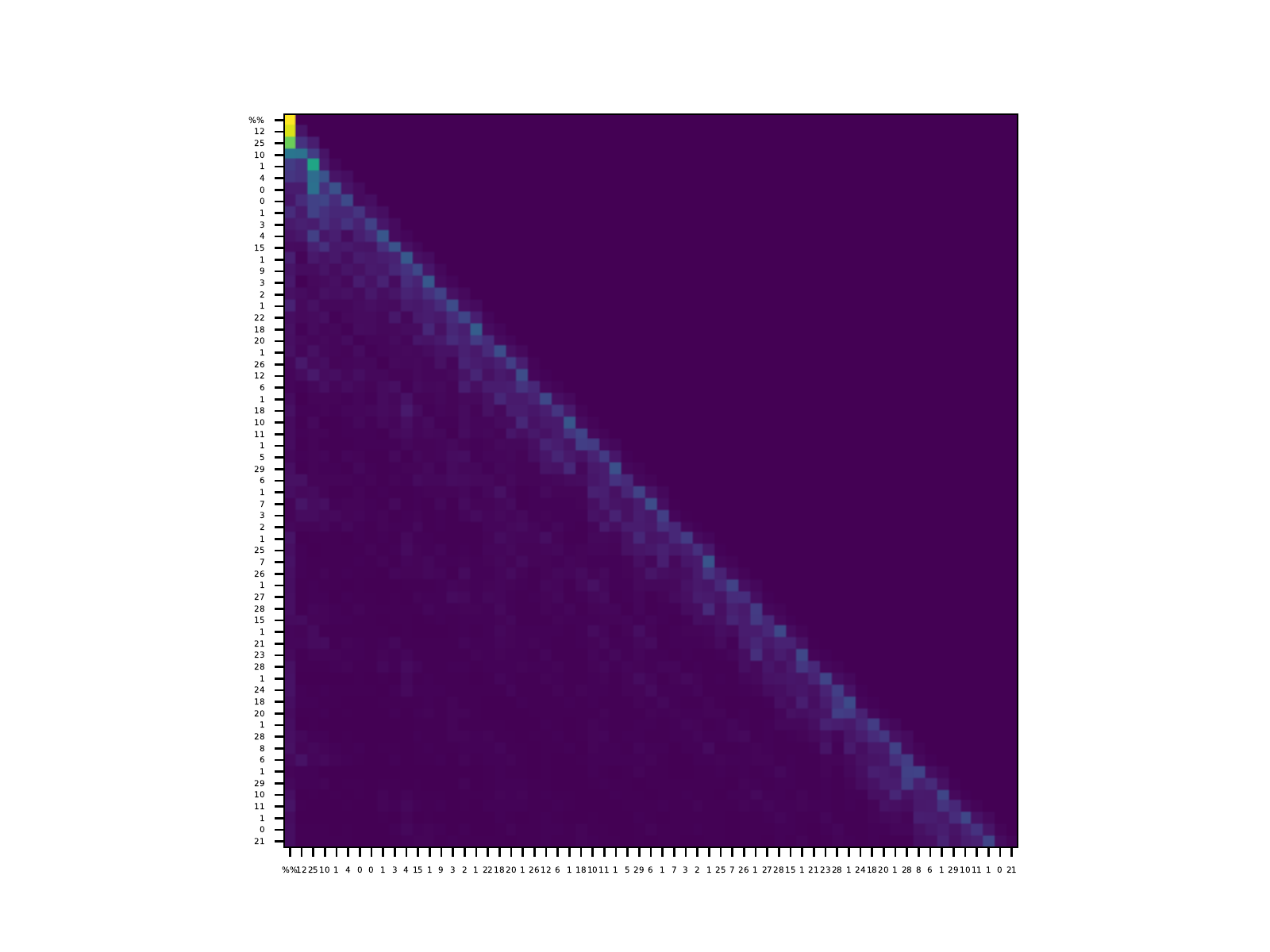}};
		\node (label) at (8,-7)[draw=none, align=left, anchor=west]{ \includegraphics[width=0.6\textwidth]{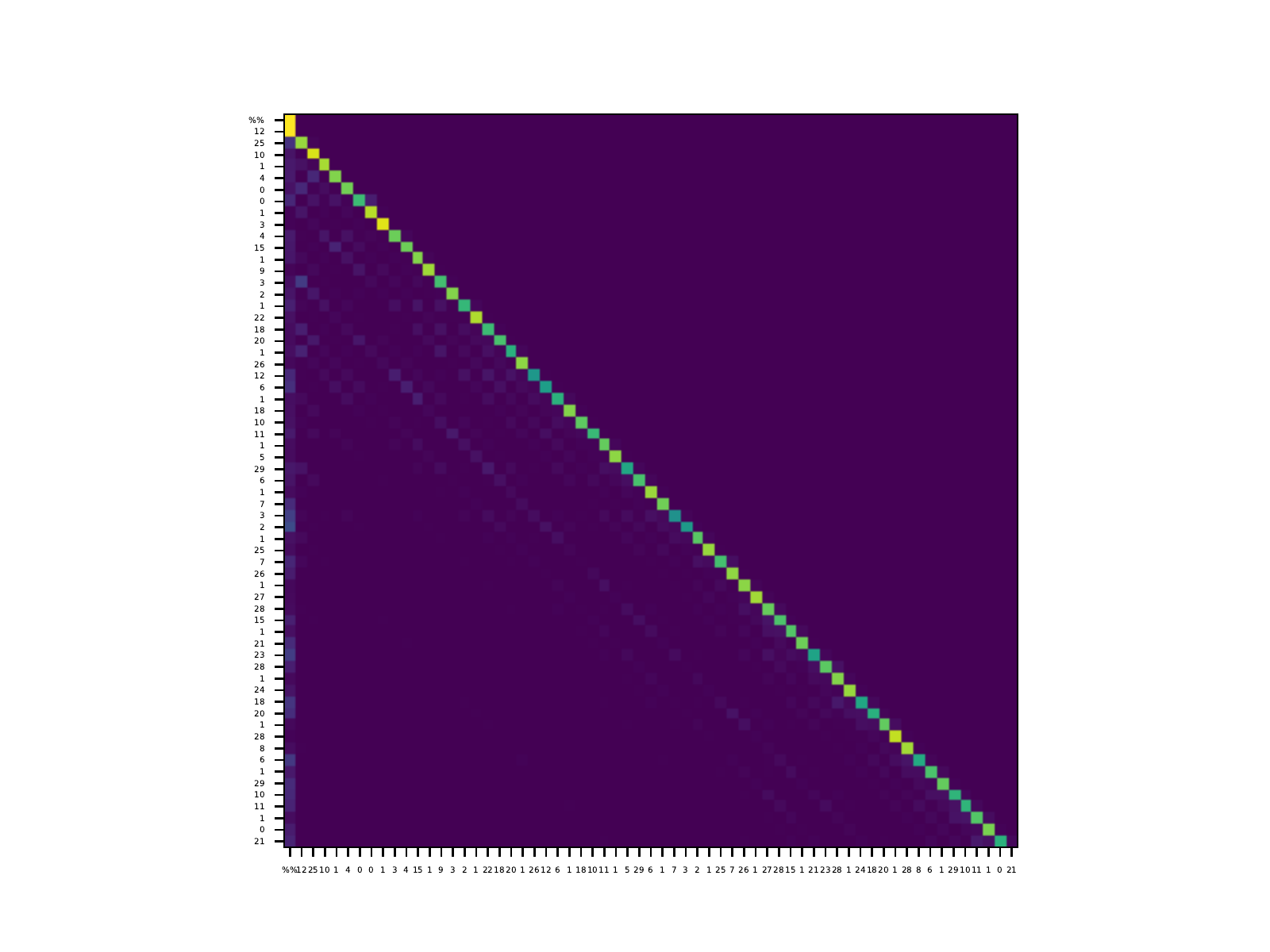}};
		\node (label) at (0,-14)[draw=none, align=left, anchor=west]{ \includegraphics[width=0.6\textwidth]{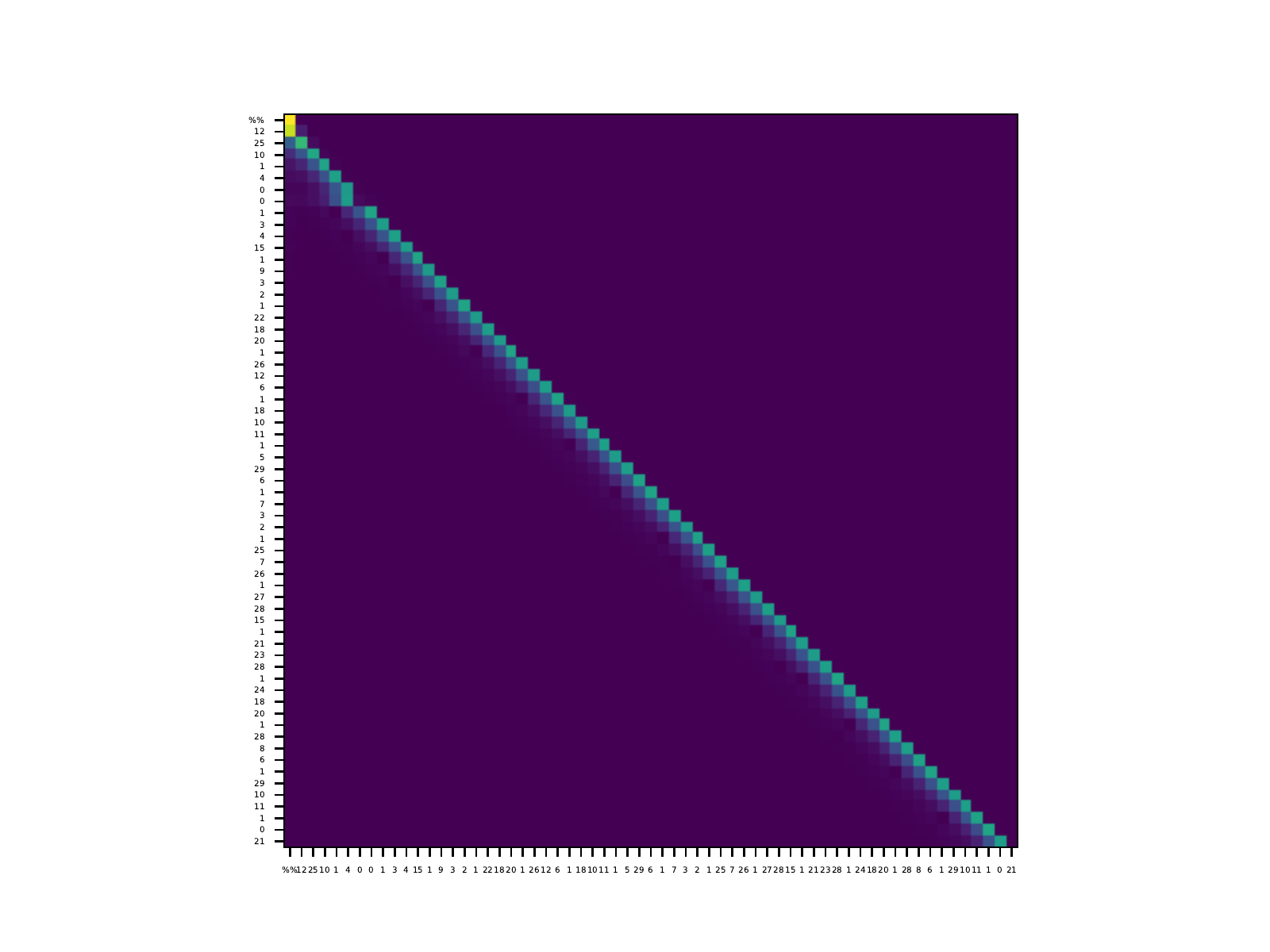}};
		\node (label) at (8,-14)[draw=none, align=left, anchor=west]{ \includegraphics[width=0.6\textwidth]{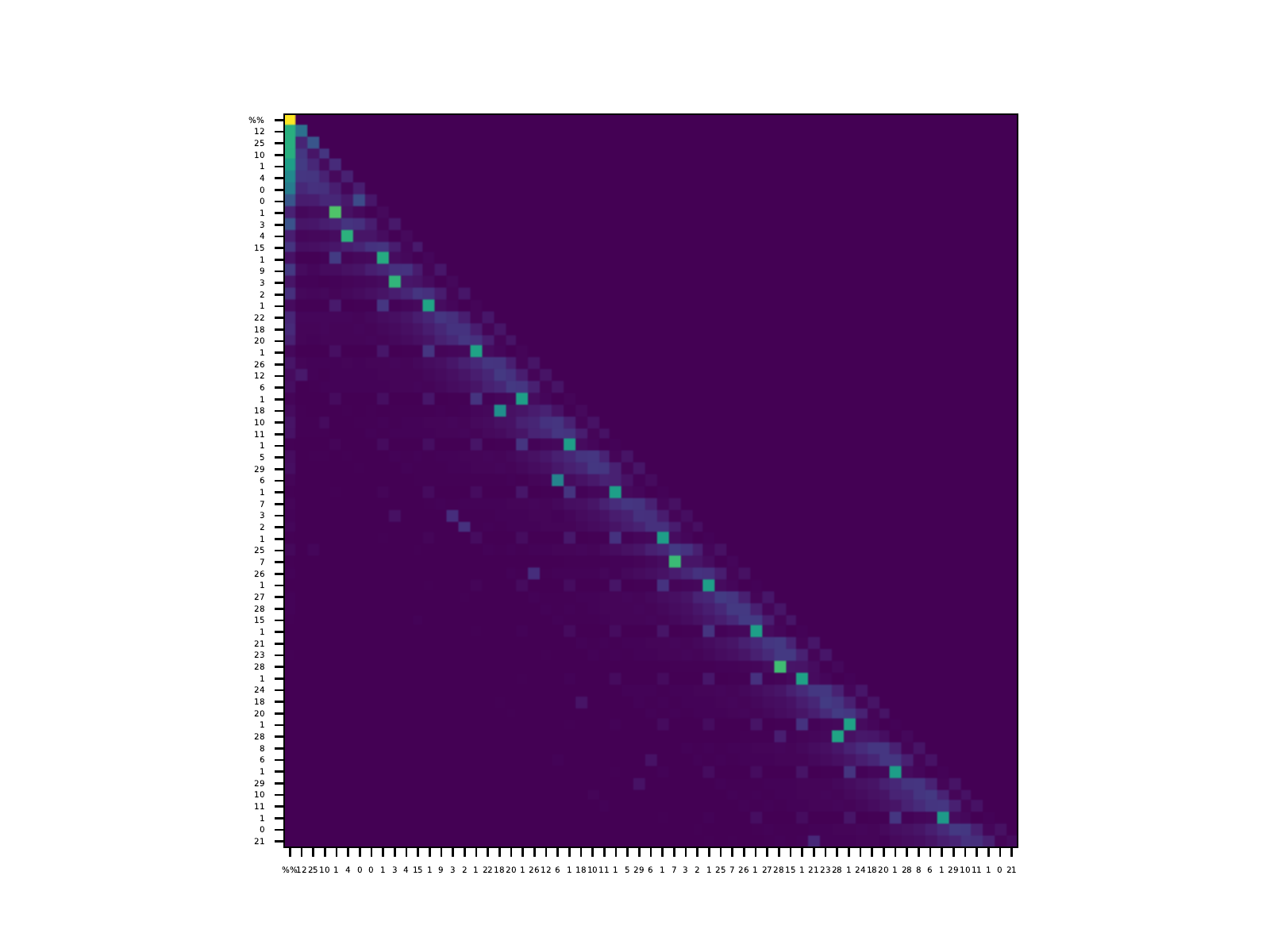}};

	\end{tikzpicture}

	\caption{Full attention maps: chain-of-thought prompting, 21M parameter model. Columns: heads. Rows: layers.
The bottom right corner is shown in Figure~\ref{eq:attention-correlation}A.
}
\end{figure*}

\begin{figure*}
\centering
        \begin{tikzpicture}

		\node (label) at (0,0)[draw=none, align=left, anchor=west]{ \includegraphics[width=0.6\textwidth]{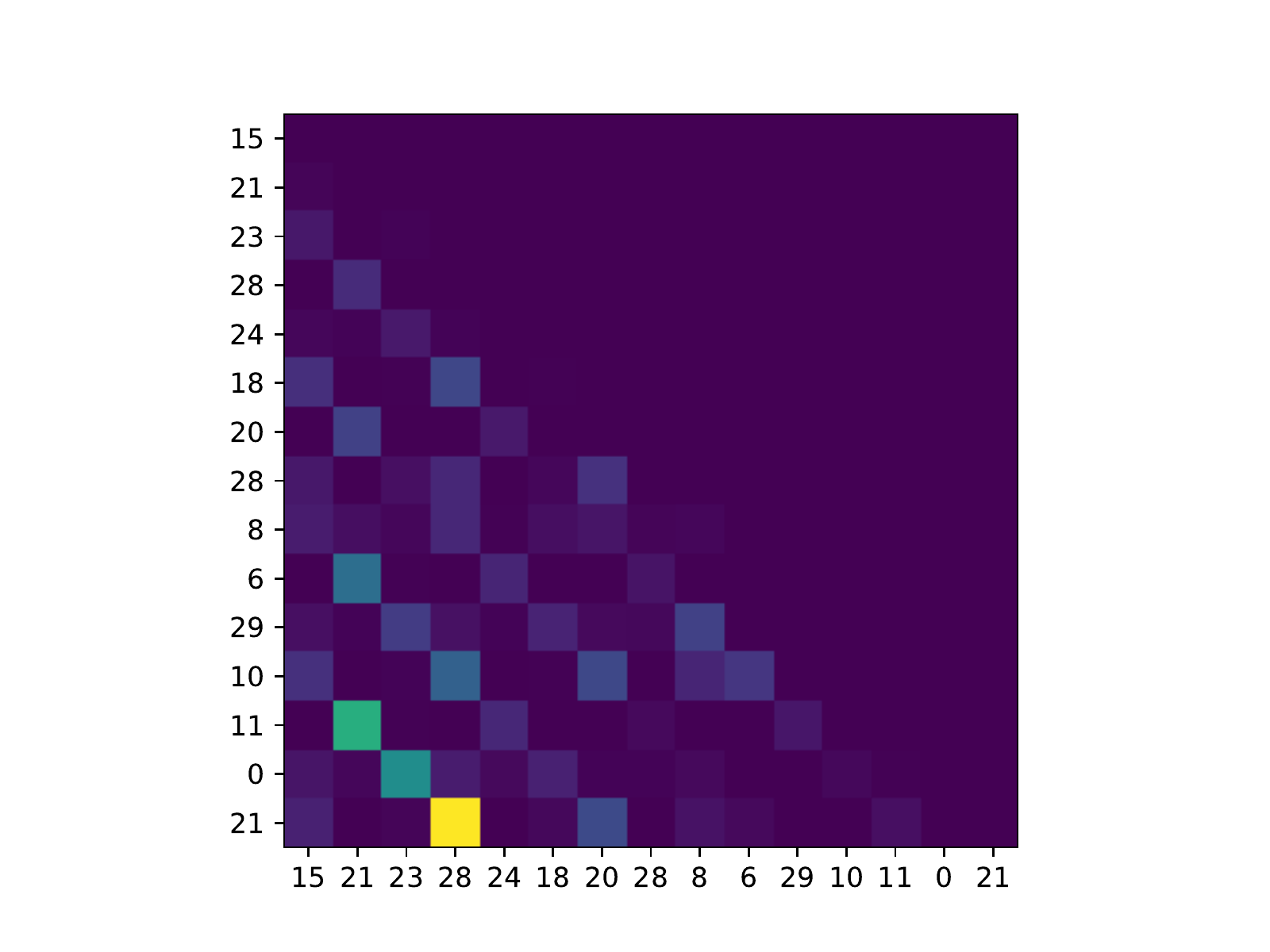}};
 \node (label) at (8,0)[draw=none, align=left, anchor=west]{\includegraphics[width=0.6\textwidth]{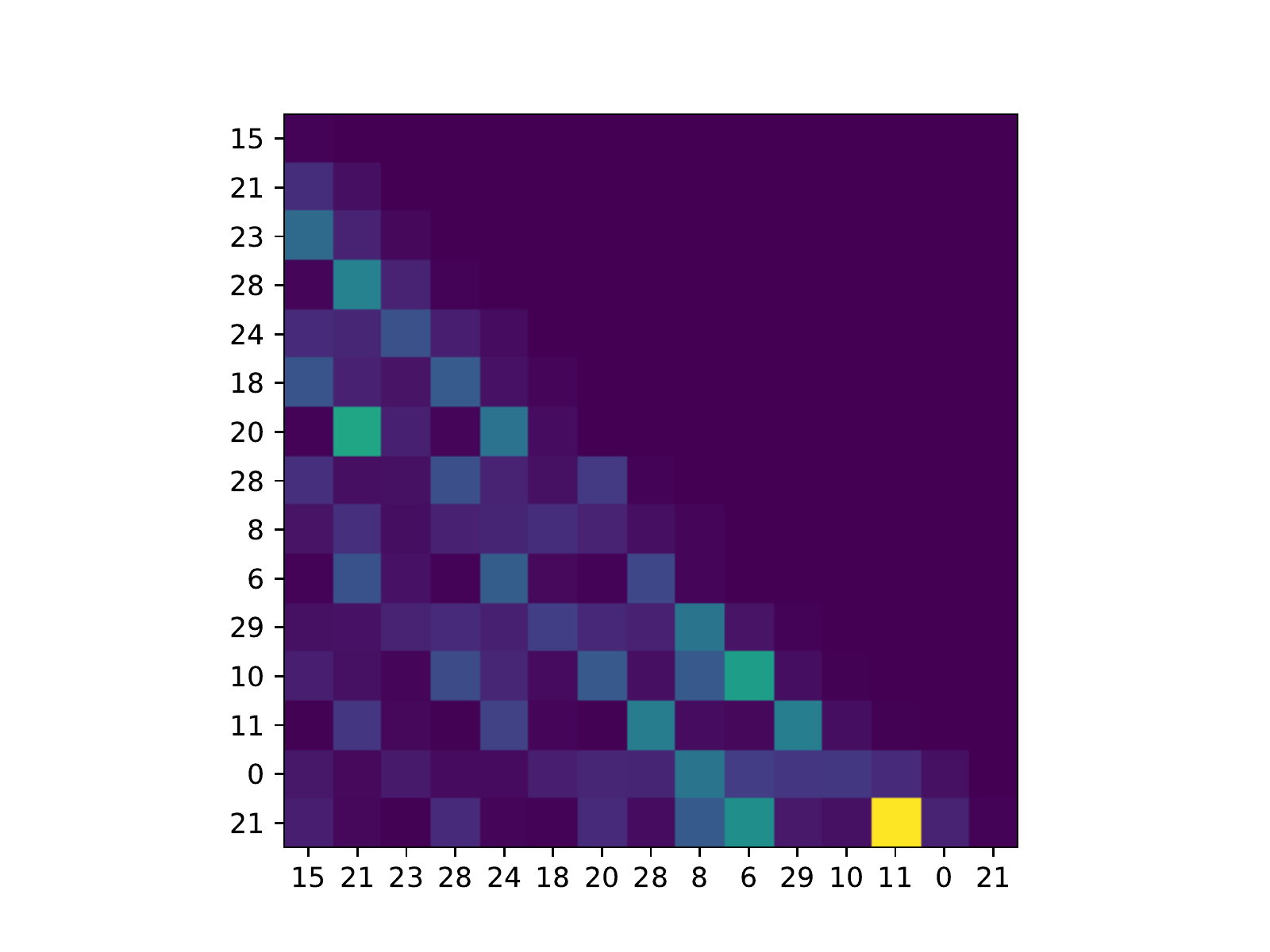}};
		\node (label) at (0,-7)[draw=none, align=left, anchor=west]{ \includegraphics[width=0.6\textwidth]{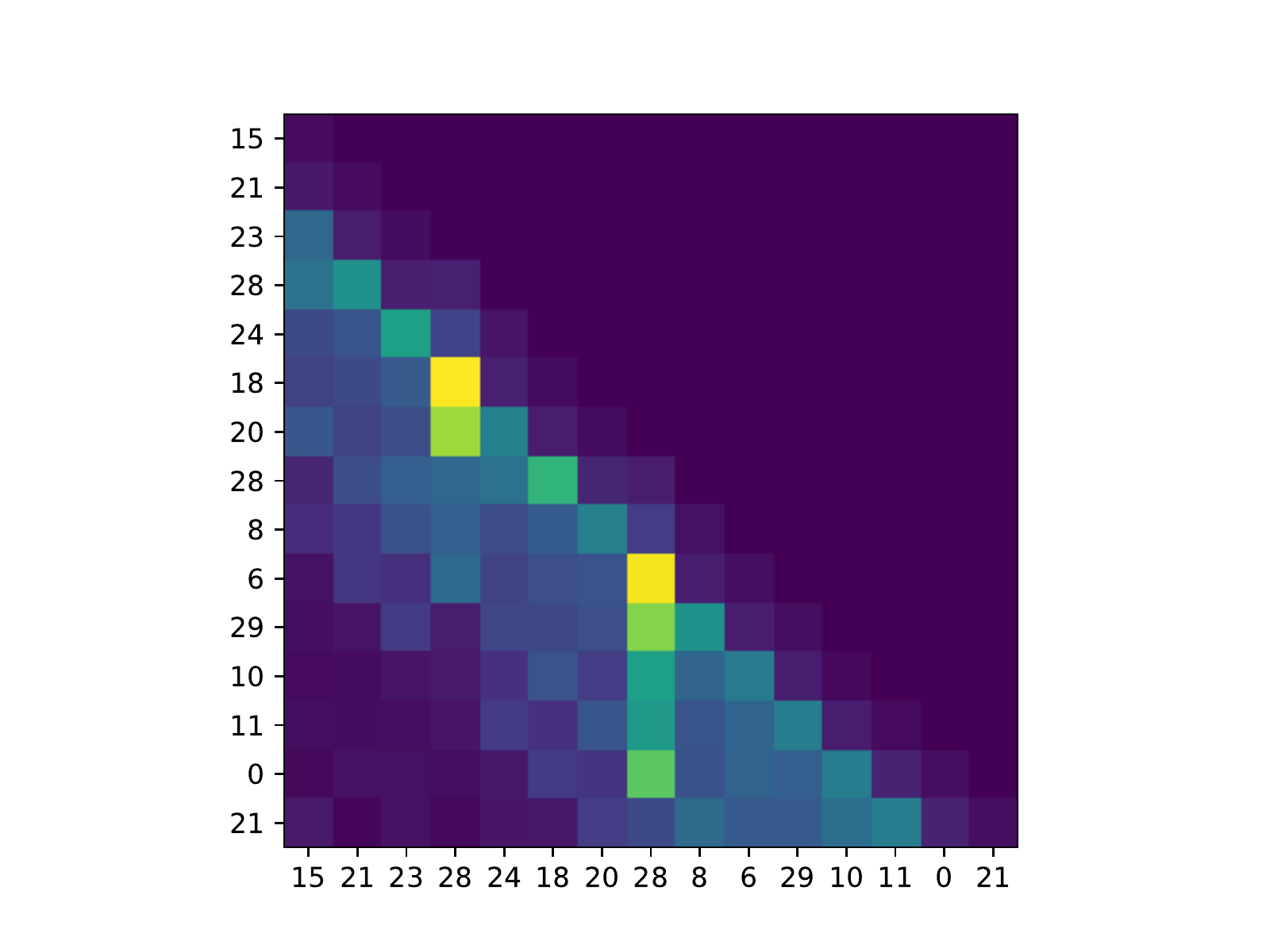}};
		\node (label) at (8,-7)[draw=none, align=left, anchor=west]{ \includegraphics[width=0.6\textwidth]{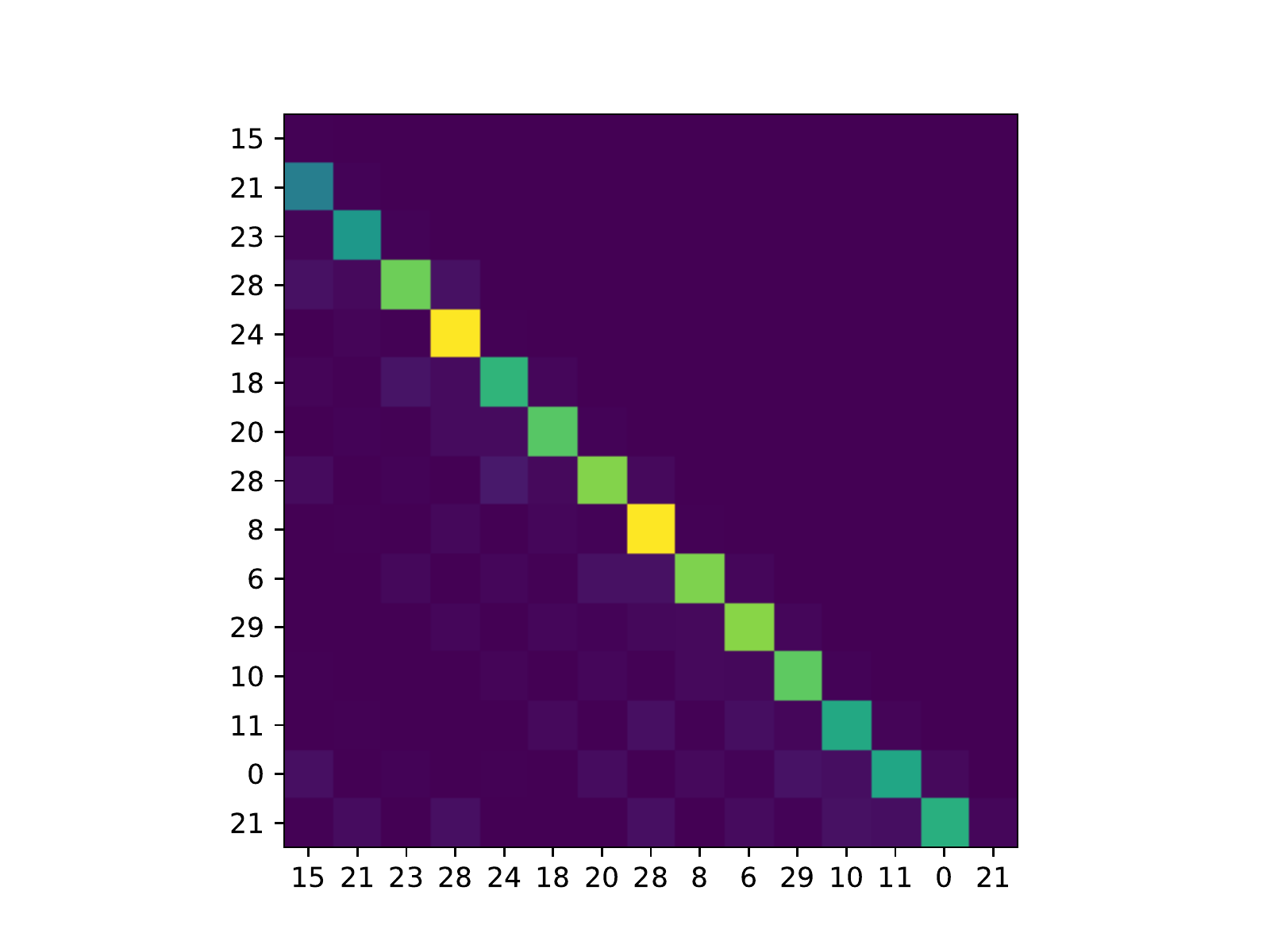}};
		\node (label) at (0,-14)[draw=none, align=left, anchor=west]{ \includegraphics[width=0.6\textwidth]{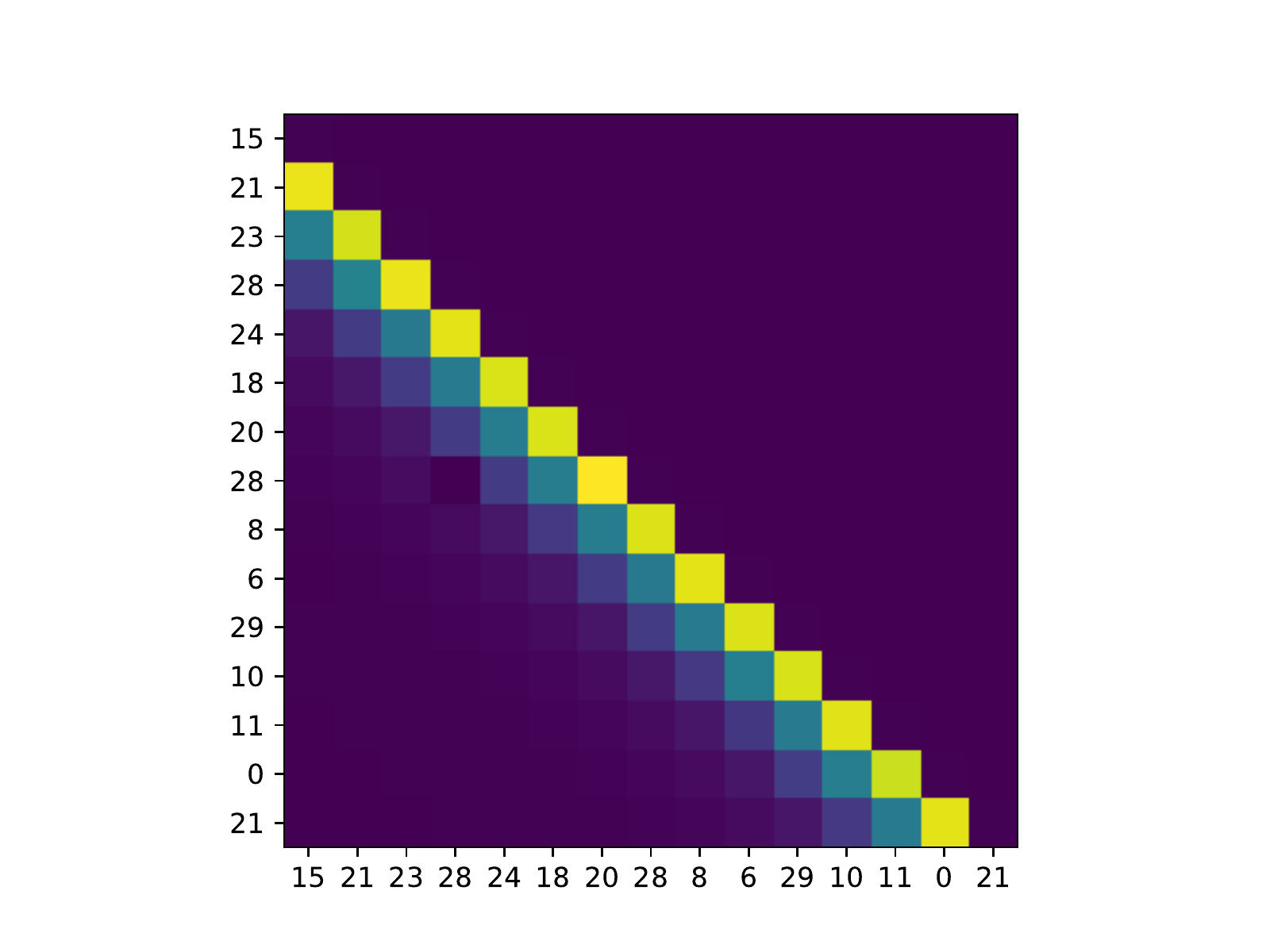}};
		\node (label) at (8,-14)[draw=none, align=left, anchor=west]{ \includegraphics[width=0.6\textwidth]{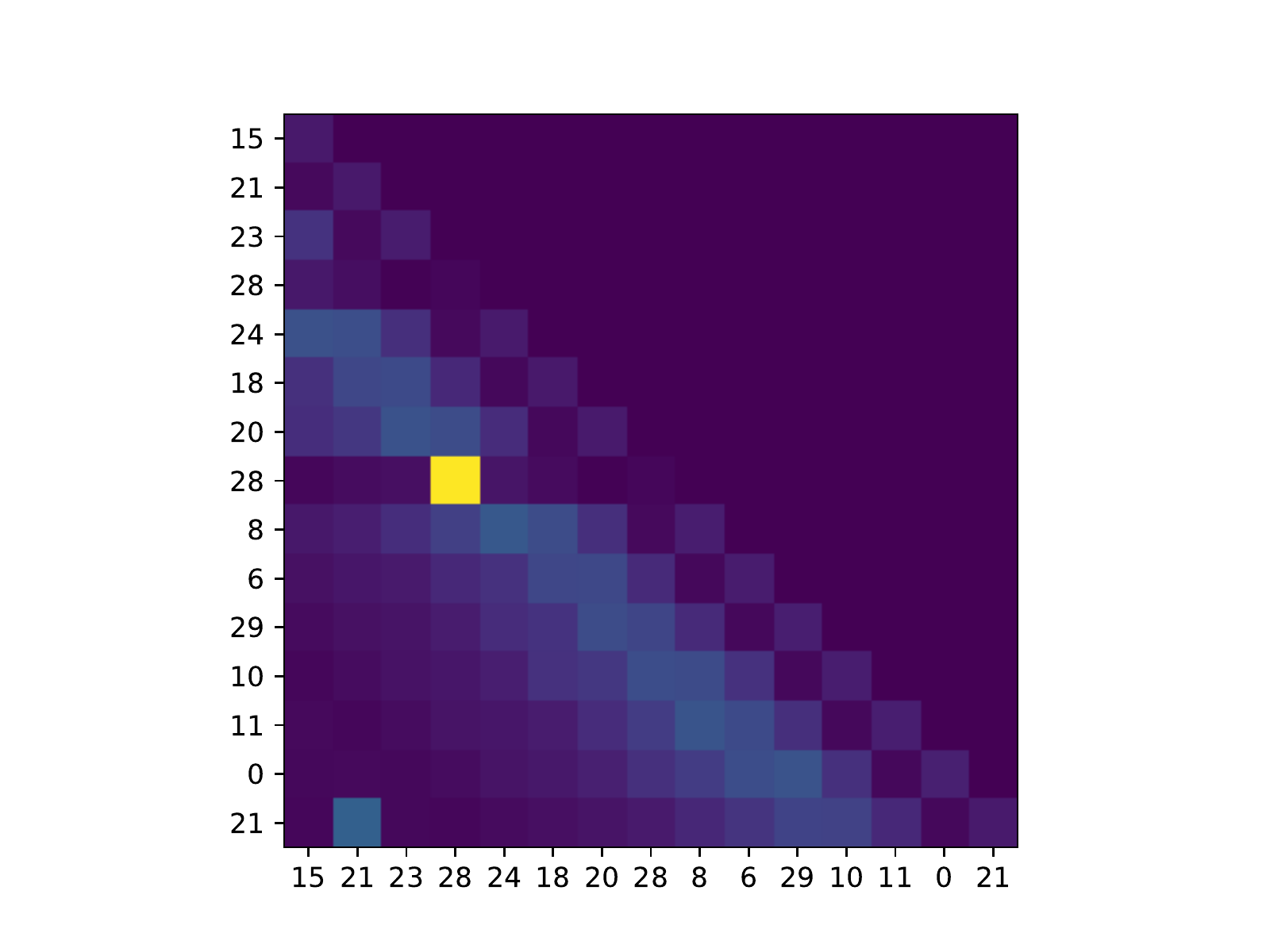}};

	\end{tikzpicture}

	\caption{Attention maps (for the last 15 tokens): chain-of-thought prompting, but with the separator removed. 21M parameter model. Columns: heads. Rows: layers.
Corresponding version with separator is shown in Figure~\ref{eq:attention-correlation}A.
}\label{fig:att-map-cot-no-sep}
\end{figure*}

\tableofcontents

\section{Task Definitions}\label{sec:task-formulas}

\paragraph{Propositional}

1 literal:
\begin{enumerate}
\item $f_i(z)=x$ (\textsc{Inverse})
\end{enumerate}

4 literals:

\begin{enumerate}[label=\arabic*.,start=2]
\item $(f_i(x) = z \wedge f_j(z) = y) \vee (f_i(y) = z \wedge f_j(z) = x)$ (\textsc{MissingLink})

\item $(f_i(x) = z \wedge f_i(z) = y) \vee (f_j(z) = x \wedge f_j(z) = y)$

\item $(f_i(x) = z \wedge f_j(z) = y) \vee (f_i(x) = y \wedge f_j(x) = z)$

\item $(f_i(x) = z \wedge f_i(z) = y) \vee (f_j(x) = y \wedge f_j(z) = x)$

\item $(f_i(x) = z  \wedge  (f_j(z) = y)) \vee    (f_i(z) = x  \wedge  (f_j(y) = z))$

\item $(f_i(x) = z  \wedge  (f_j(z) = y)) \vee   (f_i(z) = x  \wedge  (f_k(z) = y))$

\end{enumerate}

8 literals:

\begin{enumerate}[label=\arabic*.,start=8]
\item $(f_i(x) = z  \wedge  (f_j(y) = x)) \vee  (f_i(z) = x  \wedge  (f_j(x) = y)) \vee   (f_j(y) = z  \wedge  (f_k(y) = x)) \vee   (f_j(z) = y  \wedge  (f_k(x) = y))$

\item $(f_k(x) = z  \wedge  (f_j(y) = x)) \vee   (f_i(x) = z  \wedge  (f_k(x) = y)) \vee   (f_i(z) = x  \wedge  (f_k(y) = x)) \vee      (f_j(z) = y  \wedge  (f_i(y) = x))$

\item $(f_i(x) = z  \wedge  (f_j(y) = x)) \vee    (f_i(z) = x  \wedge  (f_j(x) = y)) \vee    (f_i(y) = z  \wedge  (f_k(y) = x)) \vee   (f_i(z) = y  \wedge  (f_k(x) = y))$

\end{enumerate}

\paragraph{Composed}

\begin{enumerate}[label=\arabic*.,start=11]
\item $\exists a : a=f_i(x) \wedge z=f_j(a)$
\item $\exists a : (f_i(x) = z  \wedge  f_j(a) = y \wedge f_k(z) = a)) \vee  (f_i(z) = x  \wedge  f_j(a) = z \wedge f_k(y) = a))$

This is an extension of \textsc{MissingLink} (\#2) where a chain of \emph{three} (rather than \emph{two}) functions connects $x$ and $y$; one of the two intervening objects is the target output ($z$); the other one ($a$) is unobserved.

\end{enumerate}

\paragraph{Binary Classification}
\begin{enumerate}[label=\arabic*.,start=13]
\item $(y=f_i(x) \wedge z=\ell_i) \vee (y=f_j(x) \wedge z=\ell_j)$
\item $(y=f_i(x) \wedge z=\ell_i) \vee (x=f_i(y) \wedge z=\ell_j)$
\item $(y=f_i(x) \wedge z=\ell_i) \vee (\exists a : (x=f_i(a) \wedge a = f_i(y)) \wedge z=\ell_j))$
\end{enumerate}

\section{Example Documents}\label{sec:sample-docs}

Here we provide some example document scripts.
For each example, we assign a color to each print statement, and color tokens in the document by what statement produced it.

\subsection{Example 1}

\colorlet{linecolor1}{red}
\colorlet{linecolor2}{green}
\colorlet{linecolor3}{blue}
\colorlet{linecolor4}{orange}
\colorlet{linecolor5}{purple}
\colorlet{linecolor6}{Emerald}
\colorlet{linecolor7}{yellow}
\colorlet{linecolor8}{pink}
\colorlet{linecolor9}{cyan}
\colorlet{linecolor10}{teal}
\colorlet{linecolor11}{olive}
\colorlet{linecolor12}{lime}
\colorlet{linecolor13}{lightgray}
\colorlet{linecolor14}{brown}
\colorlet{linecolor15}{magenta}
\colorlet{linecolor16}{gray}

\begin{Verbatim}[commandchars=\\\{\}]
LOOP OVER x0 DO
 FOR SOME x1 SUCH THAT f1(x0) = x1 DO
  \textcolor{linecolor1}{PRINT x1}
 END
 FOR SOME x1 SUCH THAT f0(x1) = x0 DO
  IF f5(x1) = x0 THEN
   \textcolor{linecolor2}{PRINT x0}
   \textcolor{linecolor3}{PRINT x1}
  ELSE
   IF f3(x1) = x0 THEN
    \textcolor{linecolor4}{PRINT x1}
   ELSE
    \textcolor{linecolor5}{PRINT x1}
   ENDIF
  ENDIF
  \textcolor{linecolor6}{PRINT x0}
  \textcolor{linecolor7}{PRINT x0}
 END
ENDFOR
\end{Verbatim}

Sampled Document:

\textcolor{linecolor1}{
20
}
\textcolor{linecolor5}{
1
}
\textcolor{linecolor6}{
1
}
\textcolor{linecolor7}{
1
}
\textcolor{linecolor1}{
14
}
\textcolor{linecolor5}{
25
}
\textcolor{linecolor6}{
25
}
\textcolor{linecolor7}{
25
}
\textcolor{linecolor1}{
16
}
\textcolor{linecolor5}{
16
}
\textcolor{linecolor6}{
16
}
\textcolor{linecolor7}{
16
}
\textcolor{linecolor1}{
8
}
\textcolor{linecolor5}{
12
}
\textcolor{linecolor6}{
12
}
\textcolor{linecolor7}{
12
}
\textcolor{linecolor1}{
19
}
\textcolor{linecolor2}{
29
}
\textcolor{linecolor3}{
29
}
\textcolor{linecolor6}{
29
}
\textcolor{linecolor7}{
29
}
\textcolor{linecolor1}{
25
}
\textcolor{linecolor5}{
28
}
\textcolor{linecolor6}{
28
}
\textcolor{linecolor7}{
28
}
\textcolor{linecolor1}{
20
}
\textcolor{linecolor5}{
27
}
\textcolor{linecolor6}{
27
}
\textcolor{linecolor7}{
27
}
\textcolor{linecolor1}{
6
}
\textcolor{linecolor5}{
8
}
\textcolor{linecolor6}{
8
}
\textcolor{linecolor7}{
8
}
\textcolor{linecolor1}{
7
}
\textcolor{linecolor5}{
14
}
\textcolor{linecolor6}{
14
}
\textcolor{linecolor7}{
14
}
\textcolor{linecolor1}{
28
}
\textcolor{linecolor5}{
17
}
\textcolor{linecolor6}{
17
}
\textcolor{linecolor7}{
17
}
\textcolor{linecolor1}{
17
}
\textcolor{linecolor5}{
18
}
\textcolor{linecolor6}{
18
}
\textcolor{linecolor7}{
18
}

\subsection{Example 2}

\begin{Verbatim}[commandchars=\\\{\}]
LOOP OVER x0 DO
 FOR SOME x1 SUCH THAT f0(x0) = x1 DO
  FOR SOME x2 SUCH THAT f0(x0) = x2 DO
   FOR SOME x3 SUCH THAT f6(x3) = x1 DO
    \textcolor{linecolor1}{PRINT x3}
   END
   IF f4(x2) = x0 THEN
    \textcolor{linecolor2}{PRINT x1}
   ELSE
    IF f9(x1) = x2 THEN
     \textcolor{linecolor3}{PRINT x1}
    ELSE
     \textcolor{linecolor4}{PRINT x0}
    ENDIF
   ENDIF
  END
  IF f8(x1) = x0 THEN
   \textcolor{linecolor5}{PRINT x1}
   \textcolor{linecolor6}{PRINT x1}
   \textcolor{linecolor7}{PRINT x1}
   \textcolor{linecolor8}{PRINT x1}
   \textcolor{linecolor9}{PRINT x1}
   \textcolor{linecolor10}{PRINT x1}
  ELSE
   \textcolor{linecolor11}{PRINT x0}
   \textcolor{linecolor12}{PRINT x1}
  ENDIF
 END
ENDFOR
\end{Verbatim}

Document:

\textcolor{linecolor1}{
19
}
\textcolor{linecolor4}{
1
}
\textcolor{linecolor11}{
1
}
\textcolor{linecolor12}{
1
}
\textcolor{linecolor1}{
24
}
\textcolor{linecolor4}{
4
}
\textcolor{linecolor11}{
4
}
\textcolor{linecolor12}{
4
}
\textcolor{linecolor4}{
27
}
\textcolor{linecolor11}{
27
}
\textcolor{linecolor12}{
27
}
\textcolor{linecolor2}{
15
}
\textcolor{linecolor11}{
15
}
\textcolor{linecolor12}{
15
}
\textcolor{linecolor1}{
2
}
\textcolor{linecolor4}{
10
}
\textcolor{linecolor11}{
10
}
\textcolor{linecolor12}{
10
}
\textcolor{linecolor1}{
8
}
\textcolor{linecolor4}{
21
}
\textcolor{linecolor11}{
21
}
\textcolor{linecolor12}{
21
}
\textcolor{linecolor1}{
18
}
\textcolor{linecolor4}{
17
}
\textcolor{linecolor5}{
17
}
\textcolor{linecolor6}{
17
}
\textcolor{linecolor7}{
17
}
\textcolor{linecolor8}{
17
}
\textcolor{linecolor9}{
17
}
\textcolor{linecolor10}{
17
}
\textcolor{linecolor4}{
22
}
\textcolor{linecolor11}{
22
}
\textcolor{linecolor12}{
22
}
\textcolor{linecolor1}{
1
}
\textcolor{linecolor4}{
25
}
\textcolor{linecolor11}{
25
}
\textcolor{linecolor12}{
25
}
\textcolor{linecolor1}{
26
}
\textcolor{linecolor4}{
26
}
\textcolor{linecolor11}{
26
}
\textcolor{linecolor12}{
26
}
\textcolor{linecolor1}{
22
}
\textcolor{linecolor4}{
7
}
\textcolor{linecolor11}{
7
}
\textcolor{linecolor12}{
7
}

\subsection{Example 3}

\begin{Verbatim}[commandchars=\\\{\}]
LOOP OVER x0 DO
 FOR SOME x1 SUCH THAT f5(x0) = x1 DO
  FOR SOME x2 SUCH THAT f9(x2) = x1 DO
   \textcolor{linecolor1}{PRINT x2}
  END
 END
 FOR SOME x1 SUCH THAT f7(x1) = x0 DO
  \textcolor{linecolor2}{PRINT x0}
 END
 FOR SOME x1 SUCH THAT f0(x0) = x1 DO
  IF rel8(x1) = x0 THEN
   \textcolor{linecolor3}{PRINT x1}
  ELSE
   \textcolor{linecolor4}{PRINT x1}
  ENDIF
 END
 FOR SOME x1 SUCH THAT f2(x1) = x0 DO
  \textcolor{linecolor5}{PRINT x1}
 END
 LOOP OVER x1 DO
  \textcolor{linecolor6}{PRINT x0}
 ENDFOR
 FOR SOME x1 SUCH THAT f0(x0) = x1 DO
  \textcolor{linecolor7}{PRINT x1}
  \textcolor{linecolor8}{PRINT x0}
 END
 FOR SOME x1 SUCH THAT f4(x0) = x1 DO
  \textcolor{linecolor9}{PRINT x0}
 END
ENDFOR
\end{Verbatim}

Document:

\textcolor{linecolor2}{
26
}
\textcolor{linecolor4}{
26
}
\textcolor{linecolor5}{
6
}
\textcolor{linecolor6}{
26
}
\textcolor{linecolor6}{
26
}
\textcolor{linecolor6}{
26
}
\textcolor{linecolor6}{
26
}
\textcolor{linecolor6}{
26
}
\textcolor{linecolor6}{
26
}
\textcolor{linecolor6}{
26
}
\textcolor{linecolor6}{
26
}
\textcolor{linecolor6}{
26
}
\textcolor{linecolor6}{
26
}
\textcolor{linecolor6}{
26
}
\textcolor{linecolor7}{
26
}
\textcolor{linecolor8}{
26
}
\textcolor{linecolor9}{
26
}
\textcolor{linecolor1}{
6
}
\textcolor{linecolor2}{
29
}
\textcolor{linecolor4}{
29
}
\textcolor{linecolor6}{
29
}
\textcolor{linecolor6}{
29
}
\textcolor{linecolor6}{
29
}
\textcolor{linecolor6}{
29
}
\textcolor{linecolor6}{
29
}
\textcolor{linecolor6}{
29
}
\textcolor{linecolor6}{
29
}
\textcolor{linecolor6}{
29
}
\textcolor{linecolor6}{
29
}
\textcolor{linecolor6}{
29
}
\textcolor{linecolor6}{
29
}
\textcolor{linecolor7}{
29
}
\textcolor{linecolor8}{
29
}
\textcolor{linecolor9}{
29
}
\textcolor{linecolor1}{
25
}
\textcolor{linecolor2}{
27
}
\textcolor{linecolor4}{
27
}
\textcolor{linecolor5}{
21
}
\textcolor{linecolor6}{
27
}
\textcolor{linecolor6}{
27
}
\textcolor{linecolor6}{
27
}
\textcolor{linecolor6}{
27
}
\textcolor{linecolor6}{
27
}
\textcolor{linecolor6}{
27
}
\textcolor{linecolor6}{
27
}
\textcolor{linecolor6}{
27
}
\textcolor{linecolor6}{
27
}
\textcolor{linecolor6}{
27
}
\textcolor{linecolor6}{
27
}
\textcolor{linecolor7}{
27
}
\textcolor{linecolor8}{
27
}
\textcolor{linecolor9}{
27
}
\textcolor{linecolor1}{
17
}
\textcolor{linecolor2}{
22
}
\textcolor{linecolor4}{
22
}
\textcolor{linecolor5}{
22
}
\textcolor{linecolor6}{
22
}
\textcolor{linecolor6}{
22
}
\textcolor{linecolor6}{
22
}
\textcolor{linecolor6}{
22
}
\textcolor{linecolor6}{
22
}
\textcolor{linecolor6}{
22
}
\textcolor{linecolor6}{
22
}
\textcolor{linecolor6}{
22
}
\textcolor{linecolor6}{
22
}
\textcolor{linecolor6}{
22
}
\textcolor{linecolor6}{
22
}
\textcolor{linecolor7}{
22
}
\textcolor{linecolor8}{
22
}
\textcolor{linecolor9}{
22
}
\textcolor{linecolor2}{
2
}
\textcolor{linecolor4}{
2
}
\textcolor{linecolor5}{
20
}
\textcolor{linecolor6}{
2
}
\textcolor{linecolor6}{
2
}
\textcolor{linecolor6}{
2
}
\textcolor{linecolor6}{
2
}
\textcolor{linecolor6}{
2
}
\textcolor{linecolor6}{
2
}
\textcolor{linecolor6}{
2
}
\textcolor{linecolor6}{
2
}
\textcolor{linecolor6}{
2
}
\textcolor{linecolor6}{
2
}
\textcolor{linecolor6}{
2
}
\textcolor{linecolor7}{
2
}
\textcolor{linecolor8}{
2
}
\textcolor{linecolor9}{
2
}
\textcolor{linecolor1}{
0
}
\textcolor{linecolor4}{
25
}
\textcolor{linecolor5}{
28
}
\textcolor{linecolor6}{
25
}
\textcolor{linecolor6}{
25
}
\textcolor{linecolor6}{
25
}
\textcolor{linecolor6}{
25
}
\textcolor{linecolor6}{
25
}
\textcolor{linecolor6}{
25
}
\textcolor{linecolor6}{
25
}
\textcolor{linecolor6}{
25
}
\textcolor{linecolor6}{
25
}
\textcolor{linecolor6}{
25
}
\textcolor{linecolor6}{
25
}
\textcolor{linecolor7}{
25
}
\textcolor{linecolor8}{
25
}
\textcolor{linecolor9}{
25
}
\textcolor{linecolor1}{
19
}
\textcolor{linecolor2}{
9
}
\textcolor{linecolor4}{
9
}
\textcolor{linecolor5}{
19
}
\textcolor{linecolor6}{
9
}
\textcolor{linecolor6}{
9
}
\textcolor{linecolor6}{
9
}
\textcolor{linecolor6}{
9
}
\textcolor{linecolor6}{
9
}
\textcolor{linecolor6}{
9
}
\textcolor{linecolor6}{
9
}
\textcolor{linecolor6}{
9
}
\textcolor{linecolor6}{
9
}
\textcolor{linecolor6}{
9
}
\textcolor{linecolor6}{
9
}
\textcolor{linecolor7}{
9
}
\textcolor{linecolor8}{
9
}
\textcolor{linecolor9}{
9
}
\textcolor{linecolor1}{
23
}
\textcolor{linecolor2}{
1
}
\textcolor{linecolor4}{
1
}
\textcolor{linecolor6}{
1
}
\textcolor{linecolor6}{
1
}
\textcolor{linecolor6}{
1
}
\textcolor{linecolor6}{
1
}
\textcolor{linecolor6}{
1
}
\textcolor{linecolor6}{
1
}
\textcolor{linecolor6}{
1
}
\textcolor{linecolor6}{
1
}
\textcolor{linecolor6}{
1
}
\textcolor{linecolor6}{
1
}
\textcolor{linecolor6}{
1
}
\textcolor{linecolor7}{
1
}
\textcolor{linecolor8}{
1
}
\textcolor{linecolor9}{
1
}
\textcolor{linecolor1}{
17
}
\textcolor{linecolor2}{
0
}
\textcolor{linecolor3}{
0
}
\textcolor{linecolor6}{
0
}
\textcolor{linecolor6}{
0
}
\textcolor{linecolor6}{
0
}
\textcolor{linecolor6}{
0
}
\textcolor{linecolor6}{
0
}
\textcolor{linecolor6}{
0
}
\textcolor{linecolor6}{
0
}
\textcolor{linecolor6}{
0
}
\textcolor{linecolor6}{
0
}
\textcolor{linecolor6}{
0
}
\textcolor{linecolor6}{
0
}
\textcolor{linecolor7}{
0
}
\textcolor{linecolor8}{
0
}
\textcolor{linecolor9}{
0
}
\textcolor{linecolor1}{
12
}
\textcolor{linecolor4}{
11
}
\textcolor{linecolor5}{
9
}
\textcolor{linecolor6}{
11
}
\textcolor{linecolor6}{
11
}
\textcolor{linecolor6}{
11
}
\textcolor{linecolor6}{
11
}
\textcolor{linecolor6}{
11
}
\textcolor{linecolor6}{
11
}
\textcolor{linecolor6}{
11
}
\textcolor{linecolor6}{
11
}
\textcolor{linecolor6}{
11
}
\textcolor{linecolor6}{
11
}
\textcolor{linecolor6}{
11
}
\textcolor{linecolor7}{
11
}
\textcolor{linecolor8}{
11
}
\textcolor{linecolor9}{
11
}
\textcolor{linecolor1}{
11
}
\textcolor{linecolor4}{
23
}
\textcolor{linecolor5}{
2
}
\textcolor{linecolor6}{
23
}
\textcolor{linecolor6}{
23
}
\textcolor{linecolor6}{
23
}
\textcolor{linecolor6}{
23
}
\textcolor{linecolor6}{
23
}
\textcolor{linecolor6}{
23
}
\textcolor{linecolor6}{
23
}
\textcolor{linecolor6}{
23
}
\textcolor{linecolor6}{
23
}
\textcolor{linecolor6}{
23
}
\textcolor{linecolor6}{
23
}
\textcolor{linecolor7}{
23
}
\textcolor{linecolor8}{
23
}
\textcolor{linecolor9}{
23
}

\subsection{Example 4: Representability of Test Tasks}\label{sec:test-representability}

The following script produces documents with prefixes matching the function evaluation test task, for a function $g$.
The outermost loop introduces a variable that represents the separator.

\begin{Verbatim}[commandchars=\\\{\}]
LOOP OVER x0 DO
 LOOP OVER x1 DO
  PRINT x1
  FOR SOME x2 SUCH THAT f(x1) = x2 DO
   PRINT x2
  END
  PRINT x0
 ENDFOR
ENDFOR
\end{Verbatim}

\paragraph{Expressing Tasks and Description Length.}
The function evaluation task is expressed (with free variable x1 representing the input) by
\begin{Verbatim}[commandchars=\\\{\}]
FOR SOME x2 SUCH THAT f(x1) = x2 DO
  PRINT x2
END
\end{Verbatim}
The syntax tree of this script has three nodes (variable introduction, \texttt{$\langle$block$\rangle$} printing); the function is defined by one literal.

The \textsc{MissingLink} task is expressed (with free variables x1, x2 representing the input, for functions $f$, $g$) by
\begin{Verbatim}[commandchars=\\\{\}]
FOR SOME x3 SUCH THAT f(x1) = x3 DO
  IF g(x3) = x2 DO
    PRINT x3
  ELSE:
    FOR SOME x4 SUCH THAT f(x2) = x4 DO
       PRINT x4
    END
  END
END
\end{Verbatim}
The syntax tree of the script has nine nodes (twice variable introduction, if, twice printing, four times \texttt{$\langle$block$\rangle$}); the function is defined using four literals (2-DDF with 2 disjuncts).
More generally, representing a 2-DNF with $d$ disjuncts (i.e., $2d$ literals) can be achieved with at most $6d-3$ nodes: $d$ \texttt{FOR SOME} statements, $d-1$ \texttt{IF-THEN-ELSE} statements, $3d-2$ \texttt{$<$block$>$}s, and $d$ \texttt{PRINT} statements.\footnote{Strictly speaking, representation is sometimes only approximate because functions $f_i$ are not in general invertible. Our experiments exclude inputs without unambiguous solutions, partly mitigating this.}

\paragraph{Binary Tasks}
The Relation Classification task is expressed (with free variables x1, x2, and for functions $f$, $g$), if x3 and x4 denote the labels, as 
\begin{Verbatim}[commandchars=\\\{\}]
IF f(x1) = x2 DO
   PRINT x3
ELSE
   IF g(x1) = x2 DO
      PRINT x4
   ELSE
      // undefined
   END
END
\end{Verbatim}

\paragraph{Function Composition}
Function composition is expressed as
\begin{Verbatim}[commandchars=\\\{\}]
FOR SOME x2 SUCH THAT f(x1) = x2 DO
  FOR SOME x3 SUCH THAT g(x2) = x3 DO
     PRINT x3
  END
END
\end{Verbatim}
The \textsc{ChainOfThought} and \textsc{Explanation} versions are expressed as
\begin{Verbatim}[commandchars=\\\{\}]
FOR SOME x2 SUCH THAT f(x1) = x2 DO
  PRINT x2
  FOR SOME x3 SUCH THAT g(x2) = x3 DO
     PRINT x3
  END
END
\end{Verbatim}
and
\begin{Verbatim}[commandchars=\\\{\}]
FOR SOME x2 SUCH THAT f(x1) = x2 DO
  FOR SOME x3 SUCH THAT g(x2) = x3 DO
     PRINT x3
  END
  PRINT x2
END
\end{Verbatim}

\section{Training Details}\label{sec:training-details}

1\% of the dataset was reserved as dev data.
We optimized models using AdamW at a learning rate of 1e-4 and a batch size of 64, and annealed with a cosine schedule with warmup including 1000 warmup steps and 2e6 training steps\footnote{Corresponding to $>8e9$ tokens.} 
We chose these parameters in preliminary experiments to minimize pretraining cross-entropy, without yet testing in-context learning capabilities.

\section{Stability to World}\label{sec:stability-world}

We re-trained the medium-size model on four more worlds (i.e., resampling the functions in $\Ff$), finding consistent results across worlds (Figure~\ref{fig:across-worlds}).

\begin{figure}
    \centering
    Function Evaluation
    
    \includegraphics[width=0.7\textwidth]{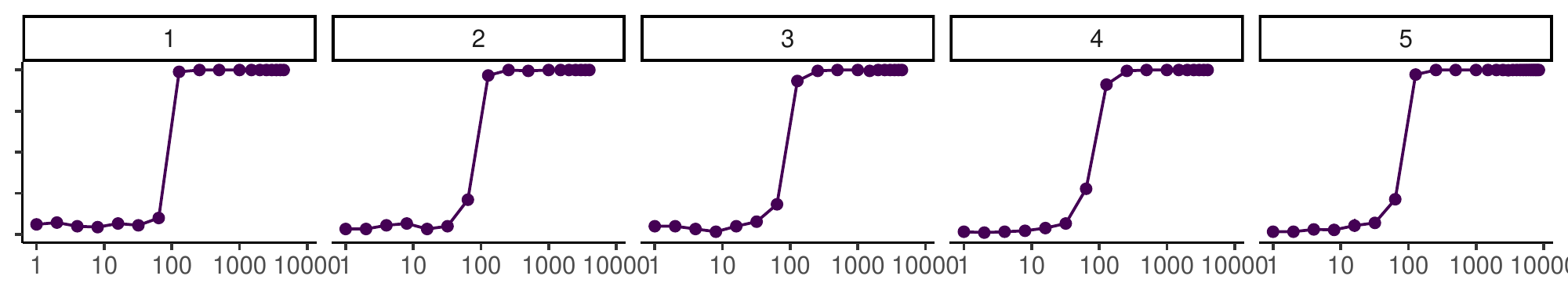}

Propositional

    \includegraphics[width=0.7\textwidth]{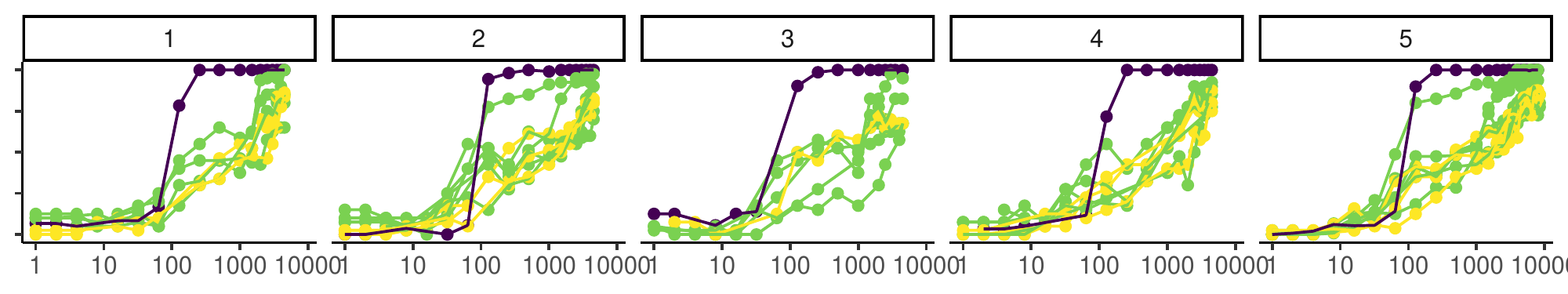}

Quantified

    \includegraphics[width=0.7\textwidth]{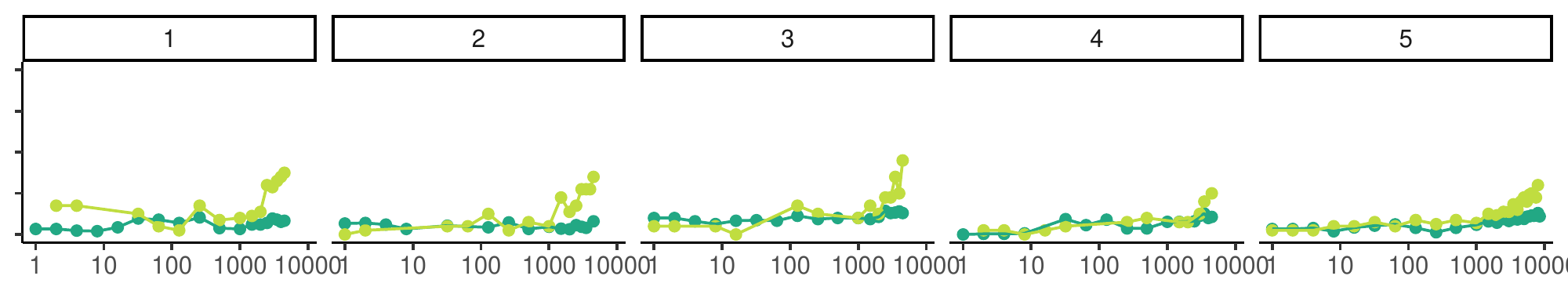}

Binary

    \includegraphics[width=0.7\textwidth]{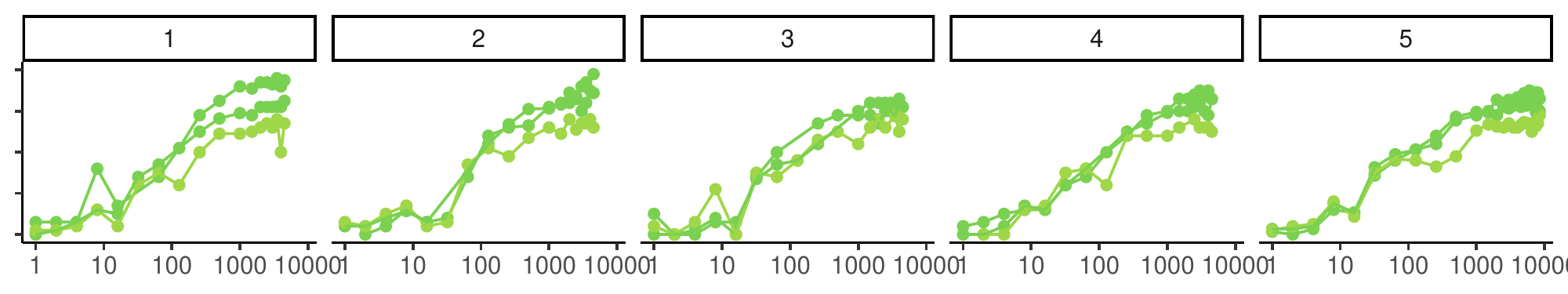}
    \caption{Results across worlds $\mathcal{M}$ for the medium-size LM.}
    \label{fig:across-worlds}
\end{figure}

\section{Attention Maps}

See Figure~\ref{fig:att-by-tasks}.

\begin{figure*}
\centering
        \begin{tikzpicture}
		\node (label) at (0,2)[draw=none, align=left, anchor=center]{Function Evaluation};
\node (label) at (0,0)[draw=none, align=left, anchor=center]{\includegraphics[width=0.25\textwidth]{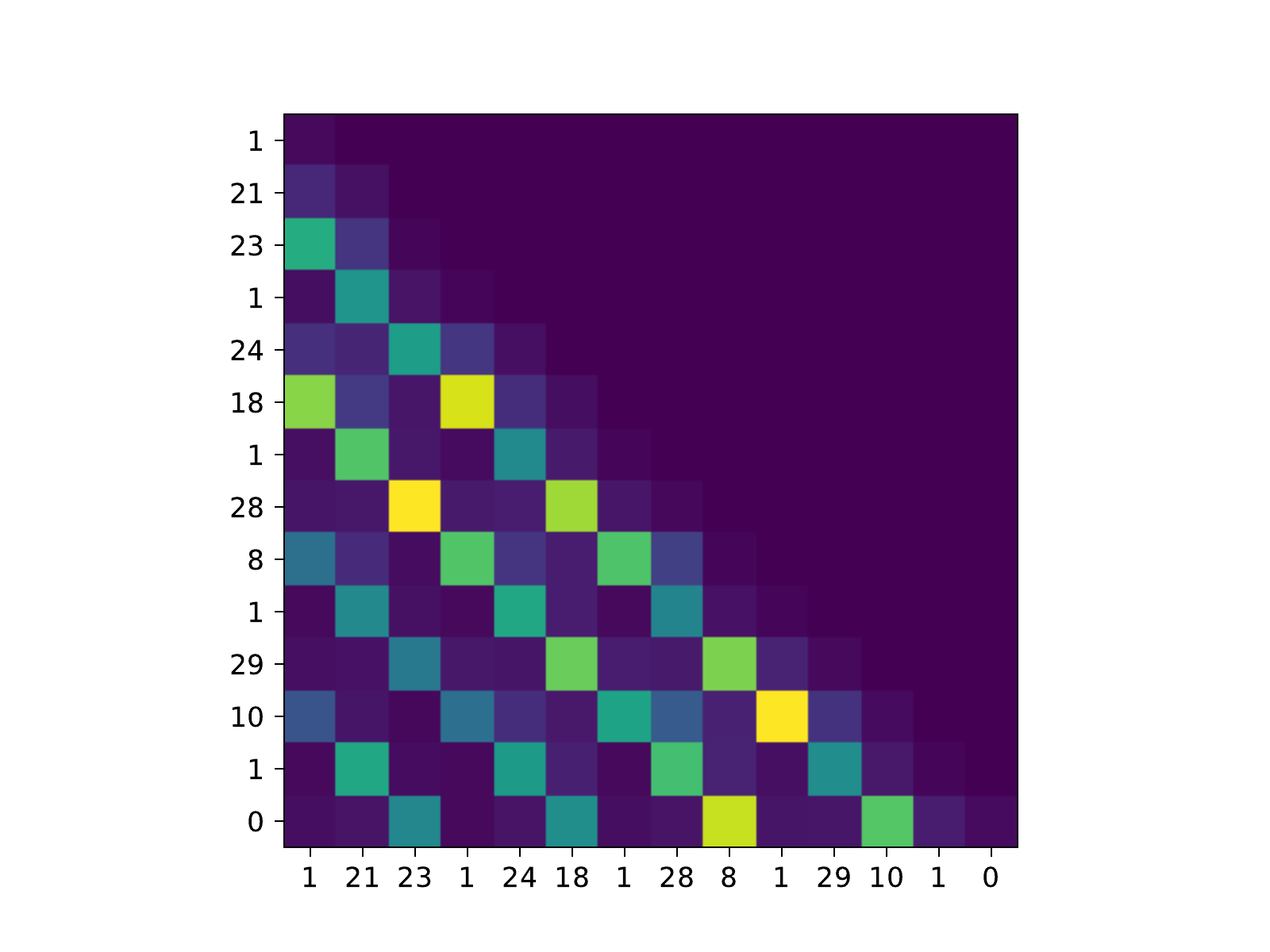}};
\node (label) at (0,-3)[draw=none, align=left, anchor=center]{\includegraphics[width=0.25\textwidth]{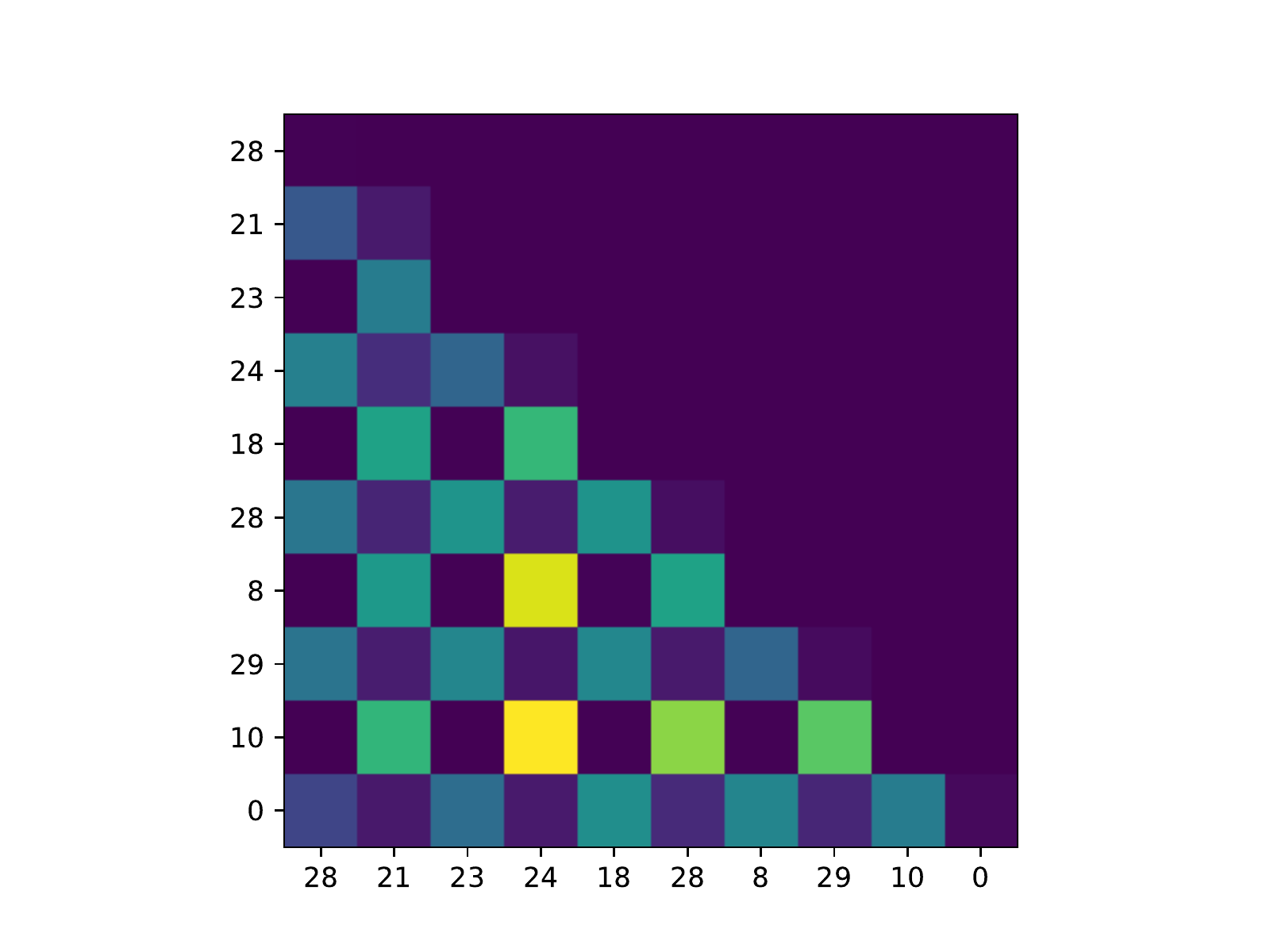}};

		\node (label) at (8,2)[draw=none, align=left, anchor=center]{Function Classification};
\node (label) at (8,0)[draw=none, align=left, anchor=center]{\includegraphics[width=0.25\textwidth]{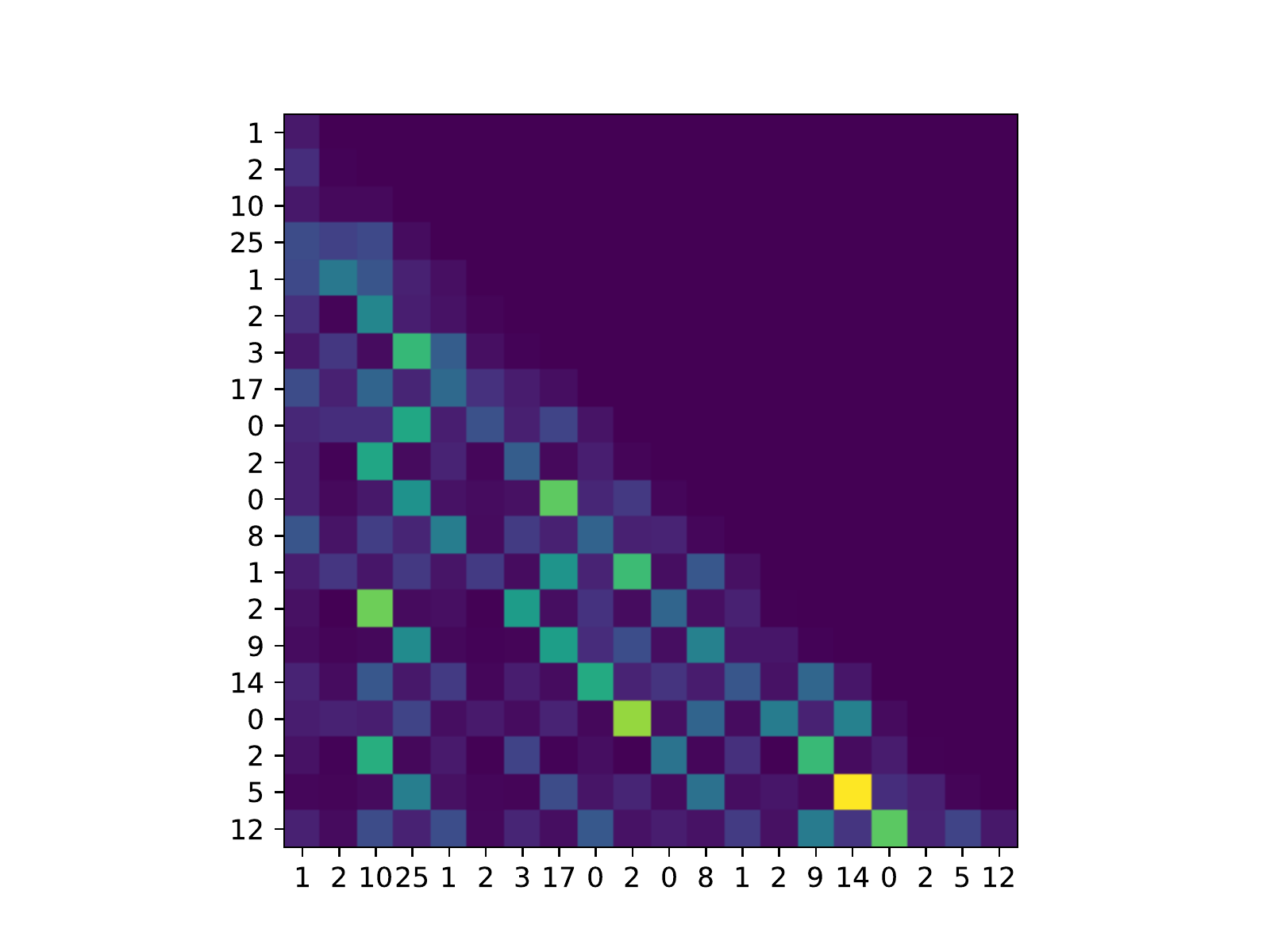}};
\node (label) at (8,-3)[draw=none, align=left, anchor=center]{\includegraphics[width=0.25\textwidth]{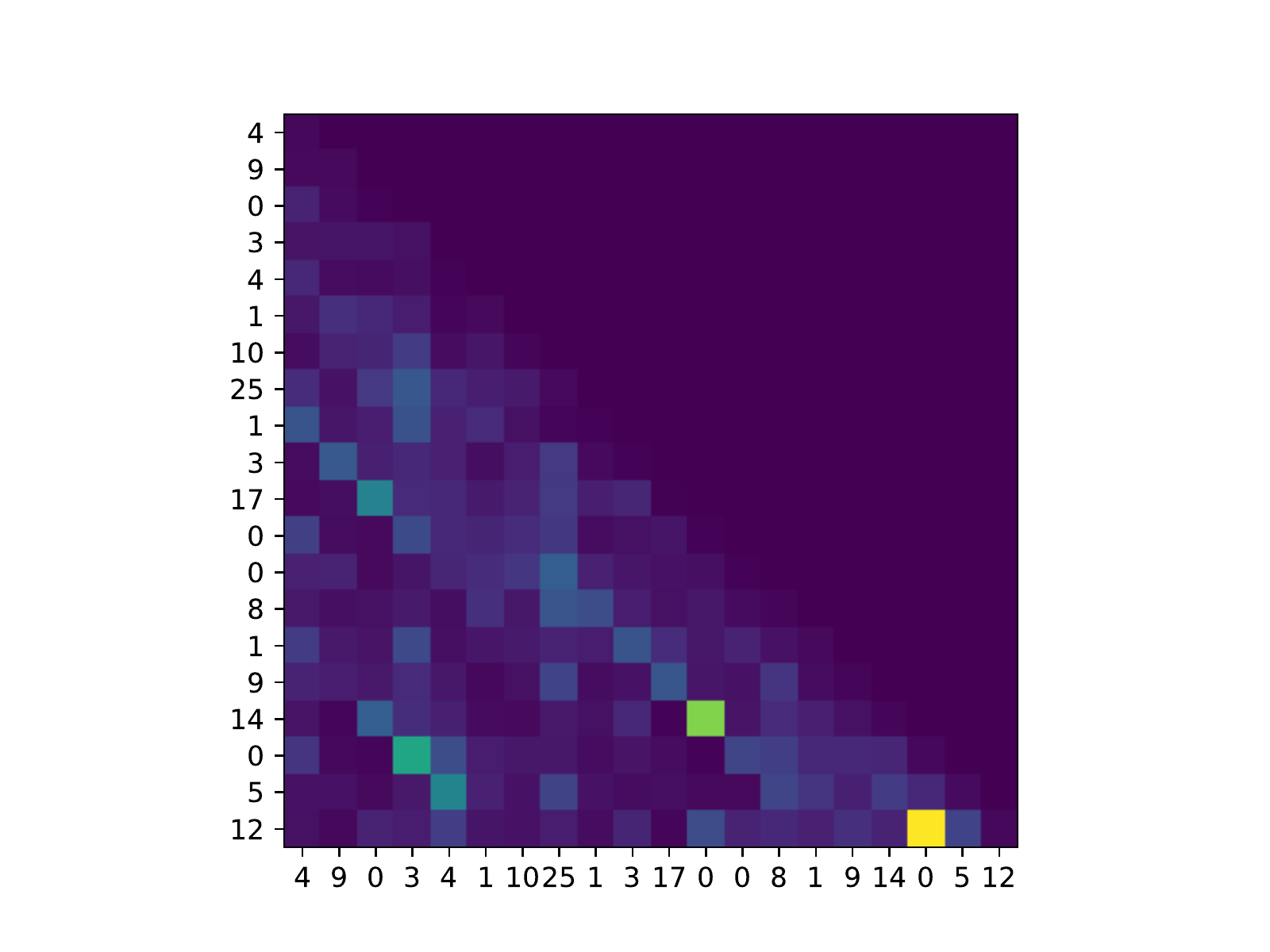}};

		\node (label) at (4,2)[draw=none, align=left, anchor=center]{Missing Link};
\node (label) at (4,0)[draw=none, align=left, anchor=center]{\includegraphics[width=0.25\textwidth]{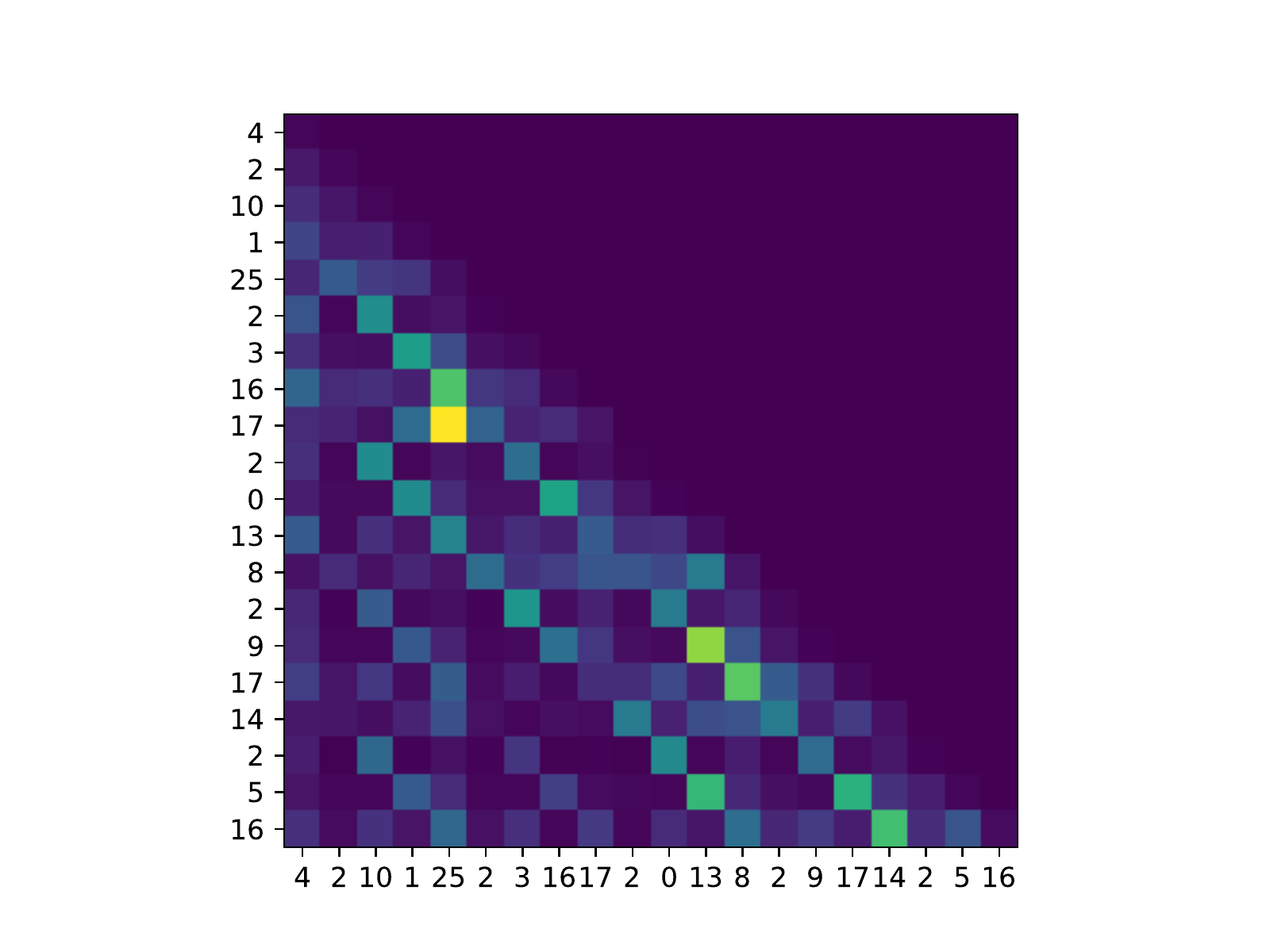}};
\node (label) at (4,-3)[draw=none, align=left, anchor=center]{\includegraphics[width=0.25\textwidth]{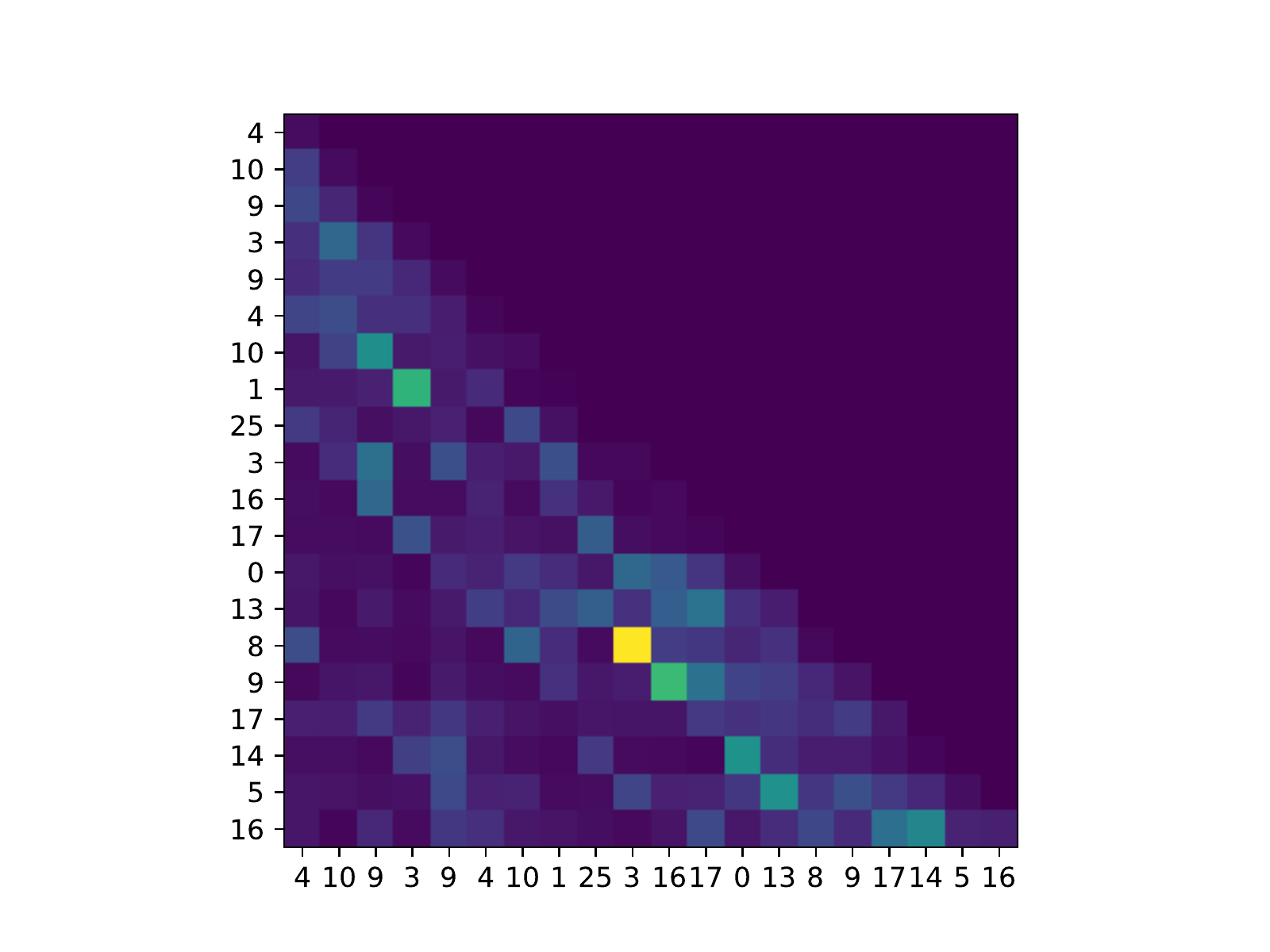}};

		\node (label) at (12,2)[draw=none, align=left, anchor=center]{Chain of Thought};
\node (label) at (12,0)[draw=none, align=left, anchor=center]{\includegraphics[width=0.25\textwidth]{probing_Sep_ChainOfThought_DiffRel_Long_End.py_medium.py_500_corpus.py_91226973.txt_293031_final.txt_2_1.pdf}};
\node (label) at (12,-3)[draw=none, align=left, anchor=center]{\includegraphics[width=0.25\textwidth]{probing_NoSep_ChainOfThought_DiffRel_Long_End.py_medium.py_500_corpus.py_91226973.txt_293031_final.txt_2_1.pdf}};

	\end{tikzpicture}

	\caption{Attention Maps for three tasks, with (top) or without (bottom) separator; only showing the end of the prompt for visibility. All results are for the 21M parameters model; for the second head in the top layer.
 FunctionEvaluation and the chain-of-thought task lead to very well-defined attention patterns. Attention pattern on FunctionClassification and MissingLink are more diffuse, mirroring increased difficulty.
 Without separators, attention patterns are more diffuse, but the periodic structure of the prompt is visible at least in the Function Evaluation task, suggesting that the separator helps with, but is not fully necesssary for, inducing the compositional structure of the prompt.
 }\label{fig:att-by-tasks}
\end{figure*}

\section{Proof of Theorem 1}

\subsection{Formal Definition of Compositional Attribute Grammars}

Here, we provide a full formal definition of our notion of compositional generative process.
All content is also contained in the main text, but we include this formal definition for reference.

We discuss the relationship to other grammar formalisms in the linguistic literature, and explain design choices, in Section~\ref{sec:design-choices}.

A \textbf{Compositional Attribute Grammar} is a tuple $(\mathcal{G}, \Omega, \Sigma, \textsf{spell} : \omega\mapsto\spell{\omega}, a, \mathcal{R}, \yield)$ consisting of the following:

\subsubsection{A Standard PCFG Generating Derivation Trees}\label{def:pcfg}
\begin{enumerate}
    \item a standard PCFG $\mathcal{G}$ consisting of (e.g., \citet{Manning1999FoundationsOS})
    \begin{enumerate}
    \item a finite set $\NT$ of nonterminals
    \item a designated start  symbol $\START \in \NT$ 
    \item a finite set $\T$ of terminals
    \item a finite set of production rules $\prodrule$ of the form
\begin{equation}
n \Rightarrow n_1 \dots n_k
\end{equation}
where $n \in \NT$; $n_i \in \NT \cup \T$.
    \item for each nonterminal $n$, a probability distribution $P_n(\cdot)$ over all production rules whose LHS has this nonterminal. In particular, the set of such production rules is nonempty for each nonterminal.
    \end{enumerate} 

such that the set $\mathcal{T}$ of derivation trees is the smallest set, with associated map $\textsf{root} : \mathcal{T} \rightarrow \T \cup \NT$ and a family of probability distributions $P(\cdot|n)$ ($n \in \T\cup \NT$), such that:

\begin{enumerate}
    \item If $t \in \T$, then $t \in \mathcal{T}$ and $\textsf{root}(t) = t$, and $P_t(t) = 1$.
    \item If $t_1, ..., t_k \in \mathcal{T}$ with root (non)terminals $n_1, ..., n_k$, and $\prodrule$ is a production rule of the form $n \Rightarrow n_1 \dots n_k$, then the tree $\tau$ with root $n$ and children $t_1, ..., t_k$ is an element of $\trees$.
    We will write
    \begin{equation}
        \tau := \prodrule[t_1, ..., t_k]
    \end{equation}
    for this tree.
    Furthermore,
    $\textsf{root}(\tau) = n$ and 
    \begin{equation}
        P(\tau|n) = P_n(\prodrule) \prod_{i=1}^k P(t_i|n_i).
    \end{equation}
\end{enumerate}

\end{enumerate}

\subsubsection{A yield operation mapping derivations to strings}\label{sec:yield}
Having defined the PCFG backbone creating derivation trees, we now need to define how these are mapped to linear strings of tokens.
This mapping is where the distinctive properties where CAGs (and other linguistic grammar formalisms, Appendix~\ref{sec:design-choices}) generalize over PCFGs come in.
The following ingredients are needed:

\begin{enumerate}[label=\arabic*.,start=2]
\item a finite alphabet $\Sigma$, a finite universe $\Omega$, and a map $\mathsf{spell} : \omega\mapsto\spell{\omega} : \Omega\rightarrow\Sigma$
        \item a map $a : \NT \cup \T \rightarrow \mathbb{N}$, called the ``\textit{arity}''

    \item a set $\mathcal{R}$ of sources of randomness\footnote{A formal definition is in terms of the set $\mathcal{R}$ of infinite trees where (1) each node is an independent coin flip, (2) every node has a countably infinite sequence of daughters.
    For any $r \in \mathcal{R}$, we will write $r_0, r_1, r_2,\dots$ for the daughters---these are independent random objects and themselves elements of $\mathcal{R}$.}
    \item A partial function mapping trees and their {\cvgarg}s to linear strings of tokens (see Figure~\ref{fig:compositional-setup}B and Figure~\ref{fig:yield} for illustration):
    \begin{equation}
        \mathcal{Y} : \mathcal{T} \times \Omega^* \times \mathcal{R} \pfun \Sigma^*
    \end{equation}
    such that
    \begin{equation}
        \mathcal{Y}(\tau, \xi, r)
    \end{equation}
    is well-defined iff the number of {\cvgarg}s matches the correct arity: $|\xi| = a_{\textsf{root}(\tau)}$,
    and the following holds:
\begin{enumerate}
    \item If $t \in \T$, then $\mathcal{Y}(t,\xi, r)$ is arbitrary.\footnote{Technically, it needs to be measurable w.r.t. $r$ for each choice of $t$ and $\xi$.}

\textit{The terminals in Figures~\ref{fig:compositional-setup} and \ref{fig:percolation}--\ref{fig:yield} all exemplify this. Some of them (shown in green) depend on their {\cvgarg} variable(s) $\xi$; others (shown in yellow) do not.}
    
    \item For a tree
    \begin{equation}
        \tau := \prodrule[t_1, .., t_k]
    \end{equation}
    arising from the production rule
    \begin{equation}
        \prodrule : n \Rightarrow n_1, ..., n_k
    \end{equation}
    with daughter trees $t_1, \dots, t_k$ and
    $\textsf{root}(\tau) = n$, the yield
    \begin{equation}
        \mathcal{Y}(\tau, \xi, r) \in \Sigma^*
    \end{equation}
    equals a concatenation of yields of children:
\begin{equation}\label{eq:yield-recursive-formal}
    \mathcal{Y}(t_{\iota_1},\eta_1,r_1) ... \mathcal{Y}(t_{\iota_q},\eta_q,r_q) \in \Sigma^*,
\end{equation}
    where
\begin{enumerate}
    \item $r_0 ,..., r_k \in \mathcal{R}$ are mutually independent sources of randomness derived from $r$,\footnote{Formally, they are different daughter trees of $r$.}

\item $q \in \mathbb{N}$ and $\iota \in \{1, ..., k\}^q$ are determined by $\prodrule, r_0, \xi$.

    \item For $j=1, \dots, q$, the tuple $\eta_j \in \Omega^{a_{n_j}}$ is determined by $\psi,r_0,\xi$.

    \end{enumerate}

  \textit{Informally, $\iota$ indicates the ordering and multiplicity of each $t_i$. In the simplest case, represented by ordinary Context-Free Grammars, each daughter tree appears exactly once in some fixed order, so we can take $q=k$ and $\iota = \langle 1,\dots,k\rangle$.
It is known that more general capabilities are needed to account for the structure of natural language, and a range of variants of (\ref{eq:yield-recursive-formal}) have been proposed in the linguistic literature.
See Appendix~\ref{sec:minimalist-iteration} for more on this.
One example is the case of ``loop'' operations, as in Figures~\ref{fig:compositional-setup} and in \ref{fig:yield}3--4: Here, the same daughter tree can be repeated, i.e., some indices appear in $\iota$ multiple times.}

\textit{Each $\eta_j$ is the {\cvgarg} vector for the tree $\tau_{\iota_j}$ and has the length matching the appropriate arity. It supports both passing of {\cvgarg}s along the derivation tree, and the introduction of new {\cvgarg}s which are then made available to daughter subtrees, both illustrated in Figure~\ref{fig:percolation} and Figure~\ref{fig:yield}3--4.}

\end{enumerate}
\end{enumerate}

\begin{figure*}
\centering
\includegraphics[width=0.5\textwidth]{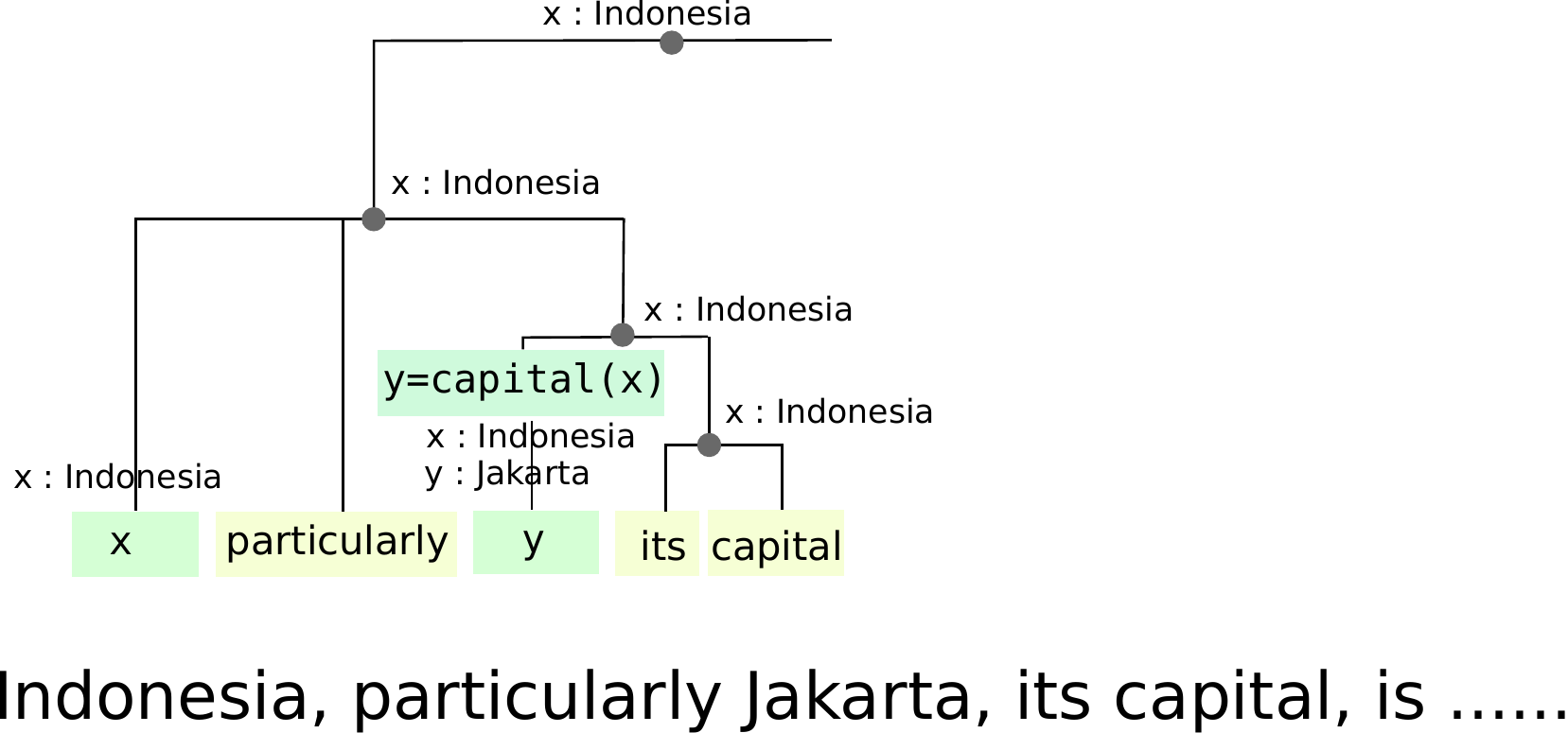}
	\caption{{\Cvgarg}s (or variables) are passed through the derivation tree. This can be formalized using {\cvgarg} lists; here, the nonterminals would have one ($\langle x \rangle$) or two ($\langle x,y \rangle$) {\cvgarg}s.}\label{fig:percolation}
\end{figure*}

\begin{figure*}
\centering
\includegraphics[width=0.5\textwidth]{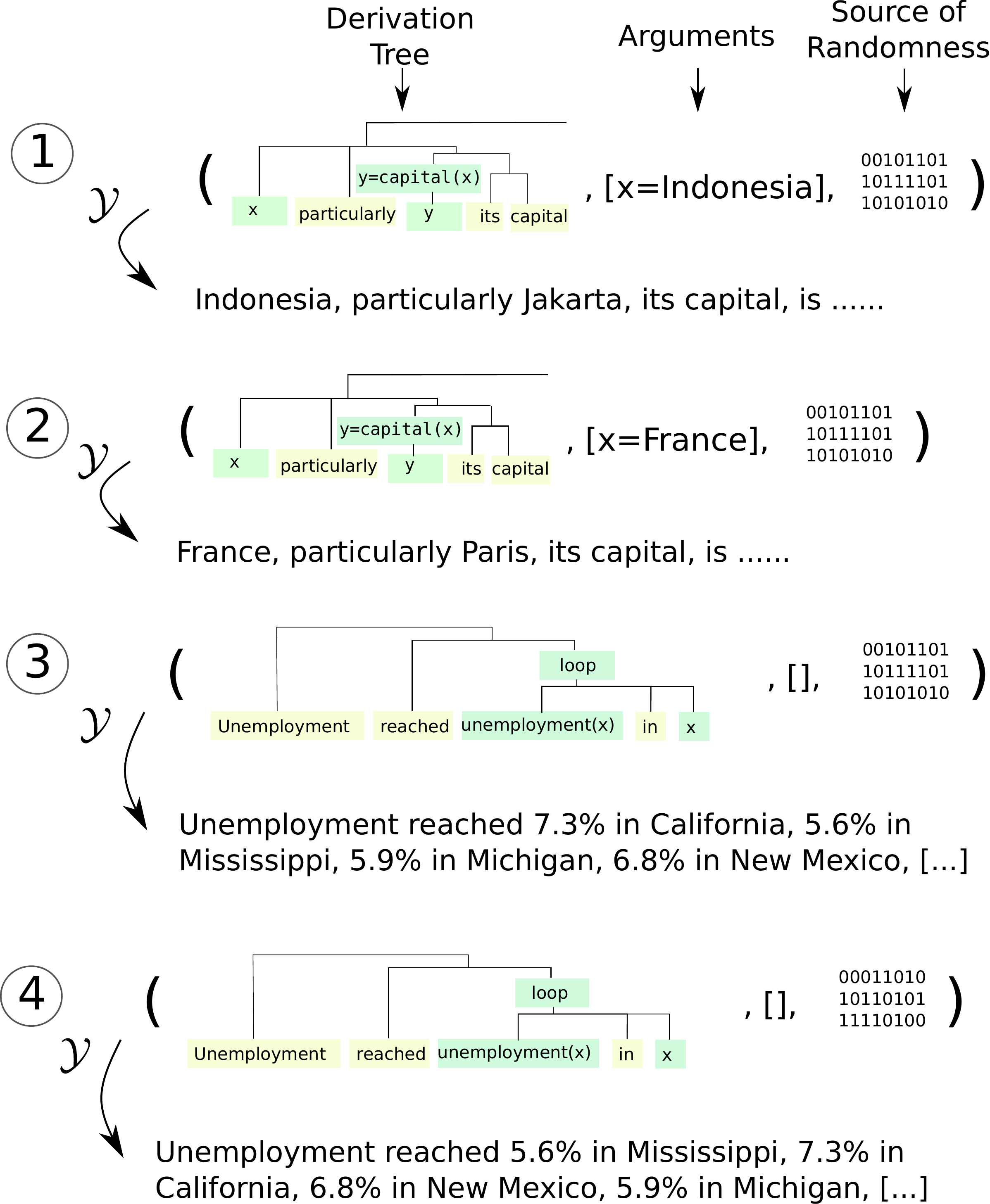}
	
\caption{The yield operation $\yield$ maps derivation trees together with {\cvgarg}s and a source of randomness to strings. Unlike the standard yield function of context-free grammars, it is capable of compositionally re-using and re-combining subtrees in two ways:
First, it ``wires'' the {\cvgarg}s of the daughter trees and its own {\cvgarg}s (1--2).
Second, it can repeat daughter trees, with different {\cvgarg}s, as in the gapping construction \citep{Ross1970GAPPINGAT} (3--4).
The first capacity can be simulated in context-free grammars.
The second capacity goes beyond context-freeness, and is possible in mildly-context sensitive grammar formalisms.
We allow both of these capabilities to be influenced by a source of randomness, for example, to allow variable ordering in iteration constructions without blowing up the set of production rules (compare 3 and 4). 
}\label{fig:yield}
\end{figure*}

\paragraph{Remark.} Extension to a typed version, where $\Omega$ is classified into a set of types and each nonterminal requires its {\cvgarg}s to have appropriate types, is straightforward.

\subsection{Regularity Assumptions}\label{sec:assumptions}\label{sec:technical}

\begin{figure*}
\centering
\includegraphics[width=0.5\textwidth]{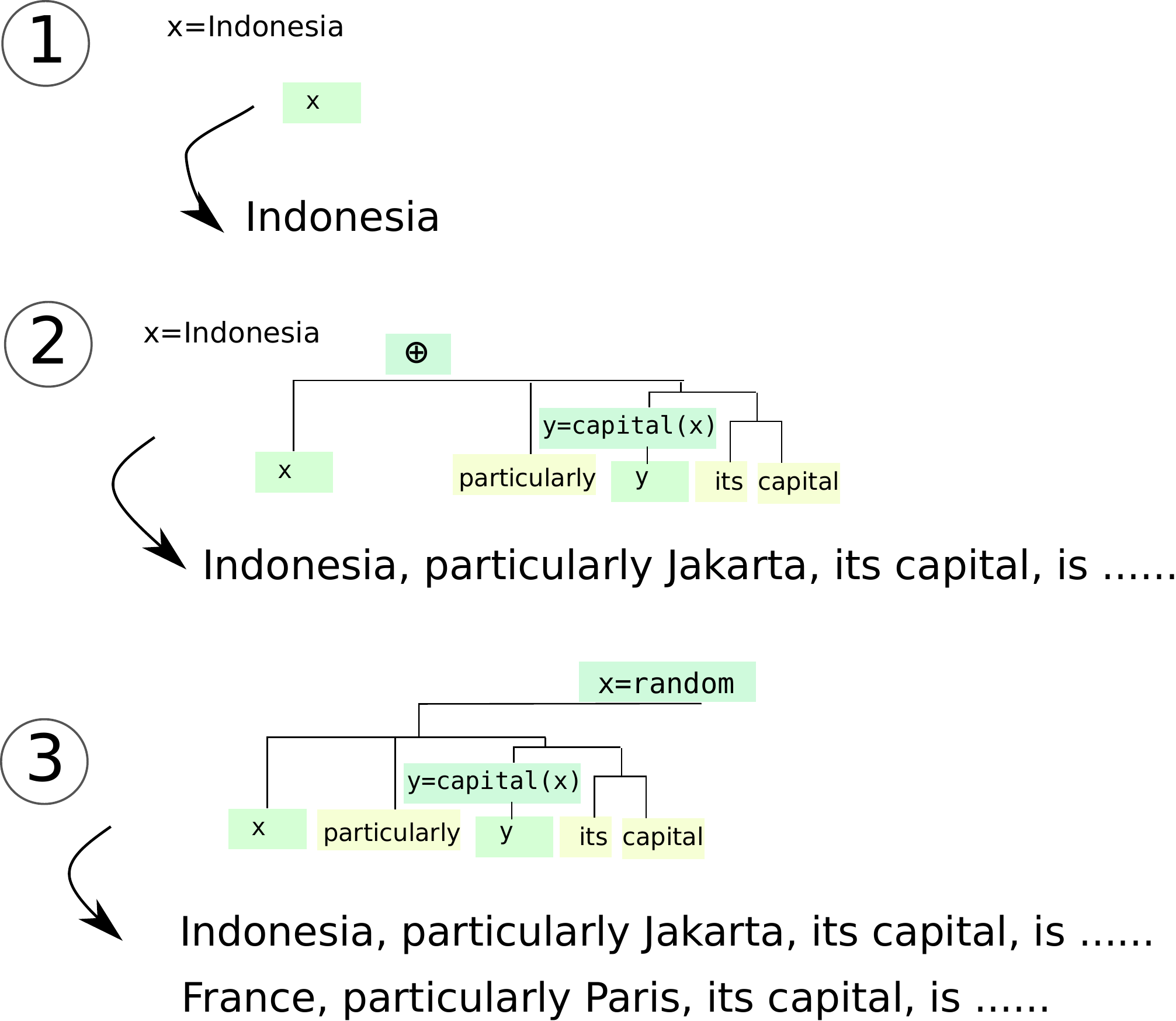}
\caption{We illustrate three closure properties stated in Section~\ref{sec:technical}: Projection onto a variable (1), concatenation of derivation trees (2), and marginalization of variables (3). A CAG permits concatenation of all trees, though production probabilities will distinguish between more probable (like this example) and less probable concatenations (e.g., ``Indonesia Paris'').
}\label{fig:closure}
\end{figure*}

To ensure that all strings in $\Sigma^*$ can be constructed, we assume the following closure properties of the set $\mathcal{T}$ of derivation trees (Figure~\ref{fig:closure}).
Let $a_{max}$ the maximum arity of any nonterminal.
\begin{enumerate}

\item Closure under \textsc{Projection} (Figure~\ref{fig:closure}A): For $1 \leq i \leq n \leq a_{max}$, there is $\tau_{projection,i,n} \in \mathcal{T}$ such that $\mathcal{Y}(\tau_{projection,i,n}, \xi ,r) \equiv \spell{\xi_i}$ for $\xi \in \Omega^n$.
This is marked with variable names in Figure~\ref{fig:compositional-setup}.

\item Closure under \textsc{Concatenation} (Figure~\ref{fig:closure}B): For each $\tau_1, \tau_2 \in \mathcal{T}$, $a_{root(\tau_1)} = a_{root(\tau_2)}$,
there is some production rule $\psi$ such that $\mathcal{Y}(\psi[\tau_1,\tau_2],\xi,r) $ equals
\begin{equation*}
\mathcal{Y}(\tau_1,\xi,r_1)\mathcal{Y}(\tau_2,\xi,r_2).
\end{equation*}

\item Closure under \textsc{Marginalization} (Figure~\ref{fig:closure}C): For each $\tau \in \mathcal{T}$, there is a production rule $\psi$ such that, for $\xi \in \Omega^{a_\tau-1}$, $\mathcal{Y}(\psi[\tau], \xi, r)$ has the infix
\begin{equation*}
     \mathcal{Y}(\tau, \langle\omega, \xi_1, ..., \xi_{a_\tau-1}\rangle, r_1),
\end{equation*}
where $\omega$ is a uniform sample from $\Omega$ determined by $r_0$.

\item Closure under \textsc{Constants}: If $\sigma \in \Sigma$  and $1 \leq n \leq a_{max}$,
then there is $\tau_{symbol(\sigma),n}$ such that $\yield(\tau_{symbol(\sigma)}, \xi, r) \equiv \sigma$ for $\xi \in \Omega^n$.

This property is trivially satisfied in standard PCFGs.

\item There are constants $c_0, c_1 > 0$ such that, for any nonterminal $n$ with arity zero ($a(n)=0$), there is a tree $\tau'$ containing a subtree $\tau$ with root nonterminal $n$ with $\textsf{root}(\tau')=\START$ with $\DL[\tau'] \leq c_0+\DL[\tau]$ such that $\tau$ is used in $\yield(\tau', \langle\rangle, r)$ with probability (over $r$) at least $c_1$. Here, by ``used'' we mean that, along the path from the root of $\tau'$ down to the root of $\tau$, all subtrees appear in the index list $\iota$ generated for the dominating production rule.
An implication of ``being used'' is that the yield of $\tau$ will occur as an infix inside the yield of $\tau'$.

Informally, this states that all nonterminals get used by some document-level derivation tree at probability bounded away from zero.
	In our experiments, $c_0 \leq 2$, $c_1 = 1$ since all nonterminals with zero arity (i.e., document scripts without free variables) can directly generate full documents (see Section~\ref{sec:scripts-are-cvcg}).

\item While documents can be unboundedly long, the expected document length is finite: $\sum_d p(d) |d| <\infty$.

This is needed to make autoregressive language modeling (Equation~\ref{eq:predictive}) well-defined.

\end{enumerate}

Assumptions 1--5 are comparable to the regularity assumptions that \cite[][Assumption 5]{DBLP:conf/iclr/XieRL022} made for their HMM model, which, inter alia, state that all transitions within a mixture component have nonzero probability bounded away from zero, and that all tokens can be emitted.

Another common assumption in the literature is that the PCFG is \emph{proper}, i.e., that the probability of all derivation trees sums up to one (rather than a smaller number). We do not require it here as it is not strictly necessary to prove our results. It is always satisfied when the PCFG is created as the MLE from an empirically observed set of derivation trees \citep{Chi1999StatisticalPO}.

\subsection{Preparatory Lemmas}

We first recall the definition of Iteration Complexity of a CAG:
\begin{definition}[Formal Definition of Iteration Complexity]\label{eq:formal-def-rn}
Let any CAG be given. For $n \leq |\Omega|$, let the CAG's \textit{Iteration Complexity} $R_n$ be the smallest number such that 
the following holds for all $\theta \in \trees$, and all pairwise distinct $x_1, ..., x_n$.
	We consider all trees $\tau\in\mathcal{T}$ such that for all $\xi\in\Omega^{a_\tau}$, $\mathcal{Y}(\tau, \xi, r)$ with probability at least $p_\tau > 0$ has an infix whose distribution (over the randomness in $r$) matches
\begin{equation}
\mathcal{Y}(\theta,\langle x_1, \xi_{1\dots a_\tau}\rangle,r_1)...\mathcal{Y}(\theta,\langle x_N, \xi_{1\dots a_\tau}\rangle,r_n).
\end{equation}
There is always at least one such tree (Lemma~\ref{eq:always-some-repetition}).
We define $R_n$ by the requirement that, for at least one of these $\tau$,
\begin{equation}\label{eq:def-rn-appendix}
\DL[\tau]  \leq  R_n +\DL[\theta] + \frac{1}{\rho} \log \left[p_\tau \cdot {{|\Omega|\choose n}}\right].
\end{equation}
\end{definition}
Intuitively, $R_n$ indicates how much more complex repetition is compared to a single occurrence; the third term accounts for the number of different choices of $x_1,\dots,x_n$; it disappears in the simple case where the yield of $\tau$ contains (\ref{eq:infix-repetition}) for each sequence $x_1,\dots,x_n$ at equal probabilities $p_\tau = {{|\Omega|\choose n}}^{-1}$. 

\begin{example}\label{ex:iteration-rn}
The prime example is a production rule mapping a nonterminal to a single nonterminal, and a corresponding yield $\yield(\psi[\tau], \langle\rangle, r)$ of the form $\yield(\tau, \langle x_1\rangle, r_1) \dots \yield(\tau, \langle x_{n}\rangle, r_{n})$, with the permutation determined by $r$, as in Figure \ref{fig:compositional-setup}B bottom.
	If each sequence $\yield(\tau, \langle x_1\rangle, r_1) \dots \yield(\tau, \langle x_{n}\rangle, r_{n})$ is generated with probability
\begin{equation}
	p(n) \cdot {|\Omega| \choose n}^{-1},
\end{equation}
then
	\begin{equation}R_n \leq 1-\frac{1}{\rho} \log \sum_{k=n}^\infty p(k)
	\end{equation}
	because substituting this for $R_n$ makes (\ref{eq:def-rn-appendix}) true (we set $w(n) :=\sum_{k=n}^\infty p(k)$):
\begin{align*}
	&R_n + \DL[\tau] + \frac{1}{\rho} \log \left[w(n) \cdot {|\Omega| \choose n}^{-1} \cdot {{|\Omega|\choose n}}\right]\\
=		&R_n +\DL[\tau] + \frac{1}{\rho} \log w(n)\\
	=	&(1-\frac{1}{\rho} \log w(n)) +\DL[\tau] + \frac{1}{\rho} \log w(n)\\
	=	&  1+\DL[\tau] \\
	= &	\DL[\psi[\tau]]
\end{align*}
For example, under a power law $p(n) \sim n^{-q}$\footnote{Finiteness of expected document length entails $q>2$.}:
\begin{equation}
	R_n \lesssim 1+\frac{q-1}{\rho} \log n
\end{equation}
Such sublinear growth guarantees convergence of ICL errors to zero as the prompt length increases.
\end{example}

\begin{lemma}\label{prop:prob-dl-bound}
    For some constant $\rho>0$, any $n \in \NT$ and any $\tau \in \mathcal{T}$ with $\mathsf{root}(\tau)=n$,
\begin{equation}
        P(\tau|n) \geq \exp(-\rho\cdot \DL[\tau]).
    \end{equation}
\end{lemma}

\begin{proof}
    Let $p_{min}$ be the minimum probability of any production rule in the PCFG, and set $\rho = -\log p_{min}$.
    Then, the result follows from induction over the height of $\tau$.    
\end{proof}

\begin{lemma}\label{eq:always-some-repetition}
Consider any CAG satisfying the regularity assumptions.
If $\func\in\mathcal{T}$, and 	$x_1, \dots, x_n \in \Omega$ are all distinct, there is a  tree $\tau\in\mathcal{T}$ such that $\mathcal{Y}(\tau, \xi_{1\dots a_\tau}, r)$ has an infix whose distribution matches
\begin{equation}
	Q(r) := \mathcal{Y}(\func,\langle x_1, \xi_{1\dots a_\tau}\rangle,r_1)...\mathcal{Y}(\func,\langle x_N, \xi_{1\dots a_\tau}\rangle,r_N)
\end{equation}
with probability $p_\tau > 0$.

\end{lemma}

\begin{proof}
Use the \textsc{Marginalization} property (Section~\ref{sec:assumptions}.3) to obtain a derivation tree whose yield is 
$\mathcal{Y}(\func, \langle x, \xi_{1\dots a_\tau}\rangle, r)$
where $x \in \Omega$ is uniformly random.
Then use the \textsc{Concatenation} property (Section~\ref{sec:assumptions}.2) to obtain a derivation tree whose yield has the same distribution as
\begin{equation}
\mathcal{Y}(\func,\langle x_1, \xi_{1\dots a_\tau}\rangle,r_1)...\mathcal{Y}(\func,\langle x_N, \xi_{1\dots a_\tau}\rangle,r_N)
\end{equation}
with probability $p_\tau = |\Omega|^{-n}$.
\end{proof}

\subsection{Proof of Theorem~\ref{theorem:theorem1}}\label{proof:theorem1}

\begin{figure*}\centering
\includegraphics[width=0.7\textwidth]{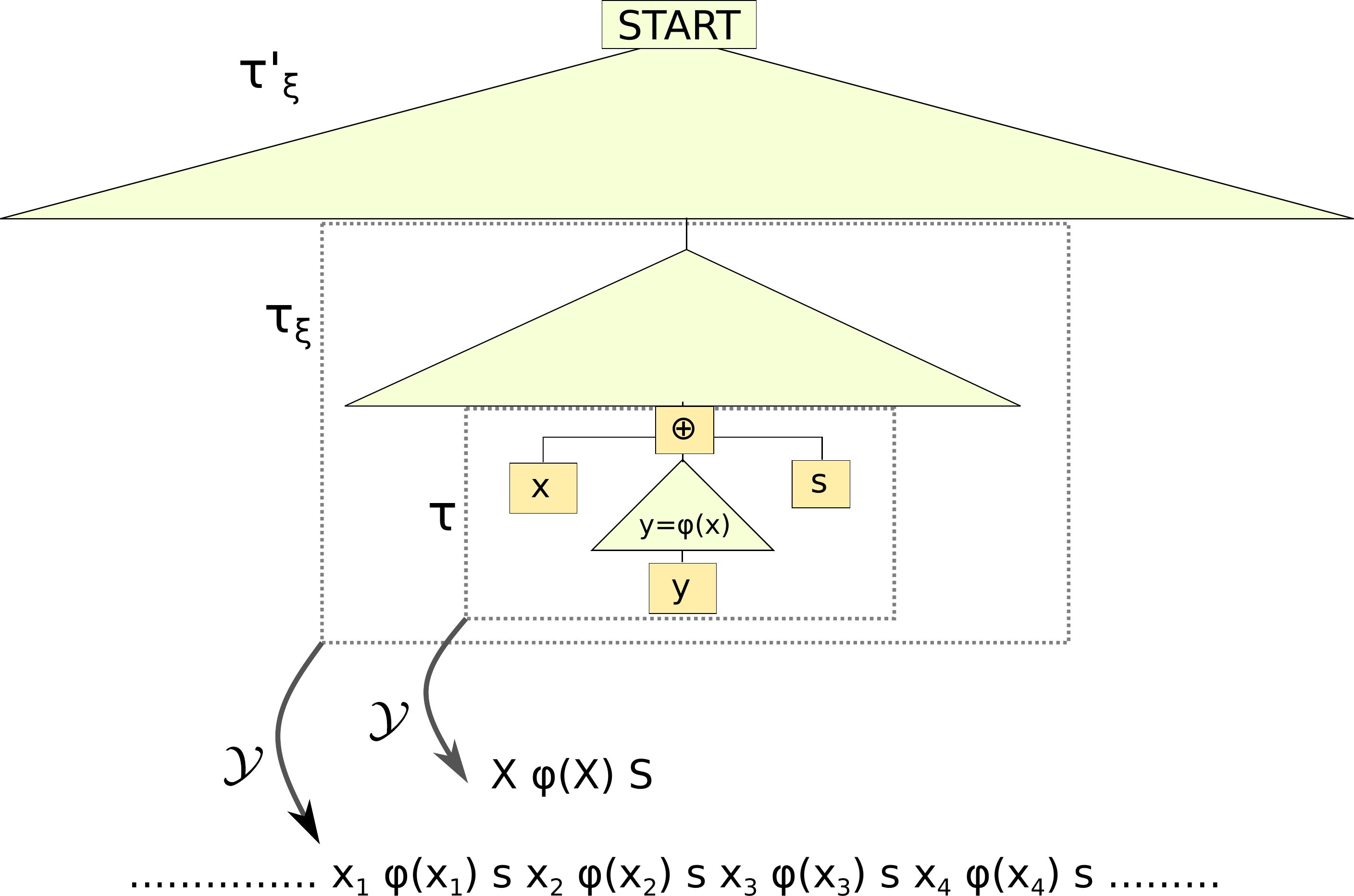}
\caption{Construction of the derivation tree $\tau'_\xi$ in the proof of Theorem~\ref{theorem:theorem1}. Rectangles denote individual nodes; triangles denote (sub)trees. The innermost derivation tree $\tau$ describes a single string $x \func(x) s$, where $x$ and $s$ are variables; it is obtained using the projection and concatenation operations (Section~\ref{sec:assumptions}). The subtree $\tau_\psi$ denoted by the lowest triangle defines some function $\psi$; the number of nodes inside it is $\DL[\tau_\psi]$.
Further outwards, the tree $\tau_\xi$ is chosen using the definition of $R_n$; it has the property that its yield contains an $n$-fold repetition of the yield of $\tau$ applied to objects $x_1, x_2, \dots, x_n$. 
	Finally, the tree $\tau'_\xi$ embeds $\tau_\xi$ inside some tree whose root is the start symbol nonterminal; its existence is guaranteed by regularity assumptions.}\label{fig:proof}
\end{figure*}
\begin{proof}[Proof of the Theorem]

In a first step, we will bound the description length of the prompt.
In the second step, we bound the probability assigned to a prompt by the autoregressive predictive model $M$.
In the third step, we will use this to bound the error on completing it.
Step 3 is similar to ideas from the completeness proof of Solomonoff Induction in Algorithmic Information  Theory (specifically Claim 5.2.2 in \citet{LiVitanyi2008introduction})---which studies Turing-complete generative processes---but we adapt these to our setting of linguistically motivated generative processes.

\paragraph{Bounding Description Length of Prompt.}

We begin by considering any sequence $\xi := x_1\dots x_n \in \Omega^n$ with pairwise distinct elements.
First, by the closure under \textsc{Projection}, \textsc{Concatenation}, and \textsc{Constant} properties (Section~\ref{sec:assumptions}), there is a derivation tree $\tau \in \mathcal{T}$ such that for all $x \in \Omega$ (Figure~\ref{fig:proof}):
	\begin{equation}\label{eq:tau-x-fx}
    \yield(\tau, \langle x\rangle, r) = \spell{x  \func(x)}  s \in \Sigma^*
\end{equation}
and
\begin{equation}
    \DL[\tau] \leq \DL[\tau_\func] + 3
\end{equation}
as in Figure~\ref{fig:compositional-setup}C.
Now let $\tau_\xi$ be a derivation tree such that its yield has an infix
\begin{equation}\label{eq:infix}
\yield(\tau,\langle x_1\rangle,r_1)...\yield(\tau,\langle x_N\rangle,r_N)
\end{equation}
with probability $\geq p_{\tau_\xi} > 0$
such that its description length $\tau_\xi$ satisfies:
\begin{equation}\label{eq:descr-len-bound}
	\DL[\tau_\xi] \leq  R_n +\DL[\tau_\func]+3 + \rho^{-1} \log \left[{{|\Omega|\choose n}}  p_{\tau_\xi}\right]
\end{equation}

\paragraph{Remark 1.}
The bound is derived for a prompt where examples have the simple form $\spell{x\func(x)}$.
In a CAG reflecting naturalistic text distributions, analogous bounds follow for more naturalistic prompts, or--if the CAG can derive those--even for prompts with unnaturally permuted labels.
The key is that a tree $\tau$ is constructed whose yield represents an individual example. 

\paragraph{Remark 2.}
If $s \in \Omega$ as in our experiments, one can obtain an on-average bound over all $s$, with constants independent of $s$, by instead considering
\begin{equation}
    \yield(\tau, \langle x,s\rangle, r) = \spell{x  \func(x)  s} \in \Sigma^*
\end{equation}
obtained using \textsc{Projection} and \textsc{Concatenation}.
Then by the closure under \textsc{Marginalization} (Section~\ref{sec:assumptions}.3) there is a tree $\tau'_{\xi}$ such that 
\begin{equation}
    \mathcal{Y}(\tau'_{\xi}, \langle\rangle, r) = \mathcal{Y}(\tau_\xi, \langle s\rangle, r')
\end{equation}
for each $s \in \Omega$ with probability $\frac{1}{|\Omega|}$.
The rest of the proof then proceeds with this $\tau'_{\xi}$ in place of $\tau_{\xi}$.
This alternative reasoning justifies why, in our experimental setup, the constants in Theorems 1--2 do not depend on $\Omega$.

\paragraph{Bounding Language Model Probability of Prompt.}
By Regularity Assumption 5 in Section~\ref{sec:technical}, there is 
some $\tau'_\xi$ with $\textsf{root}(\tau'_\xi)=\START$ containing $\tau_\xi$ as a subtree with $\DL[\tau'_\xi] \leq c_0 + \DL[\tau_\xi]$, and furthermore, with probability $c_1>0$ over $r$, $\yield(\tau'_\xi, \langle\rangle, r)$ contains an infix distributed identically to $\yield(\tau_\xi, \langle\rangle, r')$, and any such string will in turn have the infix~(\ref{eq:infix}) with probability $p_{\tau_\xi}$.
Hence, $\yield(\tau'_\xi, \langle\rangle, r)$ contains the infix~(\ref{eq:infix}) with probability  (over $r$) at least $c_1 p_{\tau_\xi}$.

For any $\xi$, let $k = (d+2) \cdot n$ be the length of the prompt $Q_n(\xi)$, i.e. the prompt $P_n$ for the sequence $\xi$.
We will now lower-bound the model assigned to the predictor to $Q_n(\xi)$.
Now by Lemma~\ref{prop:prob-dl-bound} and Equation~(\ref{eq:descr-len-bound}):
\begin{align*}
    P(\tau_\xi'|\START) \geq& \exp(-\rho c_0-\rho \DL[\tau_\xi]) \\
	\geq & \exp\left(-\rho c_0-\rho\cdot \left(  R_n +\DL[\tau_\func]+3 + \rho^{-1} \log \left[{{|\Omega|\choose n}}  p_{\tau_\xi}\right]\right)\right) \\
	=&  \exp\left(-\rho \cdot \left(R_n + 3+\DL[\tau_\func]\right)-\rho c_0\right) \cdot p_{\tau_\xi}^{-1} \cdot {|\Omega|\choose n}^{-1}.
\end{align*}
Recall, for any $k \geq 0$ and for any string $w \in \Sigma^k$,
\begin{equation}\label{eq:predictive-2}
    M(w_n|w_{1\dots n-1}) = \frac{\sum_d p(d) \cdot \#_d(w_{1\dots n})}{\sum_d p(d) \cdot \#_d(w_{1\dots n-1})}, 
\end{equation}
where, for the empty string $\epsilon$, $\#_d(\epsilon) = |d|+2$ by convention.\footnote{This convention is chosen so that $\sum_{w\in\Sigma\cup\{\$\}} M(w) = 1$.}
Equation~(\ref{eq:predictive-2}) is well-defined because of Regularity Assumptions, 2, 4 and 5 in Section~\ref{sec:technical}, by which both the numerator and the denominator are finite and non-zero.
Then, 
\begin{equation}\label{eq:predictive-product}
   M_k(w_{1\dots k}) := \prod_{n=1}^k M(w_n|w_{1\dots n-1}) =  \prod_{n=1}^N \frac{\sum_d p(d) \cdot \#_d(w_{1\dots n})}{\sum_d p(d) \cdot \#_d(w_{1\dots n-1})}  = \frac{\sum_d p(d) \cdot \#_d(w_{1\dots k})}{2+\sum_d p(d) \cdot |d|}, 
\end{equation}
where the last equality follows by cancellation of terms appearing in both the numerator and the denominator.
Recall that for $m=1,\dots,n$, the prompt $P_m(\xi)$ is given by
	\begin{equation}
		 \overline{x_1 \func(x_1)} s \spell{x_2 \func(x_2)} s \dots s \spell{x_{m-1} \func(x_{m-1})} s \spell{x_{m}},
	\end{equation}
 Then define $Q_m(\xi)$ as the prompt plus its correct completion:
	\begin{equation}
 \spell{x_1} \spell{\func(x_1)} s 
 \dots s \spell{x_{m} \func(x_{m})} s \equiv P_m \spell{\func(x_{m})} s
	\end{equation}
Now for any correctly completed prompt $Q_n(\xi) \in \Omega^k$, by the definition of $M_k$ (\ref{eq:predictive-product}),
\begin{align}
     M_k(Q_n(\xi))     \geq &  \frac{c_1 p_{\tau_\xi} p(\tau'_\xi|\START) }{2+\mathbb{E}[|d|]} \\
	\geq& c_1\cdot \exp\left(-\rho \cdot (R_n + 3 +\DL[\tau_\func])-\rho c_0\right) \cdot {|\Omega|\choose n}^{-1} \cdot \frac{1}{{2+\mathbb{E}[|d|]}}.\label{eq:mkpn-lower-bound}
\end{align}
We have derived this bound for any sequence $\xi$.

\paragraph{Bounding ICL Error.}
We now use this to derive an error bound that holds on average across all $\xi$.
For the distribution $\mu$ over $\Sigma^*$ arising as the uniform measure on all correctly completed prompts $Q_n(\xi)$: 
\begin{equation}
\mu := \frac{1}{{{{|\Omega|\choose n}}}} \sum_{\xi : \xi_1, \ldots, \xi_n \text{ pairwise distinct}} \delta_{Q_n(\xi)}
\end{equation}
it holds, by rearranging (\ref{eq:mkpn-lower-bound}), that
\begin{align}\label{eq:fraction-bound}
	\frac{\mu_k(Q_n(\xi))}{M_k(Q_n(\xi))}  \leq  & c_1 \cdot \exp(\rho  \cdot (R_n + 3+\DL[\tau_\func])+\rho c_0)  \cdot {(2+\mathbb{E}[|d|])}. 
\end{align}
 We will overload notation by writing $\mu(\sigma)$ for $\sum_{\sigma'\in\Sigma^*} \mu(\sigma\sigma')$ when $|\sigma| < |Q_n(\xi)|$, and $\mu(\sigma'|\sigma) = \frac{\sum_{\sigma'''\in\Sigma^*}\mu(\sigma\sigma'\sigma''')}{\sum_{\sigma''\in\Sigma^*} \mu(\sigma\sigma'')}$.
 We are now ready to bound the error made by the predictor on the prompts.
The expected summed cross-entropy loss on the in-context learning task satisfies (in the sums below, we use the notation $x_{1\dots n} \tilde{\in} \Omega^n$ to denote the sum over all $x_{1\dots n} \in \Omega^n$ with pairwise distinct entries) :
	\begin{align}
 		&	-	\frac{1}{{|\Omega| \choose n}}\sum_{x_{1\dots n} \tilde{\in} \Omega^n}  \sum_{m=1}^n  \log {M}(\spell{\func(x_m)}s|P_{m})  \\
=		&	-	\frac{1}{{|\Omega| \choose n}}\sum_{x_{1\dots n} \tilde{\in} \Omega^n}  \sum_{m=1}^n  \log {M}(\spell{\func(x_m)}s|Q_{m-1}\spell{x_m})  \\
=		&		\frac{1}{{|\Omega| \choose n}}\sum_{x_{1\dots n} \tilde{\in} \Omega^n} \sum_{m=1}^n  \log \frac{\mu(\spell{\func(x_m)}s|Q_{m-1}\spell{x_m})}{{M}(\spell{\func(x_m)}s|Q_{m-1}\spell{x_m})} \label{eq:expected-summed-ce-fraction} 
\end{align}
where we used the fact that $\mu(\spell{\func(x_m)}s|Q_{m-1}\spell{x_m})=1$.
The goal is to bound this in terms of (\ref{eq:fraction-bound}).
By the non-negativity of the KL-Divergence, we obtain
\begin{align*}
    0 \leq   & \sum_{m=1}^n \sum_{x_{1\dots m-1} \tilde{\in} \Omega^{m-1}} \mu(Q_{m-1}) D_{KL}(\mu(\cdot|Q_{m-1})||{M(\cdot|Q_{m-1})}) \\
     = & \sum_{m=1}^n \sum_{x_{1\dots m} \tilde{\in} \Omega^{m}} \mu(Q_{m-1}) \mu(\spell{x_m}|Q_{m-1}) \log \frac{\mu(\spell{x_m}|Q_{m-1})}{{M(\spell{x_m}|Q_{m-1})}} \\
     = & \sum_{m=1}^n \sum_{x_{1\dots n} \tilde{\in} \Omega^{n}} \mu(P_n) \log \frac{\mu(\spell{x_m}|Q_{m-1})}{{M(\spell{x_m}|Q_{m-1})}} \\
= & \frac{1}{{|\Omega| \choose n}} \sum_{x_{1\dots n} \tilde{\in} \Omega^n} \sum_{m=1}^n  \log \frac{\mu(\spell{x_{m}}|Q_{m-1})}{{M}(\spell{x_{m}}|Q_{m-1})} \\
\end{align*}
where both $\mu(\cdot|Q_m)$ and $M(\cdot|Q_m)$ are distributions over the one symbol following $Q_m$.
Hence, we can continue (using Equation~(\ref{eq:fraction-bound})):
\begin{align*}
(\ref{eq:expected-summed-ce-fraction}) \leq   & \frac{1}{{|\Omega| \choose n}} \sum_{x_{1\dots n} \tilde{\in} \Omega^n} \sum_{m=1}^n  \log \frac{\mu(\spell{x_{m}\func(x_m)}s|Q_{m-1})}{{M}(\spell{x_{m} \func(x_m)}s|Q_{m-1})} \\
=  & \frac{1}{{|\Omega| \choose n}} \sum_{x_{1\dots n} \tilde{\in} \Omega^n}  \log \frac{\mu(Q_n)}{{M_{|Q_n|}}(Q_n)} \\
=  &  \sum_{{\bm \sigma} \in \Sigma^{|Q_n|}} \mu_k({\bm \sigma}) \log \frac{\mu({\bm \sigma})}{{M_{|Q_n|}}({\bm \sigma})} \\
	\leq		&   \rho \cdot (R_n +3+\DL[\tau_\func])+\rho c_0+\log c_1+\log{(2+{\mathbb{E}[|d|]})}
	\end{align*}
 This expression is
 \begin{equation}\label{eq:zero-one-proof}
 \mathcal{O}(R_n +\DL[\tau_\func])
 \end{equation}
where the constants absorbed into $\mathcal{O}(\cdot)$ depend on: $\rho$ provided by Lemma~\ref{prop:prob-dl-bound} (depends on PCFG production probabilities); $c_0, c_1 > 0$ provided by Regularity Assumption 5; $\log \mathbb{E}[|d|] < \infty$ provided by Regularity Assumption 6, and not otherwise on $\yield$ (and thus $\Omega$).
This then also yields a bound on the zero-one error on completing the prompt:
\begin{align*}
	&	\frac{1}{{|\Omega| \choose n}}\sum_{x_{1\dots n} \tilde{\in} \Omega^n}  \sum_{m=1}^n \mathbb{1}_{\spell{\func(x_m)} \neq \operatorname{arg\ max}_{\omega \in \Sigma^{d}} M(\omega s|P_{m})}  \\ 
 \leq & 	\frac{1}{{|\Omega| \choose n}}\sum_{x_{1\dots n} \tilde{\in} \Omega^n}  \sum_{m=1}^n \mathbb{1}_{M(\spell{\func(x_m)} s|P_{m}) \leq \frac{1}{2}}  \\ 
 = & 		\frac{1}{\log 2} \frac{1}{{|\Omega| \choose n}}\sum_{x_{1\dots n} \tilde{\in} \Omega^n}  \sum_{m=1}^n \log 2 \cdot \mathbb{1}_{-\log M(\spell{\func(x_m)} s|P_{m}) \geq \log 2}  \\ 
 \leq & 	-	\frac{1}{\log 2} \frac{1}{{|\Omega| \choose n}}\sum_{x_{1\dots n} \tilde{\in} \Omega^n}  \sum_{m=1}^n \log M(\spell{\func(x_m)} s|P_{m}) \\ 
 = & \frac{1}{\log 2} \cdot (\ref{eq:expected-summed-ce-fraction}) \\
 = &  \mathcal{O}(R_n +\DL[\tau_\func])
 \end{align*}

\end{proof}

\subsection{Extension to Stochastic or Noisy Functions}\label{sec:stochastic-functions}
Extension to the case where $\phi$ is stochastic or noisy has an entirely analogous proof, and only requires some additional notation to formally define the relevant notions.
Formally, we say that a stochastic function $\func : \Omega \times \mathcal{R} \rightarrow \Omega^*$ 
is \emph{expressed} by a derivation tree $\tau_\func$ with description length $\DL[\tau_\func]$ if, for all $w \in \Sigma^*$ and all $x\in\Omega$, $Prob(\yield(\tau_\func, \langle x\rangle, r)=w)$ equals $Prob(\spell{\func(x,r)} = w)$.
The optimal cross-entropy loss, for an oracle predictor that knows the task from the start, is given by the entropy of outputs:
\begin{equation}
\ell_{opt} = \frac{1}{|\Omega|} \sum_x \operatorname{H}[\spell{\func(x,r)}]
\end{equation}
The learning bound will compare ICL loss to this optimal loss:
\begin{theorem}[Regret Bound for Nondeterministic Functions]\label{theorem:regret-bound}
Let any CAG be given, satisfying the regularity assumptions, and including the associated trees $\trees$, yield map $\yield$, and predictive distribution $M$, with the associated quantities $R_n$.
	Let $\func : \Omega\times\mathcal{R} \rightarrow \Omega^d$ be a function expressed by a derivation tree $\tau_\func\in\trees$.
 Let $\xi := x_1, x_2, ..., x_n \in \Omega$ ($n \leq |\Omega|$) be a sequence without replacement, and let $s \in \Sigma$.
For $m=1,\dots,n$, consider the prompt $P_m$ given by
	\begin{equation}
		 \overline{x_1 \func(x_1)} s \spell{x_2 \func(x_2)} s \dots s \spell{x_{m-1} \func(x_{m-1})} s \spell{x_{m}}.
	\end{equation}
 with expected completion $\spell{\func(x_m)}$.
Assume predictions are made as
\begin{equation}
     \operatorname{arg} \operatorname{max}\limits_{\omega \in \Sigma^d} M(\omega s|P_m).
\end{equation}
and cross-entropy loss is incurred as
\begin{equation}
     - \log M( \spell{\func(x_m)}s|P_m).
\end{equation}
On average across the choice of the sequence $x_1, x_2, ..., x_n$, the summed cross-entropy loss on completing $P_1, ..., P_n$,
is bounded by
 \begin{equation}\label{eq:learning-bound-regret}
  n \cdot \ell_{opt} + \mathcal{O}\left(R_n +\DL[\tau_\func]\right)
 \end{equation}
where $\mathcal{O}(\cdot)$ absorbs constants depending on the PCFG, $s$, and the average document length $\mathbb{E}[|d|]$, but not otherwise on $|\Omega|$, $\func$, or $n$.
\end{theorem}
The proof is entirely analogous to that of Theorem 1.
A bound on zero-one loss follows from a bound on cross-entropy loss analogously to (\ref{eq:zero-one-proof}).

\subsection{Necessity of Dependence on $R_n$}\label{sec:optimality-bound}
Here, we sketch (not fully formally) why the dependence of the bound (\ref{eq:learning-bound}) on $R_n$ cannot in general be avoided, by constructing a CAG with an ``adversarial'' production rule that mimics a prompt in order to lead the predictor astray. 
We do not consider this adversarial production rule linguistically realistic, but aim to show that the linear bound is tight in the absence of further assumptions.
Simultaneously, we point out that such an adversarial production rule can only slow down, but not prevent ICL. The latter point is important because it shows why the learning guarantee is stable under mixing other data into the pretraining distribution---ICL capabilities obtained in a small CAG (like the one generating our \textsc{Compositional} dataset) carry over  to larger CAGs extending it, up to a change in the constants in Equation~\ref{eq:learning-bound}.

We consider a CAG with a production rule as described in Example~\ref{ex:iteration-rn} mapping a nonterminal to a single nonterminal, and a corresponding yield $\yield(\psi[\tau], \langle\rangle, r)$ of the form
\begin{equation}
\yield(\tau, \langle x_1\rangle, r_1) \dots \yield(\tau, \langle x_{n}\rangle, r_{n})
\end{equation}
where any sequence of mutually different  $x_1, ..., x_n$ is chosen with probability
\begin{equation}\label{eq:prob-permutation-iter}
p(n) \cdot {|\Omega| \choose n}^{-1} := A \cdot n^{-q} \cdot {|\Omega| \choose n}^{-1}.
\end{equation}
where $A = (\sum_{n=1}^\infty n^{-q})^{-1}$; $q>2$. 
We now add a second ``adversarial'' production rule $\psi'$ where the yield of $\psi'[\tau]$ is 
\begin{equation}
\yield(\tau, \langle x_1\rangle, r_1) \dots \yield(\tau, \langle x_{n-1}\rangle, r_{n-1}) \spell{x_n} \spell{y}
\end{equation}
where $y$ is chosen from $\Omega$ at random.
Again, each $x_1, ..., x_n$ is chosen with probability (\ref{eq:prob-permutation-iter}).
Both $\psi$ and $\psi'$ generate derivation trees whose root nonterminal is $\START$, and they have the same production probabilities.
Now given a prompt:
\begin{equation}
P_n := \spell{x_1} \spell{\func(x_1)} s \dots \spell{x_{n}} \spell{\func(x_{n})} s \spell{x_{n+1}}
\end{equation}
by Example~\ref{ex:iteration-rn},
\begin{equation}
	R_n \leq 1-\frac{1}{\rho} \log \sum_{k\geq n} p(n) = \Theta(\log n)
\end{equation}
and we obtain an error bound
\begin{equation}\label{eq:logarithmic-error}
	\mathcal{O}(\log N)
\end{equation}
on completing prompts $P_1, \dots, P_N$ from Theorem 1.
Now we consider two derivation trees:
$\tau_1$ applying the loop production rule $\psi$ to the derivation tree $\tau_{example}$
for which (as in Equation~\ref{eq:tau-x-fx}):
\begin{equation}
    \yield(\tau_{example}, \langle x\rangle, r) = \spell{x  \func(x)}  s \in \Sigma^*
\end{equation}
and $\tau_2$ applying instead the adversarial production rule.
Now, for $y \neq \func(x_{n+1})$,
\begin{align}
\sum_d p(d|\tau_1) \cdot \#_d(P_n) = \sum_d p(d|\tau_2) \cdot \#_d(P_n) = &\sum_{k=n+1}^\infty p(k) (k-n) {|\Omega| \choose n+1}^{-1}
\\
\sum_d p(d|\tau_1) \cdot \#_d(P_n \spell{y}) = &0
\\
\sum_d p(d|\tau_2) \cdot \#_d(P_n \spell{y}) =& |\Omega|^{-1} \sum_{k=n+1}^\infty p(k)  {|\Omega| \choose n+1}^{-1}
\end{align}
When $n$ is large, the probability of the prompt context is dominated by $\tau_1, \tau_2$, because other ways of generating such a string (e.g. through concatenation of independent strings) have probability exponentially small in $n$, asymptotically smaller than ${|\Omega| \choose n+1}^{-1}$.
Hence,
\begin{align*}
    \sum_{y\neq \func(x_{n+1})} M(\spell{y}|P_n)= &\sum_{y\neq \func(x_{n+1})} \frac{\sum_d p(d) \cdot \#_d(P_n y)}{\sum_d p(d) \cdot \#_d(P_n)}  \\
        \gtrsim  &\frac{\sum_{k=n+1}^\infty k^{-q}{|\Omega| \choose n+1}^{-1}}{ 2\sum_{k=n+1}^\infty k^{-q} (k-n) {|\Omega| \choose n+1}^{-1}}  \\
        \sim &\frac{n^{-q+1}}{n^{-q+2}} = \frac{1}{n} \\
\end{align*}
and hence
\begin{align*}
 - \sum_{n=1}^N  \log M(\spell{\func(x_{n+1})}|P_n) \gtrsim   -\sum_{n=1}^N \log\left(1-\frac{1}{n}\right) \geq  \sum_{n=1}^N n^{-1} \sim \log N \\
\end{align*}
(the last inequality can be obtained from the Taylor series of $\log(1-x)$),
matching the upper bound (\ref{eq:logarithmic-error}) up to constants.

We note that the presence of this adversarial production rule can only slow down ICL up to the limit given by Theorem 1, but cannot make ICL impossible.
Intuitively, this is because even an adversarial production rule  will inadvertently lend support to the correct prediction: A correctly completed prompt $P_n \spell{\func(x_{n+1})}$ is a prefix of longer prompts $P_{n'}$, and thus will occur even in documents created using the adversarial production rule.

\subsection{Broad ICL Skills Require Non-Context-Free Generative Process}\label{sec:pcfg-no-icl}
Here, we describe informally why broad ICL learning guarantees for an idealized predictor require going beyond context-free grammars as a formal model of the generative process underlying the pretreaining distribution.
It has long been noted that context-free grammars are too restrictive to model the syntactic structure of language \citep[e.g.][]{Shieber1985EvidenceAT,Joshi1985NaturalLP}; indeed, repetition-like operations eluding context-free grammars have figured in this context \citep[e.g.][]{Steedman1990GappingAC,Kallmeyer2010OnMC}; see also Appendix~\ref{sec:minimalist-iteration}.

First, we note that iteration operations such as in Example~\ref{ex:iteration-rn} cannot be implemented in a PCFG because that would enable a CFG to express languages such as the copy language, $\{ww : w\in \Sigma^*\}$, known to be impossible.
More generally, in a PCFG, the description length of a derivation tree expressing repetition of a tree $\tau$ needs to grow with $n\cdot \DL[\tau]$; hence, $R_n \equiv +\infty$.

If the description of length of $\phi_\tau$ is bounded, prompt-like structures can be hard-coded into a PCFG, e.g., using production rules
\begin{verbatim}
A :== x f(x) s A 
A :== 
\end{verbatim}
Such a simple CFG could generate pretraining data sufficient to enable ICL for the function $f$. 
However, this only works when $f$ is hard-coded into the PCFG in such a way, and ICL will fail for other functions.
To see why this is the case, it is sufficient to consider constant functions $\func_w : \Omega \rightarrow\Omega^u$ such that $\spell{\func_w(x)} \equiv w \in \Omega^u$.
Given a prompt 
\begin{equation}
P_m(x_{1\dots m}; w) := x_1 w s x_2 w s \dots x_{m-1} w s x_m
\end{equation}
the reference answer, from the set $\{bs : b\in \Sigma^u\}$ is $ws$.
The possibility of ICL for an idealized predictor would now mean that  -- for sufficiently large $m$ depending on $|w|$ -- the CFG can ``copy'' $w$ from the prompt $P_m(x_{1\dots m}; w)$.
However, for any fixed CFG, this must fail at least for some sufficiently long $w$, essentially for the same reason as the non-context-freeness of the copy language.
\footnote{This is seen most easily for the cross-entropy loss: If $w$ is long, the LM's cross-entropy on $P_n$ (i.e., $-\log M(P_n)$) scales with $n|w|$. Hence, the cross-entropy of the correct completion given $P_m$ cannot converge to zero as $m\rightarrow\infty$.}

There is a second sense in which ICL guarantees such as Theorem 1 are not possible for CFGs:
A PCFG generating a pretraining dataset allowing ICL skills for a function $f$ (such as the hard-coded rules above) can always be extended with adversarial production rules that prevent ICL in the idealized predictor for the extended distribution.
In contrast, for general CAGs, the learning guarantee from Theorem 1 holds even if extending the CAG with other, possibly adversarial, production rules---intuitively, the guarantee is stable under mixing other data into the pretraining distribution---such additional rules would only affect the constants in Equation~\ref{eq:learning-bound} (cf. Appendix~\ref{sec:optimality-bound}).

\section{Proof of Theorem \ref{theorem:cot}}\label{sec:proof-cot}

\begin{proof}[Proof of Theorem \ref{theorem:cot}]
We reduce the statement to an application of the proof of  Theorem~\ref{theorem:theorem1} to the prompts
   \begin{equation}\label{eq:first-prompt-q}
    P^{(1)}_n({\bf q}) = \overline{x_1 \func_1(x_1) q_1} s \dots \overline{x_n \func_1(x_n) q_n} s \overline{x_{n+1}}
    \end{equation}
    with expected completion $\phi_1(x_{n+1})$,
    and
   \begin{equation}\label{eq:second-prompt-q}
    P^{(2)}_n({\bf q}) = \overline{x_1 q_1 \func_1(q_1)} s \dots \overline{x_n q_n \func_1(q_n)} s \overline{x_{n+1} q_{n+1}}
    \end{equation}
    with expected completion $\phi_1(q_{n+1})$.
    On average over the sequences ${\bf x}$ (without replacement), ${\bf q}$ (with replacement), we already have a bound of the desired form, using a variant of Theorem~\ref{theorem:theorem1} using repetition of the derivation trees $\tau^{(1)}$ and $\tau^{(2)}$:
    \begin{equation}
    \yield(\tau^{(1)}, \langle x\rangle, r) = \spell{x} \spell{\func_1(x)}\spell{q}  s \in \Sigma^*
\end{equation}
and
    \begin{equation}
    \yield(\tau^{(2)}, \langle x\rangle, r) = \spell{x} \spell{ q} \spell{ \func_1(q)}  s \in \Sigma^*
\end{equation}
where $q$ is uniformly distributed as $r$ varies,
as assured by the closure under \textsc{Marginalization} (Section~\ref{sec:assumptions}); we use closure under \textsc{Concatenation}; $x, q$ denote spellout operations $\tau_{projection}$ as assured by closure under \textsc{Projection}; and $s$ denotes $\tau_{symbol(s)}$ as assumed by closure under \textsc{Constants}.
	Substituting these for the tree $\tau$ (\ref{eq:tau-x-fx}) and analogously carrying out the remainder of Section~\ref{proof:theorem1} yields error bounds on completing (\ref{eq:first-prompt-q}) and (\ref{eq:second-prompt-q}) of the form $\mathcal{O}(R_n +\DL[\tau_{\func_1}])$.

Now we notice that, keeping $x$ and $\func_1$ fixed, the distribution of (\ref{eq:first-prompt-q}) over ${\bf q}$ equals the distribution of (\ref{eq:first-prompt}) over arbitrary functions ${\func_2}$.
An analogous statement holds for (\ref{eq:second-prompt-q}).
This proves the theorem.

\end{proof}

\section{Linguistic Grammar Formalisms and Design Choices}\label{sec:design-choices}

Here, we discuss how our definition of CAGs fits into the landscape of grammar formalisms, and how our theoretical results transfer to other formalisms.
We intend CAGs not as a full-fledged grammar formalism competing with existing proposals; rather, it aims to condense key mathematical ideas and insights from the formal grammar literature while minimizing notational burden: containing just enough material to formalize our theory of what learning from next-word prediction on compositional data entails for ICL.
We decided to go this route, rather than committing to  any individual formalism, both in order to avoid formalism-specific notational burden, and to make transparent the general features necessary (and not necessary) for proving our results.
For this reason, we focus on mathematical links to grammar formalisms, and leave the question of specific analyses of linguistic phenomena largely aside -- those questions are addressed in a substantive literature on formal grammar embedded in those formalisms.
This section only describes aspects of grammar formalisms as they are relevant to our analysis; a very comprehensive technical survey is provided by \citet{Kallmeyer2010ParsingBC}.
\citet{Mller2020GrammaticalT} discusses and compares analyses of linguistic phenomena across different formalisms.
A short survey is provided by \citet{Jger2012FormalLT}.

\subsection{Role of Yield Function}\label{sec:design-choices-lsit}
Recall that CAGs consist of a PCFG backbone generating derivation trees, and a yield operator mapping those to strings.
This architecture is extremely common across the grammar formalisms literature, even if not always stated in these terms.
The simplest case is that of a context-free grammar, where the yield function has a particularly simple form: The yield of a tree $\psi[t_1, \dots, t_k]$ consisting of an application of the production rule $\psi$ with children $t_1, \dots, t_k$ is
\begin{equation}
    \mathcal{Y}(\psi[t_1, \dots, t_k]) = \mathcal{Y}(t_1) \dots \yield(t_k)
\end{equation}
It is well-established that context-free grammars are too restrictive for modeling the syntax of natural language, even at the level of individual sentences \citep[e.g.][]{Shieber1985EvidenceAT,Joshi1985NaturalLP}.
This is addressed by a range of grammar formalisms that relax the definition of $\yield$.
An extremely general framework for formalizing this, into which many formalisms can be embedded, is the Generalized Context Free Grammar (GCFG, \citep{Pollard1984GeneralizedPS}), where each production rule $\psi$ is associated with some function $f_\psi(x_1, \dots, x_k)$ and the yield is:
\begin{equation}\label{eq:yield-gcfg}
    \mathcal{Y}(\psi[t_1, \dots, t_k]) = f_\psi(\mathcal{Y}(t_1), \dots,  \yield(t_k))
\end{equation}
The simple case where $f_\psi$ denotes string concatenation recovers context-free grammars.
Indeed, GCFGs generalize in a second direction, by allowing $\yield$ to map to \emph{tuples} of strings. This is useful for modeling syntactic relationships between structures appearing in different places in a sentence.

Grammar formalisms vary in what restrictions they place on $\yield$ -- that is, what restrictions they place on $f_\psi$.
CAGs assume that $\yield$ always maps to strings, i.e., tuples of length 1 in GCFG parlance. We do this for simplicity; nothing in our analysis depends on this.
Another assumption is more substantive:
CAG follows many grammar formalisms in assuming that $f_\psi$ performs some kind of concatenation of the argument strings (in some order and multiplicity), and cannot ``look inside'' the strings returned for the different children.
Many formalisms also limit the degree to which the yield function can repeat children; we take this up in Section~\ref{sec:minimalist-iteration}.

Some common formalisms, such as Minimalist Grammars (MGs), Tree-Adjoining Grammars (TAG), or Combinatory Categorical Grammar (CCG), are not usually stated in terms of (\ref{eq:yield-gcfg}).
MGs are weakly equivalent to formalisms commonly described in such terms (see Appendix~\ref{sec:minimalist-iteration}).
TAG and CCG themselves are also described in terms of derivation trees but with an unbounded number of node types, in fact, our analysis can be adapted to such settings directly without passing through GCFG-based description (Appendix~\ref{sec:infinite-node-types}).
Finally, some formalisms (notably HPSG \citep{pollard1994head}, and a variant of Minimalist grammars with feature percolation \citep{kobele2005features}) are Turing-complete, though only a subset similar to the other formalisms discussed here will likely be used in linguistic analysis.

Grammar formalisms are mostly applied at the level of individual sentences, but language models trained on large-scale text also learn and leverage relations across sentences within a document; indeed, grammar formalisms can be extended to model discourse without substantial changes to the mathematical formalisms \citep[e.g.][]{Kamp1993FromDT,Ginzburg2001InterrogativeIT,Ginzburg2012TheIS}. 

\paragraph{The Role of {\Cvgarg}s and Randomness}

Our definition of the yield function has two distinctive features beyond the CFG and GCFG setting (\ref{eq:yield-gcfg}), i.e., the addition of {\cvgarg}s and of an additional source of randomness:
\begin{equation}\label{eq:yield-general-si}
    \mathcal{Y}(\tau, \langle x_1,\dots,x_{a_n}\rangle, r) \in \Sigma^*, 
\end{equation}
These have important ramifications for linguistic and theoretical analysis.
The addition of attributes to nonterminals, while not present in (\ref{eq:yield-gcfg}), is in fact standard in linguistic analyses.
In some approaches, such as Minimalism \citep{Chomsky1992TheMP}, LFG \citep{Bresnan1987LexicalfunctionalG}, HPSG \citep{pollard1994head}, they are an integral aspect of the formalism; in others, such as TAG and CCG, they may be added to make linguistic analyses more parsimonious.
Attributes have many linguistic uses.
For example, in the domain of syntax, they might be used to establish subject-verb agreement without blowing up the number of nonterminal categories.
In our setting, the more relevant use of attributes  is in the syntax-semantics interface:
In particular, in order to model how forms are mapped to meanings, formal analyses typically assume that the referents of different expressions in  a sentence, and their relations among each other, are propagated through the derivation tree.
For instance, in HPSG \citep{Pollard1984GeneralizedPS,Kim2008EnglishSA,Ginzburg2001InterrogativeIT} or Sign-Based Construction Grammar \citep{Boas2012SignBasedCG}, each node in a tree is associated with semantic attributes indexing to real-world referent(s), which are propagated through the tree to establish semantic links between them. Similar approaches can also be used to define syntax-semantics interfaces for other formalsism \citep[e.g.][for TAG]{Kallmeyer2004LTAGSW}.
While we do not explicitly model meanings, they play a key role in determining the distribution of strings found in a corpus, which will respect world knowledge.
While derivational generative theories of syntax such as Minimalism \citep{Chomsky1992TheMP} focus on modeling the set of grammatical sentences irrespective of world knowledge,
a role of world knowledge is naturally part of model-theoretic (constraint-based) theories such as HPSG, where feature-based meaning representations -- and constraints on them -- are first class parts of the grammar of a language.
We formalize attributes or feature-value pairs in terms of a fixed-length list of {\cvgarg}s; our results are robust to other choices, such as modeling them in terms of feature-value pairs with some finite set of feature names.

As long as the set of possible attributes is finite -- which is the case in CAGs -- attributes can be simply compiled out into a larger set of nonterminals and production rules.
Analysis of generative capacity -- the primary cocnern of the theory of grammar formalisms--can thus ignore attributes (as is done in \citet{Kallmeyer2010ParsingBC}).
This explains why, despite the importance of attributes to linguistic analysis, they do not figure in (\ref{eq:yield-gcfg}).
For instance, in Minimalism, nonterminal categories are thought of as feature bundles, but they may be compiled out into atoms in formal description \citep[Section 6.2.2]{Kallmeyer2010ParsingBC}.
However, compiling attributes into atomic nonterminals obfuscates generalizations and blows up the size of the grammar (in the extreme case, when jointly modeling syntax and meaning, separate versions of a simple $\text{S} \rightarrow \text{NP VP}$ production rule for each combination of NP referents and verb meanings).
Keeping attributes in the formalization has the virtue of making compositional structure transparent and allowing us to prove bounds in Theorems 1--2 that do not deteriorate as $\Omega$ increases, because the PCFG component is separated from the universe.

The other addition concerns an additional source of randomness $r$ which affects how derivation trees are spelled out into strings.
This again does not change the generative capacity of our formalism, as it can be compiled out into a larger set of production rules.
Our motivation for including additional randomness is again to prevent blowup of the set of production rules.
For example, without additional nondeterminism in the $\yield$ function, a grammar would need separate production rules whenever there are different ways of assigning {\cvgarg}s to the children of a node (e.g., different possible referents for a noun phrase).

On a technical level, the use of attributes (or feature-value pairs) and additional nondeterminism permits us to decouple ICL bounds from the size of the world.
Indeed, the constants in the bounds do not change if we took the size of the world towards infinity and considered arbitrarily long prompts; we consider this an attractive mathematical property.
It would be obscured by compiling out features and nondeterminism into more nonterminals and production rules.

\paragraph{Application of Analysis to Grammar Formalisms}
Having discussed how CAGs relate to linguistic grammar formalisms, we now discuss what is needed in order to apply our Theorems 1--2 to such formalisms:
\begin{enumerate}
\item The formalism should be stated in terms of a GCFG. The fact that GCFGs model the yield as a tuple of strings does not impact the proofs of Theorems 1-2, beyond increasing notational load. For extension to formalisms not stated in terms of a GCFG, see Appendix~\ref{sec:infinite-node-types}.
\item The grammatical structure should be formalized in such a way that syntactic rules abstract over attributes, and attributes are not compiled out into atomic nonterminals.
As we discussed above, this demand satisfied by typical formal linguistic analyses.
Furthermore, we expect the scope of a grammar to be a model of the full generative distribution over sentences (or even documents), not just enumerating the grammatical sentences.

    \item In order to obtain nontrivial ICL bound, the grammar must be able to describe repetition of a single derivation tree. This is, indeed, a nontrivial condition on the grammar's generative power and requires going beyond context-free grammars (Appendix~\ref{sec:pcfg-no-icl}). We examine this further in Appendix~\ref{sec:minimalist-iteration}
\end{enumerate}

\subsection{Derivation Trees with Infinitely Many Node Types}\label{sec:infinite-node-types}
In TAG and CCG, derivation trees may have infinite sets of node or arc labels, not strictly providing a CFG backbone. 
For example, in a CCG derivation, unboudnedly complex combinatory types can appear (e.g., \texttt{VP/(NP/VP)}). Indeed, the only place where our analysis relies on the finiteness of node labels and production rules is Proposition~\ref{prop:prob-dl-bound}, bounding probability in terms of description length. An amended definition of description length accounting for unboudneness -- measuring not only the number of nodes in a tree but also their complexity -- is then sufficient to recover our theory.

\subsection{Repeating Structures in Grammar Formalisms}\label{sec:minimalist-iteration}

Key to obtaining a nontrivial ICL bound is the ability of the generative process to produce repetition of a derivation tree.
This is impossible in context-free grammars (Appendix~\ref{sec:pcfg-no-icl}), but becomes possible in more powerful linguistically adequate grammars.

Prime examples of repetition in language come from list-like enumerations (Figure~\ref{fig:compositional-setup}A.4) and from the ``gapping'' structure (Figure~\ref{fig:compositional-setup}A.6).
Repetition of derivation trees is fully possible in Range Concatenation Grammars (RNG) \citep{Boullier1999ChineseNM, Boullier2000RangeCG}.
In many other formalisms, repetition is restricted so that loop operations cannot be nested, or can only be nested subject to further constraints.
This is sufficient for proving ICL bounds as long as the function $\phi$ itself does not contain loop operations, which is satisfied in our test tasks. Even if $\phi$ contained loop operations, bounds can be proven using such formalisms if the nesting depth of loop operations is bounded, which is likely to be the case given general limits on recursion in language \citep{karlsson:2007-constraints,DBLP:conf/acl/BlasiCWSBB19}.

The formalism CNL-LMG \citep{Kallmeyer2010OnMC} specifically uses loop operations in order to account for gapping and scrambling phenomena in natural language, but it restricts the ways in which loop operations can be nested to maintain polynomial-time recognition.

Many formalisms do not allow duplication of the same subtree in the yield function ($f_\psi$ in Equation~\ref{eq:yield-gcfg}), which on first sight might exclude repetition.
However, these formalisms nonetheless can simulate loops using other capabilities -- and this is how gapping is analyzed in such formalisms.
In the remainder of the section, we illustrate how  $k$-fold iteration with {\cvgarg}s can be expressed in Minimalist Grammars  (MGs) \citep{Stabler1996DerivationalM}, a popular formalisation of  common ideas in the linguistic syntax literature \citep{Chomsky1992TheMP}.
MGs are weakly equivalent to a range of formalisms, including Multiple Context Free Grammars (MCFG) \citep{Seki1991OnMC} and linear context-free rewriting systems (LCFRS) \citep{VijayShanker1987CharacterizingSD}, and belong to the family of \textit{mildly-context sensitive language classes} \citep{Joshi1985NaturalLP,VijayShanker1994TheEO,Kallmeyer2010ParsingBC}, thought to be appropriate to describing the syntax of natural language; going beyond the power of context-free languages but allowing polynomial-time recognition.
Computational implementations with probabilistic parameterizations include, inter alia, \citet{Hale2006UncertaintyAT} using MCFGs, \citet{Hunter2013DistributionsOM,DBLP:journals/corr/abs-1710-11350,Torr2019WideCoverageNA} for MGs, \citet{Yang2022UnsupervisedDC} for LCFRSs.
Positive learnability results include \citet{Clark2021StrongLO}.
They are strictly subsumed by RCG and CNL-LMG, and in turn subsume Tree-Adjoining Grammar (TAG) \citep{Joshi1985NaturalLP} and Combinatory Categorical Grammar (CCG) \citep{Steedman2004TheSP}, two other popular non-context-free grammar formalisms.

We show that any Minimalist Grammar can be extended to achieve a constant error bound (independent of $n$) for all $n \leq |\Omega|$ for any function $\func$ expressible in the original grammar. 
This is not exactly the same as $R_k \leq c$ because it only applies to functions $\func$ that do not contain iterations themselves. But it is sufficient for achieving such a bound for functions expressed by trees $\tau_\func$ that do not contain loops themselves -- which applies to all test tasks considered in this paper.
The above-mentioned CNL-LMG formalism strictly extends MGs by allowing nested repetition of loop operations, subject to certain restrictions.
In fact, \citet{Kallmeyer2010OnMC} argued that this increased power is needed to account for gapping and scrambling in natural language; such a formalism would naturally give rise to a constant error bound even for certain functions including nested loops.
In this section, we focus on the popular Minimalist Grammars.

To make the discussion self-contained, we introduce the definition, using the equivalent  Multiple Context Free Grammars (MCFG) formalism  \citep{Seki1991OnMC} for convenience (also used by \citet{Hale2006UncertaintyAT}); this was shown to be equivalent by \citet{DBLP:conf/lacl/Michaelis98,DBLP:conf/lacl/Michaelis01}. 
Like CAGs, MCFGs can be formulated in terms of GCFGs (Equation~\ref{eq:yield-gcfg}).
As in general GCFGs, MCFG derivation trees yield tuples of strings rather than strings (thus, the ``M(ultiple)'' in ``MCFG''), with length (``dimension'') deterined by the top nonterminal.
\begin{definition}[Definition 6.1 in \cite{Kallmeyer2010ParsingBC}]\label{def:mcfg}
A MCFG consists of $(\NT, \T, F, P, S)$ such that:
\begin{enumerate}
\item $\NT$ is a finite set of non-terminals, and each $A\in \NT$ has an associated integer (``\textit{dimension}'') $dim(A)\geq 1$, $dim(A)\in\mathbb{N}$.

\item $\T$ is a finite set of terminals

\item $F$ is a finite set of mcf-functions: that is,
\begin{equation}f : (\T^*)^{d_1} \times ... \times (\T^*)^{d_k} \rightarrow (\T^*)^{d_0}
\end{equation}
such that each component of the value of $f$ is a concatenation of some constant strings and some components of its {\cvgarg}s. Furthermore, each component of the RHS of a rule is not allowed to appear in the value of $f$ more than once.

\item $P$ is a finite set of production rules of the form
\begin{equation}
  \psi :  A_0 \Rightarrow f_\psi[A_1,...,A_k]
\end{equation}
with $k\geq 0$, $f_\psi \in F$ such that
\begin{equation}
f_\psi : (\T^*)^{dim(A_1)} \times ... \times (\T^*)^{dim(A_k)} \rightarrow (\T^*)^{dim(A_0)}
\end{equation}

\item $S\in N$ is the start symbol with $dim(S) = 1$
\end{enumerate}
Derivation trees are defined as in CAGs.
The yield operation  is defined recursively as follows when $\psi \in P$:
\begin{equation}
    \yield(\func[A_1,...,A_k]) = f_\func(\yield(A_1) ... \yield(A_k))
\end{equation}
This completes the definition.
\end{definition}

Note that \citet{Kallmeyer2010ParsingBC} uses the term ``yield'' in a slightly different from our usage here, using it to refer to a set of strings derivable from a nonterminal rather than the string generated by an individual derivation tree.

The MCFG formalism does not explicitly have attributes, which would be compiled out into atomic nonterminals; we thus simply assume that every (non)terminal in $\NT \cup \T$ is associated with a tuple of attributes, and that we can write production rules as:
\begin{equation}
\psi[\xi] : \func_0(\xi) \Rightarrow \func_1(\eta_1) ... \func_k(\eta_k)
\end{equation}
where $\func_1, \dots, \func_k$ are independent of $\xi$, and $\eta_1, \dots, \eta_k$ are determined by $\psi$ and $\xi$.

We assume that some MCFG is given; we establish that we can extend it so that any function expressible using a derivation tree in the original MCFG can be iterated $N$-fold (where $N$ can go up to $|\Omega|$) in the resulting MCFG.
For notational simplicity, we assume that all nonterminals in the orignal MCFG have dimension 1; extension to higher dimension is a matter of notation.
We add all nonterminals of the form
\begin{equation}
   \textsc{Parallel}(\func; \omega_1, \dots, \omega_N)  : \func \in \NT,\ \omega_i\ \text{distinct}
\end{equation}
with dimension $N$, and of the form
\begin{equation}
    \textsc{Repeat}(\func) : \func \in \NT
\end{equation}
with dimension $1$.
For production rules
\begin{equation}
\psi : \func_0(\xi) \Rightarrow \func_1(\eta_1) ... \func_k(\eta_k)
\end{equation}
add a rule \textsc{Parallel}$(\psi,N, \omega_{1\dots N})$:
\begin{align*}
    \textsc{Parallel}(\phi_0; \omega_1, \dots, \omega_N) 
    \Rightarrow    f_\phi(& \textsc{Parallel}(\phi_1; \omega_1, \dots, \omega_N) \\
 & \dots \\
 & \textsc{Parallel}(\phi_k; \omega_1, \dots, \omega_N))
\end{align*}
where $f_\phi(x_1^{1\dots N}, ..., x_k^{1\dots N}) = [x_1^1\dots x_k^1, ..., x_1^N\dots x_k^N]$. 
Furthermore, add rules \textsc{Repeat}$(\phi, \omega_{1\dots N})$
\begin{align*}
\textsc{Repeat}(\phi_0) 
\Rightarrow 
f_\phi(\textsc{Parallel}(\phi_0; \omega_1, \dots, \omega_N)): \func \in \NT,\ \omega_i\ \text{distinct}
\end{align*}
where $f_\phi([x_1, ..., x_N]) = x_1...x_N$.

Then:
\begin{theorem}
If $\phi$ is definable in the original MCFG with $\DL[\tau_\phi]$ nodes, and $\omega_1, \dots, \omega_N \in \Omega$ is some sequence without replacement, then there is a tree $\tau'$ in the extended MCFG such that
\begin{equation}
    \yield(\tau') = \phi(\omega_1) \dots \phi(\omega_N)
\end{equation}
with $\DL[\tau_\phi]+1$ nodes.
\end{theorem}
\begin{proof}
Change any production rule in $\tau_\phi$ into the corresponding \textsc{Parallel} production rule, with the appropriate {\cvgarg}s for each of the $N$ copies.
Then apply a single \textsc{Repeat} rule to concatenate these $N$ copies.
\end{proof}

A difference between this resulting MG/MCFG and CAGs is that the MCFG needs to compile the attributes into nonterminals, inflating the size of the grammar.
This can be avoided by separating attributes from nonterminals, as done in both  CAGs and standard linguistic analysis practice.
As a result, a predictive model reflecting the extended MCFG enjoys ICL bounds (provided by Theorems 1--2) for any function $\phi$ definable in the original MCFG.

\section{Details for Document Scripts}\label{sec:acc:generating-scripts}
\label{sec:scripts-are-cvcg}

\begin{figure*}

		\begin{verbatim}
<command> ::= PRINT <variable>
           | IF <condition> THEN <block> ELSE <block> ENDIF
           | LOOP OVER <variable> DO <block> ENDFOR
           | FOR SOME <variable> SUCH THAT <condition> DO
              <block>  ENDFOR
<block> ::= <command> | <command> <command> 
           | <command> <command> <command> | ...
<condition> ::= <function>(<variable>) = <variable>
<variable> ::= x1 | x2 | x3 | ...
<function> ::= f1 | f2 | f3 | ...
\end{verbatim}

\caption{Backus-Naur(-like) Grammar for document scripts.
}\label{fig:backus-naur}
\end{figure*}

\begin{figure}
    \centering
    \includegraphics[width=0.8\textwidth]{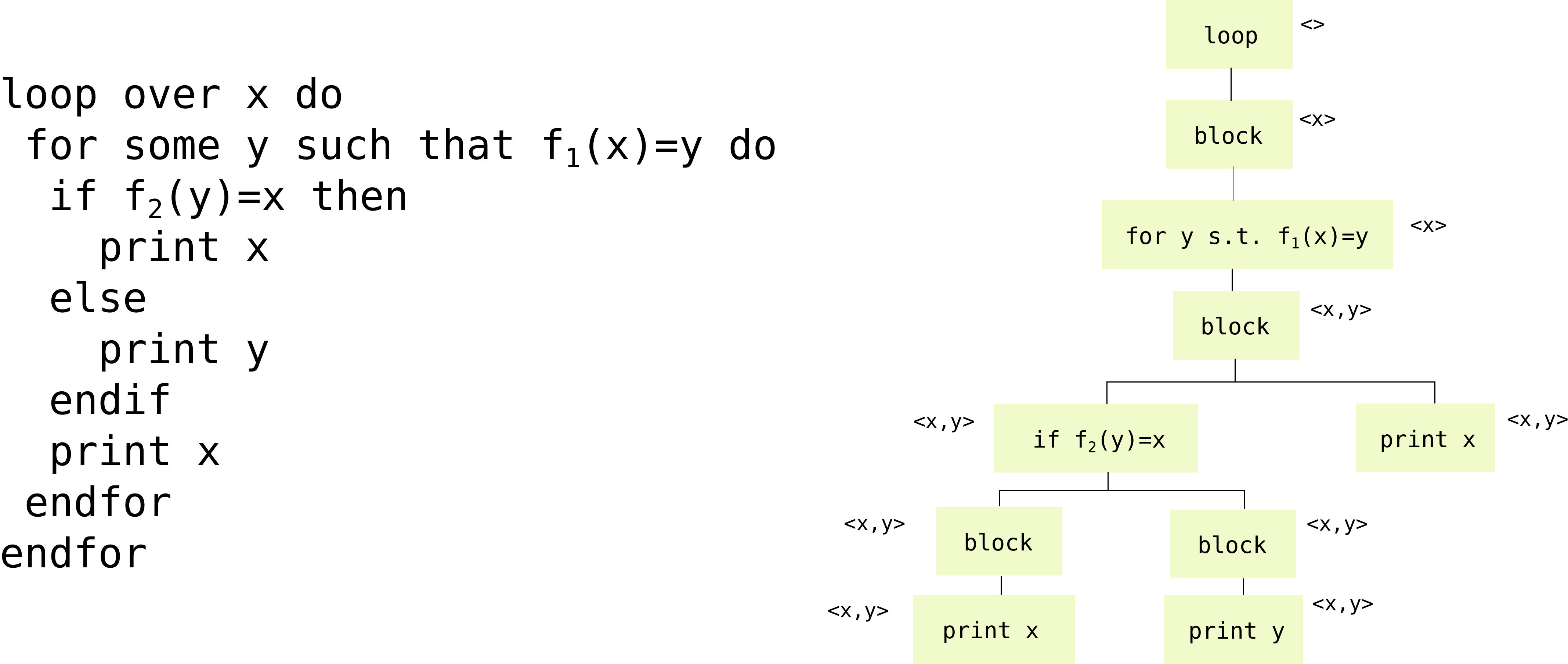}
    \caption{Left: a document script (same as in Figure~\ref{eq:setup}D). Right: a corresponding CAG derivation tree. Each node has a corresponding nonterminal in the CFG in Figure~\ref{fig:backus-naur} (\texttt{$<$command$>$} or \texttt{$<$block$>$}), and a list of {\cvgarg}s, which corresponds to the free variables available at that point in the script.}
    \label{fig:cvg-script}
\end{figure}

As programs, document scripts are defined by the Backus-Naur grammar in Figure~\ref{fig:backus-naur}.
They can simultaneously be interpreted as generated by a CAG:
The free variables within a code block correspond to {\cvgarg}s; translating to a PCFG with arity would proceed by creating separate nonterminals for each arity up to the maximum number of free variables enforced by the implementation.
An example of this correspondence is provided in Figure~\ref{fig:cvg-script}.
$R_{1}, R_2, \dots, R_{10}=1$ is satisfied by the ``for all'' statement.

The production rules for the \texttt{$\langle$function$\rangle$} nonterminal (Figure~\ref{fig:backus-naur}) depend on $|\Ff|$; as a consequence, our learning bounds will depend on $|\Ff|$.
On the other hand, the production rules do not depend on $\Omega$.
Our theory thus correctly predicts that ICL accuracy decreases with $|\Ff|$, but not with $|\Omega|$.

Because $\Sigma=\Omega$, we do not require trees specifically representing constants that may not be the names $\spell{\omega}$ of objects $\omega$, and thus do not need to satisfy closure under \textsc{Constants} (Section~\ref{sec:assumptions}).

\paragraph{Correspondence to CAG}
Here, we describe formally how a document script can be expressed as a CAG derivation tree; this simultaneously precisely defines the semantics of document scripts (see Figure~\ref{fig:cvg-script} for an example).
For each number $k$ of free variables up to some limit enforced by the implementation, we introduce nonterminals
\begin{enumerate}
\item $n_{\texttt{command},k}$
\item $n_{\texttt{block},k}$
\end{enumerate}
and, for each $1 \leq i \leq k$, terminals
\begin{enumerate}
\item $t_{\texttt{print},k,i}$
\end{enumerate}
We set $\START := n_{\texttt{command},0}$.
Further, we create production rules:
\begin{enumerate}
    \item Printing: $\psi_{print,k,i} := n_{\texttt{command},k} \Rightarrow t_{print,k,i}$
    \item Conditions: $\psi_{if, k, f_s(\xi_i) = \xi_j} : n_{\texttt{command},k} \Rightarrow t_{block,k} t_{block,k}$ for each equation of the form $f_s(\xi_i) = \xi_j$, for  $1 \leq i,j \leq k+1$; $f_s \in \mathcal{F}$.
    \item Loop: $\psi_{loop, k} : n_{\texttt{command},k} \Rightarrow t_{\texttt{block},k+1}$
    \item For some: $\psi_{for\ some, k, f_s(\xi_i) = \xi_j} :  n_{\texttt{command},k} \Rightarrow t_{\texttt{block},k+1}$
    \item Block: $\psi_{block, k, N} : n_{\texttt{command},k,N} \Rightarrow t_{block,k+1} \dots t_{block,k+1}$ ($N$ times, for each $N>0$ up to some limit enforced by the implementation)
\end{enumerate}
We specify $\yield$ as follows:
\begin{enumerate}
    \item $\yield(\psi_{print,k,i}[\tau], \xi, r) = \spell{\xi_i}$
    \item $\yield(\psi_{if, k, f_s(\xi_i) = \xi_j}[\tau_1,\tau_2], \xi, r) = \yield(\tau_1, \xi, r_1)$ if the condition is true, and the same with $\tau_2$ else.
    \item $\yield(\psi_{loop, k}[\tau], \xi, r) = \yield(\tau, \langle \omega_1, \xi_1, ..., \xi_k\rangle r_1) \dots \yield(\tau, \langle \omega_N, \xi_1, ..., \xi_k\rangle r_N)$ where the subset $\omega_1, ..., \omega_N$ is determined by $r_0$, and $N$ is fixed.
    \item $\yield(\psi_{for\ some, k, f_s(\xi_i) = \xi_j}[\tau], \xi, r) = \yield(\tau, \langle\omega, \xi_1, \dots, \xi_k\rangle, r)$ where $\omega \in \Omega$ is a random element satisfying $f_s(\xi_i) = \xi_j$ if $\omega$ is substituted for $\xi_{k+1}$; if none exist, the yield is the empty string.
    \item $\yield(\psi_{block, k, N}[\tau_1, \dots, \tau_N], \xi, r) = \yield(\tau_1, \xi, r_1) \dots \yield(\tau_N, \xi, r_N)$
\end{enumerate}

The regularity assumptions in Section~\ref{sec:assumptions} are satisfied as follows; here, we make use of the fact that $f_1$ is the identity in our experiments.
Closure under \textsc{Projection} is satisfied by $t_{print,k,i}$.
Closure under \textsc{Concatenation} is satisfied by wrapping each of $\tau_1, \tau_2$ in an application of $\psi_{for\ some, k, f_1(\xi_1) = \xi_1}$ and then wrapping those inside an application $\psi_{block, k, N}$.
Closure under \textsc{Marginalization} is satisfied by $\psi_{if, k, f_1(\xi_1) = \xi_1}$.
We omit closure under \textsc{Constants}, which is not necessary to prove Theorems 1--2 in our setup (because the separator is chosen from $\Omega$); this allows us to decouple the PCFG component of the CAG from $\Omega$.
The constants $c_0, c_1$ are small: a derivation tree with zero free variables can directly generate a document, potentially after wrapping in a \texttt{command} NT. 
Finally, expected document length is finite because production rules for loops and blocks only create strings of bounded lengths, so that $\yield(\tau, \langle\rangle, r) \leq C \cdot \DL[\tau]$ for some $C < \infty$.

\paragraph{Sampling Document Scripts.}

Sampling from the grammar requires defining a distribution over programs.
We defined a simple power-law prior for the number of $\langle\texttt{command}\rangle$ productions in a $\langle\texttt{block}\rangle$ ($p(l) \propto (1+l)^{-4}$, $1\leq l \leq 10$).
Second, as we do not have a-priori expectations for production rules for the  \texttt{$\langle$command$\rangle$} nonterminal, we defined a hyperprior so that for each document in the pretraining corpus, production probabilities were sampled individually before generating the program.
Thus, we avoided comitting to an arbitrary choice.
We constrained the production probabilities so that, for a given choice of probabilities, the expected document length was $\leq 64$, the LM's context length. Up to normalization, the satisfying set of probability vectors is approximately a half-space, which we precomputed.
We computed this set separately for each setting (varying $\Ff$ and $\Omega$) to make the mean document length comparable across settings (recall the our bounds contain constants depending on $\log \mathbb{E}[|d|]$, but not $|\Omega|$).
Thus, rather than a PCFG, we use a \emph{compound PCFG} distribution \citep{DBLP:conf/acl/KimDR19,DBLP:conf/emnlp/ZhaoT20}, i.e., sample one set of production probabilities for each generated scripts from a larger space of accepted PCFGs.
This corresponds to sampling from a mixture of different PCFGs, which we do to account for the fact that we have no a-priori knowledge of the ``correct'' production probabilities.
It does not affect the applicability of our learning bounds, as they hold for each of the mixed PCFGs individiaully.

\paragraph{Variables.}
Production of \texttt{$\langle$variable$\rangle$} was constrained to variables bound by a ''\texttt{LOOP OVER}'' or ''\texttt{FOR SOME}'' operator with scope over the variable.
Production rules whose RHS required more open variables than were contextually available were blocked.
These constraints are automatically enforced in the CAG translation.

Naively sampling syntactically correct scripts tends to produce bloated scripts with many unused variables.
We took the following steps to mitigate unused variables: With each production, we associated a probability distribution over the currently open variables; this was mixed with a uniform distribution at most productions, but with a Dirac distribution on the newly introduced variable for variable-introducing operators. Arguments to an equation $f_s(x_i) = x_j$ were sampled without replacement from the set of  available variables.

These steps do not break the PCFG independence assumptions, as they could be compiled out by splitting nonterminals and production rules.

\paragraph{Enforcing termination.}
We scale the probabilities of recursive rules with a negative power of the depth, $(\mathsf{depth})^{-2}$ (roughly based on the rate of recursion in natural language, \citet{DBLP:conf/acl/BlasiCWSBB19}), and then re-normalize the production probabilities, so that the production rules depend on the depth.
We do this for practical reasons: in order to enforce rapid termination of sampling without (near-)infinite recursive calls, as we mix a large space of PCFG parameters and thus could not tune these individually for rapid termination.
This breaks the PCFG assumptions; however, if anything, this choice should \emph{decrease} the documents' bias towards conmpositionality.

\section{HMM5 Training Datasets}\label{sec:HMM}

The HMM5 dataset closely follows the GINC dataset by \citet{DBLP:conf/iclr/XieRL022}; here, we provide a full definition to make the discussion self-contained.
The generative process is a mixture of five HMMs, whose state space is $\Omega \times \mathbb{F}$ (entities times properties), and whose transitions are independent in the two components:
\begin{equation}
p(\omega_{t+1}, f_{t+1}|\omega_t, f_t) = p(\omega_{t+1}|\omega_t) p(f_{t+1}|f_t)
\end{equation}
For the five HMMs, there is a common transition matrix for the entity component; each HMM has its own transition matrix for the property component $\theta \in \mathbb{R}^{|\Ff|\times|\Ff|}$.

A state $(\omega, f)$ emits the symbol $f(\omega)$; the corresponding property look-up table is termed  \textit{memory matrix} in \citet{DBLP:conf/iclr/XieRL022}.

Each property transition matrix $\theta$ is generated as a convex combination of $N_{perm} := 100$ random permutation matrices; the weights of the convex combination are generated as $\operatorname{softmax}((u-0.5)/0.1)$
where $u \in \mathbb{R}^{100}$ is uniformly random in $[0,1]$.

The entity transition matrix is sampled by first obtaining a matrix $T$ in the same way as the property transition matrices, and then computing
\begin{equation}
0.1 T + 0.9 I_{|\Omega| \times |\Omega|}
\end{equation}
We take the start distribution for the hidden states in all HMMs to be uniform.
For each document, one of the five HMMs is chosen randomly.
In line with the approximate length distribution of compositional documents, document length was sampled from a Gaussian with mean 50 and SD 10, clipped to $[0,128]$.

\paragraph{\textsc{HMMPerDoc}}
In the \textsc{HMMPerDoc} dataset, we created both the property and the entity transition matrix individually for each document, and set $N_{perm}=1$.
In preliminary experiments, we found that ICL was much less successful with $N_{perm}=10$ or $N_{perm}=100$.
The matrices were constrained so that either $f_2$ or $f_3$ was excluded from the transition dynamics.
In comparison to the \textsc{HMM} dataset, which mixes 5 HMMs, the resulting dataset reflects a mixture of a much larger number of HMMs, up to $|\Omega|! |\Ff-1|!$.

\section{Effect of $|\Omega|$}\label{sec:effect-of-omega}

As described in the main text, Theorem~\ref{theorem:theorem1} provides a bound independent of $|\Omega|$.
Empirically, we even observed \emph{improved} accuracy on some tasks when increasing $|\Omega|$.
This can be explained in terms of the experimental setup:
When $|\Omega|$ is large, an equation such as $y=f(x)$ (for fixed $x,y$) is less likely to be satisfied by a more than one function of small description length, so that prompts may be more distinctive.
Heuristically, on average across functions $f$ and prompts, the probability that no other function $f'$ with $\DL[f'] \leq \DL[f]$ matches the prompt is on the order of 
\begin{equation}
\left(1-\frac{1}{|\Omega|^N}\right)^{c|\Ff|} \sim \exp\left(-\frac{c|\Ff|}{|\Omega|^N}\right)
\end{equation}
if $N$ is fixed and $|\Ff|, \Omega$ are large; this increases as $|\Omega|$ increases.

\section{Ablations}\label{sec:ablations}
We created variants of the training data with (i) loops, (ii) variable introduction via ``for some'', (iii) conditions (if-then-else) 
ablated, and trained LMs with 21M parameters.
When ablating loops, we added a free variable that is set randomly, in order to introduce a variable.

See Figure~\ref{fig:ablation}.

\begin{figure}
\centering

	\textsc{FunctionEvaluation}

\includegraphics[width=0.7\textwidth]{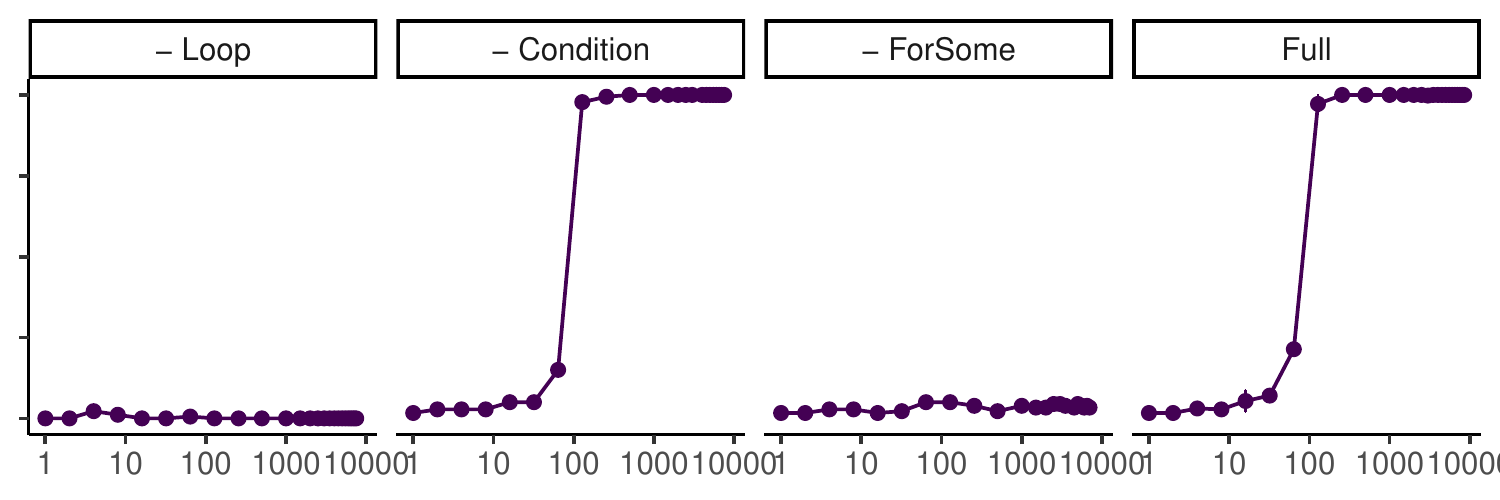}

\textsc{Propositional}

\includegraphics[width=0.7\textwidth]{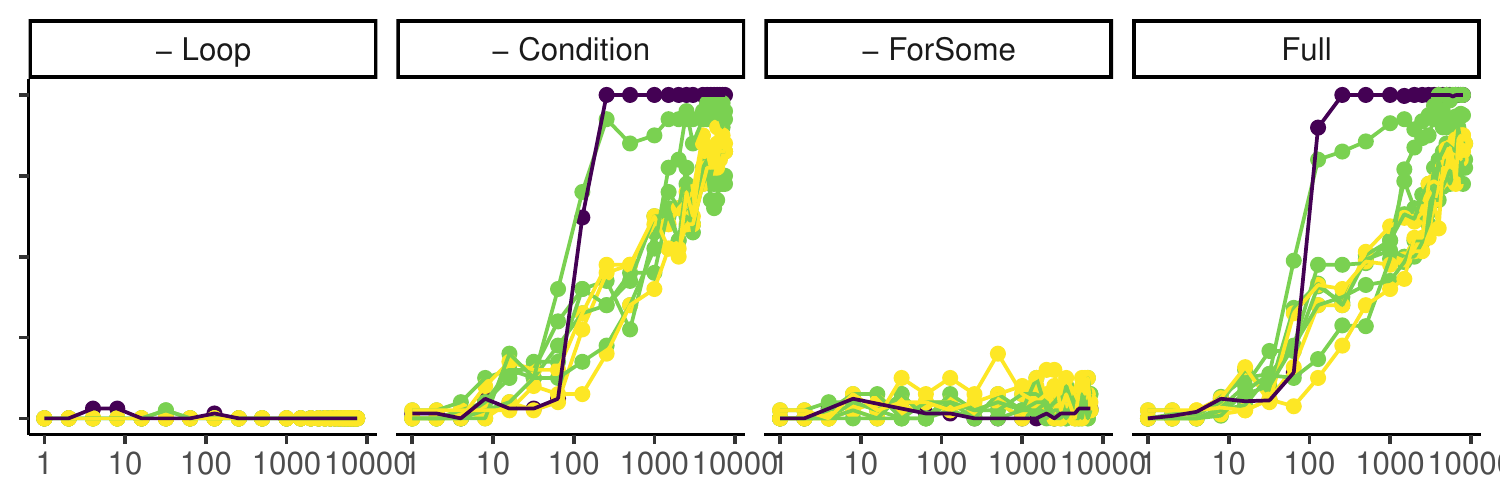}

\textsc{Composed}

\includegraphics[width=0.7\textwidth]{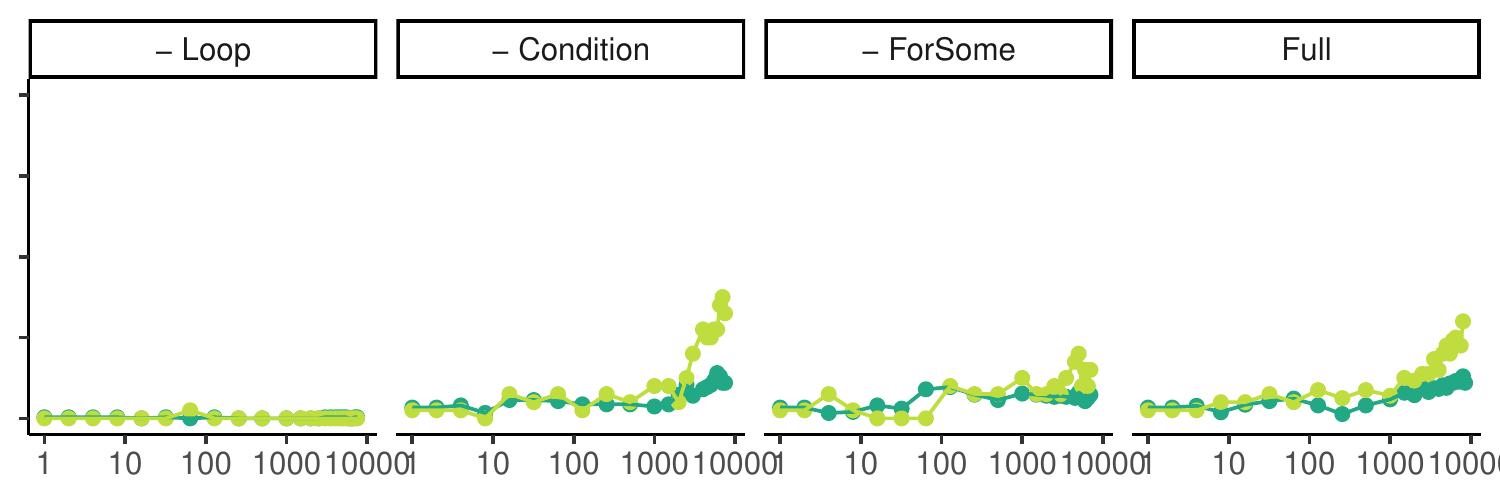}

\textsc{Binary}

\includegraphics[width=0.7\textwidth]{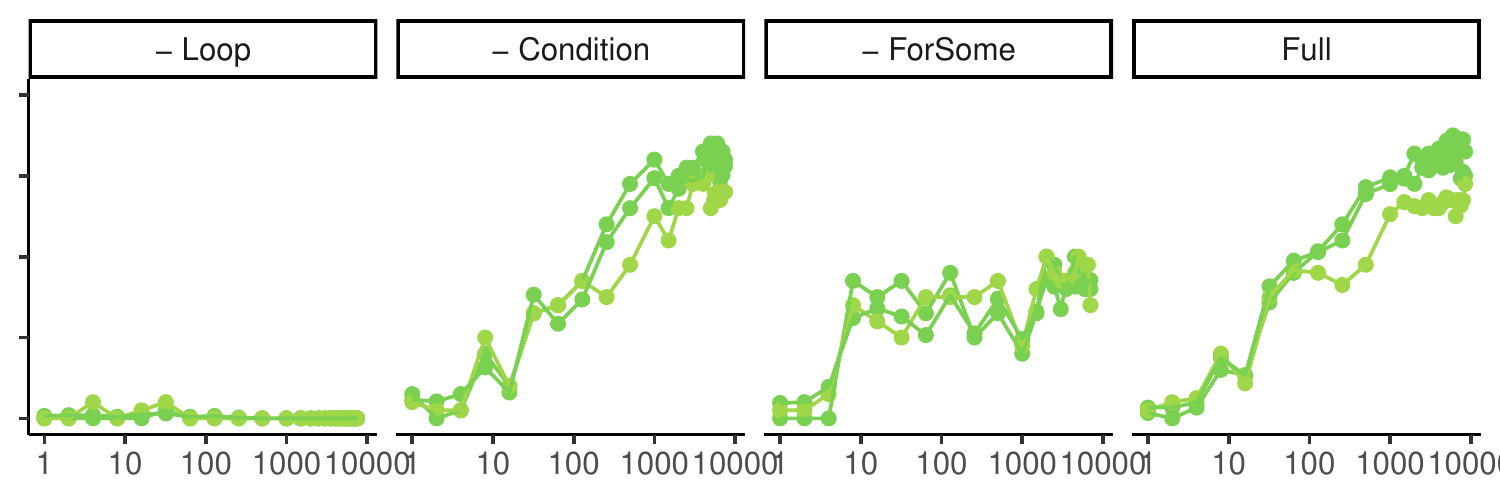}

\caption{Ablating the three components of the minimal CAG (all at 21M parameters, 14 prompt examples). 
Ablating variable-introducing constructs (loops and ``for some'') makes ICL impossible. Ablating the condition construct (``if-then-else'')  barely hurts ICL performance, if at all. 
	}\label{fig:ablation}
\end{figure}

\section{Additional Results}

See Figures \ref{fig:additional}--\ref{fig:additional-by-length-300}.

\begin{figure*}
    \centering
    All tasks by $|\Omega|$
    
    \begin{tabular}{ccccccccc}
            $|\Omega|=30$&
\includegraphics[width=0.15\textwidth]{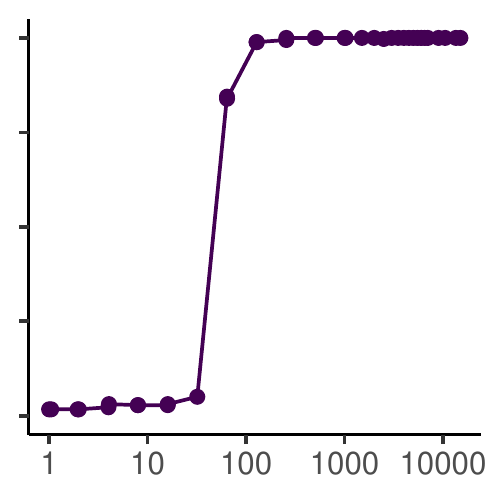} &
\includegraphics[width=0.15\textwidth]{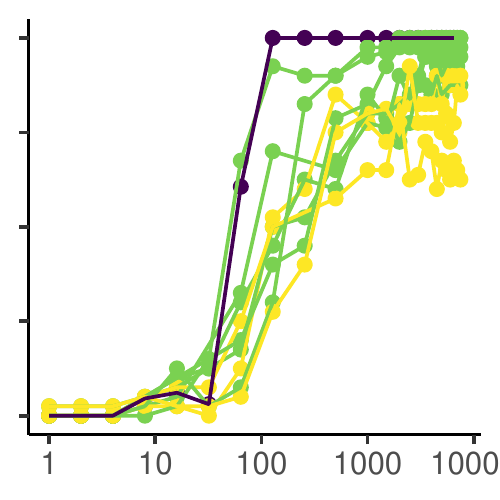} &
\includegraphics[width=0.15\textwidth]{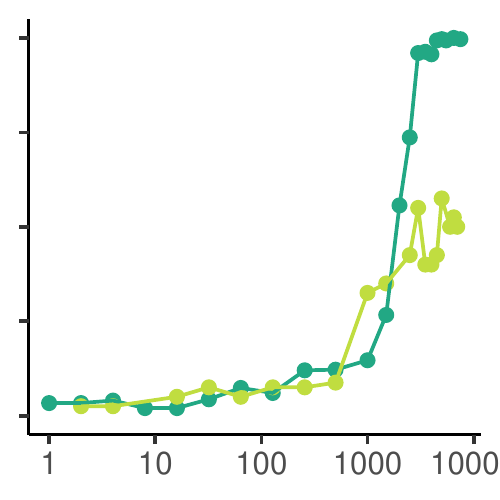} &
\includegraphics[width=0.15\textwidth]{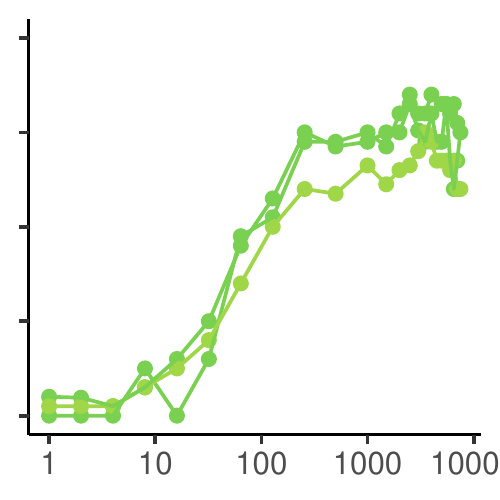}
\\
        $|\Omega|=100$&
\includegraphics[width=0.15\textwidth]{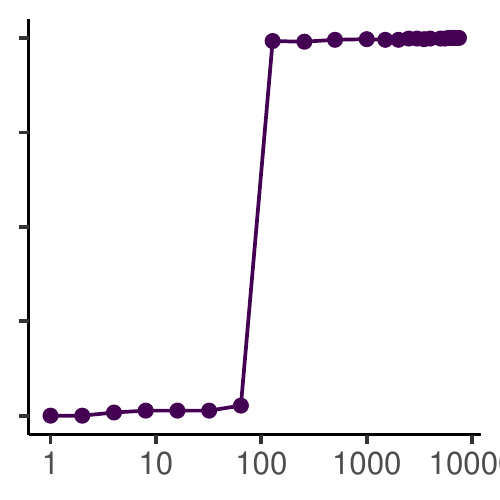} &
\includegraphics[width=0.15\textwidth]{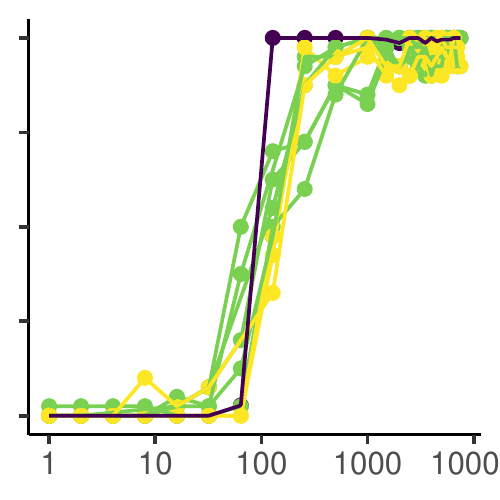} &
\includegraphics[width=0.15\textwidth]{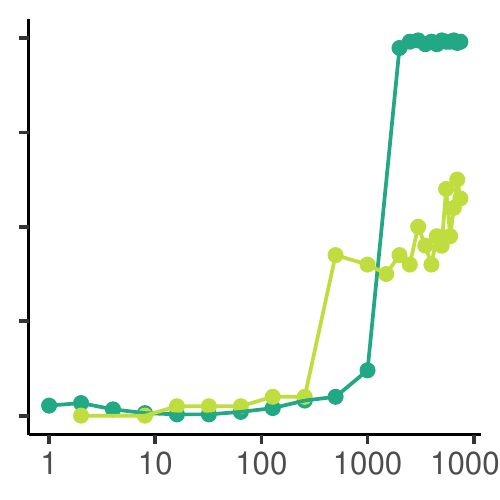} &
\includegraphics[width=0.15\textwidth]{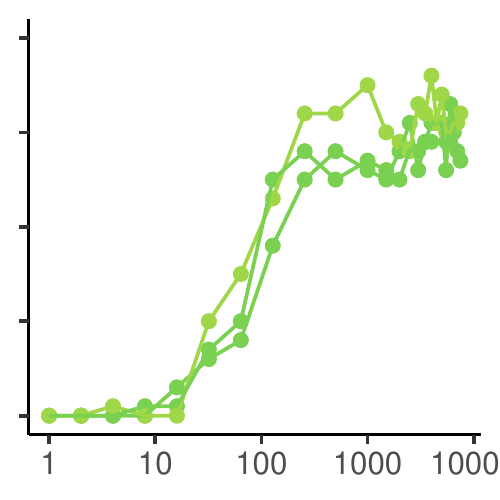}
\\
    $|\Omega|=300$&
\includegraphics[width=0.15\textwidth]{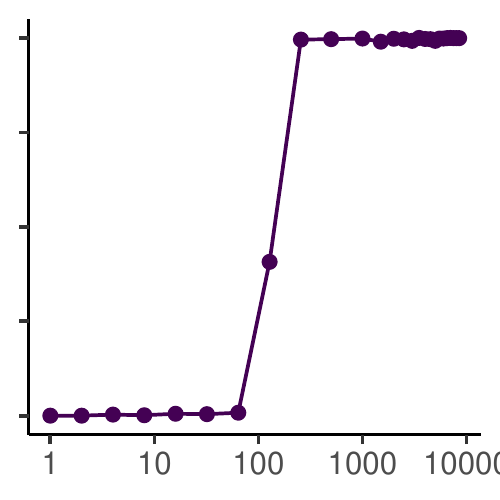} &
\includegraphics[width=0.15\textwidth]{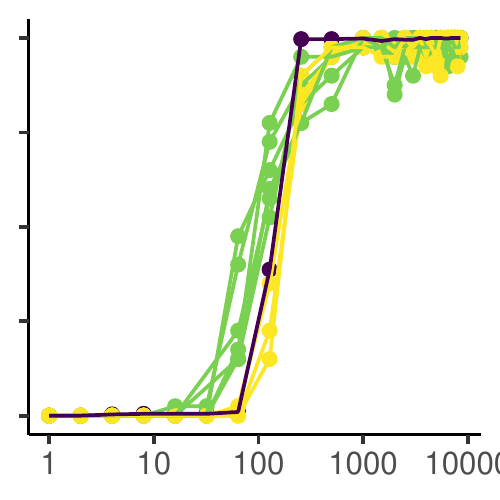} &
\includegraphics[width=0.15\textwidth]{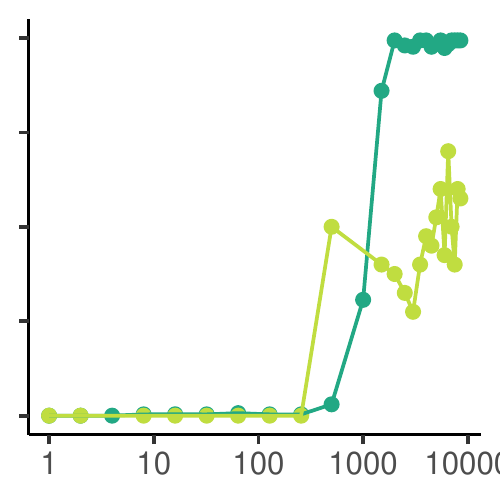} &
\includegraphics[width=0.15\textwidth]{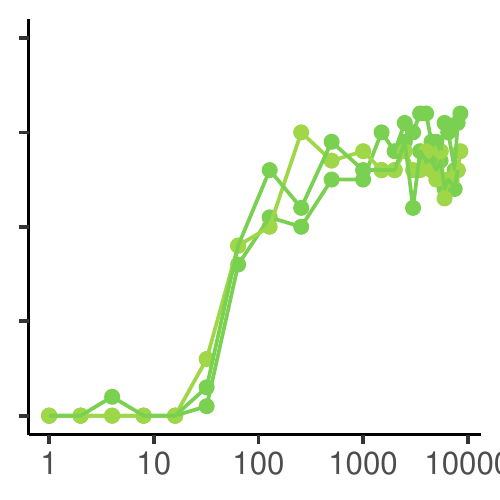}
    \end{tabular}

    \caption{Additional results by $|\Omega|$.}
    \label{fig:additional}
\end{figure*}

\begin{figure*}
    \centering
    All tasks by $|\Ff|$
    
    \begin{tabular}{ccccccccc}
            $|\Ff|=10$&
\includegraphics[width=0.15\textwidth]{simpleGrid11.pdf} &
\includegraphics[width=0.15\textwidth]{simpleGrid12.pdf} &
\includegraphics[width=0.15\textwidth]{simpleGrid13.pdf} &
\includegraphics[width=0.15\textwidth]{simpleGrid14.pdf}
\\
        $|\Ff|=20$&
\includegraphics[width=0.15\textwidth]{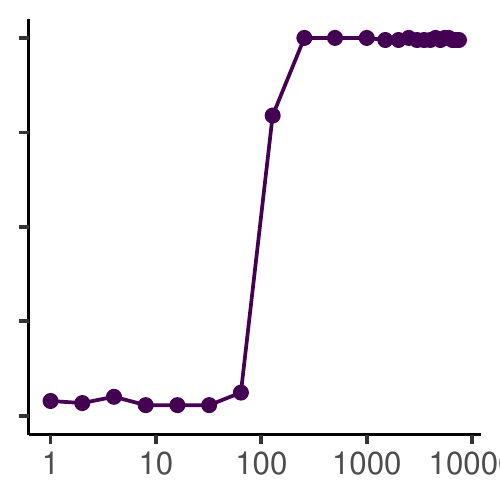} &
\includegraphics[width=0.15\textwidth]{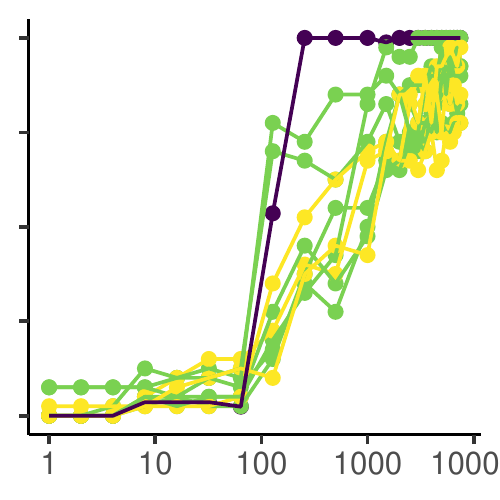} &
\includegraphics[width=0.15\textwidth]{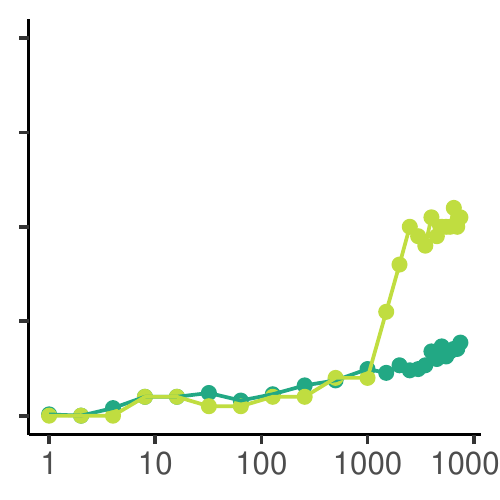} &
\includegraphics[width=0.15\textwidth]{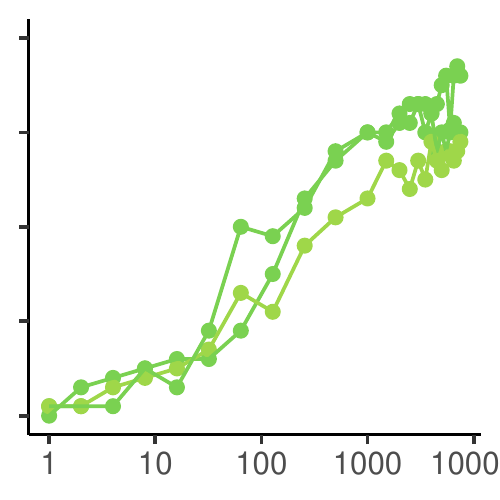}
\\
    $|\Ff|=30$&
\includegraphics[width=0.15\textwidth]{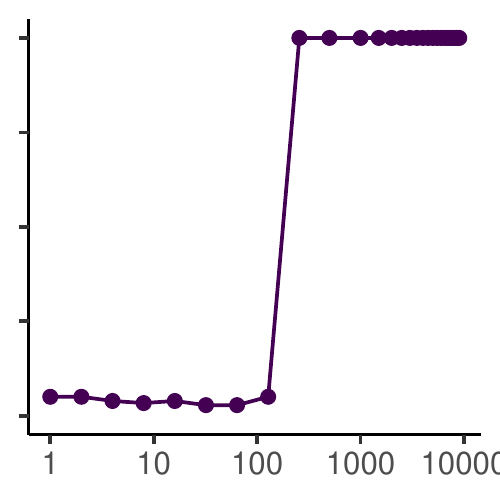} &
\includegraphics[width=0.15\textwidth]{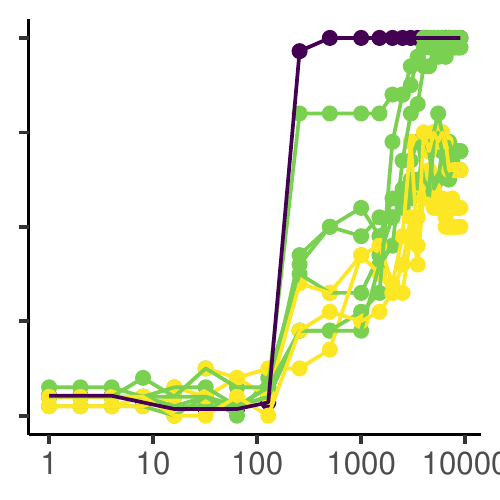} &
\includegraphics[width=0.15\textwidth]{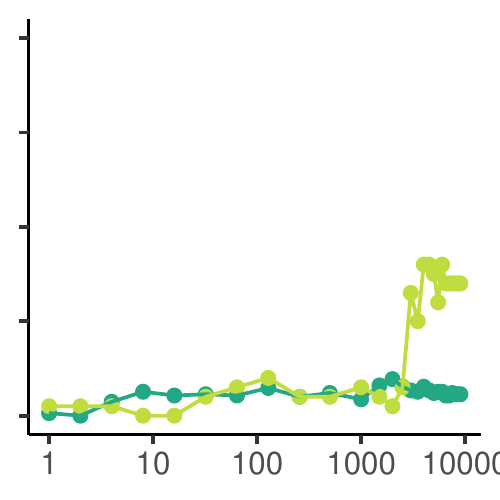} &
\includegraphics[width=0.15\textwidth]{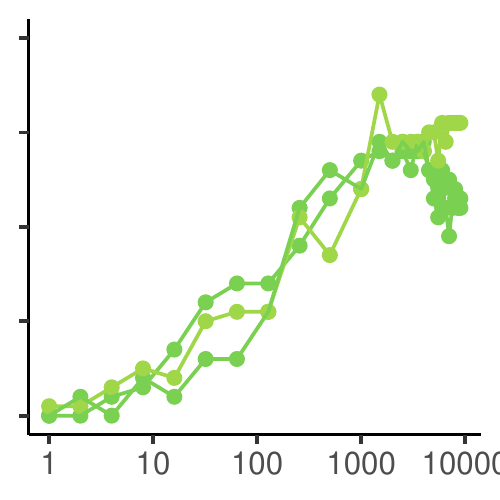}
    \end{tabular}

    \caption{Additional results, by $|\Ff|$.}
    \label{fig:additional-by-funcs}
\end{figure*}

\begin{figure*}
    \centering
	All tasks by Prompt Length

	$|\Omega|=30$, $|\Ff|=10$, 85M parameters
    
    \begin{tabular}{ccccccccc}
\includegraphics[width=0.75\textwidth]{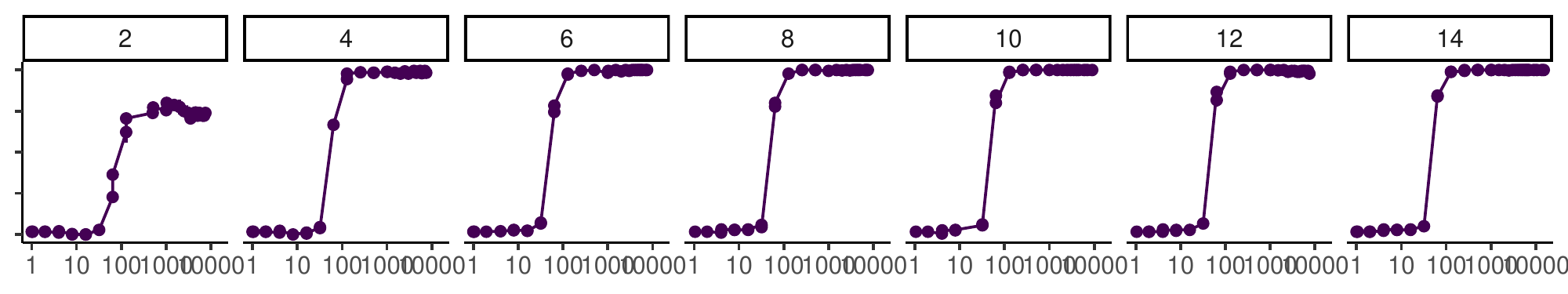} \\
\includegraphics[width=0.75\textwidth]{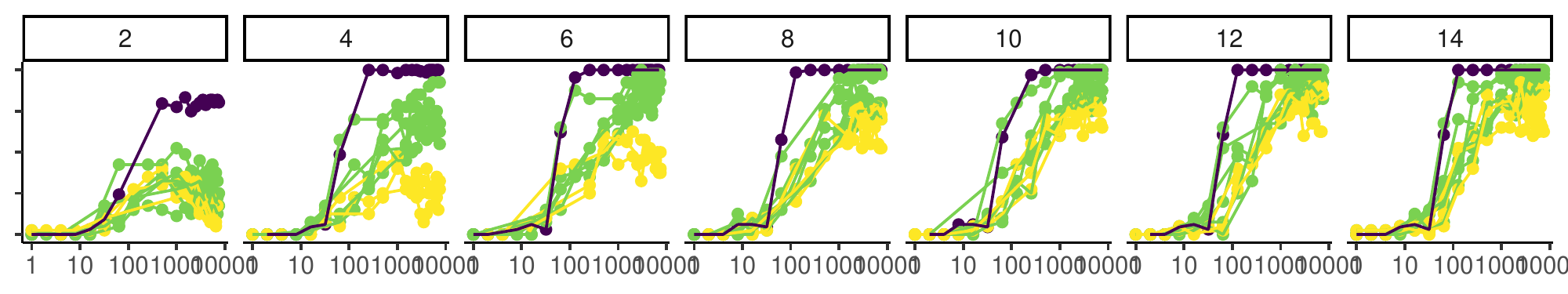} \\
\includegraphics[width=0.75\textwidth]{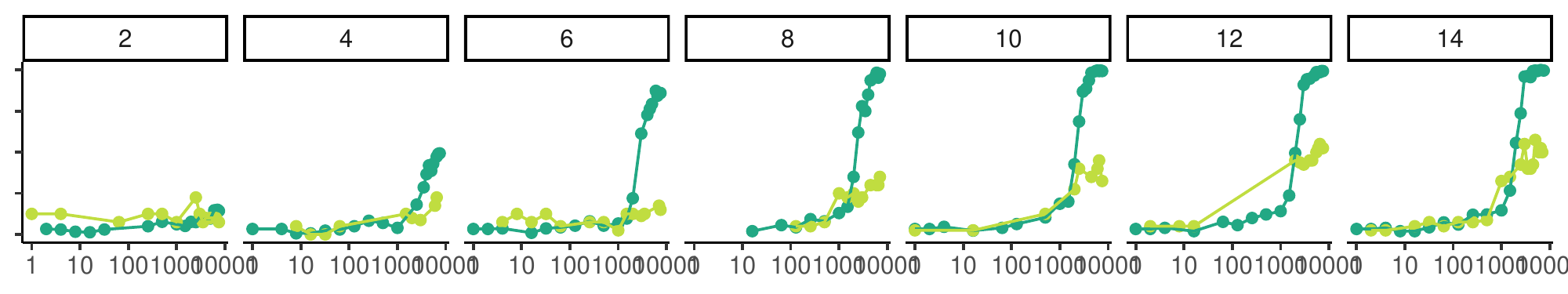} \\
\includegraphics[width=0.75\textwidth]{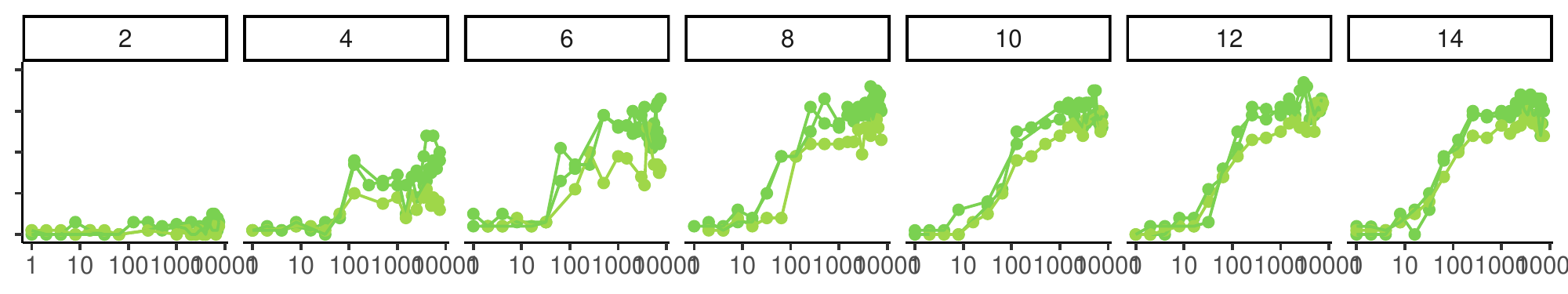}
    \end{tabular}

    \caption{Additional results.}
    \label{fig:additional-by-length}
\end{figure*}

\begin{figure*}
    \centering
	All tasks by Prompt Length

	$|\Omega|=300$, $|\Ff|=10$, 85M parameters
    
    \begin{tabular}{ccccccccc}
\includegraphics[width=0.75\textwidth]{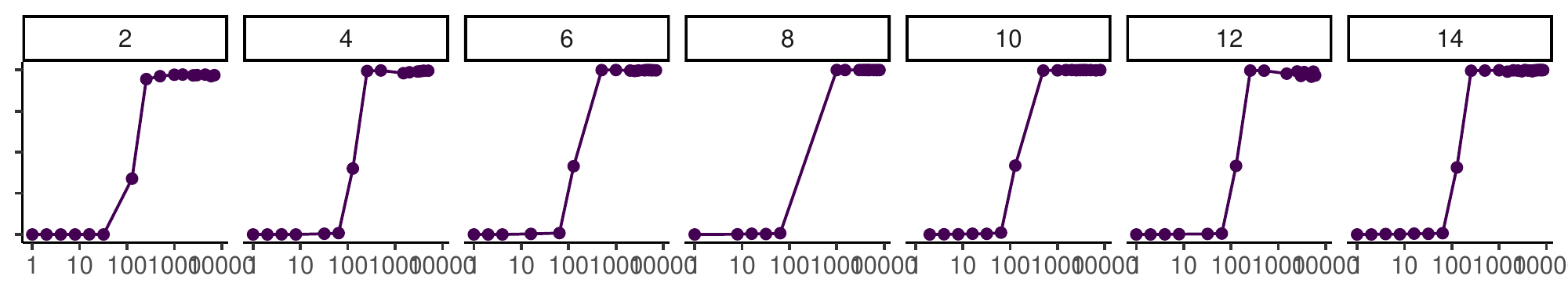} \\
\includegraphics[width=0.75\textwidth]{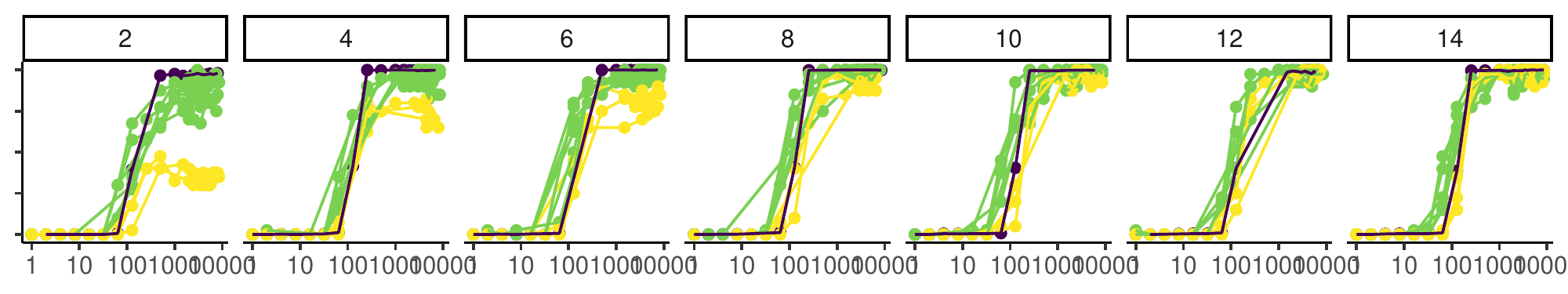} \\
\includegraphics[width=0.75\textwidth]{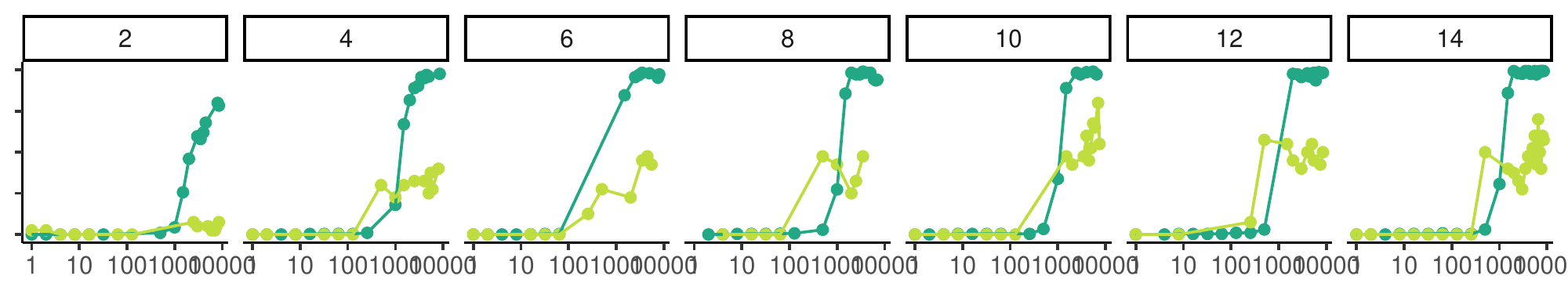} \\
\includegraphics[width=0.75\textwidth]{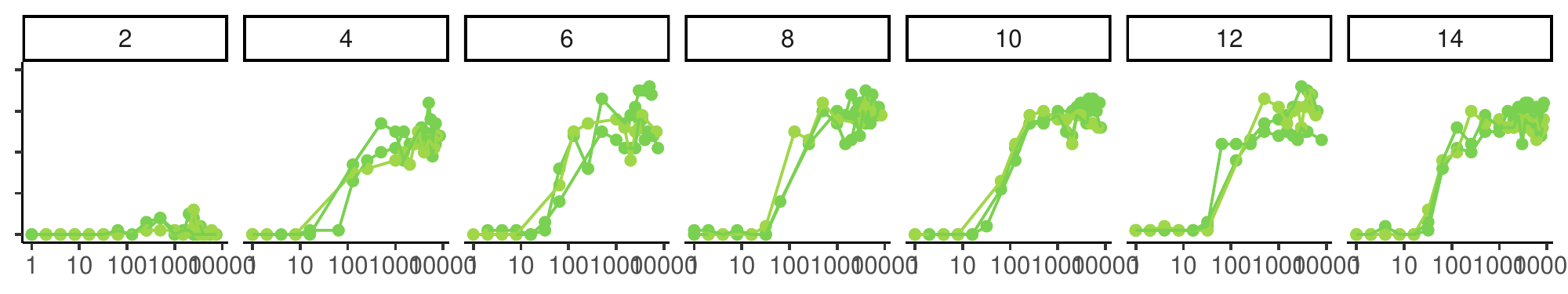}
    \end{tabular}
    
    \caption{Additional results.}
    \label{fig:additional-by-length-300}
\end{figure*}

\section{Heldout Analysis}
See Figure~\ref{fig:heldout}.
\begin{figure}
    \centering

\includegraphics[width=0.7\textwidth]{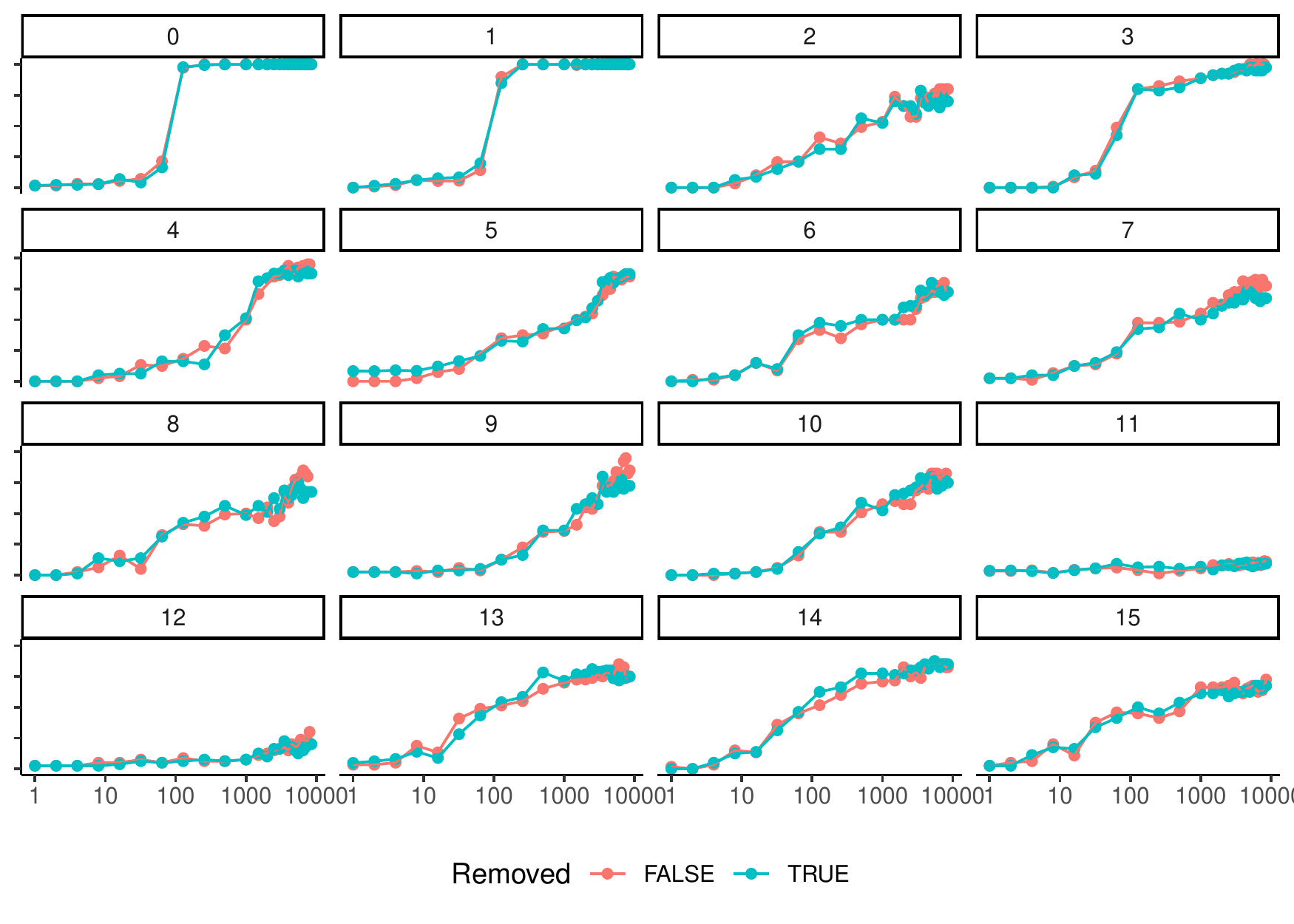}

    \caption{Accuracies on the test tasks (21M parameters; prompt length 14) after pretraining on the full dataset (red) or pretraining on a version of the same dataset where all documents containing a substring matching a valid prompt for any of the test tasks were removed (blue). Tasks are numbered as in Appendix~\ref{sec:task-formulas}.}
    \label{fig:heldout}
\end{figure}

\section{GPT-3 Experiment}\label{sec:gpt3=appendix}

\paragraph{Prompt Format}
We used single newlines to separate input from label, and double newlines to separate examples.

Sampled prompt (\textsc{FunctionEvaluation}, reversal):
\begin{verbatim}
i i x d h o y u v h\nh v u y o h d x i i
\n\n\n
n d b y p h z u h h\nh h u z h p y b d n
\n\n\n
m k e q m m j s g y\ny g s j m m q e k m
\n\n\n
u n j m u u m k t n\nn t k m u u m j n u
\n\n\n
j j z c v u t e a j
\end{verbatim}
with reference answer
\begin{verbatim}
\nj a e t u v c z j j
\end{verbatim}
As in our other experiments, only inputs $x$ or $x,y$  were included in prompts for which the response $z$ was unambiguous.
We obtained results for $\approx 14$ sampled prompts per task and prompt length (with some variability due to compute availability).

\end{document}